\documentclass[12pt,oneside,phd]{dms}
\usepackage[utf8]{inputenc} 
\usepackage[T1]{fontenc} %

\usepackage{lmodern}
\DeclareSymbolFont{largesymbols}{OMX}{cmex}{m}{n}

\anglais
\usepackage[numbers]{natbib}
\usepackage[english,frenchb]{babel}

\def\sloppy{%
  \tolerance 500
  \emergencystretch 3em%
  \hfuzz .5pt
  \vfuzz\hfuzz}
\usepackage{graphicx,amssymb,icomma}

\usepackage[labelfont=bf, width=\linewidth]{caption}

\usepackage{hyperref}  
\hypersetup{colorlinks=true,allcolors=black}
\usepackage{hypcap}   
\usepackage{bookmark} 
\numberwithin{equation}{section}
\numberwithin{table}{chapter}
\numberwithin{figure}{chapter}

\usepackage{bbold}
\usepackage{xcolor}
\usepackage[frozencache,cachedir=minted]{minted}
\usepackage{pifont}
\usepackage{cleveref}
\usepackage{enumitem}
\usepackage{booktabs}
\usepackage{wrapfig}
\usepackage{multirow}
\usepackage{subcaption}
\usepackage[export]{adjustbox}
\usepackage{algorithm}
\usepackage{algpseudocode}

\usepackage{thmtools}
\usepackage{thm-restate}
\newtheorem{theorem}{Theorem}
\newtheorem{proposition}{Proposition}
\newtheorem{lemma}{Lemma}
\newtheorem{corollary}{Corollary}
\theoremstyle{definition}
\newtheorem{definition}{Definition}

\theoremstyle{remark}

\newcommand{\restatableeq}[3]{\label{#3}#2\gdef#1{#2\tag{\ref{#3}}}}
\algnewcommand{\LineComment}[1]{\State \(\triangleright\) #1}

\counterwithout{equation}{section}
\counterwithin{equation}{chapter}
\counterwithin{algorithm}{chapter}
\counterwithin{theorem}{chapter}
\counterwithin{proposition}{chapter}
\counterwithin{lemma}{chapter}
\counterwithin{corollary}{chapter}
\counterwithin{definition}{chapter}
\counterwithin{assumption}{chapter}
\counterwithin{remark}{chapter}

\newcommand{\cmark}{\ding{51}}
\newcommand{\xmark}{\ding{55}}

\newcommand{\mathbbm}[1]{\mathbb{1}}
\newcommand{\ozero}{\text{\textcircled{\scriptsize{0}}}}
\newcommand{\oone}{\text{\textcircled{\scriptsize{1}}}}
\newcommand{\edge}[1]{$\langle$\textit{#1}$\rangle$}

\newcommand{\graph}{\mathcal{G}}
\newcommand{\grel}{\mathcal{G}_r}
\newcommand{\gtrain}{\mathcal{G}_{\textit{train}}}
\newcommand{\ginf}{\mathcal{G}_{\textit{inf}}}

\newcommand{\etrain}{\mathcal{V}_{\textit{train}}}
\newcommand{\einf}{\mathcal{V}_{\textit{inf}}}
\newcommand{\rtrain}{\mathcal{R}_{\textit{train}}}
\newcommand{\rinf}{\mathcal{R}_{\textit{inf}}}

\newcommand{\rels}{\mathcal{R}}

\newcommand{\rfund}{\gR_{\textit{fund}}}
\newcommand{\vrfund}{\mR_{\textit{fund}}}

\newcommand{\easy}[1]{#1}
\newcommand{\tp}[1]{#1 \cmark}
\newcommand{\fp}[1]{#1 \xmark}
\newcommand{\stack}[1]{\begin{tabular}{@{}c@{}}#1\end{tabular}}

\definecolor{mygreen}{RGB}{84, 130, 53}
\definecolor{myblue}{RGB}{68, 114, 196}
\definecolor{myorange}{RGB}{237, 125, 49}
\definecolor{mygray}{RGB}{246, 246, 246}
\definecolor{greenmarker}{RGB}{150, 255, 48}

\usepackage{bm}

\def\eqref#1{equation~\ref{#1}}

\def\1{\bm{1}}

\def\vb{{\bm{b}}}

\def\ve{{\bm{e}}}

\def\vg{{\bm{g}}}
\def\vh{{\bm{h}}}

\def\vm{{\bm{m}}}

\def\vq{{\bm{q}}}
\def\vr{{\bm{r}}}
\def\vs{{\bm{s}}}

\def\vv{{\bm{v}}}
\def\vw{{\bm{w}}}
\def\vx{{\bm{x}}}
\def\vy{{\bm{y}}}

\def\mA{{\bm{A}}}

\def\mD{{\bm{D}}}
\def\mE{{\bm{E}}}

\def\mR{{\bm{R}}}

\def\mW{{\bm{W}}}

\DeclareMathAlphabet{\mathsfit}{\encodingdefault}{\sfdefault}{m}{sl}
\SetMathAlphabet{\mathsfit}{bold}{\encodingdefault}{\sfdefault}{bx}{n}

\def\gA{{\mathcal{A}}}

\def\gC{{\mathcal{C}}}
\def\gD{{\mathcal{D}}}
\def\gE{{\mathcal{E}}}

\def\gG{{\mathcal{G}}}
\def\gH{{\mathcal{H}}}
\def\gI{{\mathcal{I}}}

\def\gL{{\mathcal{L}}}
\def\gM{{\mathcal{M}}}
\def\gN{{\mathcal{N}}}

\def\gP{{\mathcal{P}}}
\def\gQ{{\mathcal{Q}}}
\def\gR{{\mathcal{R}}}

\def\gT{{\mathcal{T}}}

\def\gV{{\mathcal{V}}}

\def\gX{{\mathcal{X}}}

\def\sR{{\mathbb{R}}}

\newcommand{\E}{\mathbb{E}}

\DeclareMathOperator*{\argmax}{arg\,max}

\DeclareMathOperator*{\topk}{top-k}

\begin{document}

\version{1}
\pagenumbering{roman}



\title{Learning Representations for Reasoning:\linebreak Generalizing Across Diverse Structures}

\author{Zhaocheng Zhu}

\copyrightyear{2024}

\department{Département d'informatique et de recherche opérationnelle}

\date{1 août 2024}

\sujet{Informatique}

\president{Jian-Yun Nie}

\directeur{Jian Tang}


\membrejury{Bang Liu}

\examinateur{Pasquale Minervini}   





\pdfbookmark[chapter]{Couverture}{PageUn}

\maketitle

\maketitle


\francais
\chapter*{Résumé}

Le raisonnement, la capacité de tirer des conclusions logiques à partir de connaissances existantes, est une caractéristique marquante de l'être humain. Avec la perception, ils constituent les deux thèmes majeurs de l'intelligence artificielle. Alors que l'apprentissage profond a repoussé les limites de la perception au-delà des performances humaines en vision par ordinateur et en traitement du langage naturel, les progrès dans les domaines du raisonnement sont loin derrière. L'une des raisons fondamentales est que les problèmes de raisonnement ont généralement des structures flexibles à la fois pour les connaissances (par exemple, les graphes de connaissances) et les requêtes (par exemple, les requêtes en plusieurs étapes), et de nombreux modèles existants ne fonctionnent bien que sur les structures vues pendant l'entraînement.

Dans cette thèse, nous visons à repousser les limites des modèles de raisonnement en concevant des algorithmes qui généralisent à travers les structures de connaissances et de requêtes, ainsi que des systèmes qui accélèrent le développement sur des données structurées. Cette thèse est composée de trois parties.
Dans la partie~\ref{part:inductive}, nous étudions des modèles qui peuvent généraliser de manière inductive à des graphes de connaissances invisibles, qui impliquent de nouveaux vocabulaires d'entités et de relations. Pour les nouvelles entités, nous proposons un nouveau cadre qui apprend les opérateurs neuronaux dans un algorithme de programmation dynamique calculant des représentations de chemin~\cite{zhu2021neural}. Ce cadre peut être étendu à des graphes de connaissances à l'échelle d'un million en apprenant une fonction de priorité~\cite{zhu2023net}. Pour les relations, nous construisons un graphe de relations pour capturer les interactions entre les relations, convertissant ainsi les nouvelles relations en nouvelles entités. Cela nous permet de développer un modèle pré-entraîné unique pour des graphes de connaissances arbitraires~\cite{galkin2024towards}.
Dans la partie~\ref{part:multi-step}, nous proposons deux solutions pour généraliser les requêtes en plusieurs étapes sur les graphes de connaissances et sur le texte respectivement. Pour les graphes de connaissances, nous montrons que les requêtes en plusieurs étapes peuvent être résolues par plusieurs appels de réseaux neuronaux graphes et d'opérations de logique floue~\cite{zhu2022neural}. Cette conception permet la généralisation à de nouvelles entités~\cite{galkin2022inductive}, et peut être intégrée à notre modèle pré-entraîné pour prendre en charge des graphes de connaissances arbitraires~\cite{galkin2024zero}. Pour le texte, nous concevons un nouvel algorithme pour apprendre des connaissances explicites sous forme de règles textuelles afin d'améliorer les grands modèles de langage sur les requêtes en plusieurs étapes~\cite{zhu2023large}.
Dans la partie~\ref{part:system}, nous proposons deux systèmes pour faciliter le développement de l'apprentissage automatique sur des données structurées. Notre bibliothèque open source traite les données structurées comme des citoyens de première classe et supprime la barrière au développement d'algorithmes d'apprentissage automatique sur des données structurées, y compris des graphes, des molécules et des protéines~\cite{zhu2022torchdrug}. Notre système d'intégration de nœuds résout le goulot d'étranglement de la mémoire GPU des matrices d'intégration et s'adapte aux graphes avec des milliards de nœuds~\cite{zhu2019graphvite}.

\noindent {\bf \keywords: raisonnement, apprentissage de représentation, graphes de connaissances, grands modèles de langage, systèmes d'apprentissage automatique}


\anglais
\chapter*{Abstract}

Reasoning, the ability to logically draw conclusions from existing knowledge, is a hallmark of human. Together with perception, they constitute the two major themes of artificial intelligence. While deep learning has pushed the limit of perception beyond human-level performance in computer vision and natural language processing, the progress in reasoning domains is way behind. One fundamental reason is that reasoning problems usually have flexible structures for both knowledge (e.g.\ knowledge graphs) and queries (e.g.\ multi-step queries), and many existing models only perform well on structures seen during training.

In this thesis, we aim to push the boundary of reasoning models by devising algorithms that generalize across knowledge and query structures, as well as systems that accelerate development on structured data. This thesis is composed of three parts.
In Part~\ref{part:inductive}, we study models that can inductively generalize to unseen knowledge graphs, which involve new entity and relation vocabularies. For new entities, we propose a novel framework that learns neural operators in a dynamic programming algorithm computing path representations~\cite{zhu2021neural}. This framework can be further scaled to million-scale knowledge graphs by learning a priority function~\cite{zhu2023net}. For relations, we construct a relation graph to capture the interactions between relations, thereby converting new relations into new entities. This enables us to develop a single pre-trained model for arbitrary knowledge graphs~\cite{galkin2024towards}.
In Part~\ref{part:multi-step}, we propose two solutions for generalizing across multi-step queries on knowledge graphs and text respectively. For knowledge graphs, we show multi-step queries can be solved by multiple calls of graph neural networks and fuzzy logic operations~\cite{zhu2022neural}. This design enables generalization to new entities~\cite{galkin2022inductive}, and can be integrated with our pre-trained model to accommodate arbitrary knowledge graphs~\cite{galkin2024zero}. For text, we devise a new algorithm to learn explicit knowledge as textual rules to improve large language models on multi-step queries~\cite{zhu2023large}.
In Part~\ref{part:system}, we propose two systems to facilitate machine learning development on structured data. Our open-source library treats structured data as first-class citizens and removes the barrier for developing machine learning algorithms on structured data, including graphs, molecules and proteins~\cite{zhu2022torchdrug}. Our node embedding system solves the GPU memory bottleneck of embedding matrices and scales to graphs with billion nodes~\cite{zhu2019graphvite}.

\noindent {\bf \keywords: reasoning, representation learning, knowledge graphs, large language models, machine learning systems}


\cleardoublepage
\pdfbookmark[chapter]{\contentsname}{toc}  
\tableofcontents
\cleardoublepage
\phantomsection  
\listoftables
\cleardoublepage
\phantomsection
\listoffigures


\chapter*{List of Acronyms and Abbreviations}

\begin{twocolumnlist}{.2\textwidth}{.75\textwidth}
    MLP & Multi-Layer Perceptron \\
    ReLU & Rectified Linear Unit \\
    SGD & Stochastic Gradient Descent \\
    ASGD & Asynchronous Stochastic Gradient Descent \\
    EPFO & Existential Positive First-Order \\
    FOL & First-Order Logic \\
    GNN & Graph Neural Network \\
    PNA & Principal Neighborhood Aggregation \\
    NBFNet & Neural Bellman-Ford Network \\
    A*Net & A* Network \\
    GNN-QE & Graph Neural Network Query Executor \\
    Ultra & Unified, Learnable \& Transferable Knowledge Graph Representations \\ 
    MR & Mean Rank \\
    MRR & Mean Reciprocal Rank \\
    MAPE & Mean Absolute Percentage Error \\
    AUROC & Area Under the Receiver Operating Characteristic Curve \\
    AP & Average Precision \\
    OOM & Out-of-Memory \\
    SOTA & State-of-the-Art \\
    LLM & Large Language Model \\
    GPT & Generative Pretrained Transformer \\
    CoT & Chain-of-Thought \\
    LtM & Least-to-Most \\
    CPU & Central Processing Unit \\
    GPU & Graphics Processing Unit \\    
\end{twocolumnlist}


\chapter*{Acknowledgements}

It felt like yesterday when I stepped onto the campus of UdeM and joined Mila as a Ph.D. student in 2018. I was thrilled to conduct fundamental research at one of the birthplaces of deep learning, yet anxious about whether I could tackle all the challenges along the way. As the journey comes to an end, I can confidently say that I have completed a wonderful and fruitful Ph.D. life, which would definitely not be possible without help from many people.

The highest acknowledgement should be given to my parents, Yuming Zhao and Jianguo Zhu, for their unconditional love and support. They always believe in my abilities and trust my decisions more than anyone else. When I was frustrated by research failures and the COVID pandemic, they listened to my troubles and provided help as much as they could. I feel privileged to have Ph.D. parents who fully understand the impact and hardship of research. My mother provided me with practical suggestions at every crucial moment in my career. Their support established my initial confidence and strategy for solving daunting challenges. Thank you to my parents for their profound impact on my life.

I would like to express my deep gratitude to my advisor, Jian Tang, for giving me the opportunity to study at Mila and teaching me numerous skills. Over the past years, Jian has significantly influenced me in various aspects, such as choosing research directions, maintaining high standards and presenting works to a broad audience. When I was to submit my first paper at Mila, Jian provided close guidance and rewrote almost every section of the paper, which I learned later how challenging it is to polish ``doodles'' from a junior student. Jian always encouraged me to broaden my vision with fundamental machine learning techniques, which equipped me with a solid understanding of graphical models, unsupervised learning and meta learning. Most of my writing skills were acquired from Jian's edits on my manuscripts, and most of my presentation skills were obtained from the courses and tutorials we prepared together. I also really appreciate the freedom Jian gave me, even when my research did not always align with his interests.

Mikhail Galkin deserves special thanks for being my closest collaborator and friend for years. I never imagined that I could collaborate with a well-known blogger in graph machine learning. Mikhail has long been the first one I brainstorm with about my crazy ideas, and he always provided constructive suggestions to help me realize my dreams, including several works in this thesis. Mikhail also brought me to many talented researchers, especially those in European graph communities. Look forward to seeing you again in California!

I have learned a lot from my mentors, Yuan (Emily) Xue and Hanjun Dai, when I took an internship at Google. Both are not only knowledgeable collaborators, but also role models for my career. Emily was both energetic and passionate at work, and I always felt invigorated after each discussion with her. Hanjun respected my curiosity and helped me sort out the best plan to balance curiosity and output. I wish myself could be as nice as them when I mentor interns. I am grateful to Dale Schuurmans for invaluable advice on career choices, and Xinyun Chen, Denny Zhou, Xiaowei Li, Bo Dai, Xuezhi Wang for insightful discussions.

I would like to express my heartfelt thanks to my second cousin, Hao Tang, for discussing all kinds of stuff with me, ranging from techniques, gossip to life plans. While I am not sure if there is genetic evidence for shared research interests, ours seem to have a large overlap, which renders Hao's suggestions valuable to my works. I would also extend my gratitude to Yao Lu, who played an important role in my research and life during my second year.

Mila is an awesome place to conduct research, and I am grateful to the diverse environment, computation resources and freedom Mila granted me. I am grateful to my coauthors who played an integral role in my research journey. Among many, thank you Meng Qu, Zuobai Zhang, Sophie Xhonneux, Chence Shi, Jianan Zhao and Xinyu Yuan. I would like to extend my thanks to Andreea Deac, Shengchao Liu, Weiping Song, Minkai Xu, Jiarui Lu, Farzaneh Heidari, Huiyu Cai, Minghao Xu, Yangtian Zhang, Chuanrui Wang and Zhihao Zhan in Jian's group, as well as Jie Fu, Min Lin, Ladislav Rampášek and Chen Sun across Mila. I thank my mentees who helped me practice my mentoring skills, including Zhijian Duan, Shengding Hu, Shiyu Zhao, Zhanke Zhou, Michelle Liu and Emy Yue Hu.

I have learned many precious lessons during collaboration with Jingbo Shang, Michael Bronstein, Laurent Charlin and Maxime Gazeau, as well as conversations with Jiliang Tang, Bruno Ribeiro, Ming Zhang and Shuiwang Ji. I am grateful to people who taught me important engineering skills, which laid a solid foundation for my implementation. Jifeng Dai and Xizhou Zhu educated me on how to break down model performance through ablation studies. Wenbin Hou and Shizhen Xu guided me in debugging and profiling CUDA implementation. I am fortunate to meet many like-minded researchers in the world, whom I always find excited and inspired to talk with. Thanks to Pasquale Minervini, Yihong Chen, Qian Huang, Petar Veličković and Leon Bergen for discussions on graph reasoning. Thanks to Eric Zelikman, Abulhair Saparov and Mehran Kazemi for insights on large language models. It was an honor to receive comments from people who work on classical neural-symbolic methods, such as Artur Garcez, Giuseppe Marra and Francesco Giannini.

The support and encouragement from my friends really made my Ph.D. journey an enjoyable one. Thanks Jun Gao for reminding me of our original motivation for research. Thanks to Tian Li, Ziniu Hu, Hongyu Ren, Yao Fu, Guodong Zhang and Jiaxuan You for sharing their research, stories and thoughts with me. Thanks to Yue Dong, Linqing Liu, Dinghuai Zhang and Jiayi Weng for helping me relieve the stress during job search. I am grateful to have overcome challenges in the pandemic with my roommate Maksym Korablyov.

Having a few hobbies to prevent burnout from intensive research work is an effective remedy for the Ph.D. journey. I would like to thank Daniel Kordan and Michael Shainblum for their fantastic artworks and tutorials, which motivated me to go out with cameras and curious eyes, regardless of the harsh winters in Montreal. Thanks to all friends who created unforgettable memories with me during our hikes and skiing trips.

Last but not the least, I would like to deliver special thanks to my mother's friend, Bo Huang, and her family, for helping me settle down in Montreal and navigate many life challenges, especially during the difficult pandemic period.

 %
 %

\NoChapterPageNumber
\cleardoublepage
\pagenumbering{arabic}


\chapter{Introduction}
\label{cha:introduction}

\begin{figure}[t]
    \centering
    \includegraphics[width=0.48\linewidth,valign=t]{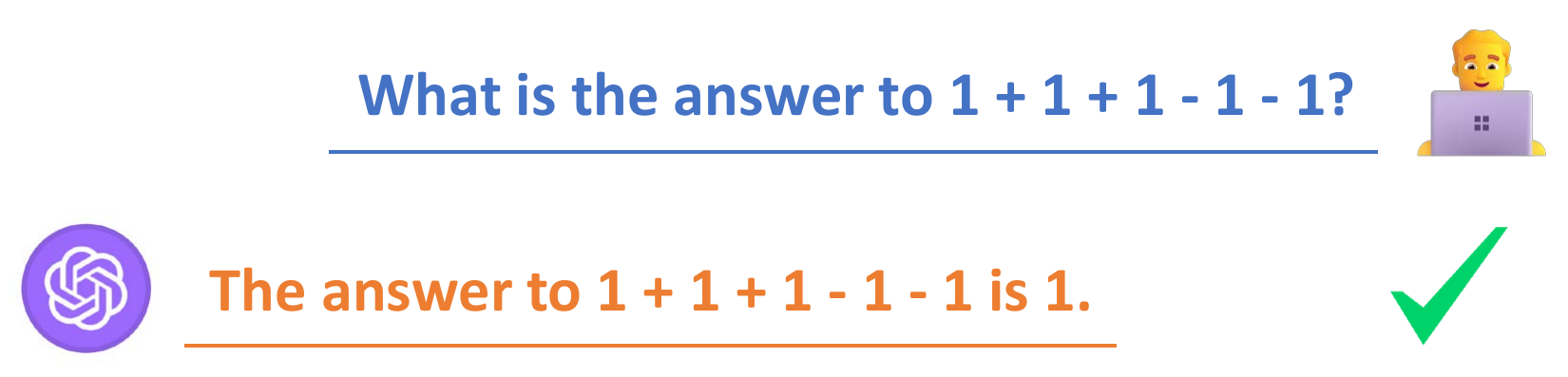}
    \hfill
    \includegraphics[width=0.48\linewidth,valign=t]{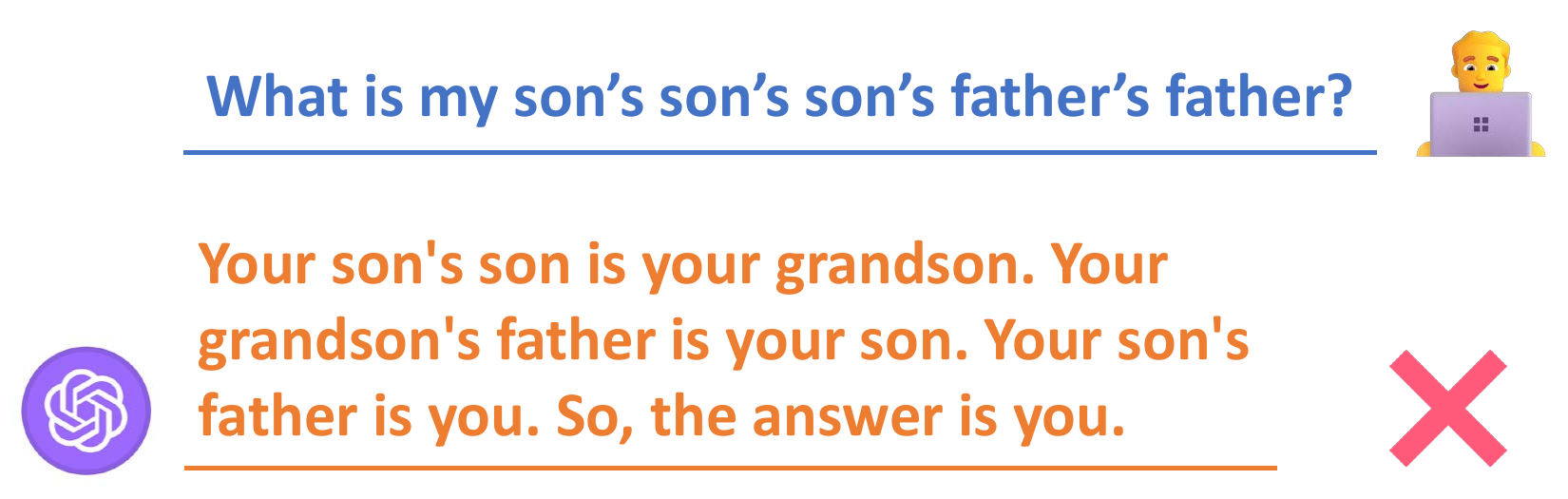}
    \\
    \vspace{1em}
    \begin{minipage}{0.48\linewidth}
        \centering
        \includegraphics[width=0.7\linewidth]{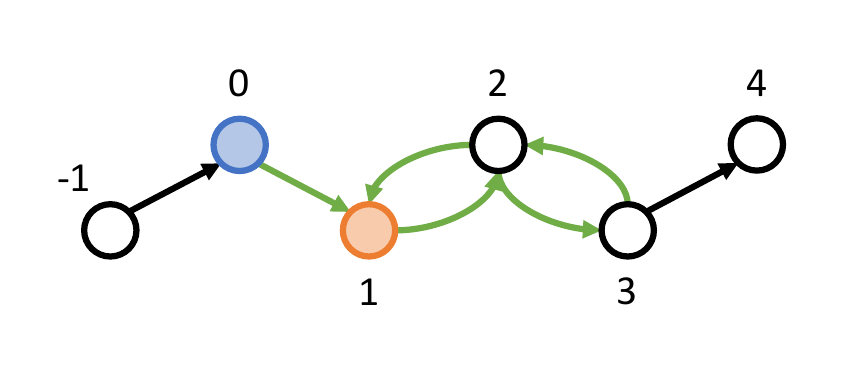}
    \end{minipage}
    \hfill
    \begin{minipage}{0.48\linewidth}
       \centering
        \includegraphics[width=0.7\linewidth]{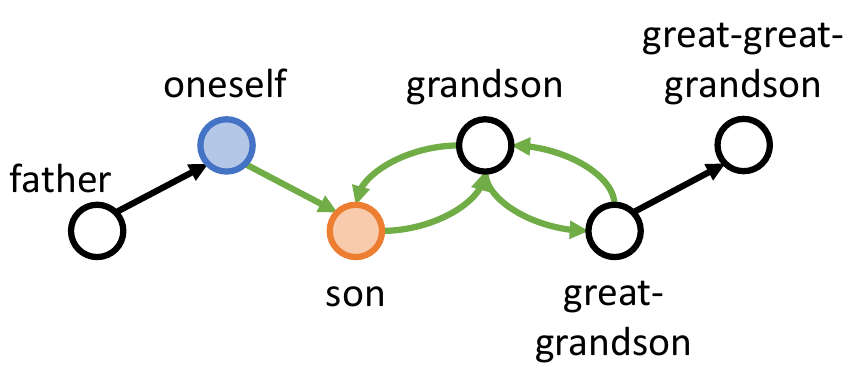}
    \end{minipage}
    \caption[Failure of GPT-4 in generalizing to new reasoning problems]{Failure of GPT-4 in generalizing to new reasoning problems. Top: GPT-4 solves the arithmetic problem, but fails to answer the kinship problem. Bottom: both problems aim to find the ending node of a given path on a path-like knowledge structure.}
    \label{fig:gpt_failure}
\end{figure}

Perception and reasoning are two major themes of artificial intelligence. Perception endows an agent with the ability to perceive the environment and process them into knowledge, while reasoning empowers it to use the stored knowledge to answer questions and draw new conclusions. With the rise of deep learning, there have been numerous advanced models in perception domains such as computer vision~\cite{he2016deep} and natural language processing~\cite{vaswani2017attention}. These models further reach a zenith when trained on massive data from the Internet, often condensing into a single large model that can serve a wide range of tasks~\cite{achiam2023gpt, team2023gemini, ramesh2022hierarchical, alayrac2022flamingo}---termed as foundation models nowadays. Nevertheless, breakthroughs of deep learning in reasoning domains are very limited, and consequently many applications are still bottlenecked by the reasoning abilities of current models.

One fundamental reason explaining this discrepancy is that reasoning domains require much more complicated generalization than perception domains. Unlike perception domains where unseen problems are mostly interpolation of seen problems, problems in reasoning domains are usually extrapolation of seen problems, e.g.\ composition, factorization, abstraction or substitution of seen problems. If a model does not possess proper inductive biases to deal with such dimensions of generalization, it will inevitably fail to solve new reasoning problems. Figure~\ref{fig:gpt_failure} (top) shows such a failure of GPT-4~\cite{achiam2023gpt}. GPT-4 can successfully answer an arithmetic problem, but not a similar kinship problem, despite the fact that it has learned kinship commonsense. This highlights the shortcomings of existing model architectures in performing reasoning, which cannot even be amended by training on the whole Internet.

In this thesis, we aim to push the boundary of representation learning models in reasoning domains. We notice that reasoning problems, despite their various surface forms, are underpinned by structures that represent the background knowledge in a reasoning process. Answering a query is thus cast to a series of actions on such structures, often resulting in a sub structure. Figure~\ref{fig:gpt_failure} (bottom) illustrates the corresponding structures for the two problems mentioned above. From this perspective, the two problems have identical knowledge and query structures, and the only difference lies in the vocabularies associated with the structures. Motivated by this observation, we are interested in representation learning models that can generalize across diverse structures, especially those unseen at training time. We consider the following types of generalization in reasoning problems
\begin{itemize}[label=$\bullet$, leftmargin=*]
    \item \textbf{Generalization across knowledge structures.} Most knowledge structures evolve over time, requiring reasoning models to adapt to new knowledge structures in the same domain. Structures of two domains can be similar even if they have distinct semantics for nodes and edges. In particular, we investigate models that can generalize across graphs with different entity and relation vocabularies.
    \item \textbf{Generalization across query structures.} Many queries of interest are composed of multiple reasoning steps. Consequently, there are an exponential number of possible query structures, and it is infeasible to train a model with all structures. Therefore, we consider studying models that can generalize to new or longer combinations of steps.
\end{itemize}

Both goals are considerably challenging for representation learning models, because representation learning models are always good at fitting all the information provided by the training set and the test set differs from that in some aspects. Hence, we need to inject certain inductive biases in the model architecture, such that the model only learns functions over information that are universal to different structures. By contrast, symbolic algorithms, such as personalized PageRank~\cite{page1999pagerank} and subgraph matching, generalize perfectly to unseen structures no matter what training structures they are implemented for\footnote{Strictly speaking, a non-trivial training structure is required to verify the correctness of an implementation of a symbolic algorithm.}. Nevertheless, symbolic algorithms are handcrafted by humans and only applicable to limited scenarios they are designed for, e.g.\ a complete knowledge structure. To meet our goal of generalization across structures, we devise representation learning models with inductive biases inspired by symbolic algorithms. This core idea is a recurring theme throughout this thesis.

Developing representation learning models for structures is often cumbersome, since modern machine learning frameworks are designed and optimized for tensors. To solve this issue, we aim to develop a library to simplify machine learning development on structured data and engage more developers into this field. Additionally, given the large scale of knowledge structures in the real world, we would like to study scalable solutions for both popular and our representation learning methods on structures, with the goal of extending their application to million-scale or even larger graphs.

To summarize, this thesis addresses the challenge of generalization across diverse structures, including entity vocabularies, relation vocabularies, multi-step queries in both graph and text modalities. Our works in this thesis reveal the possibility of unifying various knowledge structures and query structures, leading to the first foundation model for both single- and multi-step queries on knowledge graphs. These works have changed the long-held convention of learning shallow embeddings of structures and unlocked many opportunities in reasoning domains. We anticipate our assets and findings will accelerate progress towards human-level reasoning models and the ultimate goal of artificial intelligence.

\section{Summary of Contributions}

Here we provide an overview of this thesis and summary of our contributions. In Chapter~\ref{cha:background}, we discuss the definition of problems and the goal of generalization across structures. We summarize related works on representation learning for structures in order to provide readers a better understanding of our contributions. Chapter~\ref{cha:nbfnet} \& \ref{cha:ultra} study models that generalize to structures with unseen entity and relation vocabularies. Chapter~\ref{cha:gnn-qe} \& \ref{cha:htt} demonstrate models for solving multi-step queries in both graph and text modalities. Chapter~\ref{cha:torchdrug} \& ~\ref{cha:graphvite} present systems to facilitate development of machine learning on structured data.

\smallskip\noindent\textbf{(Chapter~\ref{cha:nbfnet}) Generalization to Knowledge Graphs with Unseen Entities.}
Knowledge graph reasoning is typically solved by embedding methods, which learn an embedding vector for each entity and relation in a knowledge graph. Such embedding vectors restrict the prediction of these methods to entities they are trained on. We proposed NBFNet~\cite{zhu2021neural} to learn the representation of every entity pair as a function of the relational paths between them, which eliminates the need for entity embeddings and generalizes to new entities or even new knowledge graphs of the same relation vocabulary. NBFNet achieved state-of-the-art results in both transductive and inductive settings, and can be made more efficient by learning a priority function to select nodes and edges on the fly~\cite{zhu2023net}. NBFNet ranked 12 out of 39 teams in 1st OGB large-scale challenge, being the strongest single model and the most parameter efficient on a knowledge graph with 87 million entities and 504 million facts.

\smallskip\noindent\textbf{(Chapter~\ref{cha:ultra}) Generalization to Any Knowledge Graph with Arbitrary Entity and Relation Vocabularies.}
Generalization across knowledge structures plays a key role in the era of foundation models. In order to train a single generic model that performs reasoning on arbitrary input, we need to enable generalization to new relation vocabularies in addition to the entity vocabularies studied in Chapter~\ref{cha:nbfnet}. In our work Ultra~\cite{galkin2024towards}, we solved this challenge by parameterizing relative relation representations as a function of relation interactions, resulting in two nested NBFNet, one for entities and one for relations. By training on 3 standard knowledge graphs, Ultra shows strong zero-shot generalization performance on 40 knowledge graphs of various domains and sizes, on par or even surpassing state-of-the-art methods on 32 datasets. Ultra eliminated the need of training models separately for each graph and established the first foundation model for knowledge graph reasoning.

\smallskip\noindent\textbf{(Chapter~\ref{cha:gnn-qe}) Solving Multi-hop Queries on Knowledge Graphs.}
Many applications of reasoning require to deal with queries that inherently contain multiple steps, e.g. \emph{at what universities do Turing Award winners in the field of deep learning work?} Common embedding methods modeling such queries simulate logic operations with neural networks, which do not generalize well to different combinations of steps. We proposed GNN-QE~\cite{zhu2022neural} to decompose queries into individual steps, and parameterize each step with graph neural networks (GNNs) or fuzzy logic. Such a design aligns with the subgraph matching algorithm that generalizes perfectly when the graph is complete. GNN-QE not only achieves a relative gain of 22.3\% on existential positive first-order (EPFO) queries and 95.1\% on negation queries, but is also applicable to knowledge graphs with unseen entities~\cite{galkin2022inductive}. Meanwhile, GNN-QE supports visualization of entity distributions for every intermediate steps, and can be further integrated with Ultra to answer queries on any knowledge graph~\cite{galkin2024zero}.

\smallskip\noindent\textbf{(Chapter~\ref{cha:htt}) Solving Multi-step Queries with Large Language Models.}
With the popularization of large language models (LLMs), reasoning in natural languages has gradually drawn the attention of the community. Chain-of-Thought (CoT) prompting~\cite{wei2022chain} showed that we can teach LLMs to solve multi-step queries using a small set of in-context examples with intermediate steps. However, CoT relies on implicit knowledge stored in the parameters of LLMs, of which errors may exacerbate in multi-step reasoning. We rectified this issue with Hypotheses-to-Theories (HtT) prompting~\cite{zhu2023large} that learns explicit knowledge as textual rules and generalizes to queries longer than the training ones. Because the rules are expressed and learned as text, HtT opened up a transparent and interpretable learning paradigm for LLMs. The rules discovered by LLMs align well with human knowledge, and naturally transfer to different models and to different forms of the same problem.

\smallskip\noindent\textbf{(Chapter~\ref{cha:torchdrug}) A Library for Structured Data and Applications.}
Modern machine learning heavily relies on batch processing of tensors on GPUs to accelerate the computation, which structured data does not conform to. A common compromise is to pad structured data into grid data that can be represented by tensors, leading to much unnecessary computation and cognitive load. We addressed this issue by developing a library~\cite{zhu2022torchdrug} that treats structured data as first-class citizens and provides intuitive yet efficient GPU implementation for both graph and domain-specific operations. This infrastructure enables us to efficiently build solutions for various applications, such as knowledge graph reasoning, molecular property prediction and many protein representation learning tasks, all released as an open-source software TorchDrug. TorchDrug has brought many researchers and developers into the fields of knowledge graph reasoning and drug discovery.

\smallskip\noindent\textbf{(Chapter~\ref{cha:graphvite}) A System for Training Embeddings on Large Graphs.}
While embedding methods are the \emph{de facto} model for knowledge graph reasoning, they often cannot scale to knowledge graphs with more than 1 million entities, due to the large size of parameter matrices. Moreover, embedding methods do not benefit from techniques such as mini-batch SGD or data parallel, since they involve much more memory access per computation than neural networks. To scale up embedding methods, we developed GraphVite~\cite{zhu2019graphvite} to leverage the unique advantages of CPUs and GPUs for different stages in embedding training. GraphVite was the first system to train node embeddings of a billion-scale graph with only 4 P100 GPUs, and accelerate training on million-scale graphs by 51 times. It supports 10 different embedding methods, covering homogeneous graphs, knowledge graphs and high-dimensional visualization.

\section{Other Works}

While this thesis mainly focuses on representation learning techniques for reasoning tasks, some of our works do not directly fall into this range. GraphAny~\cite{zhao2024graphany} studied generalization across feature and label spaces in the node classification task. BioKGC~\cite{hu2024path} applied NBFNet to predict new interactions in biomedical networks. We have collaborated in a position paper discussing multi-hop reasoning and neural graph databases~\cite{ren2023neural}, and two cross-modal reasoning projects, including KEPLER~\cite{wang2021kepler} for enhancing language models with knowledge graph supervision, and GraphText~\cite{zhao2023graphtext} for training-free graph reasoning with LLMs. We have also developed a few works in drug discovery~\cite{shi2020graphaf, xu2022peer} based on our TorchDrug library, and large-scale multi-task datasets for molecular representation learning~\cite{beaini2024towards}. These works are not included as part of this thesis.

\section{Reading Guide}

This thesis is organized in an order from fundamental techniques to application solutions for reasoning and structured data. Readers interested in reasoning domains are encouraged to follow the original order of this thesis. Additionally, this thesis may be interesting to those working on topics related to the techniques we developed. For readers who want to focus on specific machine learning topics, we recommend the following chapters

\begin{itemize}[label=$\bullet$]
    \item Graph neural networks: Chapter~\ref{cha:nbfnet}, \ref{cha:ultra} and \ref{cha:gnn-qe}.
    \item Inductive generalization: Chapter~\ref{cha:nbfnet}, \ref{cha:ultra} and \ref{cha:gnn-qe}.
    \item Compositional generalization: Chapter~\ref{cha:gnn-qe} and \ref{cha:htt}.
    \item Zero-shot learning: Chapter~\ref{cha:ultra} and \ref{cha:gnn-qe}.
    \item Large language model reasoning: Chapter~\ref{cha:htt}.
    \item Batching irregular structure: Chapter~\ref{cha:nbfnet}, \ref{cha:gnn-qe} and \ref{cha:torchdrug}.
    \item Scalability: Chapter~\ref{cha:nbfnet} and \ref{cha:graphvite}.
\end{itemize}
\chapter{Background}
\label{cha:background}

In this chapter, we introduce the background knowledge and discuss related work of this thesis. We provide definitions for knowledge graph reasoning, inductive generalization and compositional generalization, as well as discuss challenges for achieving such generalization in representation learning models. We summarize related work in both knowledge graph reasoning and large language model reasoning literature, highlight their drawbacks and refer interested readers to referenced materials.

\section{Preliminary}

\subsection{Knowledge Graph Reasoning}

We adopt knowledge graphs as the major testbed for studying generalization, since they are a common discrete format of knowledge and free of confounding factors such as linguistic variations. A knowledge graph is denoted by $\gG = (\gV, \gE, \gR)$, where $\gV$ and $\gE$ represent the set of entities (nodes) and relations (edges) respectively, and $\gR$ is the set of relation types. Each relation is expressed as a triplet $(h, r, t)$, with $h$ and $t$ being the head and tail entities, and $r$ being the relation type. Due to this formulation, knowledge graphs are often referred to as a collection of triplets, where each triplet is a training or test sample from the perspective of machine learning.

The goal of knowledge graph reasoning is to predict all answer entities in a knowledge graph given a query. In its simplest form, the query contains only an entity and a relation, and the goal is to find either head entities or tail entities that form valid triplets with the query. A query example may be \emph{Who are Turing Award winners?} We denote such queries as $(u, q, ?)$ or $(?, q, u)$. Usually, due to the incomplete nature of knowledge graphs, the answer entities cannot be directly retrieved from the knowledge graph, and we need to reason about such missing triplets with representation learning models.

Going beyond single-hop queries, multi-hop queries aim to solve queries with multiple entities, relations and logical operations, including conjunction ($\land$), disjunction ($\lor$) and negation ($\neg$). For example, multi-hop queries may represent complex questions like \emph{at what universities do Turing Award winners in the field of deep learning work?}, which can be written as $?v \exists u: \textit{Win}(u, \textit{Turing Award}) \land \textit{Field}(u, \textit{Deep Learning}) \land \textit(u, v)$. Multi-hop queries are challenging in the sense that one not only needs to deal with incomplete knowledge graphs, but also models multiple operations and satisfies logical properties.

\begin{figure}[t]
    \centering
    \includegraphics[width=0.8\textwidth]{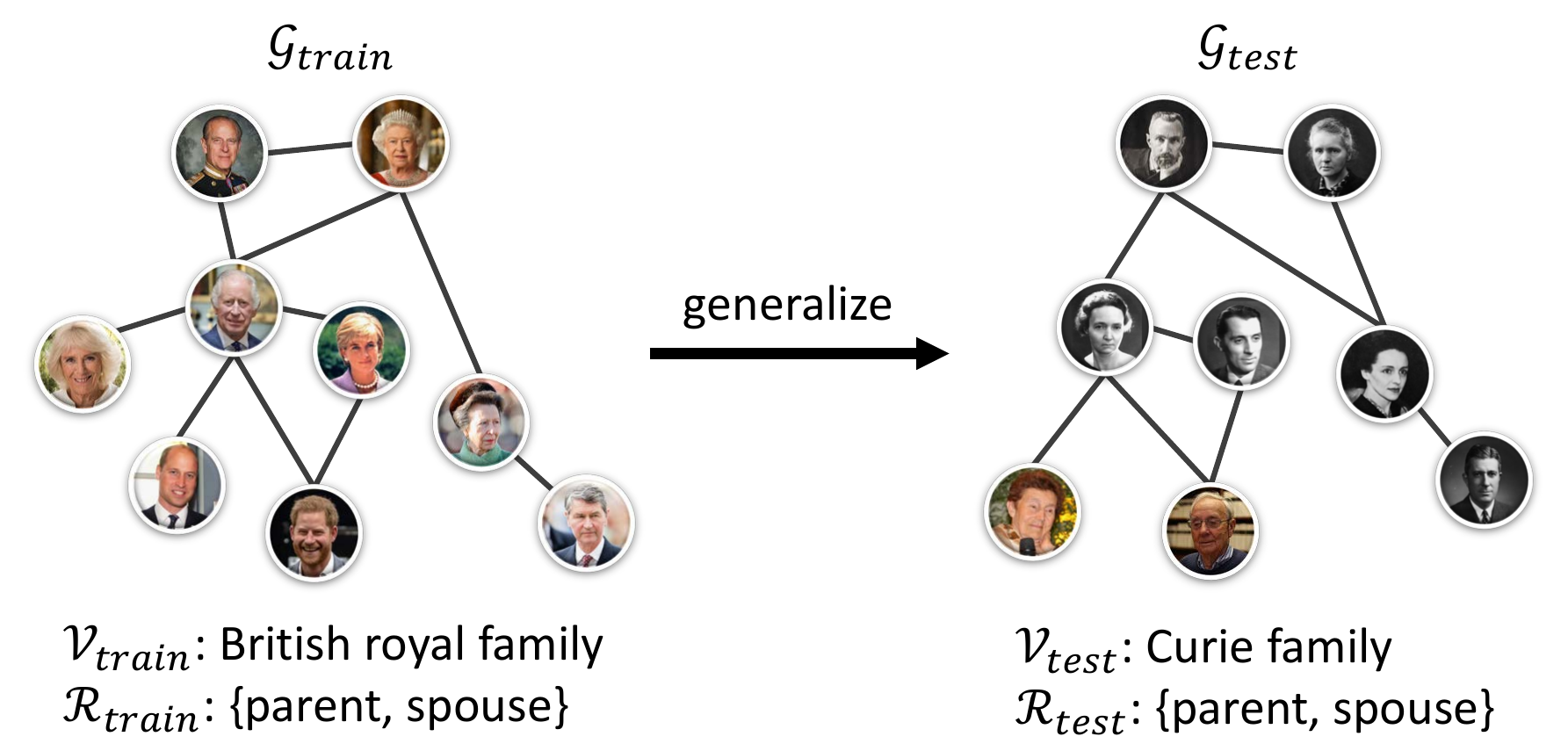}
    \\[0.5em]
    \includegraphics[width=0.8\textwidth]{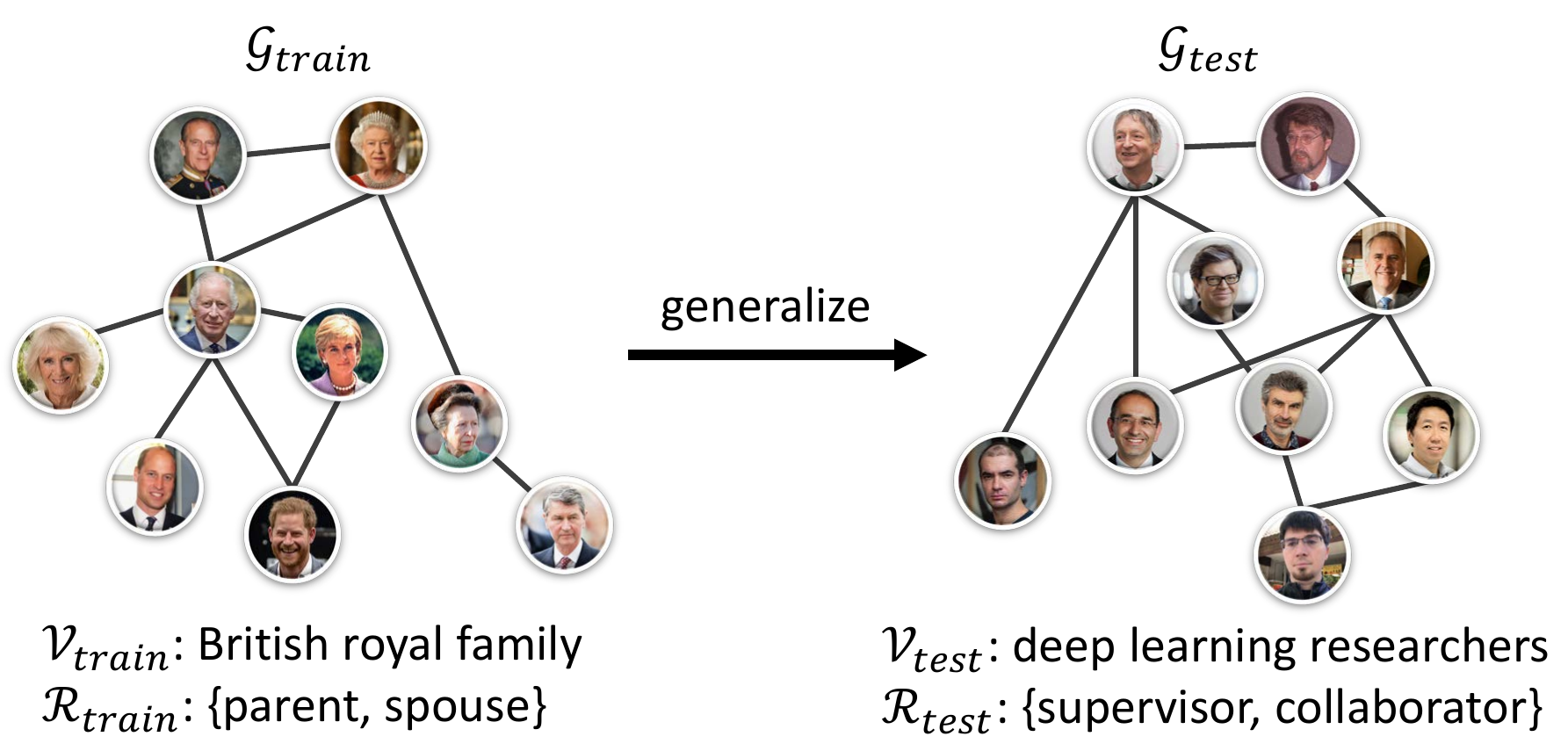}
    \caption[Illustration of inductive generalization on knowledge graphs]{Illustration of inductive generalization on knowledge graphs. The upper example generalizes from the British royal family tree to the Curie family tree, which involves new entities. The lower example generalizes from the British royal family tree to the deep learning researcher tree, which involves both new entities and new relations.}
    \label{fig:inductive}
\end{figure}

\subsection{Inductive Generalization}
\label{sec:inductive}

Traditionally, knowledge graph reasoning methods are evaluated on the \pagebreak knowledge graph they are trained on, with test queries that are not seen during training, which is referred as \emph{transductive setting}. By contrast, inductive setting uses a test knowledge graph $\gG_{test} = (\gV_{test}, \gE_{test}, \gR_{test})$ different from the training one $\gG_{train} = (\gV_{train}, \gE_{train}, \gR_{train})$, and evaluates models with queries on $\gG_{test}$. Typically, this new graph consists of some entities unseen during training, or a completely new vocabulary of entities, but shares the same relation vocabulary with the training graph, i.e.\ $\gR_{train} = \gR_{test}$. The inductive setting requires models to induce principles that generalize to new entities, rather than memorizing certain properties of entities. Figure~\ref{fig:inductive} (upper) shows an example of the inductive setting.

In addition to the inductive setting above, we consider a more challenging setting where $\gG_{test}$ and $\gG_{train}$ have completely different entity and relation vocabularies, termed as inductive entity and relation setting. The rationale behind this setup is that knowledge graphs may share some reasoning patterns in common (e.g.\ symmetric rules, composition rules) despite differences in their entity and relation semantics. Figure~\ref{fig:inductive} (bottom) illustrates the inductive entity and relation setting. Since the test graph can be arbitrarily different from the training one, this setting evaluates the ability of generalizing to any knowledge graph.

\subsection{Compositional Generalization}

Compositional generalization is required for answering multi-step queries, since there are an exponential number of combinations in multi-step queries and most of them cannot be covered in the training set. In other words, models should learn skills for individual steps in the training set, and adaptively re-combine skills for these steps to solve a test query. There are two dimensions of compositional generalization: (1) generalizing to new combinations of steps, which is implicitly covered in the evaluation of multi-step queries; (2) generalizing to longer combinations of steps. We explicitly investigate generalization to longer combinations of steps in Chapter~\ref{cha:htt}. Table~\ref{tab:chapter} summarizes the knowledge structure, query structure and generalization studied in each chapter, along with baselines.
\section{Related Work}

In this section, we discuss literature related to reasoning and generalization across structures, grouped by graph representation learning, multi-hop query answering, and reasoning over natural languages. Graph representation learning covers techniques commonly used for learning representations of elements in a graph structure, such as entities, relations, subgraphs or paths. For multi-hop query answering, we discuss neural and neural-symbolic approaches for solving first-order logic queries on knowledge graphs. For reasoning over natural languages, we summarize the latest techniques in finetuning or prompting LLMs for solving reasoning tasks.

\begin{table}[t]
    \centering
    \caption[List of generalization across structure studied in this thesis]{List of knowledge structure, query structure and generalization across structure studied by baselines and chapters in this thesis.}
    \label{tab:chapter}
    \begin{adjustbox}{max width=\textwidth}
        \begin{tabular}{llll}
            \toprule
            & \bf{Knowledge Structure} & \bf{Query Structure} & \bf{Generalization} \\
            \midrule
            Embeddings & Knowledge graph & Single-hop query & New queries \\
            Chapter~\ref{cha:nbfnet} & Knowledge graph & Single-hop query & New entities and queries \\
            Chapter~\ref{cha:ultra} & Knowledge graph & Single-hop query & New entities, relations and queries \\
            \midrule
            Embeddings & Knowledge graph & Multi-hop query & New queries \\
            \multirow{2}{*}{Chapter~\ref{cha:gnn-qe}} & Knowledge graph & Multi-hop query & New entities and queries \\
            & Knowledge graph & Multi-hop query & New entities, relations and queries \\
            \midrule
            CoT & Natural language (latent) & Natural language (multi-step) & New and longer queries\footnotemark \\
            Chapter~\ref{cha:htt} & Natural language (latent) & Natural language (multi-step) & New and longer queries \\
            \bottomrule
        \end{tabular}
    \end{adjustbox}
\end{table}
\footnotetext{While CoT shows some abilities to generalize to new and longer queries, it often generates incorrect answers when solving long queries.}

\subsection{Graph Representation Learning}

Figure~\ref{fig:paradigm} shows all graph representation learning paradigms developed for knowledge graph reasoning. For a comprehensive guide on graph representation learning, we encourage readers to browse the Graph Representation Learning book~\cite{hamilton2020graph}.

\medskip \noindent \textbf{Embedding Methods.} Embedding methods compute the likelihood of a triplet as a function over its entity and relation embeddings. In early methods, the entity embeddings are represented by vectors, while the relations are represented by matrices. For example, SE~\cite{bordes2011learning} scores a triplet as the distance between two entities projected by the relation matrix. RESCAL~\cite{nickel2011three} scores the triplet as a bilinear model over the entities and the relation, which is generalized to non-linear neural networks by NTN~\cite{socher2013reasoning}. However, these models lack regularization for relations and tend to overfit the datasets~\cite{nickel2015review}.

\begin{figure}[t]
    \centering
    \begin{subfigure}{0.3\textwidth}
        \centering
        \includegraphics[width=0.9\textwidth]{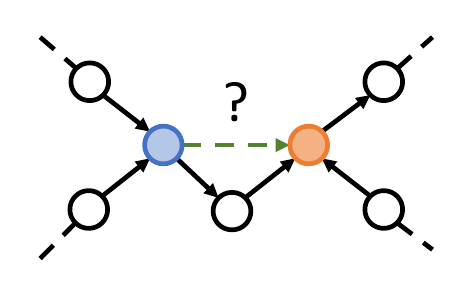}
        \caption{\smaller Knowledge graph reasoning}
        \label{fig:kg_reasoning}
    \end{subfigure}
    \begin{subfigure}{0.3\textwidth}
        \centering
        \includegraphics[width=0.9\textwidth]{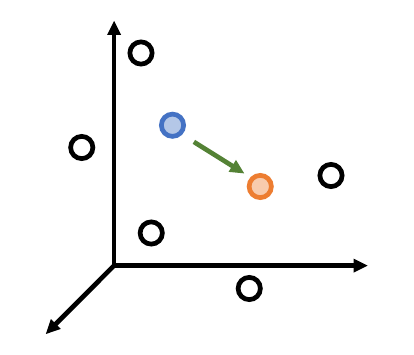}
        \caption{\smaller Embedding methods}
    \end{subfigure} \\
    \begin{subfigure}{0.3\textwidth}
        \centering
        \includegraphics[width=0.9\textwidth]{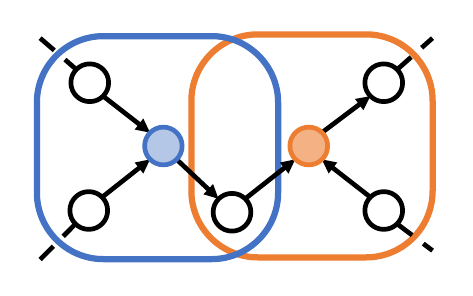}
        \caption{\smaller Node GNN encoders}
        \label{fig:node_gnn}
    \end{subfigure}
    \begin{subfigure}{0.3\textwidth}
        \centering
        \includegraphics[width=0.9\textwidth]{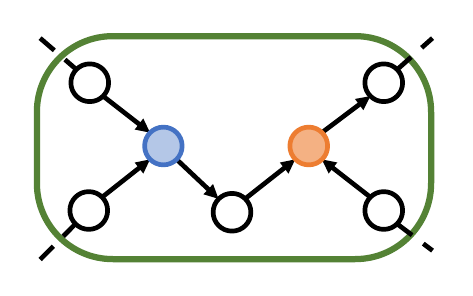}
        \caption{\smaller Subgraph GNN encoders}
        \label{fig:subgraph_gnn}
    \end{subfigure}
    \begin{subfigure}{0.3\textwidth}
        \centering
        \includegraphics[width=0.9\textwidth]{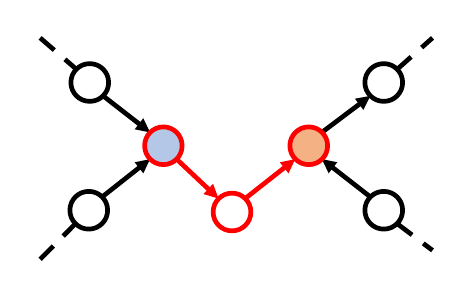}
        \caption{\smaller Path-based methods}
    \end{subfigure}
    \caption[Graph representation learning paradigms for knowledge graph reasoning]{Different graph representation learning paradigms for knowledge graph reasoning. All the methods aim to predict the probability of the \textcolor{mygreen}{green} relation between the \textcolor{myblue}{blue} entity and the \textcolor{myorange}{orange} entity in Figure~\ref{fig:kg_reasoning}.}
    \label{fig:paradigm}
\end{figure}

Later works reduce the number of parameters in such models by defining relations as vector embeddings with the same dimension as entity embeddings. For example, TransE~\cite{bordes2013translating} simplifies the parameterization of SE and interprets relations as a translation in the entity embedding space, and scores triplets based on the distance between translated embeddings and target embeddings. The embeddings are optimized by maximizing the likelihood of observed triplets and minimizing the likelihood of unobserved triplets. Following TransE, a bunch of works improve embedding methods with new score functions that satisfy specific patterns of relations. Here we summarize some prominent works in literature. Table~\ref{tab:score_function} compares the relation patterns that different methods can model.

\smallskip \noindent \textbf{TransE}~\cite{bordes2013translating} represents relations as a translation in the entity embedding space. For a triplet $(h, r, t)$, the entity embedding $\ve_h$ after translation should be close to the entity embedding $\ve_t$ if the triplet is true. Mathematically, the score function can be written as
\begin{equation}
    d_r(h, t) = -||\ve_h + \vr_r - \ve_t||
\end{equation}
where $\vr_r$ is the embedding for relation $r$. TransE is capable of modeling inversion and composition patterns. For instance, consider a pair of inverse relations $r_1$ and $r_2$ (e.g. \emph{husband} and \emph{wife}), we have $d_{r_1}(h, t)$ and $d_{r_2}(t, h)$ hold, which implies $\vr_{r_1} = -\vr_{r_2}$, i.e. inverse relations are modeled as opposite translations in TransE. Similarly, if a relation is equivalent to the composition of two relations (e.g. \emph{uncle} is a composition of \emph{father} and \emph{brother}), its embedding can be represented by the sum of the two translations.

\smallskip \noindent \textbf{DistMult}~\cite{yang2015embedding} is simplified parameterization of bilinear score functions~\cite{nickel2011three}. By restricting the relation matrix to be diagonal, DistMult only requires computation linear to the dimension $d$, achieving the same scalability as TransE. Specifically, DistMult uses the following score function
\begin{equation}
    d_r(h, t) = \ve_h^\top diag(\vr_r) \ve_t = \langle\ve_h, \vr_r, \ve_t\rangle
\end{equation}
From the perspective of relation patterns, DistMult is able to model symmetry patterns (e.g. \emph{friend}), which is a defect of TransE. However, DistMult is always symmetric and cannot model inversion patterns. DistMult can neither model composition patterns due to its product formulation.

\smallskip \noindent \textbf{ComplEx}~\cite{trouillon2016complex} is proposed to solve the limitations of DistMult. \cite{trouillon2016complex} shows that DistMult is equivalent to an eigen decomposition of the adjacency matrix of each relation. Since DistMult learns real embeddings, which can only model real symmetric matrices, ComplEx learns more general complex embeddings to model asymmetric adjacency matrices. The corresponding score function is
\begin{equation}
    d_r(h, t) = \operatorname{Re}(\ve_h^\top diag(\vr_r) \bar{\ve_t}) = \operatorname{Re}(\langle\ve_h, \vr_r, \bar{\ve_t}\rangle)
\end{equation}
Compared to DistMult, ComplEx is capable of modeling symmetry, antisymmetry and inversion relation patterns.

\smallskip \noindent \textbf{SimplE}~\cite{kazemi2018simple} is another work that tackles the symmetric issue of DistMult. It takes the observation that canonical Polyadic (CP) decomposition~\cite{hitchcock1927expression} can handle asymmetric tensor decomposition, but the subject and object embeddings of each entity is learned independently. Therefore, SimplE proposes to jointly learn each relation and its inverse relation. The score function of SimplE is
\begin{equation}
    d_r(h, t) = \langle\ve_h, \vr_r, \ve'_t\rangle + \langle\ve_t, \vr_{r^{-1}}, \ve'_h\rangle
\end{equation}
where $\ve$ and $\ve'$ are two separate embeddings for subjects and objects. If we let $\ve^*_h = [\ve_h, \ve'_h]$  and $\vr^*_r = [\vr_r, \vr_{r^{-1}}]$, SimplE can be rewritten as $\langle\ve^*_h, \vr^*_r, \operatorname{flip}(\ve^*_t)\rangle$ where $\operatorname{flip}(\cdot)$ flips the first and the second half of the embedding. Same as ComplEx, SimplE can handle symmetry, antisymmetry and inversion relation patterns.

\smallskip \noindent \textbf{RotatE}~\cite{sun2019rotate} is built on the motivation that none of the previous methods can model all common relation patterns. Specifically, it proposes to model symmetry, antisymmetry, inversion and composition relation patterns. This is achieved by defining each relation as a rotation in complex space. Mathematically, RotatE uses the following score function
\begin{equation}
    d_r(h, t) = -||\ve_h \odot \vr_r - \ve_t||
\end{equation}
where $\vr_r$ is a vector of unitary complex numbers and $\odot$ is the Hadamard product. $\vr_r$ can be reparameterized by phase vectors to remove the constraint on unitary norm. RotatE handles symmetry patterns by embedding such relations as phase $0$ or $\pi$ in each dimension. Inversion patterns are modeled by conjugate rotations. Because any composition of rotations is still a valid rotation, composition patterns are naturally modeled in RotatE.

\smallskip \noindent \textbf{QuatE}~\cite{zhang2019quaternion} extends RotatE to hypercomplex space, which enjoys two planes of rotations. With Hamilton quaternions, the score function for QuatE is
\begin{equation}
    d_r(h, t) = \ve_h \otimes \frac{\vr_r}{||\vr_r||} \cdot \ve_t
\end{equation}
where $\otimes$ is the Hamilton product. QuatE can model all the relation patterns in RotatE, and additionally supports anticommutativity patterns when composing relations. For example, QuatE can model \emph{father's wife} and \emph{wife's father} as different compositions of \emph{father} and \emph{wife}, which is not possible in previous score functions.

There are some research extending or unifying the search space of score functions in knowledge graph embeddings. For instance, AutoSF~\cite{zhang2020autosf} searches the score function in a tractable subspace of bilinear models. DURA~\cite{zhang2020duality} unifies distance-based models~\cite{bordes2011learning, bordes2013translating, sun2019rotate} and factorization models~\cite{nickel2011three, yang2015embedding, trouillon2016complex, kazemi2018simple, zhang2019quaternion} by viewing distance-based models as factorization models with L2 regularization.

Note all the above methods \textbf{can only be applied to transductive settings}. To extend embedding methods to unseen entities, NodePiece~\cite{galkin2022nodepiece} parameterizes each entity as a function of its incident relations, and optionally a few anchor entities shared by the training and test graphs. It learns embeddings only for these shared elements, and computes the representations of new entities based on the learned embeddings. NodePiece partially solves the challenge of inductive generalization when unseen entities are connected to the training graph. However, for general cases where training and test graphs have disjoint sets of entities, NodePiece does not provide ideal results since many entities cannot be distinguished solely based on their incident relations.

\begin{table}[t]
    \centering
    \caption[Score functions of knowledge graph embeddings]{Different score functions of knowledge graph embeddings and their support patterns. $\langle\cdot\rangle$ denotes the generalized dot product, $\bar{\cdot}$ denotes conjugate for complex vectors, $\odot$ denotes Hadamard product and $\otimes$ denotes Hamilton product.}
    \begin{adjustbox}{max width=\textwidth}
    \begin{tabular}{lcccccc}
        \toprule
        \multirow{2}{*}{\bf{Method}} & \multirow{2}{*}{\bf{Score Function}} & \multirow{2}{*}{\bf{Symmetry}} & \bf{Anti-} & \multirow{2}{*}{\bf{Inversion}} & \multirow{2}{*}{\bf{Composition}} & \bf{Anti-} \\
        & & & \bf{symmetry} & & & \bf{commutativity} \\
        \midrule
        TransE~\cite{bordes2013translating} & $-||\ve_h + \vr_r - \ve_t||$ & \cmark & \cmark & \cmark & \xmark & \xmark \\
        DistMult~\cite{yang2015embedding} & $\langle\ve_h, \vr_r, \ve_t\rangle$ & \cmark & \xmark & \xmark & \xmark & \xmark \\
        ComplEx~\cite{trouillon2016complex} & $\operatorname{Re}(\langle\ve_h, \vr_r, \bar{\ve_t}\rangle)$ & \cmark & \cmark & \cmark & \xmark & \xmark \\
        SimplE~\cite{kazemi2018simple} & $\langle\ve_h, \vr_r, \ve'_t\rangle + \langle\ve_t, \vr_{r^{-1}}, \ve'_h\rangle$ & \cmark & \cmark & \cmark & \xmark & \xmark \\
        RotatE~\cite{sun2019rotate} & $-||\ve_h \odot \vr_r - \ve_t||$ & \cmark & \cmark & \cmark & \cmark & \xmark \\
        QuatE~\cite{zhang2019quaternion} & $\ve_h \otimes \frac{\vr_r}{||\vr_r||} \cdot \ve_t$ & \cmark & \cmark & \cmark & \cmark & \cmark \\
        \bottomrule
    \end{tabular}
    \end{adjustbox}
    \label{tab:score_function}
\end{table}

\medskip \noindent \textbf{Graph Neural Networks.}
Graph neural networks (GNNs)~\cite{scarselli2008graph} are a family of representation learning models that encode topological structures of graphs. Many GNN variants~\cite{kipf2017semi, hamilton2017inductive, velivckovic2018graph, xu2019powerful, wu2019simplifying} have been developed to learn better representations for nodes or graphs. These methods have been adapted to knowledge graphs to learn representations for triplets. Based on the way a triplet representation is defined, we classify existing GNN methods on knowledge graphs into 2 categories, namely node GNN encoders (Figure~\ref{fig:node_gnn}) and subgraph GNN encoders (Figure~\ref{fig:subgraph_gnn}).

\smallskip \noindent \textbf{Node GNN Encoders.}
Node GNN encoders are the most prevalent framework for applying GNNs to knowledge graphs. GAE~\cite{kipf2016variational} and RGCN~\cite{schlichtkrull2018modeling} adopt an auto-encoder formulation, which uses GNNs to encode entity representations, and decodes triplets from entity representations and relation representations with a score function from embedding methods~\cite{bordes2013translating, yang2015embedding, trouillon2016complex, kazemi2018simple, sun2019rotate}. Some methods adopt a variational auto-encoder~\cite{kingma2014auto} formulation to regularize the entity representations with a prior distribution, such as a Gaussian distribution~\cite{kipf2016variational} or a von Mises-Fisher distribution~\cite{davidson2018hyperspherical}. Recent works improve node GNN encoders with advanced GNN architectures for knowledge graphs~\cite{vashishth2020composition, cai2019transgcn}. However, the capacity node GNN encoders is somehow limited, since the two entities in a triplet are encoded independently by GNN. One remedy is to adopt an expressive pooling layer~\cite{kong2022geodesic} over the representations learned by node GNN encoders. Note that node GNN encoders are inductive when each entity has its input features, but are \textbf{not inductive for knowledge graphs without features}.

\smallskip \noindent \textbf{Subgraph GNN Encoders.}
Subgraph GNN encoders~\cite{zhang2018link, teru2020inductive} explicitly encode the subgraph enclosing each query triplet as its representation. Typically, these methods extract a h-hop subgraph around the query entities, label each entity with its distance to the query entities, and learn the representation of the subgraph with a GNN. Subgraph GNN encoders are proved to be more powerful than node GNN encoders~\cite{zhang2021labeling}, and can be naturally applied to the inductive setting~\cite{teru2020inductive}. However, subgraph GNN encoders require to materialize a subgraph for each link, which significantly restricts their scalability for large graphs.

\medskip \noindent \textbf{Path-based Methods.}
Path-based methods have a long history in the literature of reasoning on graphs. Early methods on homogeneous graphs compute the similarity between two nodes based on the weighted count of paths (Katz index~\cite{katz1953new}), random walk probability (personalized PageRank~\cite{page1999pagerank}) or the length of the shortest path (graph distance~\cite{liben2007link}). All these methods define some handcrafted metrics over the full set of paths between two nodes, and can be efficiently solved via some polynomial algorithms (e.g. iterative fixed-point algorithm for PageRank~\cite{page1999pagerank}, Bellman-Ford algorithm for graph distance~\cite{liben2007link}). SimRank~\cite{jeh2002simrank} uses advanced metrics such as the expected meeting distance on homogeneous graphs, which is extended to heterogeneous graphs by PathSim~\cite{sun2011pathsim}.

\begin{figure}[t]
    \centering
    \includegraphics[width=0.55\textwidth]{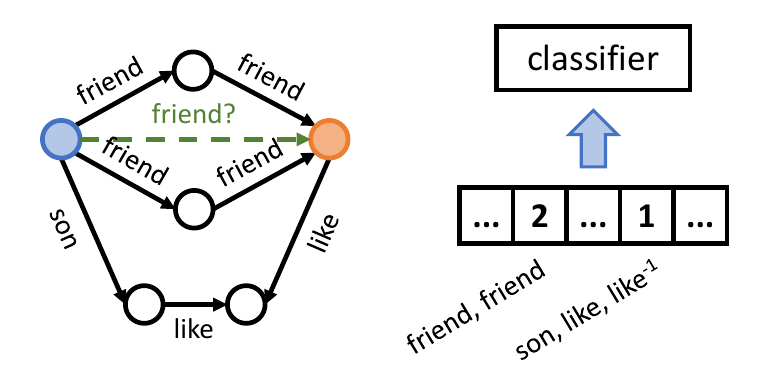}
    \caption[Illustration of Path-Ranking algorithm]{Illustration of Path Ranking algorithm~\cite{lao2010relational}. For a pair of entities, Path Ranking counts different types of relational paths between them to generate a feature vector, which is fed into a classifier to predict the triplet.}
    \label{fig:path_ranking}
\end{figure}

On knowledge graphs, Path Ranking~\cite{lao2010relational} directly uses relational paths between two entities as symbolic features for prediction. Given a query triplet, Path Ranking generates a feature vector based on the number of each type of path. Such a feature vector can be viewed as a handcrafted representation for this pair of entities, and is fed into an SVM~\cite{cortes1995support} to predict the likelihood of the query relation between the entity pair. Figure~\ref{fig:path_ranking} illustrates the algorithm of Path Ranking. Each type of path in Path Ranking can be interpreted as a probabilistic logical rule, weighted by the parameters learned by SVM. Take the case in Figure~\ref{fig:path_ranking} as an example, the logical rules corresponding to the paths are $\textit{friend}(a, b) \land \textit{friend}(b, c) \rightarrow \textit{friend}(a, c)$ and $\textit{son}(a, b) \land \textit{like}(b, c) \land \textit{like}(d, c) \rightarrow \textit{friend}(a, d)$ respectively. However, the number of relational paths is generally exponential w.r.t. the length of path, which restricts the scalability of Path Ranking~\cite{gardner2015efficient}. Moreover, Path Ranking uses handcrafted counting representations for paths, which is likely to be suboptimal for real-world datasets.

Recently, many path-based methods have been developed based on deep neural networks. By connecting them with the Path Ranking algorithm, we can classify these methods into 3 main categories. The first category, logical rules, inherits the rule property of paths in Path Ranking and uses better models to extract logical rules from the data. The second category, path sampler, tries to learn a neural network to sample a sparse set of paths for computing the representation of the query triplet. The last category, path representation learning, enhances Path Ranking by learning the representations for paths with deep models, such as an LSTM~\cite{hochreiter1997long}. Table~\ref{tab:path_based} compares popular path-based methods from the perspective of scalability and representation learning. Some path-based methods~\cite{yang2017differentiable, sadeghian2019drum} model path representations only through relations, and are inductive by construction.

\smallskip \noindent \textbf{Logical Rules.} Logical rule methods learn probabilistic logical rules to weight different paths in the triplet representation. NeuralLP~\cite{yang2017differentiable} uses an LSTM controller to learn the weight for each step in a path, and executes the logical rules on the knowledge graph to get a score for the triplet. Since NeuralLP only uses a scalar score to represent rules and triplets, it can only approximate rule weights with a rank-1 matrix. To address this issue, DRUM~\cite{sadeghian2019drum} uses a vector representation for rules and triplets, which better approximates the true rule weights. Both NeuralLP and DRUM operate on the full set of logical rules, but approximate the rule weights to achieve a polynomial time complexity. On the other hand, RNNLogic~\cite{qu2021rnnlogic} assumes a sparse set of logical rules is useful, and learns the weight for each rule with an LSTM model. For each rule generated, RNNLogic searches all its grounding paths in the knowledge graph to get the representation for the triplet. RNNLogic is more expressive than previous methods in modeling rule weights, but the reasoning algorithm has exponential time complexity and is only feasible when applied to a sparse set of rules.

\begin{table}[t]
    \centering
    \caption[Path-based methods for knowledge graph reasoning]{Path-based methods for knowledge graph reasoning. For simplicity, we group methods and their follow-up variants together when they have the same properties. For a set of paths, full methods are more expressive than sparse ones, but may lack regularization. For complexity, sampled exponential methods are empirically better than (full) exponential methods, and polynomial methods are better than both exponential methods. For path weight and representation, learned ones are more flexible and better than handcrafted ones.}
    \begin{adjustbox}{max width=\textwidth}
    \begin{tabular}{llcccc}
        \toprule
        \bf{Class} & \bf{Method} & \bf{Set of Paths} & \bf{Complexity} & \bf{Path Weight} & \bf{Path Representation} \\
        \midrule
        \multirow{5}{*}{\bf{\shortstack[l]{Traditional\\Methods}}}
        & Katz~\cite{katz1953new}      & Full         & Polynomial & Handcrafted   & Handcrafted           \\
        & Personalized PageRank~\cite{page1999pagerank} & Full         & Polynomial & Handcrafted   & Handcrafted           \\
        & Graph Distance~\cite{liben2007link}        & Full         & Polynomial & Handcrafted   & Handcrafted           \\
        & SimRank~\cite{jeh2002simrank, sun2011pathsim} & Full & Polynomial & Handcrafted & Handcrafted \\
        & Path Ranking~\cite{lao2010relational, gardner2015efficient} & Full         & Exponential & Learned    & Handcrafted           \\
        \midrule
        \multirow{2}{*}{\bf{\shortstack[l]{Logical\\Rules}}}
        & NeuralLP~\cite{yang2017differentiable, sadeghian2019drum}      & Full         & Polynomial & Learned     & Handcrafted           \\
        & RNNLogic~\cite{qu2021rnnlogic}              & Sparse       & Sampled exponential & Learned    & Handcrafted           \\
        \midrule
        \multirow{3}{*}{\bf{\shortstack[l]{Path Sampler}}}
        & DeepPath~\cite{xiong2017deeppath, lin2018multi, shen2018m, hildebrandt2020reasoning} & Sparse & Polynomial & Learned    & Handcrafted           \\
        & MINERVA~\cite{das2018go} & Sparse       & Polynomial & Learned    & Handcrafted           \\
        & DIVA~\cite{chen2018variational} & Sparse & Polynomial & Learned    & Learned             \\
        \midrule
        \multirow{3}{*}{\bf{\shortstack[l]{Path\\Representation\\Learning}}}
        & Path-RNN~\cite{neelakantan2015compositional, das2017chains} & Full         & Exponential & Learned    & Learned             \\
        & PathCon~\cite{wang2021entity} & Full         & Exponential & Learned    & Learned             \\
        & All-Paths~\cite{toutanova2016compositional} & Full         & Polynomial  & Learned    & Learned             \\
        \midrule
        \bf{Ours}
        & NBFNet~\cite{zhu2021neural} & Full         & Polynomial  & Learned    & Learned             \\
        \bottomrule
    \end{tabular}
    \end{adjustbox}
    \label{tab:path_based}
\end{table}

\smallskip \noindent \textbf{Path Samplers.}
Path samplers learn to sample important paths for computing the triplet representation, typically by reinforcement learning. For example, DeepPath~\cite{xiong2017deeppath} and MINERVA~\cite{das2018go} learn an agent to walk on a knowledge graph and collect useful paths. However, these methods suffer from extremely sparse rewards, since most sampled paths cannot even reach the tail entity. Several followup works mitigated such an issue by engineering the reward function~\cite{lin2018multi} and the search strategy~\cite{shen2018m}, or learning separate agents for positive and negative paths. DIVA~\cite{chen2018variational} extends DeepPath to a variational auto-encoder~\cite{kingma2014auto} framework, where the prior and the posterior distributions provide path samples, and the likelihood model learns the reward function. The time complexity of these methods is proportional to the number of sampled paths, which is polynomial given a fix sampling budget.

\smallskip \noindent \textbf{Path Representation Learning.}
The above two categories are focused on extracting important subsets of paths. However, most of the above methods use simple representations, such as counting, to represent the query triplet. Another stream of work learn the representations of paths with deep neural networks. For example, Path-RNN~\cite{neelakantan2015compositional, das2017chains} enumerates all the paths between two nodes, and encodes each of them with an RNN. The triplet representation is then obtained by aggregating the set of path representations. Recently, PathCon~\cite{wang2021entity} jointly learns the entity context and the path representations, and combines them into the triplet representation. These methods learn good representations for paths, but they require exponential time complexity, and are usually limited to very short paths, e.g. $\leq$3 edges. The only exception is All-Paths~\cite{toutanova2016compositional}, which uses a bilinear representation for each path and can be solved via dynamic programming in polynomial time.

\subsection{Multi-hop Query Answering.}

We divide multi-hop query answering methods into two categories, neural methods and neural-symbolic methods, based on how they represent intermediate variables. Figure~\ref{fig:complex_query} illustrate these categories of methods. We refer readers to our survey~\cite{ren2023neural} for a comprehensive discussion on multi-hop query answering.

\smallskip \noindent \textbf{Neural methods.} Neural methods represent operations and intermediate variables in a query with learned embeddings. MPQE~\cite{daza2020message} learns a representation for the query graph using RGCN~\cite{schlichtkrull2018modeling}, and select the closest entity based on the cosine similarity between the query representation and entity embeddings. In \cite{guu2015traversing}, the authors proposed compositional training to train cascades of relation projections for answering path queries. Extending the query types to conjunctive queries ($\land$), GQE~\cite{hamilton2018embedding} learns a geometric intersection operator $\gI$ to model conjunctions
\begin{equation}
    \gI(\{\vq_1, ..., \vq_n\}) = \mW \bigoplus_{i=1}^n \text{MLP}(\vq_i)
    \label{eqn:deepset}
\end{equation}
where $\vq_i$ are the embeddings of partial queries involved in the conjunction, $MLP$ is a multi-layer perceptron and $\bigoplus$ is a symmetric vector function (e.g.\ mean or min over a set of vectors) followed by a learnable transformation matrix $\mW$. Such an operator is known as DeepSets and is invariant to the permutation of its input~\cite{zaheer2017deep}. Following GQE, later works try to inject more geometric inductive bias into logical operators to achieve better performance, such as Query2Box~\cite{ren2020query2box} and BetaE~\cite{ren2020beta}. Query2Box~\cite{ren2020query2box} represents each intermediate variable as a high-dimensional box with a center embedding and an offset embedding, resulting in the following intersection operator
\begin{align}
    &\gI(\{\vq_1, ..., \vq_n\})^{\text{center}} = \sum_i a_i \odot \vq_i^{\text{center}} \qquad a_i = \frac{\exp(\text{MLP}(\vq_i))}{\sum_j \exp(\text{MLP}(\vq_j))} \\
    &\gI(\{\vq_1, ..., \vq_n\})^{\text{offset}} = \min({\vq_1^{\text{offset}}, ..., \vq_n^{\text{offset}}}) \odot \sigma(\text{DeepSet}({\vq_1, ..., \vq_n}))
\end{align}
where $\odot$ is element-wise multiplication,$\sigma(\cdot)$ is the sigmoid function and DeepSet refers to the architecture in Equation~\ref{eqn:deepset}. Query2Box aligns better with the intuition of conjunction, since intersection of boxes is always a box. To solve existential positive first-order (EPFO) queries ($\exists$, $\land$, $\lor$), the authors rewrite them into disjunctive normal form (DNF), i.e.\ disjunction of conjunctive queries, where answers can be obtained by solving each conjunctive branch and merging the predictions. BetaE~\cite{ren2020beta} further models negation operators and extends neural methods to first-order logic (FOL) queries ($\exists$, $\land$, $\lor$, $\neg$) by parameterizing each intermediate variable with a Beta distribution.

Contrary to methods that process multi-hop queries step by step, CQD-CO~\cite{arakelyan2021complex} formulates multi-hop queries as a structure optimization problem and adopts pre-trained knowledge graph embeddings~\cite{trouillon2016complex} to compute scores for each hop. Specifically, CQD-CO maximizes the following objective for embeddings of each variable $\ve^i_j$ in the query
\begin{equation}
    \hspace{-0.8em}\argmax_{\ve^i_j \in \sR^k} \left(\phi_{r_1}(\vh^1, \ve^1_1) \top ... \top \phi_{r_{n_1}}(\ve^1_{n_1-1}, \ve^1_{n_1})\right) \perp ... \perp \left(\phi_{r_1}(\vh^d, \ve^d_1) \top ... \top \phi_{r_{n_d}}(\ve^d_{n_d-1}, \ve^d_{n_d})\right)
    \label{eqn:cqd}
\end{equation}
where $\phi_{r_i}(\cdot, \cdot) \in [0, 1]$ is the score function for entity embeddings based on relation $r_i$, and $\vh^i$ corresponds to the pre-trained embeddings of the constant entities given in the query. $\top$ and $\perp$ are a t-norm and a t-conorm respectively, which are continuous generalization of boolean conjunction and disjunction for variables between $[0, 1]$. The embeddings of variables can be optimized via gradient-based methods, such as Adam~\cite{kingma2015adam}, and the final answer can be obtained by replacing the target embedding with the pre-trained embeddings of all entities to maximize the objective.

There are some other works improving the design of logical operators in neural methods. FuzzQE~\cite{chen2022fuzzy} leverages t-norm fuzzy logic to model FOL queries, which satisfies the axiomatic system of classical logic. Some recent works utilize advanced geometric embeddings to achieve desired properties for operators, e.g. hyperboloid embeddings in HypE~\cite{choudhary2021self} and cone embeddings in ConE~\cite{zhang2021cone}. Generally, all these methods compute embeddings for intermediate variables without aligning them with entities in the knowledge graph, which limits their interpretability. Besides, these methods require entity embeddings and \textbf{do not generalize to new knowledge structures}.

\begin{figure}[!h]
    \centering
    \begin{subfigure}[M]{0.58\textwidth}
        \centering
        \includegraphics[width=0.9\textwidth]{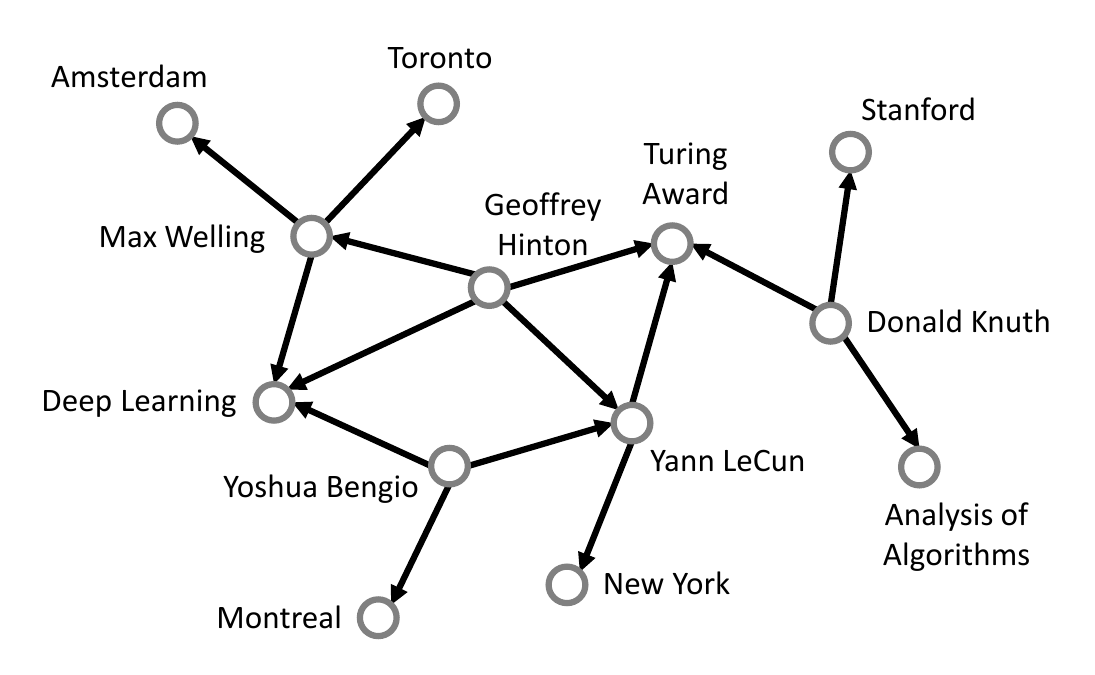}
        \caption{\smaller An incomplete knowledge graph}
    \end{subfigure}
    \begin{subfigure}[M]{0.34\textwidth}
        \centering
        \vspace{-1em}
        \includegraphics[width=0.9\textwidth]{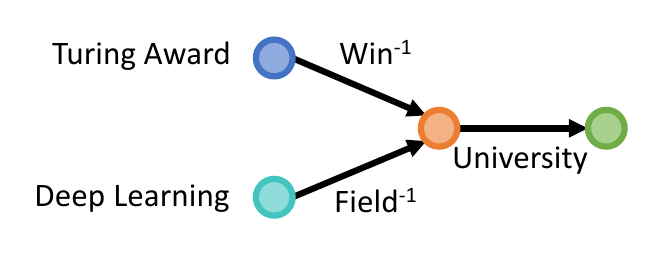}
        \caption{\smaller A multi-hop query}
    \end{subfigure} \\[-4em]
    \hfill
    \begin{subfigure}[b]{0.3\linewidth}
        \centering
        \includegraphics[width=0.9\textwidth]{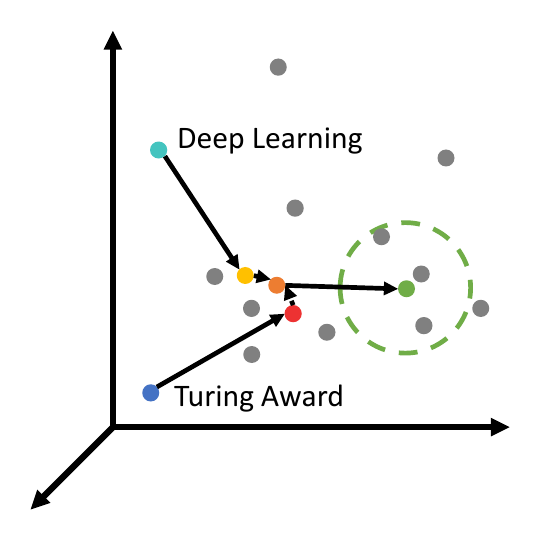}
        \caption{\smaller Neural methods}
        \label{fig:neural_methods}
    \end{subfigure}
    \hspace{0.04\textwidth}
    \begin{minipage}[b]{0.50\textwidth}
        \centering
        \begin{subfigure}{0.79\linewidth}
            \centering
            \includegraphics[width=0.9\textwidth]{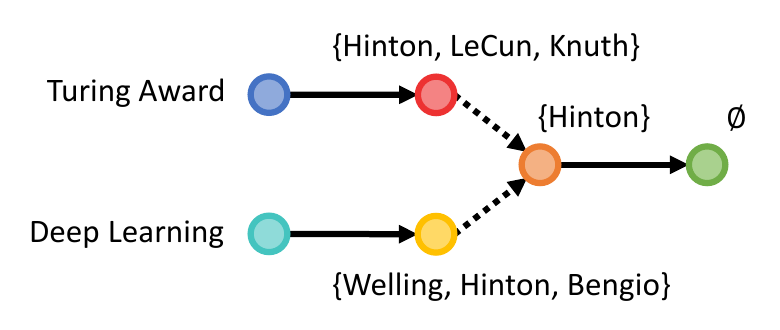}
            \caption{\smaller Symbolic methods}
            \label{fig:symbolic_methods}
        \end{subfigure} \\[1em]
        \begin{subfigure}{\linewidth}
            \centering
            \includegraphics[width=0.9\textwidth]{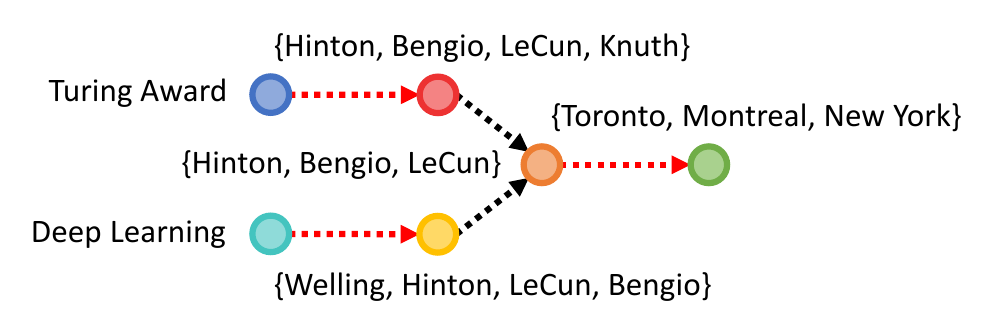}
            \caption{\smaller Neural-symbolic methods}
            \label{fig:symbolic_methods}
        \end{subfigure}
    \end{minipage}
    \caption[Methods for multi-hop query answering on an incomplete knowledge graph]{Different methods for multi-hop query answering on an incomplete knowledge graph. Neural methods can deal with missing links, but are not interpretable. Symbolic methods are interpretable and generalize across query structures, but fail to deal with missing links. Neural-symbolic methods combine the advantages of both worlds.}
    \label{fig:complex_query}
\end{figure}

\smallskip \noindent \textbf{Neural-symbolic methods.}
Neural-symbolic methods adopt embeddings to model each hop in multi-hop queries, and meanwhile take the symbolic constraint of entity assignment into consideration, providing interpretability for intermediate variables. EmQL~\cite{sun2020faithful} simultaneously maintains an embedding and a count-min sketch, i.e.\ a hash compression of a set of entities. To decode an intermediate variable or an answer, EmQL first finds the top-k entities with the highest dot product with the embedding, and then filters these entities using the sketch, which helps it to find answers logically entailed by the knowledge graph.

CQD-Beam~\cite{arakelyan2021complex} extends pre-trained knowledge graph embeddings to infer answers for complex queries based on beam search. Using the objective in Equation~\ref{eqn:cqd}, CQD-Beam searches the top-k assignments from all entities for each variable in the multi-hop query. To eliminates the approximation caused by beam search, QTO~\cite{bai2023answering} computes Equation~\ref{eqn:cqd} via dynamic programming to cover the whole space of variable assignments, and the optimal variable assignments can be extracted by backtracking the search states.

\subsection{Reasoning over Natural Languages}

Solving reasoning problems in natural language can be traced back to the bAbI benchmark~\cite{weston2016towards}, which consists of many multi-step reasoning tasks that evaluate deduction, induction and other reasoning abilities. Early attempts to solve bAbI designed models to read and write a differentiable memory component when processing the input~\cite{weston2015memory, kumar2016ask}. With the rise of Transformer architecture~\cite{vaswani2017attention} and later pretrained language models~\cite{radford2018improving, devlin2019bert, raffel2020exploring}, several works have demonstrated that Transformers can be finetuned on specific datasets to acquire reasoning abilities. \cite{clark2020transformers} finetunes Transformers to answer deductive questions based on input facts and rules. \cite{talmor2020leap} shows that Transformers are able to combine input knowledge with their implicit knowledge in reasoning tasks. \cite{nye2021show} finds that Transformers can perform program execution when finetuned with program traces.

On the other hand, the tremendous scale of LLMs enables emergent reasoning abilities in the form of prompting. In GPT-3~\cite{brown2020language}, the authors find LLMs can answer reasoning questions when primed with a few examples, an ability referred to as in-context learning or few-shot prompting. By finetuning these models on a collection of tasks with instructions, they can directly produce answers to questions in a zero-shot manner~\cite{wei2022finetuned, sanh2022multitask}. However, both few-shot and zero-shot prompting fall short of predicting the right answer for multi-step queries. Chain-of-Thought (CoT)~\cite{wei2022chain} elicits the ability of multi-step reasoning with intermediate steps in few-shot examples. Zero-shot CoT~\cite{kojima2022large} shows that a similar ability can be triggered in zero-shot prompting with a magic instruction \emph{let's think step by step}, eliminating the need of engineering the few-shot examples. Figure~\ref{fig:prompting} shows an illustration of these prompting methods on a grade school math problem.

Many works have been developed to extend CoT with certain inductive biases for reasoning. Least-to-most (LtM) prompting~\cite{zhou2023least} and decomposed prompting~\cite{khot2023decomposed} explicitly decompose multi-step reasoning into sub tasks and solve each sub task in isolation. For tasks that cannot be solved by greedy reasoning, ToT~\cite{yao2023tree} and RAP~\cite{hao2023reasoning} navigate through possible reasoning steps and expand states that are most likely to lead to the correct answer. \cite{wang2024unleashing} employs multiple LLM instances equipped with different personas to solve a task collaboratively. \cite{lightman2023let} shows that learning a verifier to check each step significantly improves the performance of multi-step reasoning. All these works rely on parameteric knowledge stored in the LLM's weights~\cite{petroni2019language}. Some other works consider taking explicit facts and rules as input, and searching possible proof traces through forward chaining~\cite{creswell2023selection} or backward chaining~\cite{kazemi2023lambada}.

\begin{figure}[t]
    \centering
    \begin{subfigure}{0.45\textwidth}
        \centering
        \includegraphics[width=0.95\textwidth]{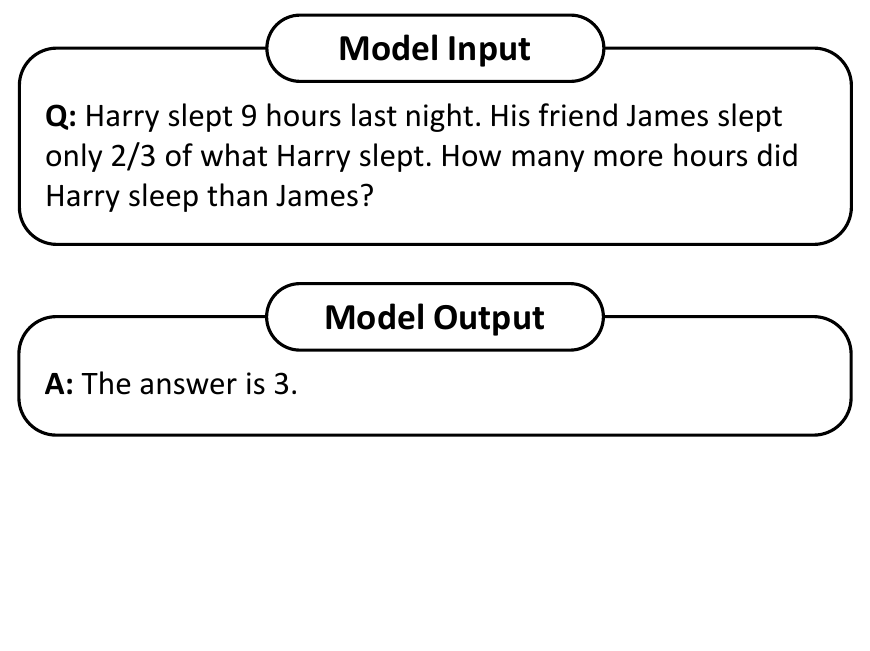} \\
        \caption{\smaller Zero-shot prompting}
        \label{fig:zero-shot}
    \end{subfigure}
    \begin{subfigure}{0.45\textwidth}
        \centering
        \includegraphics[width=0.95\textwidth]{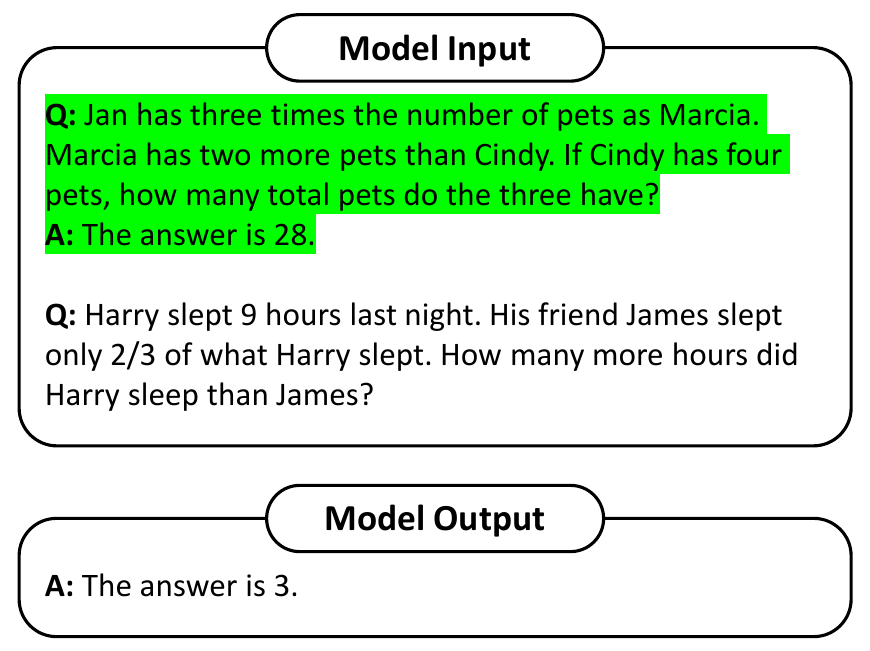} \\
        \caption{\smaller Few-shot prompting}
        \label{fig:few-shot}
    \end{subfigure} \\[0.5em]
    \begin{subfigure}{0.45\textwidth}
        \centering
        \includegraphics[width=0.95\textwidth]{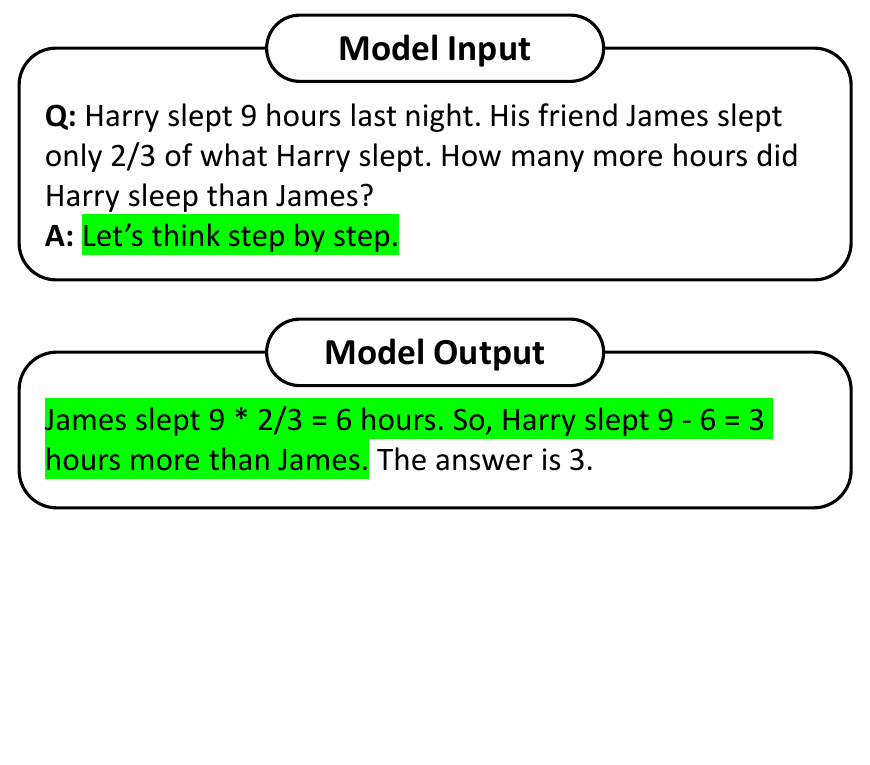}
        \caption{\smaller Zero-shot Chain-of-Thought}
        \label{fig:zero-shot_cot}
    \end{subfigure}
    \begin{subfigure}{0.45\textwidth}
        \centering
        \includegraphics[width=0.95\textwidth]{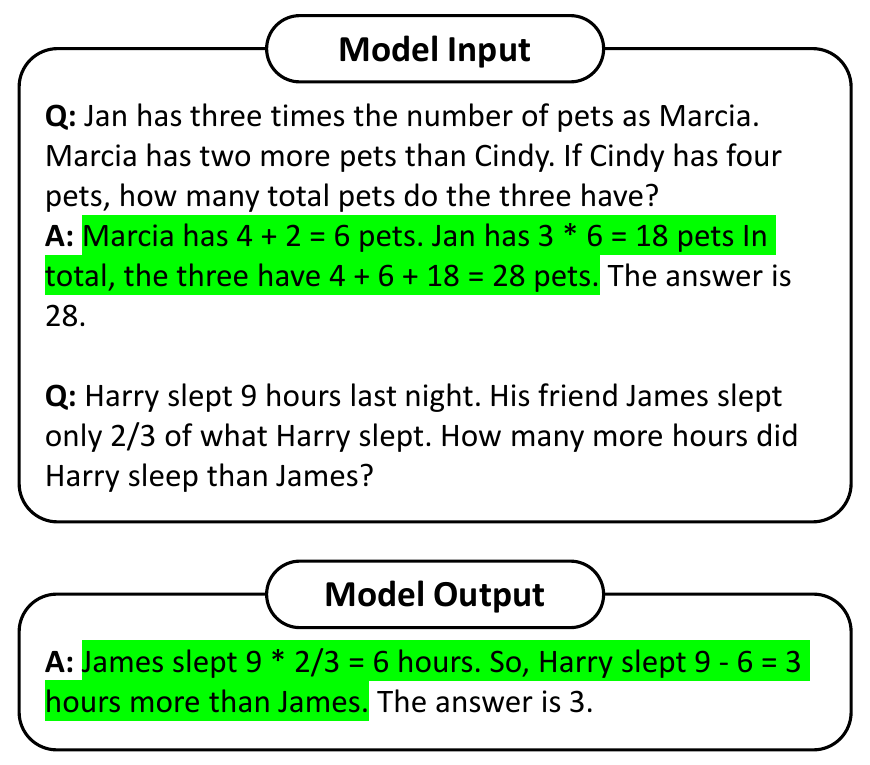}
        \caption{\smaller Few-shot Chain-of-Thought}
        \label{fig:few-shot_cot}
    \end{subfigure}
    \caption[Common prompting methods for solving reasoning tasks]{Common prompting methods for solving reasoning tasks. Major changes in each method are marked in \textcolor{greenmarker}{green}. Base LLMs can follow examples and perform few-shot prompting, while instruction-finetuned models can directly answer questions in a zero-shot format. For multi-step queries, few-shot CoT significantly improves the performance over prompting methods that directly predict the answer. Zero-shot CoT simplifies the engineering efforts of prompting, yet retains most performance of few-shot CoT.}
    \label{fig:prompting}
\end{figure}

\part{Inductive Representations}
\label{part:inductive}
\chapter[Representation Learning for Generalizing to Unseen Entities]{Representation Learning for\linebreak Generalizing to Unseen Entities}
\label{cha:nbfnet}

Knowledge graph reasoning has long been dominated by embedding methods due to their simplicity. However, embedding methods always have to be re-trained whenever the underlying knowledge graph is updated. While there are a few attempts~\cite{zhang2018link, teru2020inductive} on encoding local subgraphs for generalization to new entities, they are largely limited by their scalability. Is there an inductive, strong, yet scalable model that can substitute embedding methods in most scenarios?

In this chapter, we propose such a solution, NBFNet, by combining representation learning with traditional path-based methods and dynamic programming. We further improve the scalability of NBFNet with a learned priority function in each iteration, resulting in A*Net. NBFNet achieved significantly better performance than embedding methods, and A*Net extended such an advantage to million-scale knowledge graphs. As of the year 2024, NBFNet remains a strong baseline in knowledge graph reasoning, and many of its insights, such as path representations and efficient computation via dynamic programming, continue to benefit various reasoning tasks.

\smallskip \emph{This chapter is based on our work published at NeurIPS 2021~\cite{zhu2021neural}\footnote{The code of NBFNet is available at \url{https://github.com/DeepGraphLearning/NBFNet}} and NeurIPS 2023~\cite{zhu2023net}\footnote{The code of A*Net is available at  \url{https://github.com/DeepGraphLearning/AStarNet}}. Xinyu Yuan has contributed significantly to \cite{zhu2023net}.}

\section{Overview}

Predicting the interactions between nodes (a.k.a.\ link prediction) is a fundamental task in the field of graph machine learning. Given the ubiquitous existence of graphs, such a task has many applications, such as recommender system~\cite{koren2009matrix}, knowledge graph completion~\cite{nickel2015review} and drug repurposing~\cite{ioannidis2020few}.

Traditional methods of link prediction usually define different heuristic metrics over the paths between a pair of nodes. For example, Katz index~\cite{katz1953new} is defined as a weighted count of paths between two nodes. Personalized PageRank~\cite{page1999pagerank} measures the similarity of two nodes as the random walk probability from one to the other. Graph distance~\cite{liben2007link} uses the length of the shortest path between two nodes to predict their association.
These methods can be directly applied to new graphs, i.e.\ inductive setting, enjoy good interpretability and scale up to large graphs. However, they are designed based on handcrafted metrics and may not be optimal for link prediction on real-world graphs.

To address these limitations, some link prediction methods adopt graph neural networks (GNNs)~\cite{kipf2016variational, schlichtkrull2018modeling, vashishth2020composition} to automatically extract important features from local neighborhoods for link prediction. Thanks to the high expressiveness of GNNs, these methods have shown state-of-the-art performance. However, these methods can only be applied to predict new links on the training graph, i.e. transductive setting, and lack interpretability. While some recent methods~\cite{zhang2018link, teru2020inductive} extract features from local subgraphs with GNNs and support inductive setting, the scalability of these methods is compromised.

Therefore, we wonder if there exists an approach that enjoys the advantages of both traditional path-based methods and recent approaches based on graph neural networks, i.e.\ \textbf{generalization in the inductive setting}, \textbf{interpretability}, \textbf{high model capacity} and \textbf{scalability}. To this end, we propose a representation learning framework, Neural Bellman-Ford Networks (NBFNet), along with a more efficient variant, A* Networks (A*Net).

\smallskip \noindent \textbf{Neural Bellman-Ford Networks.}
Inspired by traditional path-based methods, our goal is to develop a general and flexible representation learning framework for link prediction based on the paths between two nodes. Specifically, we define the representation of a pair of nodes as the \emph{generalized sum} of all the path representations between them, where each path representation is defined as the \emph{generalized product} of the edge representations in the path. Many link prediction methods, such as Katz index~\cite{katz1953new}, personalized PageRank~\cite{page1999pagerank}, graph distance~\cite{liben2007link}, as well as graph theory algorithms like widest path~\cite{baras2010path} and most reliable path~\cite{baras2010path}, are special instances of this path formulation with different \emph{summation} and \emph{multiplication} operators. Motivated by the polynomial-time algorithm for the shortest path problem~\cite{bellman1958routing}, we show that such a formulation can be efficiently solved via the generalized Bellman-Ford algorithm~\cite{baras2010path} under mild conditions and scale up to large graphs.

The operators in the generalized Bellman-Ford algorithm---\emph{summation} and \emph{multiplication}---are handcrafted, which have limited flexibility. Therefore, we further propose NBFNet, a graph neural network framework that solves the above path formulation with learned operators in the generalized Bellman-Ford algorithm. Specifically, NBFNet parameterizes the generalized Bellman-Ford algorithm with three neural components, namely \textsc{Indicator}, \textsc{Message} and \textsc{Aggregate} functions. The \textsc{Indicator} function initializes a representation on each node, which is taken as the boundary condition of the generalized Bellman-Ford algorithm. The \textsc{Message} and the \textsc{Aggregate} functions learn the \emph{multiplication} and \emph{summation} operators respectively.

We show that the \textsc{Message} function can be defined according to the relational operators in knowledge graph embeddings~\cite{bordes2013translating, yang2015embedding, trouillon2016complex, kazemi2018simple, sun2019rotate}, e.g.\ as a translation in Euclidean space induced by the relational operators of TransE~\cite{bordes2013translating}. The \textsc{Aggregate} function can be defined as learnable set aggregation functions~\cite{zaheer2017deep, xu2019powerful, corso2020principal}. With such parameterization, NBFNet can generalize to the inductive setting, meanwhile achieve one of the lowest time complexity among inductive GNN methods. A comparison of NBFNet and other GNN frameworks for link prediction is showed in Table~\ref{tab:comparison}. With other instantiations of \textsc{Message} and \textsc{Aggregate} functions, our framework can also recover some existing works on learning logic rules~\cite{yang2017differentiable, sadeghian2019drum} for link prediction on knowledge graphs (Table~\ref{tab:semiring_instance}).

Our NBFNet framework can be applied to several link prediction variants, covering not only single-relational graphs (e.g.\ homogeneous graphs) but also multi-relational graphs (e.g.\ knowledge graphs). We empirically evaluate the proposed NBFNet for link prediction on homogeneous graphs and knowledge graphs in both transductive and inductive settings. Experimental results show that the proposed NBFNet outperforms existing state-of-the-art methods by a large margin in all settings, with an average relative performance gain of 18\% on knowledge graph completion (HITS@1) and 22\% on inductive relation prediction (HITS@10). We also show that the proposed NBFNet is indeed interpretable by visualizing the top-k relevant paths for link prediction on knowledge graphs.

\begin{table}[!h]
    \centering
    \caption[Comparison of GNN frameworks for link prediction]{Comparison of GNN frameworks for link prediction. The time complexity refers to the \emph{amortized time} for predicting a single edge or triplet. $|\gV|$ is the number of nodes, $|\gE|$ is the number of edges, and $d$ is the dimension of representations. The wall time is measured on FB15k-237 test set with 40 CPU cores and 4 GPUs.}
    \label{tab:comparison}
    \begin{adjustbox}{max width=\textwidth}
    \begin{tabular}{lccccc}
         \toprule
         \bf{Method} & \bf{Inductive}\footnotemark & \bf{Interpretable} & \bf{Learned Representation}  & \bf{Time Complexity} & \bf{Wall Time} \\
         \midrule
         VGAE~\cite{kipf2016variational} / &  &  & \multirow{2}{*}{\checkmark} & \multirow{2}{*}{$O(d)$} & \multirow{2}{*}{18 secs} \\
         RGCN~\cite{schlichtkrull2018modeling} \\
         NeuralLP~\cite{yang2017differentiable} / & \multirow{2}{*}{\checkmark} & \multirow{2}{*}{\checkmark} &  & \multirow{2}{*}{$O\left(\frac{|\gE|d}{|\gV|} + d^2\right)$} & \multirow{2}{*}{2.1 mins} \\
         DRUM~\cite{sadeghian2019drum} \\
         SEAL~\cite{zhang2018link} / & \multirow{2}{*}{\checkmark} & & \multirow{2}{*}{\checkmark} & \multirow{2}{*}{$O(|\gE|d^2)$} & \multirow{2}{*}{$\approx$1 month} \\
         GraIL~\cite{teru2020inductive} \\
         \midrule
         NBFNet & \checkmark & \checkmark & \checkmark & $O\left(\frac{|\gE|d}{|\gV|} + d^2\right)$ & 4.0 mins \\
         \bottomrule
    \end{tabular}
    \end{adjustbox}
\end{table}
\footnotetext{We consider the inductive setting where a model can generalize to entirely new graphs without node features.}

\begin{wrapfigure}{r}{0.44\textwidth}
    \centering
    \vspace{-0.7em}
    \includegraphics[width=0.44\textwidth]{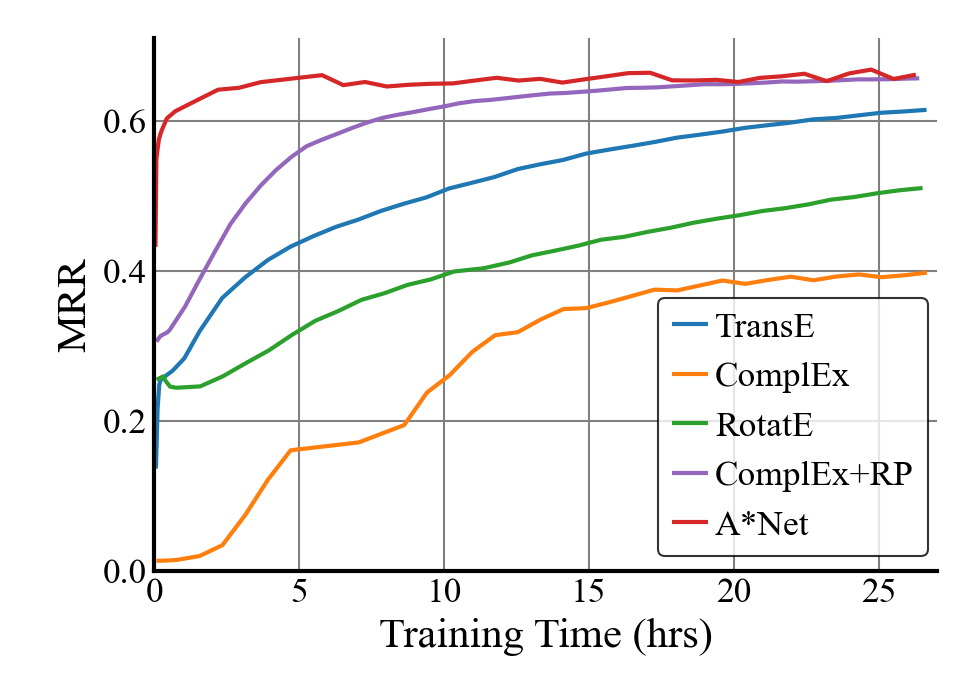}
    \caption[Results w.r.t. training time of A*Net on ogbl-wikikg2]{Validation MRR w.r.t.\ training time on ogbl-wikikg2 (1 A100 GPU). A*Net achieves state-of-the-art performance and the fastest convergence.}
    \label{fig:ogbl-wikikg2}
    \vspace{-0.5em}
\end{wrapfigure}
\smallskip \noindent \textbf{A* Networks.}
We propose A*Net to tackle the scalability issue of path-based methods. Our key idea is to search for important paths rather than use all possible paths for reasoning, thereby reducing time and memory in training and inference. Inspired by the A* algorithm~\cite{hart1968formal} for shortest path problems, given a head entity $u$ and a query relation $q$, we compute a priority score for each entity to guide the search towards more important paths. At each iteration, we select $K$ nodes and $L$ edges according to their priority, and use message passing to update nodes in their neighborhood. Due to the complex semantics of knowledge graphs, it is hard to use a handcrafted priority function like the A* algorithm without a significant performance drop (Table~\ref{tab:handcraft}). Instead, we design a neural priority function based on the node representations at the current iteration, which can be trained end-to-end by the objective function of the reasoning task without any additional supervision.

We evaluate our method on 4 transductive and 2 inductive knowledge graph datasets. Experiments show that A*Net achieves competitive performance against state-of-the-art path-based methods on FB15k-237, WN18RR and YAGO3-10, even with only 10\% of nodes and 10\% edges at each iteration (Section~\ref{sec:result}). To verify the scalability of our method, we also evaluate A*Net on ogbl-wikikg2, a million-scale knowledge graph that is 2 magnitudes larger than datasets solved by previous path-based methods. Surprisingly, with only 0.2\% nodes and 0.2\% edges, our method outperforms existing embedding methods and establishes new state-of-the-art results (Section~\ref{sec:result}) as the first non-embedding method on ogbl-wikikg2. By adjusting the ratios of selected nodes and edges, one can trade off between performance and efficiency (Section~\ref{sec:ablation_astarnet}). A*Net also converges significantly faster than embedding methods (Figure~\ref{fig:ogbl-wikikg2}), which makes it a promising model for deployment on large-scale knowledge graphs. Additionally, A*Net offers interpretability that embeddings do not possess. Visualization shows that A*Net captures important paths for reasoning (Section~\ref{sec:vis_astarnet}).

\section{Method: NBFNet}

In this section, we first define a path formulation for link prediction. Our path formulation generalizes several traditional methods, and can be efficiently solved by the generalized Bellman-Ford algorithm. Then we propose Neural Bellman-Ford Networks to learn the path formulation with neural functions.

\subsection{Path Formulation for Link Prediction}
\label{sec:path_formulation}

We consider the link prediction problem on both knowledge graphs and homogeneous graphs. \pagebreak[0] A knowledge graph is denoted by $\gG=(\gV, \gE, \gR)$, where $\gV$ and $\gE$ represent the set of entities (nodes) and relations (edges) respectively, and $\gR$ is the set of relation types. We use $\gN(u)$ to denote the set of nodes connected to $u$, and $\gE(u)$ to denote the set of edges ending with node $u$. A homogeneous graph $\gG=(\gV, \gE)$ can be viewed as a special case of knowledge graphs, with only one relation type for all edges. Throughout this paper, we use \textbf{bold} terms, $\vw_q(e)$ or $\vh_q(u, v)$, to denote vector representations, and \textit{italic} terms, $w_e$ or $w_{uv}$, to denote scalars like the weight of edge $(u, v)$ in homogeneous graphs or triplet $(u, r, v)$ in knowledge graphs. Without loss of generality, we derive our method based on knowledge graphs, while our method can also be applied to homogeneous graphs.

\smallskip \noindent \textbf{Path Formulation.} Link prediction is aimed at predicting the existence of a query relation $q$ between a head entity $u$ and a tail entity $v$. From a representation learning perspective, this requires to learn a pair representation $\vh_q(u, v)$, which captures the local subgraph structure between $u$ and $v$ w.r.t.\ the query relation $q$. In traditional methods, such a local structure is encoded by counting different types of random walks from $u$ to $v$~\cite{lao2010relational, gardner2015efficient}. Inspired by this construction, we formulate the pair representation as a \emph{generalized sum} of path representations between $u$ and $v$ with a commutative \emph{summation} operator $\oplus$. Each path representation $\vh_q(P)$ is defined as a \emph{generalized product} of the edge representations in the path with the \emph{multiplication} operator $\otimes$.
\begin{align}
    \restatableeq{\link}{&\vh_q(u, v) = \vh_q(P_1) \oplus \vh_q(P_2) \oplus ... \oplus \vh_q(P_{|\gP_{uv}|}) \vert_{P_i \in \gP_{uv}} \triangleq \bigoplus_{P \in \gP_{uv}} \vh_q(P)}{eqn:link} \\
    \restatableeq{\path}{&\vh_q(P = (e_1, e_2, ..., e_{|P|})) = \vw_q(e_1) \otimes \vw_q(e_2) \otimes ... \otimes \vw_q(e_{|P|}) \triangleq \bigotimes_{i=1}^{|P|} \vw_q(e_i)}{eqn:path}
\end{align}
where $\gP_{uv}$ denotes the set of paths from $u$ to $v$ and $\vw_q(e_i)$ is the representation of edge $e_i$. Note the \emph{multiplication} operator $\otimes$ is not required to be commutative (e.g.\ matrix multiplication), therefore we define $\bigotimes$ to compute the product following the exact order. Intuitively, the path formulation can be interpreted as a depth-first-search (DFS) algorithm, where one searches all possible paths from $u$ to $v$, computes their representations (Equation~\ref{eqn:path}) and aggregates the results (Equation~\ref{eqn:link}). Such a formulation is capable of modeling several traditional link prediction methods, as well as graph theory algorithms. Formally, Theorem~\ref{thm:katz_index}-\ref{thm:reliable_path} state the corresponding path formulations for 3 link prediction methods and 2 graph theory algorithms respectively. See Section~\ref{app:path_formulation} for proofs.
\begin{restatable}{theorem}{katz}
    Katz index is a path formulation with $\oplus = +$, $\otimes = \times$ and $\vw_q(e) = \beta w_e$.
    \label{thm:katz_index}
\end{restatable}
\vspace{-0.8em}
\begin{restatable}{theorem}{pagerank}
    Personalized PageRank is a path formulation with $\oplus = +$, $\otimes = \times$ and $\vw_q(e) = \alpha w_{uv} / \sum\nolimits_{v'\in\gN(u)}w_{uv'}$.
    \label{thm:pagerank}
\end{restatable}
\vspace{-0.8em}
\begin{restatable}{theorem}{distance}
    Graph distance is a path formulation with $\oplus = \min$, $\otimes = +$ and $\vw_q(e) = w_e$.
    \label{thm:proximity}
\end{restatable}
\vspace{-2.1em}
\begin{restatable}{theorem}{widest}
    Widest path is a path formulation with $\oplus = \max$, $\otimes = \min$ and $\vw_q(e) = w_e$.
    \label{thm:widest_path}
\end{restatable}
\vspace{-2.1em}
\begin{restatable}{theorem}{reliable}
    Most reliable path is a path formulation with $\oplus = \max$, $\otimes = \times$ and $\vw_q(e) = w_e$.
    \label{thm:reliable_path}
\end{restatable}
\smallskip \noindent \textbf{Generalized Bellman-Ford Algorithm.}
While the above formulation is able to model important heuristics for link prediction, it is computationally expensive since the number of paths grows exponentially with the path length.
Previous works~\cite{neelakantan2015compositional, das2017chains, wang2021relational} that directly computes the exponential number of paths can only afford a maximal path length of 3.
A more scalable solution is to use the generalized Bellman-Ford algorithm~\cite{baras2010path}. Specifically, assuming the operators $\langle\oplus, \otimes\rangle$ satisfy a semiring system~\cite{hebisch1998semirings} with \emph{summation identity} $\ozero_q$ and \emph{multiplication identity} $\oone_q$, we have the following algorithm.
\begin{align}
    \vspace{-0.2em}
    \restatableeq{\boundary}{\vh_q^{(0)}(u, v) &\leftarrow \mathbbm{1}_q(u = v)}{eqn:boundary_nbfnet} \\
    \restatableeq{\iteration}{\vh_q^{(t)}(u, v) &\leftarrow \left(\bigoplus_{(x, r, v)\in\gE(v)} \vh_q^{(t-1)}(u, x) \otimes \vw_q(x, r, v)\right) \oplus \vh_q^{(0)}(u, v)}{eqn:iteration}
    \vspace{-0.2em}
\end{align}
where $\mathbbm{1}_q(u = v)$ is the \textit{indicator} function that outputs $\oone_q$ if $u = v$ and $\ozero_q$ otherwise. $\vw_q(x, r, v)$ is the representation for edge $e = (x, r, v)$ and $r$ is the relation type of the edge. Equation~\ref{eqn:boundary_nbfnet} is known as the boundary condition, while Equation~\ref{eqn:iteration} is known as the Bellman-Ford iteration. The high-level idea of the generalized Bellman-Ford algorithm is to \textbf{compute the pair representation $\vh_q(u, v)$ for a given entity $u$, a given query relation $q$ and all $v \in \gV$ in parallel}, and reduce the total computation by the distributive property of \emph{multiplication} over \emph{summation}. Since $u$ and $q$ are fixed in the generalized Bellman-Ford algorithm, we may abbreviate $\vh_q^{(t)}(u, v)$ as $\vh^{(t)}_v$ when the context is clear. When $\oplus = min$ and $\otimes = +$, it recovers the original Bellman-Ford algorithm for the shortest path problem~\cite{bellman1958routing}. See Section~\ref{app:bellman_ford} for preliminaries and the proof of the above algorithm.

\begin{restatable}{theorem}{bellman}
    Katz index, personalized PageRank, graph distance, widest path and most reliable path can be solved via the generalized Bellman-Ford algorithm.
    \label{th:bellman-ford}
\end{restatable}

\begin{table}[!h]
    \centering
    \caption[Comparison of operators in NBFNet and other methods]{Comparison of operators in NBFNet and other methods from the view of path formulation.}
    \label{tab:semiring_instance}
    \begin{adjustbox}{max width=\textwidth}
        \begin{tabular}{llcccc}
            \toprule
            \multirow{2}{*}{\bf{Class}} & \multirow{2}{*}{\bf{Method}} & \bf{\textsc{Message}} & \bf{\textsc{Aggregate}} & \bf{\textsc{Indicator}} & \bf{Edge Representation} \\
            & & $\vw_q(e_i) \otimes \vw_q(e_j)$ & $\vh_q(P_i) \oplus \vh_q(P_j)$ & $\ozero_q$, $\oone_q$ & $\vw_q(e)$ \\
            \midrule
            \multirow{3}{*}{\bf{\shortstack[l]{Traditional\\Link\\Prediction}}}
            & Katz Index~\cite{katz1953new} & $\vw_q(e_i) \times \vw_q(e_j)$ & $\vh_q(P_i) + \vh_q(P_j)$ & $0, 1$ & $\beta w_e$ \\
            & Personalized PageRank~\cite{page1999pagerank} & $\vw_q(e_i) \times \vw_q(e_j)$ & $\vh_q(P_i) + \vh_q(P_j)$ & $0, 1$ & $\alpha w_{uv} / \sum_{v'\in\gN(u)}w_{uv'}$ \\
            & Graph Distance~\cite{liben2007link} & $\vw_q(e_i) + \vw_q(e_j)$ & $\min(\vh_q(P_i), \vh_q(P_j))$ & $+\infty, 0$ & $w_e$ \\
            \midrule
            \multirow{2}{*}{\bf{\shortstack[l]{Graph Theory\\Algorithms}}}
            & Widest Path~\cite{baras2010path} & $\min(\vw_q(e_i), \vw_q(e_j))$ & $\max(\vh_q(P_i), \vh_q(P_j))$ & $-\infty, +\infty$ & $w_e$ \\
            & Most Reliable Path~\cite{baras2010path} & $\vw_q(e_i) \times \vw_q(e_j)$ & $\max(\vh_q(P_i), \vh_q(P_j))$ & $0, 1$ & $w_e$ \\
            \midrule
            \multirow{2}{*}{\bf{Logic Rules}}
            & NeuralLP~\cite{yang2017differentiable} / & \multirow{2}{*}{$\vw_q(e_i) \times \vw_q(e_j)$} & \multirow{2}{*}{$\vh_q(P_i) + \vh_q(P_j)$} & \multirow{2}{*}{0, 1} & Weights learned \\
            & DRUM~\cite{sadeghian2019drum} & & & & by LSTM~\cite{hochreiter1997long} \\
            \midrule
            & \multirow{3}{*}{NBFNet} & Relational operators of & \multirow{3}{*}{\shortstack[c]{Learned set\\aggregators~\cite{corso2020principal}}} & \multirow{3}{*}{\shortstack[c]{Learned indicator\\functions}} & \multirow{3}{*}{\shortstack[c]{Learned relation\\embeddings}} \\
            & & knowledge graph \\
            & & embeddings~\cite{bordes2013translating, yang2015embedding, sun2019rotate} \\
            \bottomrule
        \end{tabular}
    \end{adjustbox}
\end{table}

\subsection{Neural Bellman-Ford Networks}
\label{sec:nbfn}

\begin{wrapfigure}{R}{0.52\textwidth}
\begin{minipage}{0.52\textwidth}
    \vspace{-2.4em}
    \begin{algorithm}[H]
        \footnotesize
        \captionsetup{font=footnotesize}\caption{Neural Bellman-Ford Networks}
        \textbf{Input:} source node $u$, query relation $q$, \#layers $T$ \\
        \textbf{Output:} pair representations $\vh_q(u, v)$ for all $v \in \gV$
        \begin{algorithmic}[1]
            \For{$v \in \gV$} \Comment{Boundary condition}
                \State{$\vh_v^{(0)} \gets \Call{Indicator}{u, v, q}$}
            \EndFor
            \For{$t \gets 1$ to $T$} \Comment{Bellman-Ford iteration}
                \For{$v \in \gV$}
                    \State{$\gM^{(t)}_v \gets \left\{\vh_v^{(0)}\right\}$} \Comment{Message augmentation}
                    \For{$(x, r, v) \in \gE(v)$}
                        \State{$\vm_{(x, r, v)}^{(t)} \gets$ \Call{Message$^{(t)}$}{$\vh_x^{(t-1)}, \vw_q(x, r, v)$}}
                        \State{$\gM^{(t)}_v \gets \gM^{(t)}_v \cup \left\{\vm_{(x, r, v)}^{(t)}\right\}$}
                    \EndFor
                    \State{$\vh_v^{(t)} \gets$ \Call{Aggregate$^{(t)}$}{$\gM^{(t)}_v$}}
                \EndFor
            \EndFor
            \State{\Return{$\vh_v^{(T)}$ as $\vh_q(u, v)$ for all $v \in \gV$}}
        \end{algorithmic}
        \label{alg:framework}
    \end{algorithm}
    \vspace{-2em}
\end{minipage}
\end{wrapfigure}

While the generalized Bellman-Ford algorithm can solve many classical methods (Theorem~\ref{th:bellman-ford}), these methods instantiate the path formulation with handcrafted operators (Table~\ref{tab:semiring_instance}), and may not be optimal for link prediction. To improve the capacity of path formulation, we propose a general framework, Neural Bellman-Ford Networks (NBFNet), to learn the operators in the pair representations.

\smallskip \noindent \textbf{Neural Parameterization.} We relax the semiring assumption and parameterize the generalized Bellman-Ford algorithm (Equation~\ref{eqn:boundary_nbfnet} and \ref{eqn:iteration}) with 3 neural functions, namely \textsc{Indicator}, \textsc{Message} and \textsc{Aggregate} functions. The \textsc{Indicator} function replaces the \emph{indicator} function $\mathbbm{1}_q(u = v)$. The \textsc{Message} function replaces the binary \emph{multiplication} operator $\otimes$. The \textsc{Aggregate} function is a permutation invariant function over sets that replaces the n-ary \emph{summation} operator $\bigoplus$. Note that one may alternatively define \textsc{Aggregate} as the commutative binary operator $\oplus$ and apply it to a sequence of messages. However, this will make the parameterization more complicated.

Now consider the generalized Bellman-Ford algorithm for a given entity $u$ and relation $q$. In this context, we abbreviate $\vh_q^{(t)}(u, v)$ as $\vh^{(t)}_v$, i.e.\ a representation on entity $v$ in the $t$-th iteration. It should be stressed that $\vh^{(t)}_v$ is still a pair representation, rather than a node representation. By substituting the neural functions into Equation~\ref{eqn:boundary_nbfnet} and \ref{eqn:iteration}, we get our Neural Bellman-Ford Networks.
\begin{align}
    \vh^{(0)}_v &\leftarrow \textsc{Indicator}(u, v, q) \\
    \vh^{(t)}_v &\leftarrow \textsc{Aggregate}\left(\left\{\textsc{Message}\left(\vh^{(t-1)}_x, \vw_q(x, r, v)\right) \middle\vert (x, r, v) \in \gE(v)\right\} \cup \left\{\vh^{(0)}_v\right\}\right)
\end{align}
NBFNet can be interpreted as a novel GNN framework for learning pair representations. Compared to common GNN frameworks~\cite{kipf2016variational, schlichtkrull2018modeling} that compute the pair representation as two independent node representations $\vh_q(u)$ and $\vh_q(v)$, NBFNet initializes a representation on the source node $u$, and readouts the pair representation on the target node $v$. Intuitively, our framework can be viewed as a source-specific message passing process, where every node learns a representation conditioned on the source node. The pseudo code of NBFNet is outlined in Algorithm~\ref{alg:framework}.

\smallskip \noindent \textbf{Design Space.} Now we discuss some principled designs for \textsc{Message}, \textsc{Aggregate} and \textsc{Indicator} functions by drawing insights from traditional methods. Note the potential design space for NBFNet is way larger than what is presented here, as one can always borrow \textsc{Message} and \textsc{Aggregate} from the arsenal of message-passing GNNs~\cite{hamilton2017inductive, gilmer2017neural, velivckovic2018graph, xu2019powerful}.

For the \textsc{Message} function, traditional methods instantiate it as natural summation, natural multiplication or min over scalars. Therefore, we may use the vectorized version of summation or multiplication. Intuitively, summation of $\vh^{(t-1)}_x$ and $\vw_q(x, r, v)$ can be interpreted as a translation of $\vh^{(t-1)}_x$ by $\vw_q(x, r, v)$ in the pair representation space, while multiplication corresponds to scaling. Such transformations correspond to the relational operators~\cite{hamilton2018embedding, ren2020query2box} in knowledge graph embeddings~\cite{bordes2013translating, yang2015embedding, trouillon2016complex, kazemi2018simple, sun2019rotate}. For example, translation and scaling are the relational operators used in TransE~\cite{bordes2013translating} and DistMult~\cite{yang2015embedding} respectively. We also consider the rotation operator in RotatE~\cite{sun2019rotate}.

The \textsc{Aggregate} function is instantiated as natural summation, max or min in traditional methods, which are reminiscent of set aggregation functions~\cite{zaheer2017deep, xu2019powerful, corso2020principal} used in GNNs. Therefore, we specify the \textsc{Aggregate} function to be sum, mean, or max, followed by a linear transformation and a non-linear activation. We also consider the principal neighborhood aggregation (PNA) proposed in a recent work~\cite{corso2020principal}, which jointly learns the types and scales of the aggregation function.

The \textsc{Indicator} function is aimed at providing a non-trivial representation for the source node $u$ as the boundary condition. Therefore, we learn a query embedding $\vq$ for $\oone_q$ and define \textsc{Indicator} function as $\mathbbm{1}(u = v) * \vq$. Note it is also possible to additionally learn an embedding for $\ozero_q$. However, we find a single query embedding works better in practice.

The edge representations are instantiated as transition probabilities or length in traditional methods. We notice that an edge may have different contribution in answering different query relations. Therefore, we parameterize the edge representations as a linear function over the query relation, i.e.\ $\vw_q(x, r, v) = \mW_r \vq + \vb_r$. For homogeneous graphs or knowledge graphs with very few relations, we simplify the parameterization to $\vw_q(x, r, v) = \vb_r$ to prevent overfitting. Note that one may also parameterize $\vw_q(x, r, v)$ with learnable entity embeddings $\vx$ and $\vv$, but such a parameterization cannot solve the inductive setting. Similar to NeuralLP~\cite{yang2017differentiable} \& DRUM~\cite{sadeghian2019drum}, we use different edge representations for different iterations, which is able to distinguish noncommutative edges in paths, e.g.\ \emph{father's wife} and \emph{wife's father}.

\smallskip \noindent \textbf{Link Prediction.} We now show how to apply the learned pair representations $\vh_q(u, v)$ to the link prediction problem. We predict the conditional likelihood of the tail entity $v$ as $p(v | u, q) = \sigma(f(\vh_q(u, v)))$, where $\sigma(\cdot)$ is the sigmoid function and $f(\cdot)$ is a feed-forward neural network. The conditional likelihood of the head entity $u$ can be predicted by $p(u|v, q^{-1}) = \sigma(f(\vh_{q^{-1}}(v, u)))$ with the same model. Following previous works~\cite{bordes2013translating, sun2019rotate}, we minimize the negative log-likelihood of positive and negative triplets (Equation~\ref{eqn:kg_loss}). The negative samples are generated according to Partial Completeness Assumption (PCA)~\cite{galarraga2013amie}, which corrupts one of the entities in a positive triplet to create a negative sample. For undirected graphs, we symmetrize the representations and define $p_q(u, v) = \sigma(f(\vh_q(u, v) + \vh_q(v, u)))$. Equation~\ref{eqn:homo_loss} shows the loss for homogeneous graphs.
\begin{align}
    &\gL_{KG} = -\log p(u, q, v) - \sum_{i=1}^n \frac{1}{n}\log (1 - p(u_i', q, v_i'))
    \label{eqn:kg_loss} \\
    &\gL_{homo} = -\log p(u, v) -
    \sum_{i=1}^n \frac{1}{n}\log (1 - p(u_i', v_i')),
    \label{eqn:homo_loss}
\end{align}
where $n$ is the number of negative samples per positive sample and $(u_i',q,v_i')$ and $(u_i',v_i')$ are the $i$-th negative samples for knowledge graphs and homogeneous graphs, respectively.

\smallskip \noindent \textbf{Time Complexity.} One advantage of NBFNet is that it has a relatively low time complexity during inference\footnote{Although the same analysis can be applied to training on a fixed number of samples, we note it is less instructive since one can trade-off samples for performance, and the trade-off varies from method to method.}.
Consider a scenario where a model is required to infer the conditional likelihood of all possible triplets $p(v | u, q)$. We group triplets with the same condition $u, q$ together, where each group contains $|\gV|$ triplets. For each group, we only need to execute Algorithm~\ref{alg:framework} once to get their predictions. Since a small constant number of iterations $T$ is enough for NBFNet to converge (Table~\ref{tab:num_layer}), Algorithm~\ref{alg:framework} has a time complexity of $O(|\gE|d + |\gV|d^2)$, where $d$ is the dimension of representations. Therefore, the amortized time complexity for a single triplet is $O\left(\frac{|\gE|d}{|\gV|} + d^2\right)$.
\section{Experiments of NBFNet}

\subsection{Experiment Setup}
\label{sec:exp_setup}

We evaluate NBFNet in three settings, knowledge graph completion, homogeneous graph link prediction and inductive relation prediction. The former two are transductive settings, while the last is an inductive setting. For knowledge graphs, we use FB15k-237~\cite{toutanova2015observed} and WN18RR~\cite{dettmers2018convolutional}. We use the standard transductive splits~\cite{toutanova2015observed, dettmers2018convolutional} and inductive splits~\cite{teru2020inductive} of these datasets. For homogeneous graphs, we use Cora, Citeseer and PubMed~\cite{sen2008collective}. Following previous works~\cite{kipf2016variational, davidson2018hyperspherical}, we split the edges into train/valid/test with a ratio of 85:5:10. Statistics of datasets can be found in Section~\ref{app:dataset_nbfnet}. Additional experiments of NBFNet on OGB~\cite{hu2020ogb} datasets can be found in Section~\ref{app:ogb}.

\smallskip \noindent \textbf{Implementation Details.}
Our implementation generally follows the open source codebases of knowledge graph completion\footnote{\url{https://github.com/DeepGraphLearning/KnowledgeGraphEmbedding}. MIT license.\label{fn:kg_url}} and homogeneous graph link prediction\footnote{\url{https://github.com/tkipf/gae}. MIT license.\label{fn:homo_url}}. For knowledge graphs, we follow \cite{yang2017differentiable, sadeghian2019drum} and augment each triplet \edge{u, q, v} with a flipped triplet \edge{v, q$^{-1}$, u}. For homogeneous graphs, we follow \cite{kipf2017semi, kipf2016variational} and augment each node $u$ with a self loop \edge{u, u}. We instantiate NBFNet with 6 layers, each with 32 hidden units. The feed-forward network $f(\cdot)$ is set to a 2-layer MLP with 64 hidden units. ReLU is used as the activation function for all hidden layers. We drop out edges that directly connect query node pairs during training to encourage the model to capture longer paths and prevent overfitting. Our model is trained on 4 Tesla V100 GPUs for 20 epochs. We select the models based on their performance on the validation set.

\smallskip \noindent \textbf{Evaluation.} We follow the filtered ranking protocol~\cite{bordes2013translating} for knowledge graph completion. For a test triplet \edge{u, q, v}, we rank it against all negative triplets \edge{u, q, v'} or \edge{u', q, v} that do not appear in the knowledge graph. We report mean rank (MR), mean reciprocal rank (MRR) and HITS at N (H@N) for knowledge graph completion. For inductive relation prediction, we follow~\cite{teru2020inductive} and draw 50 negative triplets for each positive triplet and use the above filtered ranking. We report HITS@10 for inductive relation prediction. For homogeneous graph link prediction, we follow~\cite{kipf2016variational} and compare the positive edges against the same number of negative edges. We report area under the receiver operating characteristic curve (AUROC) and average precision (AP) for homogeneous graphs.

\smallskip \noindent \textbf{Baselines.} We compare NBFNet against path-based methods, embedding methods, and GNNs. These include 11 baselines for knowledge graph completion, 10 baselines for homogeneous graph link prediction and 4 baselines for inductive relation prediction. Note the inductive setting only includes path-based methods and GNNs, since existing embedding methods cannot handle this setting.

\subsection{Main Results}

Table~\ref{tab:kg_result} summarizes the results on knowledge graph completion. NBFNet significantly outperforms existing methods on all metrics and both datasets. NBFNet achieves an average relative gain of 21\% in HITS@1 compared to the best path-based method, DRUM~\cite{sadeghian2019drum}, on two datasets. Since DRUM is a special instance of NBFNet with natural summation and multiplication operators, this indicates the importance of learning \textsc{Message} and \textsc{Aggregate} functions in NBFNet. NBFNet also outperforms the best embedding method, LowFER~\cite{amin2020lowfer}, with an average relative performance gain of 18\% in HITS@1 on two datasets. Meanwhile, NBFNet requires much less parameters than embedding methods. NBFNet only uses 3M parameters on FB15k-237, while TransE needs 30M parameters.

\begin{table}[!h]
    \centering
    \caption[Knowledge graph completion results]{Knowledge graph completion results. Results of NeuraLP and DRUM are taken from \cite{sadeghian2019drum}. Results of RotatE, HAKE and LowFER are taken from their original papers~\cite{sun2019rotate, zhang2020learning, amin2020lowfer}. Results of the other embedding methods are taken from \cite{sun2019rotate}. Since GraIL has scalability issues in this setting, we evaluate it with 50 and 100 negative triplets for FB15k-237 and WN18RR respectively and report MR based on an unbiased estimation.}
    \label{tab:kg_result}
    \begin{adjustbox}{max width=\textwidth}
        \begin{tabular}{llcccccccccc}
            \toprule
            \multirow{2}{*}{\bf{Class}} & \multirow{2}{*}{\bf{Method}}
            & \multicolumn{5}{c}{\bf{FB15k-237}} & \multicolumn{5}{c}{\bf{WN18RR}} \\
            & & \bf{MR} & \bf{MRR} & \bf{H@1} & \bf{H@3} & \bf{H@10} & \bf{MR} & \bf{MRR} & \bf{H@1} & \bf{H@3} & \bf{H@10} \\
            \midrule
            \multirow{3}{*}{\bf{Path-based}}
            & Path Ranking~\cite{lao2010relational} & 3521 & 0.174 & 0.119 & 0.186 & 0.285 & 22438 & 0.324 & 0.276 & 0.360 & 0.406 \\
            & NeuralLP~\cite{yang2017differentiable} & - & 0.240 & - & - & 0.362 & - & 0.435 & 0.371 & 0.434 & 0.566 \\
            & DRUM~\cite{sadeghian2019drum} & - & 0.343 & 0.255 & 0.378 & 0.516 & - & 0.486 & 0.425 & 0.513 & 0.586 \\
            \midrule
            \multirow{6}{*}{\bf{Embeddings}}
            & TransE~\cite{bordes2013translating} & 357 & 0.294 & - & - & 0.465 & 3384 & 0.226 & - & - & 0.501 \\
            & DistMult~\cite{yang2015embedding} & 254 & 0.241 & 0.155 & 0.263 & 0.419 & 5110 & 0.43 & 0.39 & 0.44 & 0.49 \\
            & ComplEx~\cite{trouillon2016complex} & 339 & 0.247 & 0.158 & 0.275 & 0.428 & 5261 & 0.44 & 0.41 & 0.46 & 0.51 \\
            & RotatE~\cite{sun2019rotate} & 177 & 0.338 & 0.241 & 0.375 & 0.533 & 3340 & 0.476 & 0.428 & 0.492 & 0.571 \\
            & HAKE~\cite{zhang2020learning} & - & 0.346 & 0.250 & 0.381 & 0.542 & - & 0.497 & 0.452 & 0.516 & 0.582 \\
            & LowFER~\cite{amin2020lowfer} & - & 0.359 & 0.266 & 0.396 & 0.544 & - & 0.465 & 0.434 & 0.479 & 0.526 \\
            \midrule
            \multirow{3}{*}{\bf{GNNs}}
            & RGCN~\cite{schlichtkrull2018modeling} & 221 & 0.273 & 0.182 & 0.303 & 0.456 & 2719 & 0.402 & 0.345 & 0.437 & 0.494 \\
            & GraIL~\cite{teru2020inductive} & 2053 & - & - & - & - & 2539 & - & - & - & - \\
            & NBFNet & \bf{114} & \bf{0.415} & \bf{0.321} & \bf{0.454} & \bf{0.599} & \bf{636} & \bf{0.551} & \bf{0.497} & \bf{0.573} & \bf{0.666} \\
            \bottomrule
        \end{tabular}
    \end{adjustbox}
\end{table}

\begin{table}[!h]
    \centering
    \caption[Homogeneous graph link prediction results]{Homogeneous graph link prediction results. Results of VGAE and S-VGAE are taken from their original papers~\cite{kipf2016variational, davidson2018hyperspherical}.}
    \label{tab:homo_result}
    \begin{adjustbox}{max width=\textwidth}
        \begin{tabular}{llcccccc}
            \toprule
            \multirow{2}{*}{\bf{Class}} & \multirow{2}{*}{\bf{Method}}
                    & \multicolumn{2}{c}{\bf{Cora}} & \multicolumn{2}{c}{\bf{Citeseer}} & \multicolumn{2}{c}{\bf{PubMed}} \\
                &   & \bf{AUROC} & \bf{AP} & \bf{AUROC} & \bf{AP} & \bf{AUROC} & \bf{AP} \\
            \midrule
            \multirow{3}{*}{\bf{Path-based}}
            & Katz Index~\cite{katz1953new} & 0.834 & 0.889 & 0.768 & 0.810 & 0.757 & 0.856 \\
            & Personalized PageRank~\cite{page1999pagerank} & 0.845 & 0.899 & 0.762 & 0.814 & 0.763 & 0.860 \\
            & SimRank~\cite{jeh2002simrank} & 0.838 & 0.888 & 0.755 & 0.805 & 0.743 & 0.829 \\
            \midrule
            \multirow{3}{*}{\bf{Embeddings}}
            & DeepWalk~\cite{perozzi2014deepwalk} & 0.831 & 0.850 & 0.805 & 0.836 & 0.844 & 0.841 \\
            & LINE~\cite{tang2015line} & 0.844 & 0.876 & 0.791 & 0.826 & 0.849 & 0.888 \\
            & node2vec~\cite{grover2016node2vec} & 0.872 & 0.879 & 0.838 & 0.868 & 0.891 & 0.914 \\
            \midrule
            \multirow{5}{*}{\bf{GNNs}}
            & VGAE~\cite{kipf2016variational} & 0.914 & 0.926 & 0.908 & 0.920 & 0.944 & 0.947 \\
            & S-VGAE~\cite{davidson2018hyperspherical} & 0.941 & 0.941 & \bf{0.947} & \bf{0.952} & 0.960 & 0.960 \\
            & SEAL~\cite{zhang2018link} & 0.933 & 0.942 & 0.905 & 0.924 & 0.978 & 0.979 \\
            & TLC-GNN~\cite{yan2021link} & 0.934 & 0.931 & 0.909 & 0.916 & 0.970 & 0.968 \\
            & NBFNet             & \bf{0.956} & \bf{0.962} & 0.923 & 0.936 & \bf{0.983} & \bf{0.982} \\
            \bottomrule
        \end{tabular}
    \end{adjustbox}
\end{table}

\begin{table}[!h]
    \centering
    \caption[Inductive relation prediction results]{Inductive relation prediction results (HITS@10). V1-v4 corresponds to the 4 standard versions of inductive splits. Results of compared methods are taken from \cite{teru2020inductive}.}
    \label{tab:inductive_result}
    \footnotesize
    \begin{tabular}{llcccccccc}
        \toprule
        \multirow{2}{*}{\bf{Class}} & \multirow{2}{*}{\bf{Method}} & \multicolumn{4}{c}{\bf{FB15k-237}} & \multicolumn{4}{c}{\bf{WN18RR}} \\
        & & \bf{v1} & \bf{v2} & \bf{v3} & \bf{v4} & \bf{v1} & \bf{v2} & \bf{v3} & \bf{v4} \\
        \midrule
        \multirow{3}{*}{\bf{Path-based}}
        & NeuralLP~\cite{gilmer2017neural} & 0.529 & 0.589 & 0.529 & 0.559 & 0.744 & 0.689 & 0.462 & 0.671 \\
        & DRUM~\cite{sadeghian2019drum} & 0.529 & 0.587 & 0.529 & 0.559 & 0.744 & 0.689 & 0.462 & 0.671 \\
        & RuleN~\cite{meilicke2018fine} & 0.498 & 0.778 & 0.877 & 0.856 & 0.809 & 0.782 & 0.534 & 0.716 \\
        \midrule
        \multirow{2}{*}{\bf{GNNs}}
        & GraIL~\cite{teru2020inductive} & 0.642 & 0.818 & 0.828 & 0.893 & 0.825 & 0.787 & 0.584 & 0.734 \\
        & NBFNet & \bf{0.834} & \bf{0.949} & \bf{0.951} & \bf{0.960} & \bf{0.948} & \bf{0.905} & \bf{0.893} & \bf{0.890} \\
        \bottomrule
    \end{tabular}
\end{table}

Table~\ref{tab:homo_result} shows the results on homogeneous graph link prediction. NBFNet gets the best results on Cora and PubMed, meanwhile achieves competitive results on CiteSeer. Note CiteSeer is extremely sparse (Section~\ref{app:dataset_nbfnet}), which makes it hard to learn good representations with NBFNet. One thing to note here is that unlike other GNN methods, NBFNet does not use the node features provided by the datasets but is still able to outperform most other methods. We leave how to effectively combine node features and structural representations for link prediction as our future work.

Table~\ref{tab:inductive_result} summarizes the results on inductive relation prediction. On all inductive splits of two datasets, NBFNet achieves the best result. NBFNet outperforms the previous best method, GraIL~\cite{teru2020inductive}, with an average relative performance gain of 22\% in HITS@10. Note that GraIL explicitly encodes the local subgraph surrounding each node pair and has a high time complexity (Table~\ref{tab:comparison}). Usually, GraIL can at most encode a 2-hop subgraph, while our NBFNet can efficiently explore longer paths. 

\begin{table}
    \caption[Ablation studies of NBFNet on FB15k-237]{Ablation studies of NBFNet on FB15k-237. All the entries are the MRR metric.}
    \label{tab:ablation_nbfnet}
    \footnotesize
    \begin{subtable}[t]{0.52\textwidth}
        \centering
        \caption{Different \textsc{Message} and \textsc{Aggregate} functions.\label{tab:msg_agg}}
        \vspace{-0.2em}
        \begin{adjustbox}{max width=\textwidth}
            \begin{tabular}{lcccc}
                \toprule
                \multirow{2}{*}{\bf{\textsc{Message}}} & \multicolumn{4}{c}{\bf{\textsc{Aggregate}}} \\
                                                  & Sum & Mean & Max & PNA~\cite{corso2020principal} \\
                \midrule
                TransE~\cite{bordes2013translating} & 0.297 & 0.310 & 0.377 & 0.383 \\
                DistMult~\cite{yang2017differentiable} & 0.388 & 0.384 & 0.374 & \bf{0.415} \\
                RotatE~\cite{sun2019rotate}   & 0.392 & 0.376 & 0.385 & \bf{0.414} \\
                \bottomrule
            \end{tabular}
        \end{adjustbox}
    \end{subtable}
    \hspace{0.3em}
    \begin{subtable}[t]{0.46\textwidth}
        \centering
        \caption{Different number of layers.\label{tab:num_layer}}
        \vspace{-0.2em}
        \begin{adjustbox}{max width=\textwidth}
            \begin{tabular}{lcccc}
                \toprule
                \multirow{2}{*}{\bf{Method}} & \multicolumn{4}{c}{\bf{\#Layers ($T$)}} \\
                & 2 & 4 & 6 & 8 \\
                \midrule
                NBFNet & 0.345 & 0.409 & \bf{0.415} & \bf{0.416} \\
                \bottomrule
            \end{tabular}
        \end{adjustbox}
    \end{subtable}
    \\[1em]
    \begin{subtable}[t]{\textwidth}
        \centering
        \caption{Performance per relation category. The two scores are rankings over heads and tails respectively.\label{tab:rel_category}}
        \vspace{-0.2em}
        \begin{adjustbox}{max width=0.8\textwidth}
            \begin{tabular}{lcccc}
                \toprule
                \multirow{2}{*}{\bf{Method}} & \multicolumn{4}{c}{\bf{Relation Category}} \\
                & \bf{1-to-1} & \bf{1-to-N} & \bf{N-to-1} & \bf{N-to-N} \\
                \midrule
                TransE~\cite{bordes2013translating} & 0.498/0.488 & 0.455/0.071 & 0.079/0.744 & 0.224/0.330 \\
                RotatE~\cite{sun2011pathsim} & 0.487/0.484 & 0.467/0.070 & 0.081/0.747 & 0.234/0.338 \\
                NBFNet & \bf{0.578}/\bf{0.600} & \bf{0.499}/\bf{0.122} & \bf{0.165}/\bf{0.790} & \bf{0.348}/\bf{0.456} \\
                \bottomrule
            \end{tabular}
        \end{adjustbox}
    \end{subtable}
\end{table}

\subsection{Ablation Studies}

\smallskip \noindent \textbf{\textsc{Message} \& \textsc{Aggregate} Functions.} Table~\ref{tab:msg_agg} shows the results of different \textsc{Message} and \textsc{Aggregate} functions. Generally, NBFNet benefits from advanced embedding methods (DistMult, RotatE > TransE) and aggregation functions (PNA > sum, mean, max). Among simple \textsc{Aggregate} functions (sum, mean, max), combinations of \textsc{Message} and \textsc{Aggregate} functions (TransE \& max, DistMult \& sum) that satisfy the semiring assumption\footnote{Here semiring is discussed under the assumption of linear activation functions. Rigorously, no combination satisfies a semiring if we consider non-linearity in the model.} of the generalized Bellman-Ford algorithm, achieve locally optimal performance. PNA significantly improves over simple counterparts, which highlights the importance of learning more powerful \textsc{Aggregate} functions.

\smallskip \noindent \textbf{Number of GNN Layers.} Table~\ref{tab:num_layer} compares the results of NBFNet with different number of layers. Although it has been reported that GNNs with deep layers often result in significant performance drop~\cite{li2018deeper, zhao2019pairnorm}, we observe NBFNet does not have this issue. The performance increases monotonically with more layers, hitting a saturation after 6 layers. We conjecture the reason is that longer paths have negligible contribution, and paths not longer than 6 are enough for link prediction.

\smallskip \noindent \textbf{Performance by Relation Category.} We break down the performance of NBFNet by the categories of query relations: one-to-one, one-to-many, many-to-one and many-to-many\footnote{The categories are defined same as \cite{wang2014knowledge}. We compute the average number of tails per head and the average number of heads per tail. The category is \textit{one} if the average number is smaller than 1.5 and \textit{many} otherwise.}. Table~\ref{tab:rel_category} shows the prediction results for each category. It is observed that NBFNet not only improves on easy one-to-one cases, but also on hard cases where there are multiple true answers for the query.

\subsection{Path Interpretations of Predictions}

One advantage of NBFNet is that we can interpret its predictions through paths, which may be important for users to understand and debug the model. Intuitively, the interpretations should contain paths that contribute most to the prediction $p(u, q, v)$. Following local interpretation methods~\cite{baehrens2010explain, zeiler2014visualizing}, we approximate the local landscape of NBFNet with a linear model over the set of all paths, i.e.\ 1st-order Taylor polynomial. We define the importance of a path as its weight in the linear model, which can be computed by the partial derivative of the prediction w.r.t.\ the path. Formally, the top-k path interpretations for $p(u, q, v)$ are defined as
\begin{equation}
    P_1, P_2, ..., P_k = \topk_{P \in \gP_{uv}} \frac{\partial{p(u, q, v)}}{\partial{P}}
\end{equation}
Note this formulation generalizes the definition of logical rules~\cite{yang2017differentiable, sadeghian2019drum} to non-linear models. While directly computing the importance of all paths is intractable, we approximate them with edge importance. Specifically, the importance of each path is approximated by the sum of the importance of edges in that path, where edge importance is obtained via auto differentiation. Then the top-k path interpretations are equivalent to the top-k longest paths on the edge importance graph, which can be solved by a Bellman-Ford-style beam search. Better approximation is left as a future work.

Table~\ref{tab:vis_nbfnet} visualizes path interpretations from FB15k-237 test set. While users may have different insights towards the visualization, here is our understanding. 1) In the first example, NBFNet learns soft logical entailment, such as $\emph{impersonate}^{-1} \land \emph{nationality} \implies \emph{nationality}$ and $\emph{ethnicity}^{-1} \land \emph{distribution} \implies \emph{nationality}$. 2) In second example, NBFNet performs analogical reasoning by leveraging the fact that \emph{Florence} is similar to \emph{Rome}. 3) In the last example, NBFNet extracts longer paths, since there is no obvious connection between \emph{Pearl Harbor (film)} and \emph{Japanese language}.
\begin{table}[!h]
    \centering
    \caption[Path interpretations of NBFNet predictions on FB15k-237]{Path interpretations of predictions on FB15k-237 test set. For each query triplet, we visualize the top-2 path interpretations and their weights. Inverse relations are denoted with a superscript $^{-1}$.}
    \label{tab:vis_nbfnet}
    \begin{adjustbox}{max width=\textwidth}
        \footnotesize
        \begin{tabular}{ll}
            \toprule
            \bf{Query} & \edge{u, q, v}: \edge{O. Hardy, nationality, U.S.} \\
            \midrule
            0.243 & \edge{O. Hardy, impersonate$^{-1}$, R. Little} $\land$ \edge{R. Little, nationality, U.S.} \\
            0.224 & \edge{O. Hardy, ethnicity$^{-1}$, Scottish American} $\land$ \edge{Scottish American, distribution, U.S.} \\
            \midrule
            \bf{Query} & \edge{u, q, v}: \edge{Florence, vacationer, D.C. Henrie} \\
            \midrule
            0.251 & \edge{Florence, contain$^{-1}$, Italy} $\land$ \edge{Italy, capital, Rome} $\land$ \edge{Rome, vacationer, D.C. Henrie} \\
            0.183 & \edge{Florence, place live$^{-1}$, G.F. Handel} $\land$ \edge{G.F. Handel, place live, Rome} $\land$ \edge{Rome, vacationer, D.C. Henrie} \\
            \midrule
            \bf{Query} & \edge{u, q, v}: \edge{Pearl Harbor (film), language, Japanese} \\
            \midrule
            0.211 & \edge{Pearl Harbor (film), film actor, C.-H. Tagawa} $\land$ \edge{C.-H. Tagawa, nationality, Japan} \\
            & $\land$ \edge{Japan, country of origin, Yu-Gi-Oh!} $\land$ \edge{Yu-Gi-Oh!, language, Japanese} \\
            0.208 & \edge{Pearl Harbor (film), film actor, C.-H. Tagawa} $\land$ \edge{C.-H. Tagawa, nationality, Japan} \\
            & $\land$ \edge{Japan, official language, Japanese} \\
            \bottomrule
        \end{tabular}
    \end{adjustbox}
\end{table}

\subsection{Results on Large Graphs}
\label{app:ogb}

To show the effectiveness of NBFNet on large graphs, we additionally evaluate our method on two knowledge graph datasets from OGB~\cite{hu2020ogb}, ogbl-biokg and WikiKG90M. Ogbl-biokg is a biomedical knowledge graph containing 93,773 entities, 51 relations and 5,088,434 triplets. WikiKG90M is an extremely large knowledge graph used in OGB large-scale challenge~\cite{hu2021ogb}, with 87,143,637 entities, 1,315 relations and 504,220,369 triplets. We follow the standard evaluation protocol of OGB link property prediction, and compute the mean reciprocal rank (MRR) of the true entity against 1,000 negative entities given in the test set.

\smallskip \noindent \textbf{Bidirectional BFS Sampling.}
In order to fit WikiKG90M into NBFNet, we use a bidirectional breath-first-search (BFS) algorithm to sample a local subgraph for each query. Given a query, we generate a $k$-hop neighborhood for each of the head entity and the candidate tail entities, based on a BFS search. The union of all generated neighborhoods is then collected as the sampled graph. With this sampling algorithm, any path within a length of $2k$ between the head entity and any tail candidate is guaranteed to present in the sampled graph. See Figure~\ref{fig:bfs_sampling} for illustration. While a standard single BFS algorithm computing the $2k$-hop neighborhood of the head entity has the same guarantee, a bidirectional BFS algorithm significantly reduces the number of nodes and edges in the sampled graph.

\begin{figure}[!h]
    \centering
    \begin{subfigure}{0.32\textwidth}
        \centering
        \includegraphics[width=\textwidth]{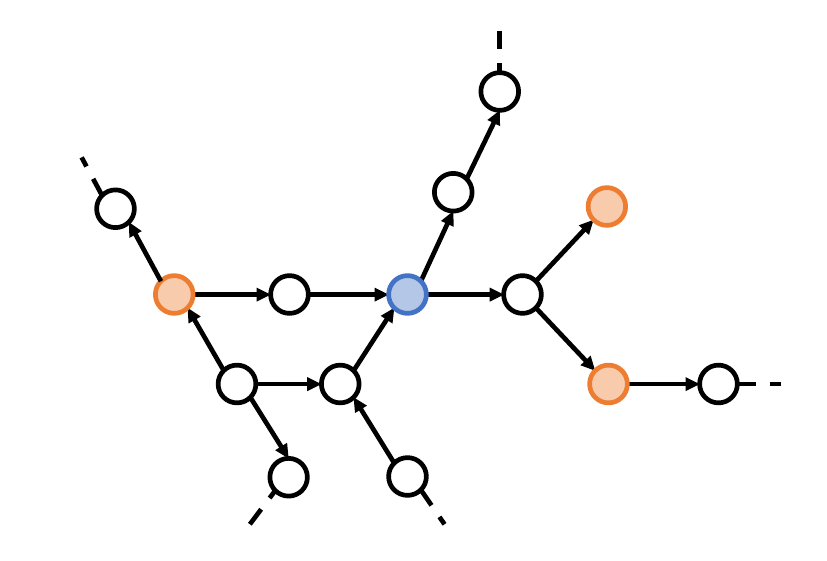}
        \caption{Original graph}
    \end{subfigure}
    \begin{subfigure}{0.32\textwidth}
        \centering
        \includegraphics[width=\textwidth]{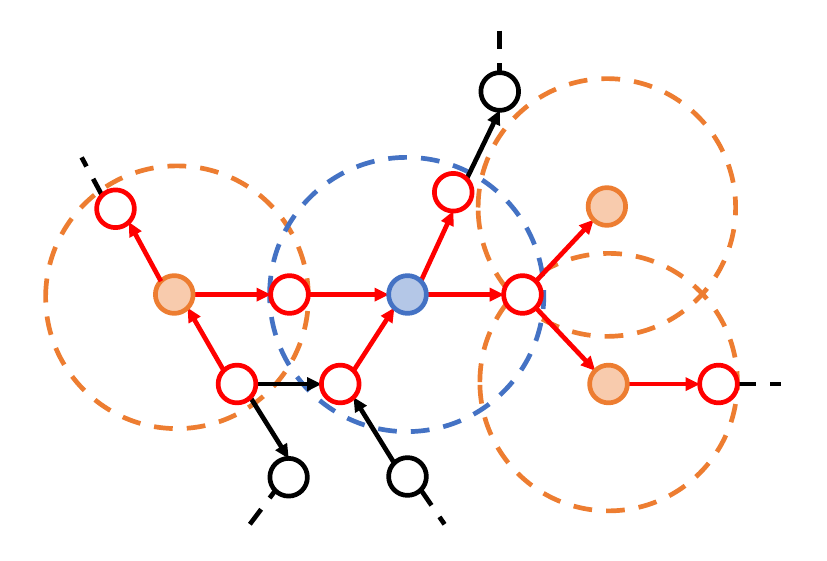}
        \caption{Bidirectional BFS}
    \end{subfigure}
    \begin{subfigure}{0.32\textwidth}
        \centering
        \includegraphics[width=\textwidth]{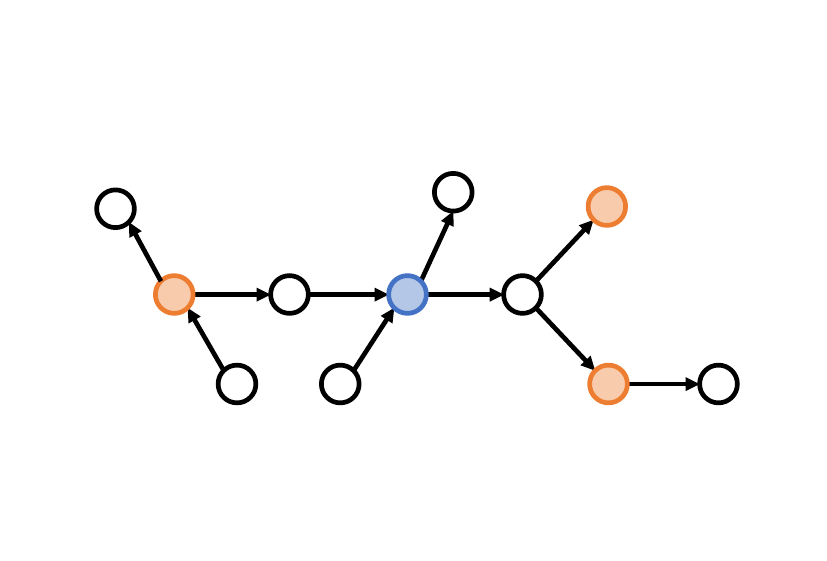}
        \caption{Sampled graph}
    \end{subfigure}
    \caption[Illustration of bidirectional BFS sampling]{Illustration of bidirectional BFS sampling. For a \textcolor{myblue}{head entity} and multiple \textcolor{myorange}{tail candidates}, we use BFS to sample a $k$-hop neighborhood around each entity. The neighborhood is denoted by dashed circles. The nodes and edges visited by the BFS algorithm are extracted to generate the sampled graph. Best viewed in color.}
    \label{fig:bfs_sampling}
\end{figure}

\begin{table}[!h]
    \centering
    \caption[Knowledge graph completion results on ogbl-biokg]{Knowledge graph completion results on ogbl-biokg. Results of compared methods are taken from the OGB leaderboard.}
    \label{tab:ogbl-biokg}
    \footnotesize
    \begin{tabular}{llccc}
        \toprule
        \bf{Class} & \bf{Method} & \bf{Test MRR} & \bf{Validation MRR} & \bf{\#Params}\\
        \midrule
        \multirow{6}{*}{\bf{Embeddings}}
        & TransE~\cite{bordes2013translating} & 0.7452 & 0.7456 & 187,648,000 \\
        & DistMult~\cite{yang2015embedding} & 0.8043 & 0.8055 & 187,648,000 \\
        & ComplEx~\cite{trouillon2016complex} & 0.8095 & 0.8105 & 187,648,000 \\
        & RotatE~\cite{sun2019rotate} & 0.7989 & 0.7997 & 187,597,000 \\
        & AutoSF~\cite{zhang2020autosf} & 0.8309 & \bf{0.8317} & 93,824,000 \\
        & PairRE~\cite{chao2021pairre} & 0.8164 & 0.8172 & 187,750,000 \\
        \midrule
        \bf{GNNs}
        & NBFNet & \bf{0.8317} & \bf{0.8318} & 734,209\\
        \bottomrule
    \end{tabular}
\end{table}

We additionally downsample the neighbors when expanding the neighbors of an entity, to tackle entities with large degrees. For each entity visited during the BFS algorithm, we downsample its outgoing neighbors and incoming neighbors to $m$ entities respectively.

\smallskip \noindent \textbf{Results on ogbl-biokg.}
Ogbl-biokg is a large biomedical knowledge graph that contains 93,773 entities, 51 relations and 5,088,434 triplets. We compare NBFNet with 6 embedding methods on this dataset. Note by the time of this work, only embedding methods are available for such large-scale datasets. Table~\ref{tab:ogbl-biokg} shows the results on ogbl-biokg. NBFNet achieves the best result compared to all methods reported on the official leaderboard\footnotemark with much fewer parameters. Note the previous best model AutoSF is based on architecture search and requires more computation resource than NBFNet for training.

\begin{wraptable}{r}{0.452\textwidth}
    \vspace{-1.6em}
    \centering
    \caption[Knowledge graph completion results on WikiKG90M validation set]{Knowledge graph completion results on WikiKG90M validation set.}
    \label{tab:ogb_lsc}
    \footnotesize
    \begin{tabular}{lc}
        \toprule
        \bf{Model} & \bf{MRR} \\
        \midrule
        NBFNet & 0.924 \\
        NBFNet (6 model ensemble) & 0.930 \\
        \bottomrule
    \end{tabular}
    \vspace{-1.6em}
\end{wraptable}

\smallskip \noindent \textbf{Results on WikiKG90M.}
Table~\ref{tab:ogb_lsc} shows the results of NBFNet on WikiKG90M validation set. Our best single model uses $k=2$ and $m=100$. While the validation set requires to rank the true entity against 1,000 negative entities, in practice it is not mandatory to draw 1,000 negative samples for each positive sample during training. We find that reducing the negative samples from 1,000 to 20 and increasing the batch size from 4 to 64 provides a better result, although it creates a distribution shift between sampled graphs in training and validation. We leave further research of such distribution shift as a future work.
\footnotetext{\url{https://ogb.stanford.edu/docs/leader_linkprop/\#ogbl-biokg}}
\section{Theories and Proofs}

\subsection{Path Formulations for Traditional Methods}
\label{app:path_formulation}

Here we demonstrate our path formulation is capable of modeling traditional link prediction methods like Katz index~\cite{katz1953new}, personalized PageRank~\cite{page1999pagerank} and graph distance~\cite{liben2007link}, as well as graph theory algorithms like widest path~\cite{baras2010path} and most reliable path~\cite{baras2010path}.

Recall the path formulation is defined as
\begin{align}
    \link \\
    \path
\end{align}
which can be written in the following compact form
\begin{equation}
    \vh_q(u,v) = \bigoplus_{P\in\gP_{uv}} \bigotimes_{i=1}^{|P|} \vw_q(e_i)
    \label{eqn:compact}
\end{equation}

\smallskip \noindent \textbf{Katz Index.}
The Katz index for a pair of nodes $u$, $v$ is defined as a weighted count of paths between $u$ and $v$, penalized by an attenuation factor $\beta\in(0,1)$. Formally, it can be written as
\begin{equation}
    \text{Katz}(u, v)
    =\sum_{t=1}^{\infty} \beta^t \ve_u^\top \mA^t \ve_v
\end{equation}
where $\mA$ denotes the adjacency matrix and $\ve_u$, $\ve_v$ denote the one-hot vector for nodes $u$, $v$ respectively. The term $\ve_u^\top \mA^t \ve_v$ counts all paths of length $t$ between $u$, and $v$ and shorter paths are assigned with larger weights.

\katz*
\begin{proof}
We show that $\text{Katz}(u, v)$ can be transformed into a summation over all paths between $u$ and $v$, where each path is represented by a product of damped edge weights in the path.
Mathematically, it can be derived as
\begin{align}
    \text{Katz}(u, v)
    &=\sum_{t=1}^{\infty} \beta^t \sum_{P \in \gP_{uv}:|P|=t}\prod_{e \in P} w_{e}\\
    &=\sum_{P \in \gP_{uv}}\prod_{e \in P} \beta w_{e}
\end{align}
Therefore, the Katz index can be viewed as a path formulation with the \emph{summation} operator $+$, the \emph{multiplication} operator $\times$ and the edge representations $\beta w_e$.
\end{proof}

\smallskip \noindent \textbf{Personalized PageRank.}
The personalized PageRank (PPR) for $u$ computes the stationary distribution over nodes generated by an infinite random walker, where the walker moves to a neighbor node with probability $\alpha$ and returns to the source node $u$ with probability $1-\alpha$ at each step. The probability of a node $v$ from a source node $u$ has the following closed-form solution~\cite{jeh2003scaling}
\begin{equation}
    \text{PPR}(u, v)
    =(1-\alpha)\sum_{t=1}^{\infty} \alpha^t \ve_u^\top (\mD^{-1}\mA)^t \ve_v
\end{equation}
where $\mD$ is the degree matrix and $\mD^{-1}\mA$ is the (random walk) normalized adjacency matrix.
Note that $\ve_u^\top (\mD^{-1}\mA)^t \ve_v$ computes the probability of $t$-step random walks from $u$ to $v$.

\pagerank*
\begin{proof}
We omit the coefficient $1-\alpha$, since it is always positive and has no effect on the ranking of different node pairs.
Then we have 
\begin{align}
    \text{PPR}(u, v)
    &\propto\sum_{t=1}^{\infty} \alpha^t \sum_{P \in \gP_{uv}:|P|=t}\prod_{(a,b) \in P} \frac{w_{ab}}{\sum_{b'\in\gN(a)}w_{ab'}} \\
    &=\sum_{P \in \gP_{uv}}\prod_{(a,b) \in P} \frac{\alpha w_{ab}}{\sum_{b'\in\gN(a)}w_{ab'}}
\end{align}
where the \emph{summation} operator is $+$, the \emph{multiplication} operator is $\times$ and edge representations are random walk probabilities scaled by $\alpha$.
\end{proof}

\smallskip \noindent \textbf{Graph Distance.}
Graph distance (GD) is defined as the minimum length of all paths between $u$ and $v$.

\distance*
\begin{proof}
Since the length of a path is the sum of edge lengths in the path, we have
\begin{equation}
    \text{GD}(u, v) = \min_{P \in \gP_{uv}} \sum_{e \in P} w_e
    \label{eqn:shortest_path}
\end{equation}
Here the \emph{summation} operator is $\min$, the \emph{multiplication} operator is $+$ and the edge representations are the lengths of edges.
\end{proof}

\smallskip \noindent \textbf{Widest Path.}
The widest path (WP), also known as the maximum capacity path, is aimed at finding a path between two given nodes, such that the path maximizes the minimum edge weight in the path.

\widest*
\begin{proof}
Given two nodes $u$ and $v$, we can write the widest path as
\begin{align}
    \text{WP}(u, v) =\max_{P \in \gP_{uv}} \min_{e \in P} w_e
\end{align}
Here the \emph{summation} operator is $\max$, the \emph{multiplication} operator is $\min$ and the edge representations are plain edge weights.
\end{proof}

\smallskip \noindent \textbf{Most Reliable Path.}
For a graph with non-negative edge probabilities, the most reliable path (MRP) is the path with maximal probability from a start node to an end node. This is also known as Viterbi algorithm~\cite{viterbi1967error} used in the maximum a posterior (MAP) inference of hidden Markov models (HMM).

\reliable*
\begin{proof}
For a start node $u$ and an end node $v$, the probaility of their most reliable path is 
\begin{align}
    \text{MRP}(u, v)
    =\max_{P \in \gP_{uv}} \prod_{e \in P} w_e
\end{align}
Here the \emph{summation} operator is $\max$, the \emph{multiplication} operator is $\times$ and the edge representations are edge probabilities.
\end{proof}

\subsection{Generalized Bellman-Ford Algorithm}
\label{app:bellman_ford}

First, we prove that the path formulation can be efficiently solved by the generalized Bellman-Ford algorithm when the operators $\langle\oplus, \otimes\rangle$ satisfy a semiring.
Then, we show that traditional methods satisfy the semiring assumption and therefore can be solved by the generalized Bellman-Ford algorithm.

\smallskip \noindent \textbf{Preliminaries on Semirings.}
Semirings are algebraic structures with two operators, \emph{summation} $\oplus$ and \emph{multiplication} $\otimes$, that share similar properties with the natural summation and the natural multiplication defined on integers. Specifically, $\oplus$ should be commutative, associative and have an identity element $\ozero$. $\otimes$ should be associative and have an identity element $\oone$. Mathematically, the \emph{summation} $\oplus$ satisfies
\begin{itemize}[label=$\bullet$]
    \setlength{\parskip}{0pt}
    \setlength{\itemsep}{0pt plus 1pt}
    \item \textbf{Commutative Property.} $a \oplus b = b \oplus a$
    \item \textbf{Associative Property.} $(a \oplus b) \oplus c = a \oplus (b \oplus c)$
    \item \textbf{Identity Element.} $a \oplus \ozero = a$
\end{itemize}
The \emph{multiplication} $\otimes$ satisfies
\begin{itemize}[label=$\bullet$]
    \setlength{\parskip}{0pt}
    \setlength{\itemsep}{0pt plus 1pt}
    \item \textbf{Associative Property.} $(a \otimes b) \otimes c = a \otimes (b \otimes c)$
    \item \textbf{Absorption Property.} $a \otimes \ozero = \ozero \otimes a = \ozero$
    \item \textbf{Identity Element.} $a \otimes \oone = \oone \otimes a = a$
\end{itemize}
Additionally, $\otimes$ should be distributive over $\oplus$.
\begin{itemize}[label=$\bullet$]
    \setlength{\parskip}{0pt}
    \setlength{\itemsep}{0pt plus 1pt}
    \item \textbf{Distributive Property (Left).} $a \otimes (b \oplus c) = (a \otimes b) \oplus (a \otimes c)$
    \item \textbf{Distributive Property (Right).} $(b \oplus c) \otimes a = (b \otimes a) \oplus (c \otimes a)$
\end{itemize}
Note semirings differ from natural arithmetic operators in two aspects. First, the \emph{summation} operator $\oplus$ does not need to be invertible, e.g., min or max. Second, the \emph{multiplication} operator $\otimes$ does not need to be commutative nor invertible, e.g., matrix multiplication.

\smallskip \noindent \textbf{Generalized Bellman-Ford Algorithm for Path Formulation.}
Now we prove that the generalized Bellman-Ford algorithm computes the path formulation when the operators $\langle\oplus, \otimes\rangle$ satisfy a semiring. It should be stressed that the generalized Bellman-Ford algorithm for path problems has been proved in~\cite{baras2010path}, and not a contribution of this paper. Here we apply the proof to our proposed path formulation.

The generalized Bellman-Ford algorithm computes the following iterations for all $v \in \gV$
\begin{align}
    \boundary \\
    \iteration
\end{align}
\begin{lemma}\label{lem:induction}
After $t$ Bellman-Ford iterations, the intermediate representation $\vh^{(t)}_q(u,v)$ aggregates all path representations within a length of $t$ edges for all $v$. That is
\begin{equation}
    \vh_q^{(t)}(u,v) = \bigoplus_{P\in\gP_{uv}:|P|\le t} \bigotimes_{i=1}^{|P|} \vw_q(e_i)
\end{equation}
\end{lemma}
\begin{proof}
We prove Lemma~\ref{lem:induction} by induction. For the base case $t=0$, there is a single path of length $0$ from $u$ to itself and no path to other nodes. Due to the product definition of path representations, a path of length $0$ is equal to the \emph{multiplication} identity $\oone_q$. Similarly, a summation of no path is equal to the \emph{summation} identity $\ozero_q$. Therefore, we have $\vh_q^{(0)}(u, v) = \mathbbm{1}_q(u = v) = \bigoplus_{P\in\gP_{uv}:|P|=0} \bigotimes_{i=1}^{|P|} \vw_q(e_i)$.

For the inductive case $t > 0$, we consider the second-to-last node $x$ in each path if the path has a length larger than $0$. To avoid overuse of brackets, we use the convention that $\otimes$ and $\bigotimes$ have a higher priority than $\oplus$ and $\bigoplus$. 
\begin{align}
    \vh^{(t)}_q(u,v) 
    &= \left(\bigoplus_{(x, r, v) \in \gE(v)}\vh_q^{(t-1)}(u,x) \otimes \vw_q(x, r, v)\right) \oplus \vh^{(0)}_q(u,v)\\
    &= \left[\bigoplus_{(x, r, v) \in \gE(v)}\left(\bigoplus_{P\in\gP_{ux}:|P|\le t-1} \bigotimes_{i=1}^{|P|} \vw_q(e_i)\right) \otimes \vw_q(x, r, v)\right] \oplus \vh^{(0)}_q(u,v)\label{eq:substitution}\\
    &=
    \left\{\bigoplus_{(x, r, v) \in \gE(v)}\left[\bigoplus_{P\in\gP_{ux}:|P|\le t-1}\left(\bigotimes_{i=1}^{|P|} \vw_q(e_i)\right) \otimes \vw_q(x, r, v)\right]\right\} \oplus \vh^{(0)}_q(u,v)\label{eq:distributive}\\
    &= \left(\bigoplus_{P\in\gP_{uv}:1\le|P|\le t} \bigotimes_{i=1}^{|P|} \vw_q(e_i)\right) \oplus \left( \bigoplus_{P\in\gP_{uv}:|P|=0} \bigotimes_{i=1}^{|P|} \vw_q(e_i)\right)\label{eq:associative}\\
    &= \bigoplus_{P\in\gP_{uv}:|P|\le t} \bigotimes_{i=1}^{|P|} \vw_q(e_i),
\end{align}
where Equation~\ref{eq:substitution} substitutes the inductive assumption for $\vh_q^{(t-1)}(u, x)$, Equation~\ref{eq:distributive} uses the distributive property of $\otimes$ over $\oplus$.
\end{proof}

By comparing Lemma~\ref{lem:induction} and Equation~\ref{eqn:compact}, we can see the intermediate representation converges to our path formulation $\lim_{t\rightarrow \infty} \vh_q^{(t)}(u, v) = \vh_q(u, v)$. More specifically, at most $|\gV|$ iterations are required if we only consider simple paths, i.e., paths without repeating nodes. In practice, for link prediction we find it only takes a very small number of iterations (e.g., $T = 6$) to converge, since long paths make negligible contribution to the task.

\bellman*
\begin{proof}
Given that the generalized Bellman-Ford algorithm solves the path formulation when $\langle\oplus, \otimes\rangle$ satisfy a semiring, we only need to show that the operators of the path formulations for traditional methods satisfy semiring structures.

Katz index (Theorem~\ref{thm:katz_index}) and personalized PageRank (Theorem~\ref{thm:pagerank}) use the natural summation $+$ and the natural multiplication $\times$, which obviously satisfy a semiring.

Graph distance (Theorem~\ref{thm:proximity}) uses $\min$ for \emph{summation} and $+$ for \emph{multiplication}. The corresponding identities are $\ozero = +\infty$ and $\oone = 0$. It is obvious that $+$ satisfies the associative property and has identity element $0$.
\begin{itemize}[label=$\bullet$]
    \setlength{\parskip}{0pt}
    \setlength{\itemsep}{0pt plus 1pt}
    \item \textbf{Commutative Property.} $\min(a, b) = \min(b, a)$
    \item \textbf{Associative Property.} $\min(\min(a, b), c) = \min(a, \min(b, c))$
    \item \textbf{Identity Element.} $\min(a, +\infty) = a$
    \item \textbf{Absorption Property.} $a + \infty = \infty + a = +\infty$
    \item \textbf{Distributive Property (Left).} $a + \min(b, c) = \min(a + b, a + c)$
    \item \textbf{Distributive Property (Right).} $\min(b, c) + a = \min(b + a, c + a)$
\end{itemize}

Widest path (Theorem~\ref{thm:widest_path}) uses $\max$ for \emph{summation} and $\min$ for \emph{multiplication}. The corresponding identities are $\ozero = -\infty$ and $\oone = +\infty$. We have
\begin{itemize}[label=$\bullet$]
    \setlength{\parskip}{0pt}
    \setlength{\itemsep}{0pt plus 1pt}
    \item \textbf{Commutative Property.} $\max(a, b) = \max(b, a)$
    \item \textbf{Associative Property.} $\max(\max(a, b), c) = \max(a, \max(b, c))$
    \item \textbf{Identity Element.} $\max(a, -\infty) = a$
    \item \textbf{Associative Property.} $\min(\min(a, b), c) = \min(a, \min(b, c))$
    \item \textbf{Absorption Property.} $\min(a, -\infty) = \min(-\infty, a) = -\infty$
    \item \textbf{Identity Element.} $\min(a, +\infty) = \min(+\infty, a) = a$
    \item \textbf{Distributive Property (Left).} $\min(a, \max(b, c)) = \max(\min(a, b), \min(a, c))$
    \item \textbf{Distributive Property (Right).} $\min(\max(b, c), a) = \max(\min(b, a), \min(c, a))$
\end{itemize}
where the distributive property can be proved by enumerating all possible orders of $a$, $b$ and $c$.

Most reliable path (Theorem~\ref{thm:reliable_path}) uses $\max$ for \emph{summation} and $\times$ for \emph{multiplication}.  The corresponding identities are $\ozero = 0$ and $\oone = 1$, since all path representations are probabilities in $[0, 1]$. It is obvious that $\times$ satisfies the associative property, the absorption property and has identity element $0$.
\begin{itemize}[label=$\bullet$]
    \setlength{\parskip}{0pt}
    \setlength{\itemsep}{0pt plus 1pt}
    \item \textbf{Commutative Property.} $\max(a, b) = \max(b, a)$
    \item \textbf{Associative Property.} $\max(\max(a, b), c) = \max(a, \max(b, c))$
    \item \textbf{Identity Element.} $\max(a, 0) = a$
    \item \textbf{Distributive Property (Left).} $a \times \max(b, c) = \max(a \times b, a \times c)$
    \item \textbf{Distributive Property (Right).} $\max(b, c) \times a = \max(b \times a, c \times a)$
\end{itemize}
where the identity element and the distributive property hold for non-negative elements.
\end{proof}

\section{Method: A*Net}
\label{sec:method}

We propose A*Net to scale up path-based methods with the A* algorithm. We show that the A* algorithm can be derived from the observation that only a small set of paths are important for reasoning (Section~\ref{sec:method_A*}). Since it is hard to handcraft a good priority function for knowledge graph reasoning (Table~\ref{tab:handcraft}), we design a neural priority function, and train it end-to-end for reasoning (Section~\ref{sec:method_neural}).

\subsection{Preliminary: A* Algorithm}
\label{sec:A*}
A* algorithm~\cite{hart1968formal} is an extension of the Bellman-Ford algorithm for shortest path problems. Unlike the Bellman-Ford algorithm that propagates through every node uniformly, the A* algorithm prioritizes propagation through nodes with higher priority according to a heuristic function specified by the user. With an appropriate heuristic function, A* algorithm can reduce the search space of paths. Formally, with the notation from Equation~\ref{eqn:path}, the priority function for node $x$ is
\begin{equation}
    s(x) = d(u, x) \otimes g(x, v)
    \label{eqn:priority_func}
\end{equation}
where $d(u, x)$ is the length of current shortest path from $u$ to $x$, and $g(x, v)$ is a heuristic function estimating the cost from $x$ to the target node $v$. For instance, for a grid-world shortest path problem (Figure~\ref{fig:correspondence}(a)), $g(x, v)$ is usually defined as the $L_1$ distance from $x$ to $v$, $\otimes$ is the addition operator, and $s(x)$ is a lower bound for the shortest path length from $u$ to $v$ through $x$. During each iteration, the A* algorithm prioritizes propagation through nodes with smaller $s(x)$.

\begin{figure*}[t]
    \centering
    \includegraphics[width=0.98\textwidth]{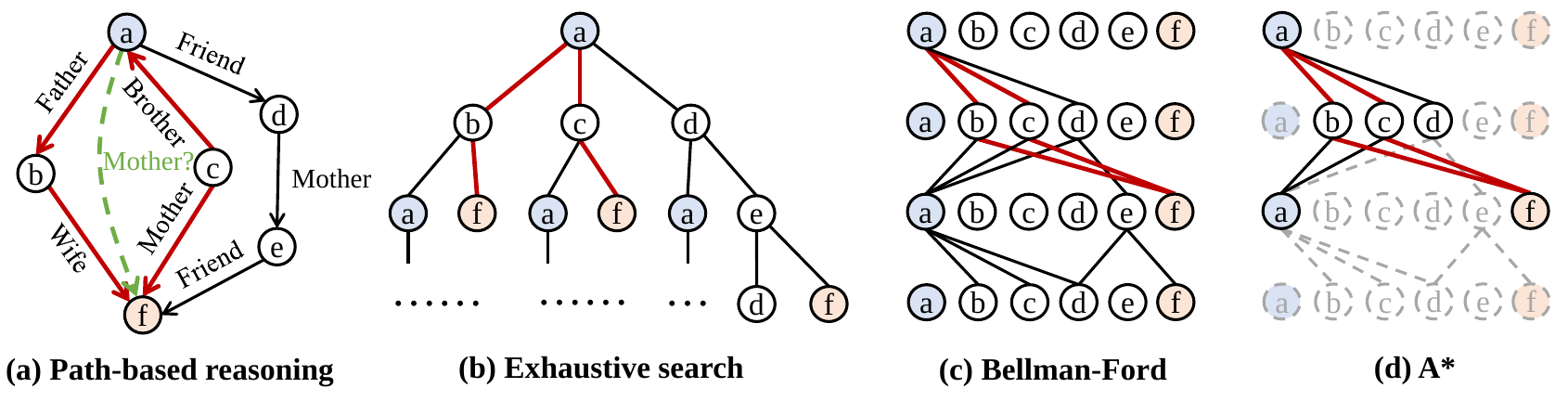}
    \caption[Overview of path-based methods, exhaustive search, NBFNet and A*Net]{\textbf{(a)} Given a query $(a, \textit{Mother}, ?)$, only a few important paths (showed in red) are necessary for reasoning. Note that paths can go in the reverse direction of relations. \textbf{(b)} Exhaustive search algorithm (e.g.\ Path-RNN, PathCon) enumerates all paths in exponential time. \textbf{(c)} Bellman-Ford algorithm (e.g.\ NeuralLP, DRUM, NBFNet, RED-GNN) computes all paths in polynomial time, but needs to propagate through all nodes and edges. \textbf{(d)} A*Net learns a priority function to select a subset of nodes and edges at each iteration, and avoids exploring all nodes and edges.}
    \label{fig:method_comparison}
\end{figure*}

\subsection{Path-based Reasoning with A* Algorithm}
\label{sec:method_A*}

The Bellman-Ford algorithm used to compute path-based methods needs to visit all $|\gV|$ nodes and $|\gE|$ edges in a knowledge graph. However, in real-world knowledge graphs, only a small portion of paths is related to the query. Based on this observation, we introduce the concept of important paths. We then show that the representations of important paths can be iteratively computed with the A* algorithm under mild assumptions.

\smallskip \noindent \textbf{Important Paths for Reasoning.} Given a query relation and a pair of entities, only some of the paths between the entities are important for answering the query. Consider the example in Figure~\ref{fig:method_comparison}(a), the path \emph{a $\xrightarrow{\text{Friend}}$ d $\xrightarrow{\text{Mother}}$ e $\xrightarrow{\text{Friend}}$ f} cannot determine whether \emph{f} is an answer to \emph{Mother(a, ?)} due to the use of the \emph{Friend} relation in the path. On the other hand, kinship paths like \emph{a $\xrightarrow{\text{Father}}$ b $\xrightarrow{\text{Wife}}$ f} or \emph{a $\xleftarrow{\text{Brother}}$ c $\xrightarrow{\text{Mother}}$ f} are able to predict that \emph{Mother(a, f)} is true. Formally, we define $\gP_{u \leadsto v|q} \subseteq \gP_{u \leadsto v}$ to be the set of paths from $u$ to $v$ that is important to the query relation $q$. Mathematically, we have
\begin{equation}
    \vh_q(u, v) = \bigoplus_{P \in \gP_{u \leadsto v}}\vh_q(P) \approx \bigoplus_{P \in \gP_{u \leadsto v|q}}\vh_q(P)
    \label{eqn:important_path}
\end{equation}
In other words, any path $P \in \gP_{u \leadsto v}\setminus\gP_{u \leadsto v|q}$ has negligible contribution to $\vh_q(u, v)$. In real-world knowledge graphs, the number of important paths $|\gP_{u \leadsto v|q}|$ may be several orders of magnitudes smaller than the number of paths $|\gP_{u \leadsto v}|$~\cite{chen2018variational}. If we compute the representation $\vh_q(u, v)$ using only the important paths, we can scale up path-based reasoning to large-scale knowledge graphs.

\begin{wrapfigure}{r}{0.45\textwidth}
    \vspace{-1.3em}
    \centering
    \includegraphics[width=0.45\textwidth]{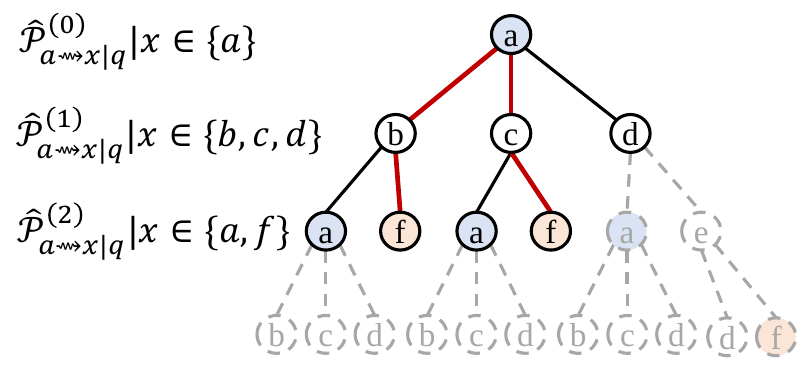}
    \vspace{-0.5em}
    \caption[Iterative computation of important paths]{The colored paths are important paths $\gP_{u \leadsto v|q}$, while the solid paths are the superset $\hat{\gP}_{u \leadsto v|q}$ used in Equation~\ref{eqn:path_selection}.}
    \label{fig:recursive}
    \vspace{-0.8em}
\end{wrapfigure}
\smallskip \noindent \textbf{Iterative Computation of Important Paths.}
Given a query $(u, q, ?)$, we need to discover the set of important paths $\gP_{u \leadsto v|q}$ for all $v \in \gV$. However, it is challenging to extract important paths from $\gP_{u \leadsto v}$, since the size of $\gP_{u \leadsto v}$ is exponentially large. Our solution is to explore the structure of important paths and compute them iteratively. We first show that we can cover important paths with iterative path selection (Equation~\ref{eqn:path_boundary} and \ref{eqn:path_selection}). Then we approximate iterative path selection with iterative node selection (Equation~\ref{eqn:node_selection}).

Notice that paths in $\gP_{u \leadsto v}$ form a tree structure (Figure~\ref{fig:recursive}). On the tree, a path is not important if any prefix of this path is not important for the query. For example, in Figure~\ref{fig:method_comparison}(a), \emph{a $\xrightarrow{\text{Friend}}$ d $\xrightarrow{\text{Mother}}$ e $\xrightarrow{\text{Friend}}$ f} is not important, as its prefix \emph{a $\xrightarrow{\text{Friend}}$ d} is not important for the query \emph{Mother}. Therefore, we assume there exists a path selection function $m_q: 2^\gP \mapsto 2^\gP$ that selects important paths from a set of paths given the query relation $q$. $2^\gP$ is the set of all subsets of $\gP$. With $m_q$, we construct the following set of paths $\hat{\gP}_{u \leadsto v|q}^{(t)}$ iteratively
\begin{align}
    &\hat{\gP}_{u \leadsto v|q}^{(0)} \leftarrow \{(u, \text{self loop}, v)\}~\text{if}~u = v~\text{else}~\varnothing 
    \label{eqn:path_boundary} \\
    &\hat{\gP}_{u \leadsto v|q}^{(t)} \leftarrow \bigcup_{\substack{x \in \gV \\ (x,r,v) \in \gE(v)}} \left\{P + \{(x,r,v)\} \middle| P \in m_q(\hat{\gP}_{u \leadsto x|q}^{(t-1)})\right\}
    \label{eqn:path_selection}
\end{align}
where $P + \{(x,r,v)\}$ concatenates the path $P$ and the edge $(x, r, v)$. The paths $\hat{\gP}_{u \leadsto v|q}^{(t)}$ computed by the above iteration is a superset of the important paths $\gP_{u \leadsto v|q}^{(t)}$ of length $t$ (see Theorem~\ref{thm:superset} in Section~\ref{app:method_A*}). Due to the tree structure of paths, the above iterative path selection still requires exponential time. Hence we further approximate iterative path selection with iterative node selection, by assuming paths with the same length and the same stop node can be merged. The iterative node selection replacing Equation~\ref{eqn:path_selection} is (see Proposition~\ref{prop:node_selection} in Section~\ref{app:method_A*})
\begin{align}
    \hat{\gP}_{u \leadsto v|q}^{(t)} \leftarrow \bigcup_{\substack{x \in n_{uq}^{(t-1)}(\gV) \\ (x,r,v) \in \gE(v)}} \left\{P + \{(x,r,v)\} \middle| P \in \hat{\gP}_{u \leadsto x|q}^{(t-1)}\right\}
    \label{eqn:node_selection}
\end{align}
where $n_{uq}^{(t)}: 2^\gV \mapsto 2^\gV$ selects ending nodes of important paths of length $t$ from a set of nodes.

\smallskip \noindent \textbf{Reasoning with A* Algorithm.}
Equation~\ref{eqn:node_selection} iteratively computes the set of important paths $\hat{\gP}_{u \leadsto v|q}$. In order to perform reasoning, we need to compute the representation $\vh_q(u, v)$ based on the important paths, which can be achieved by an iterative process similar to Equation~\ref{eqn:node_selection} (see Theorem~\ref{thm:A*} in Section~\ref{app:method_A*})
\begin{align}
    \vh_q^{(t)}(u, v) \leftarrow &\vh_q^{(0)}(u, v) \oplus \bigoplus_{\substack{x \in n_{uq}^{(t-1)}(\gV) \\ (x, r, v)\in \gE(v)}} \vh_q^{(t-1)}(u, x) \otimes \vw_q(x, r, v)
    \label{eqn:A*}
\end{align}
Equation~\ref{eqn:A*} is the A* iteration (Figure~\ref{fig:method_comparison}(d)) for path-based reasoning. Note the A* iteration uses the same boundary condition as Equation~\ref{eqn:boundary_nbfnet}. Inspired by the classical A* algorithm, we parameterize $n_{uq}^{(t)}(\gV)$ with a node priority function $s_{uq}^{(t)}: \gV \mapsto [0, 1]$ and select top-$K$ nodes based on their priority. However, there does not exist an oracle for the priority function $s_{uq}^{(t)}(x)$. We will discuss how to learn the priority function $s_{uq}^{(t)}(x)$ in the following sections.

\subsection{Path-based Reasoning with A*Net}
\label{sec:method_neural}

Both the performance and the efficiency of the A* algorithm heavily rely on the heuristic function. While it is straightforward to use $L_1$ distance as the heuristic function for grid-world shortest path problems, it is not clear what a good priority function for knowledge graph reasoning is due to the complex relation semantics in knowledge graphs. Indeed, our experiments suggest that handcrafted priority functions largely hurt the performance of path-based methods (Table~\ref{tab:handcraft}). In this section, we discuss a neural priority function, which can be end-to-end trained by the reasoning task.

\begin{wrapfigure}{r}{0.475\textwidth}
    \centering
    \vspace{-1.3em}
    \includegraphics[width=0.465\textwidth]{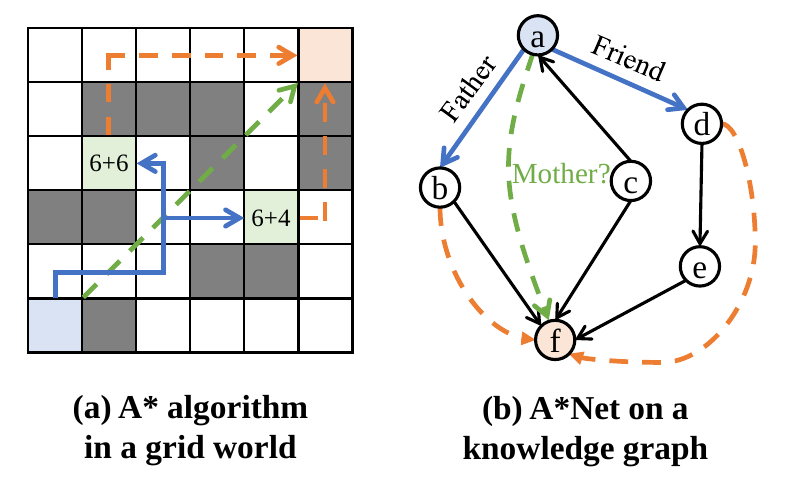}
    \caption[Correspondence between A* algorithm and A*Net]{\textbf{(a)} A* algorithm computes the current distance $d(u, x)$ (blue), estimates the remaining distance $g(x, v)$ (orange), and prioritizes shorter paths. \textbf{(b)} A*Net computes the current representations $\vh_q^{(t)}(u, x)$ (blue), estimates the remaining representations $\vg([\vh_q^{(t)}(u, x), \vq])$ (orange) based on the query $\vq$ (green), and prioritizes paths more relevant to the query.}
    \label{fig:correspondence}
\end{wrapfigure}

\smallskip \noindent \textbf{Neural Priority Function.}
To design the neural priority function $s_{uq}(x)$, we draw inspiration from the priority function in the A* algorithm for shortest path problems (Equation~\ref{eqn:priority_func}). The priority function has two terms $d(u, x)$ and $g(x, v)$, where $d(u, x)$ is the current distance from node $u$ to $x$, and $g(x, v)$ estimates the remaining distance from node $x$ to $v$.

From a representation learning perspective, we need to learn a representation $\vs_{uq}(x)$ to predict the priority score $s_{uq}(x)$ for each node $x$. Inspired by Equation~\ref{eqn:priority_func}, we use the current representation $\vh_q^{(t)}(u, x)$ to represent $d^{(t)}(u, x)$. However, it is challenging to find a representation for $g^{(t)}(x, v)$, since we do not know the answer entity $v$ beforehand. Noticing that in the A* algorithm, the target node $v$ can be expressed by the source node plus a displacement (Figure~\ref{fig:correspondence}(a)), we reparameterize the answer entity $v$ with the head entity $u$ and the query relation $q$ in A*Net. By replacing $g^{(t)}(x, v)$ with another function $g^{(t)}(u, x, q)$, the representation $\vs_{uq}(x)$ is parameterized as
\begin{equation}
    \vs_{uq}^{(t)}(x) = \vh_q^{(t)}(u, x) \otimes \vg([\vh_q^{(t)}(u, x), \vq])
    \label{eqn:priority_repr}
\end{equation}
where $\vg(\cdot)$ is a feed-forward network that outputs a vector representation and $[\cdot, \cdot]$ concatenates two representations. Intuitively, the learned representation $\vq$ captures the semantic of query relation $q$, which serves the goal for answering query $(u, q, ?)$. The function $\vg([\vh_q^{(t)}(u, x), \vq])$ compares the current representation $\vh_q^{(t)}(u, x)$ with the goal $\vq$ to estimate the remaining representation (Figure~\ref{fig:correspondence}(b)). If $\vh_q^{(t)}(u, x)$ is close to $\vq$, the remaining representation will be close to 0, and $x$ is likely to be close to the correct answer. The final priority score is predicted by
\begin{equation}
    s_{uq}^{(t)}(x) = \sigma(f(\vs_{uq}^{(t)}(x)))
    \label{eqn:neural_func}
\end{equation}
where $f(\cdot)$ is a feed-forward network and $\sigma$ is the sigmoid function that maps the output to $[0, 1]$.

\smallskip \noindent \textbf{Learning.} To learn the neural priority function, we incorporate it as a weight for each message in the A* iteration. For simplicity, let $\gX^{(t)} = n_{uq}^{(t-1)}(\gV)$ be the nodes we try to propagate through at $t$-th iteration. We modify Equation~\ref{eqn:A*} to be
\begin{equation}
    \vh_q^{(t)}(u, v) \leftarrow \vh_q^{(0)}(u, v) \oplus \bigoplus_{\substack{x \in \gX^{(t)} \\ (x, r, v)\in \gE(v)}} s_{uq}^{(t-1)}(x)\left(\vh_q^{(t-1)}(u, x) \otimes \vw_q(x, r, v)\right)
    \label{eqn:weight}
\end{equation}
Equation~\ref{eqn:weight} encourages the model to learn larger weights $s_{uq}^{(t)}(x)$ for nodes that are important for reasoning. In practice, as some nodes may have very large degrees, we further select top-$L$ edges from the neighborhood of $n_{uq}^{(t-1)}(\gV)$ according to the priority of node $v$, i.e.\ the tail node of an edge. The intuition here is that if an edge $(x, r, v)$ goes to a node with a higher priority, it is likely that we are propagating towards the answer entities. A pseudo code of A*Net is illustrated in Algorithm~\ref{alg:A*Net}. Note the top-$K$ and top-$L$ functions are not differentiable.

\begin{wrapfigure}{R}{0.54\textwidth}
    \begin{minipage}{0.54\textwidth}
        \vspace{-3.6em}
        \begin{algorithm}[H]
            \footnotesize
            \captionsetup{font=footnotesize}\caption{A*Net}
            \begin{flushleft}
                \textbf{Input:} head entity $u$, query relation $q$, \#iterations $T$ \\
                \textbf{Output:} $p(v|u, q)$ for all $v \in \gV$
            \end{flushleft}
            \begin{algorithmic}[1]
                \For{$v \in \gV$}
                    \State{$\vh_q^{(0)}(u, v) \gets \mathbbm{1}_q(u = v)$}
                \EndFor
                \For{$t \gets 1$ to $T$}
                    \State{$\gX^{(t)} \gets \textrm{TopK}(s_{uq}^{(t-1)}(x) | x \in \gV)$}
                    \State{$\gE^{(t)} \gets \bigcup_{x \in \gX^{(t)}} \gE(x) $}
                    \State{$\gE^{(t)} \gets \textrm{TopL}(s_{uq}^{(t-1)}(v) | (x, r, v) \in \gE^{(t)})$}
                    \State{$\gV^{(t)} \gets \bigcup_{(x, r, v) \in \gE^{(t)}} \{v\}$}
                    \For{$v \in \gV^{(t)}$}
                        \State{Compute $\vh_q^{(t)}(u, v)$ with Equation~\ref{eqn:weight}}
                        \State{Compute priority $s_{uq}^{(t)}(v)$ with Equation~\ref{eqn:priority_repr}, \ref{eqn:neural_func}}
                    \EndFor
                \EndFor
                \LineComment{Share weights between $s_{uq}(v)$ and the predictor}
                \State{\Return $s_{uq}^{(T)}(v)$ as $p(v|u, q)$ for all $v \in \gV$}
            \end{algorithmic}
            \label{alg:A*Net}
        \end{algorithm}
        \vspace{-2.4em}
    \end{minipage}
\end{wrapfigure}
Nevertheless, it is still too challenging to train the neural priority function, since we do not know the ground truth for important paths, and there is no direct supervision for the priority function. Our solution is to share the weights between the priority function and the predictor for the reasoning task. The intuition is that the reasoning task can be viewed as a weak supervision for the priority function. Recall that the goal of $s_{uq}^{(t)}(x)$ is to determine whether there exists an important path from $u$ to $x$ (Equation~\ref{eqn:node_selection}). In the reasoning task, any positive answer entity must be present on at least one important path, while negative answer entities are less likely to be on important paths. Our ablation experiment demonstrates that sharing weights improve the performance of neural priority function (Table~\ref{tab:share_weights}). Following \cite{sun2019rotate}, A*Net is trained to minimize the binary cross entropy loss over triplets
\begin{equation}
    \gL = - \log p(u, q, v) - \sum_{i=1}^{n} \frac{1}{n} \log (1 - p(u_i', q, v_i'))
    \label{eqn:training_loss}
\end{equation}
where $(u, q, v)$ is a positive sample and $\{(u_i', q, v_i')\}_{i=1}^n$ are negative samples. Each negative sample $(u_i, q, v_i)$ is generated by corrupting the head or the tail in a positive sample.

\begin{figure*}[!h]
    \centering
    \includegraphics[width=0.98\textwidth]{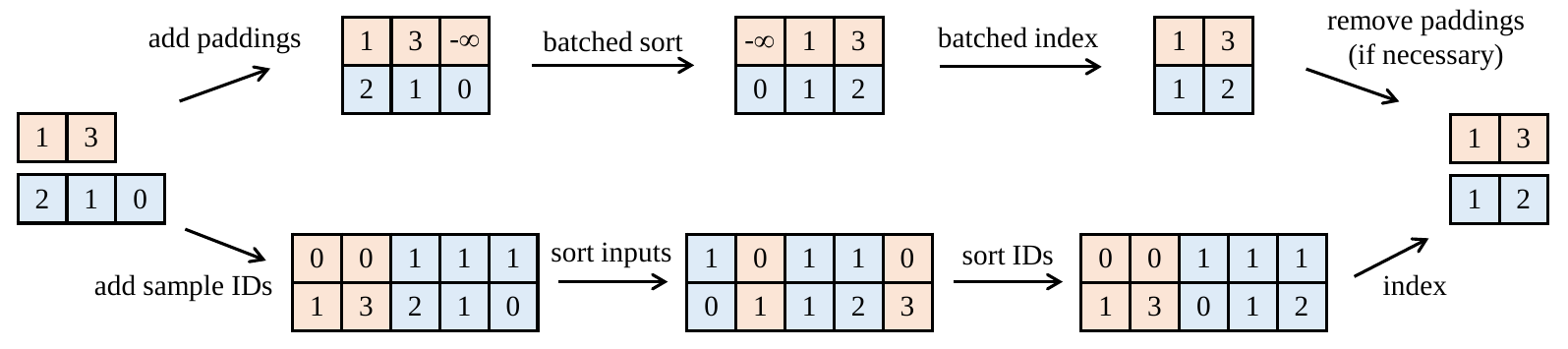}
    \caption[Comparison between padding-based and padding-free topk]{Comparison between padding-based (up) and padding-free (down) \emph{topk} for $K=2$. Padding-based operations first add paddings to create a tensor for batched operations, and then remove the paddings. Padding-free operations pair the inputs with their sample IDs (shown in colors), and then apply single-sample operations over the whole batch.}
    \label{fig:padding_free}
\end{figure*}

\smallskip \noindent \textbf{Efficient Implementation with Padding-Free Operations.}
Modern neural networks heavily rely on batched execution to unleash the parallel capacity of GPUs. While Algorithm~\ref{alg:A*Net} is easy to implement for a single sample $(u, q, ?)$, it is not trivial to batch A*Net for multiple samples. The challenge is that different samples may have very different sizes for nodes $\gV^{(t)}$ and edges $\gE^{(t)}$. A common approach is to pad the set of nodes or edges to a predefined constant, which would severely counteract the acceleration brought by A*Net.

Here we introduce padding-free $topk$ operation to avoid the overhead in batched execution. The key idea is to convert batched execution of different small samples into execution of a single large sample, which can be paralleled by existing operations in deep learning frameworks. Considering the example in Figure~\ref{fig:padding_free}, the batched execution of $topk([[1, 3], [2, 1, 0]])$ can be converted into a multi-key sort problem over $[[0, 1], [0, 3], [1, 2], [1, 1], [1, 0]]$, where the first key is the sample index in the batch and the second key is the original input. The multi-key sort is then implemented by composing stable single-key sort operations in deep learning frameworks. We then apply indexing operations and remove the sample index to get the desired output. Algorithm~\ref{alg:topk} provides the pseudo code for padding-free \emph{topk}.

\begin{algorithm}[!h]
    \footnotesize
    \captionsetup{font=footnotesize}\caption{Padding-free implementation of \emph{topk} in PyTorch}
\begin{flushleft}
    \textbf{Input:} Input values of each sample \mintinline{python}{inputs}, size of each sample \mintinline{python}{sizes}, \mintinline{python}{K} \\
    \textbf{Output:} TopK values of each sample, indices of topk values
\end{flushleft}
\begin{minted}[baselinestretch=1,fontsize=\footnotesize,linenos,numbersep=-6pt]{python}
   # the sample id of each element
   sample_ids = torch.arange(batch_size).repeat_interleave(sizes)
   # multi-key sort of (sample_ids, inputs)
   indices = inputs.argsort()
   indices = sample_ids[indices].argsort(stable=True)
   sorteds = inputs[indices]
   # take top-k values of each sample
   ranges = torch.arange(K).repeat(batch_size)
   ranges = ranges + sizes.cumsum(0).repeat_interleave(K) - K
   return sorteds[ranges], indices[ranges]
\end{minted}
    \label{alg:topk}
    \vspace{-0.3em}
\end{algorithm}
\section{Experiments of A*Net}
\label{sec:experiment}

We evaluate A*Net on standard transductive and inductive knowledge graph datasets, including a million-scale one ogbl-wikikg2. We conduct ablation studies to verify our design choices and visualize the important paths learned by the priority function in A*Net.

\subsection{Experiment Setup}
\label{sec:exp_setup}

\smallskip \noindent \textbf{Datasets \& Evaluation.} We evaluate A*Net on 4 standard knowledge graphs, FB15k-237~\cite{toutanova2015observed}, WN18RR~\cite{dettmers2018convolutional}, YAGO3-10~\cite{mahdisoltani2014yago3} and ogbl-wikikg2~\cite{hu2021ogb}. For the transductive setting, we use the standard splits from their original works~\cite{toutanova2015observed, dettmers2018convolutional}. For the inductive setting, we use the splits provided by \cite{teru2020inductive}, which contains 4 different versions for each dataset. As for evaluation, we use the standard filtered ranking protocol~\cite{bordes2013translating} for knowledge graph reasoning. Each triplet $(u, q, v)$ is ranked against all negative triplets $(u, q, v')$ or $(u', q, v)$ that are not present in the knowledge graph. We measure the performance with mean reciprocal rank (MRR) and HITS at K (H@K). Efficiency is measured by the average number of messages (\#message) per step, wall time per epoch and memory cost. To plot the convergence curves for each model, we dump checkpoints during training with a high frequency, and evaluate the checkpoints later on the validation set. See more details in Section~\ref{app:dataset_nbfnet}.

\smallskip \noindent \textbf{Implementation Details.}
Our work is developed based on the open-source codebase of path-based reasoning with Bellman-Ford algorithm\footnote{\url{https://github.com/DeepGraphLearning/NBFNet}. MIT license. \label{fn:codebase}}. For a fair comparison with existing path-based methods, we follow the implementation of NBFNet~\cite{zhu2021neural} and parameterize $\bigoplus$ with principal neighborhood aggregation (PNA)~\cite{corso2020principal} or sum aggregation, and parameterize $\bigotimes$ with the relation operation from DistMult~\cite{yang2015embedding}, i.e.\ vector multiplication. The indicator function (Equation~\ref{eqn:boundary_nbfnet}) $\mathbbm{1}_q(u = v) = \mathbbm{1}(u = v)\vq$ is parameterized with a query embedding $\vq$ for all datasets except ogbl-wikikg2, where we augment the indicator function with learnable embeddings based on the personalized PageRank~\cite{page1999pagerank} score from $u$ to $v$. The edge representation (Equation~\ref{eqn:weight}) $\vw_q(x, r, v) = \mW_r \vq + \vb_r$ is parameterized as a linear function over the query relation $q$ for all datasets except WN18RR, where we use a simple embedding $\vw_q(x, r, v) = \vr$. We use the same preprocessing steps as in \cite{zhu2021neural}, including augmenting each triplet with a flipped triplet, and dropping out query edges during training.

For the neural priority function, we have two hyperparameters: $K$ for the maximum number of nodes and $L$ for the maximum number of edges. To make hyperparameter tuning easier, we define maximum node ratio $\alpha=K/|\gV|$ and maximum average degree ratio $\beta=L|\gV|/K|\gE|$, and tune the ratios for each dataset. The maximum edge ratio is determined by $\alpha\beta$. The other hyperparameters are kept the same as the values in \cite{zhu2021neural}. We train A*Net with 4 Tesla A100 GPUs (40 GB), and select the best model based on validation performance.

\smallskip \noindent \textbf{Baselines.} We compare A*Net against embedding methods, GNNs and path-based methods. The embedding methods are TransE~\cite{bordes2013translating}, ComplEx~\cite{trouillon2016complex}, RotatE~\cite{sun2019rotate}, HAKE~\cite{zhang2020learning}, RotH~\cite{chami2020low}, PairRE~\cite{chao2021pairre}, ComplEx+Relation Prediction~\cite{chen2021relation} and ConE~\cite{bai2021modeling}. The GNNs are RGCN~\cite{schlichtkrull2018modeling}, CompGCN~\cite{vashishth2020composition} and GraIL~\cite{teru2020inductive}. The path-based methods are MINERVA~\cite{das2018go}, Multi-Hop~\cite{lin2018multi}, CURL~\cite{zhang2022learning}, NeuralLP~\cite{yang2017differentiable}, DRUM~\cite{sadeghian2019drum}, NBFNet~\cite{zhu2021neural} and RED-GNN~\cite{zhang2022knowledge}. Note that path-finding methods~\cite{das2018go, lin2018multi, zhang2022learning} that use reinforcement learning and assume sparse answers can only be evaluated on tail prediction. Training time of all baselines are measured based on their official open-source implementations, except that we use a more recent implementation\footnote{\url{https://github.com/DeepGraphLearning/KnowledgeGraphEmbedding}. MIT license.} of TransE and ComplEx.

\begin{table}[t]
    \centering
    \begin{minipage}[b]{0.63\textwidth}
        \centering
        \footnotesize
        \caption[Performance on transductive knowledge graph reasoning]{Performance on transductive knowledge graph reasoning. Results of embedding methods are from \cite{bai2021modeling}. Results of other baseline methods are from \cite{zhu2021neural}.}
        \vspace{-0.2em}
        \begin{adjustbox}{max width=\textwidth}
            \begin{tabular}{lcccccccc}
                \toprule
                \multirow{2}{*}{\bf{Method}} & \multicolumn{4}{c}{\bf{FB15k-237}} & \multicolumn{4}{c}{\bf{WN18RR}} \\
                & MRR & H@1 & H@3 & H@10 & MRR & H@1 & H@3 & H@10 \\
                \midrule
                TransE & 0.294 & - & - & 0.465 & 0.226 & - & 0.403 & 0.532 \\
                RotatE & 0.338 & 0.241 & 0.375 & 0.533 & 0.476 & 0.428 & 0.492 & 0.571 \\
                HAKE & 0.341 & 0.243 & 0.378 & 0.535 & 0.496 & 0.451 & 0.513 & 0.582 \\
                RotH & 0.344 & 0.246 & 0.380 & 0.535 & 0.495 & 0.449 & 0.514 & 0.586 \\
                ComplEx+RP & 0.388 & 0.298 & 0.425 & 0.568 & 0.488 & 0.443 & 0.505 & 0.578 \\
                ConE & 0.345 & 0.247 & 0.381 & 0.540 & 0.496 & 0.453 & 0.515 & 0.579 \\
                \midrule
                RGCN & 0.273 & 0.182 & 0.303 & 0.456 & 0.402 & 0.345 & 0.437 & 0.494 \\
                CompGCN & 0.355 & 0.264 & 0.390 & 0.535 & 0.479 & 0.443 & 0.494 & 0.546 \\
                \midrule
                NeuralLP & 0.240 & - & - & 0.362 & 0.435 & 0.371 & 0.434 & 0.566 \\
                DRUM & 0.343 & 0.255 &0.378 & 0.516 & 0.486 & 0.425 & 0.513 & 0.586 \\
                NBFNet & \bf{0.415} & \bf{0.321} & \bf{0.454} & \bf{0.599} & \bf{0.551} & \bf{0.497} & \bf{0.573} & \bf{0.666} \\
                RED-GNN & 0.374 & 0.283 & - & 0.558 & 0.533 & 0.485 & - & 0.624 \\
                A*Net & \bf{0.411} & \bf{0.321} & \bf{0.453} & 0.586 & \bf{0.549} & \bf{0.495} & \bf{0.573} & \bf{0.659} \\
                \bottomrule
            \end{tabular}
        \end{adjustbox}
        \label{tab:transductive}
        \caption[Tail prediction performance on transductive datasets]{Tail prediction performance on transductive datasets. Results of compared methods are from \cite{lin2018multi, zhang2022learning}.}
        \vspace{-0.2em}
        \begin{adjustbox}{max width=\textwidth}
            \begin{tabular}{lcccccccc}
                \toprule
                \multirow{2}{*}{\bf{Method}} & \multicolumn{4}{c}{\bf{FB15k-237}} & \multicolumn{4}{c}{\bf{WN18RR}} \\
                & MRR & H@1 & H@3 & H@10 & MRR & H@1 & H@3 & H@10 \\
                \midrule
                MINERVA & 0.293 & 0.217 & 0.329 & 0.456 & 0.448 & 0.413 & 0.456 & 0.513 \\
                Multi-Hop & 0.393 & 0.329 & - & 0.544 & 0.472 & 0.437 & - & 0.542 \\
                CURL & 0.306 & 0.224 & 0.341 & 0.470 & 0.460 & 0.429 & 0.471 & 0.523 \\
                NBFNet & \bf{0.509} & \bf{0.411} & \bf{0.562} & \bf{0.697} & \bf{0.557} & \bf{0.503} & \bf{0.579} & \bf{0.669} \\
                A*Net & \bf{0.505} & \bf{0.410} & \bf{0.556} & 0.687 & \bf{0.557} & \bf{0.504} & \bf{0.580} & \bf{0.666} \\
                \bottomrule
            \end{tabular}
        \end{adjustbox}
        \label{tab:tail_prediction}
        \caption[Efficiency on transductive datasets]{Efficiency on transductive datasets.}
        \vspace{-0.2em}
        \begin{adjustbox}{max width=\textwidth}
            \begin{tabular}{lcccccc}
                \toprule
                \multirow{2}{*}{\bf{Method}} &  \multicolumn{3}{c}{\bf{FB15k-237}} & \multicolumn{3}{c}{\bf{WN18RR}} \\
                & \#message & time & memory & \#message & time & memory \\
                \midrule
                NBFNet & 544,230 & 16.8 min & 19.1 GiB & 173,670 & 9.42 min &  26.4 GiB \\
                A*Net & 38,610 & 8.07 min & 11.1 GiB & 4,049 & 1.39 min &  5.04 GiB \\
                \midrule
                Improvement & 14.1$\times$ & 2.1$\times$ & 1.7$\times$ & 42.9$\times$ & 6.8$\times$ & 5.2$\times$ \\
                \bottomrule
            \end{tabular}
        \end{adjustbox}
        \label{tab:transductive_efficiency}
    \end{minipage}
    \hspace{0.5em}
    \begin{minipage}[b]{0.34\textwidth}
        \centering
        \includegraphics[width=\textwidth]{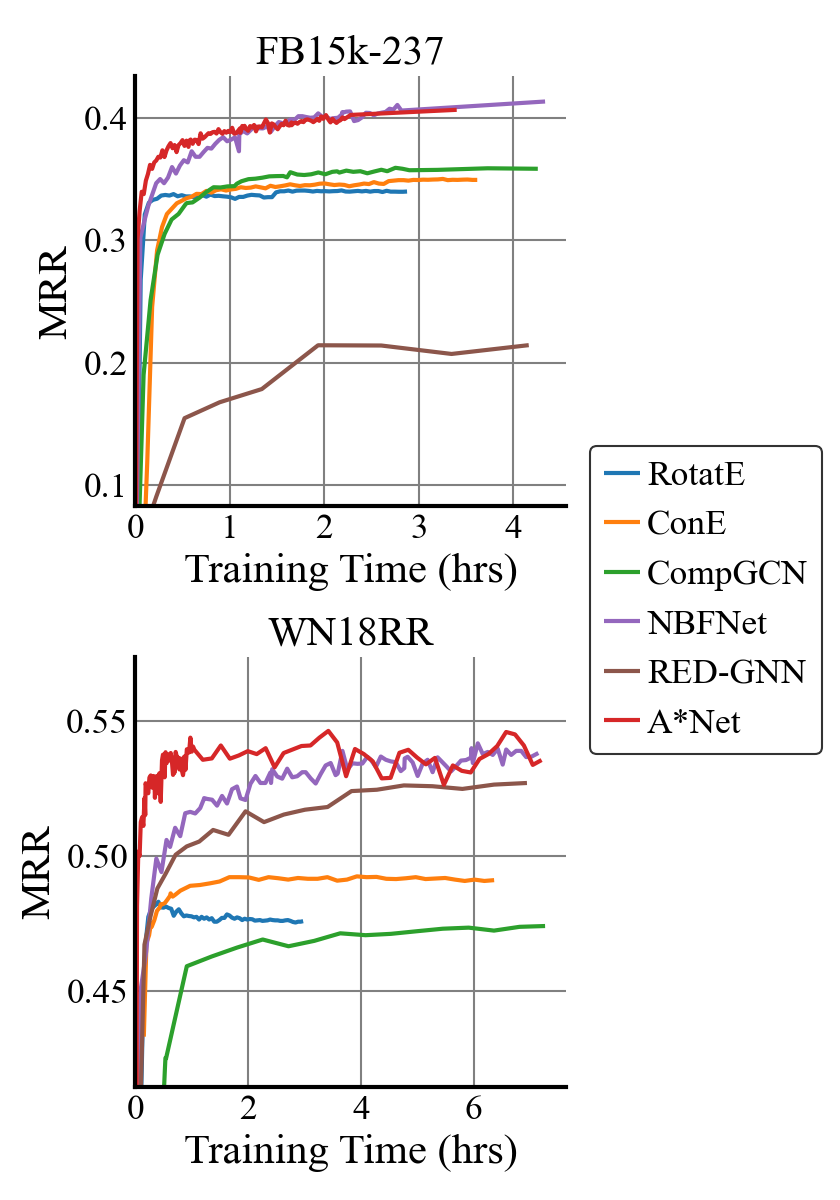}
        \captionof{figure}[Results w.r.t.\ training time]{Validation MRR w.r.t.\ training time (1 A100 GPU).}
        \label{fig:convergence}
        \vspace{0.2em}
        \captionof{table}[Performance on ogbl-wikikg2]{Performance on ogbl-wikikg2 (MRR). Results of compared methods are from \cite{chao2021pairre, chen2021relation}. NBFNet runs out of memory (OOM) on this dataset.}
        \footnotesize
        \begin{adjustbox}{max width=\textwidth}
            \begin{tabular}{lccc}
                \toprule
                \multirow{2}{*}{\bf{Method}} & \multicolumn{3}{c}{\bf{ogbl-wikikg2}} \\
                & Test & Valid & \#Params \\
                \midrule
                TransE & 0.4256 & 0.4272 & 1,251 M \\
                ComplEx & 0.4027 & 0.3759 & 1,251 M \\
                RotatE & 0.4332 & 0.4353 & 1,251 M \\
                PairRE & 0.5208 & 0.5423 & 500 M \\
                ComplEx+RP & 0.6392 & 0.6561 & 250 M \\
                \midrule
                NBFNet & OOM & OOM & OOM \\
                A*Net & \bf{0.6767} & \bf{0.6851} & \bf{6.83 M} \\
                \bottomrule
            \end{tabular}
        \end{adjustbox}
        \label{tab:large-scale}
    \end{minipage}
\end{table}

\subsection{Main Results}
\label{sec:result}

Table~\ref{tab:transductive} shows that A*Net outperforms all embedding methods and GNNs, and is on par with NBFNet on transductive knowledge graph reasoning. We also observe a similar trend of A*Net and NBFNet over path-finding methods on tail prediction (Table~\ref{tab:tail_prediction}). Since path-finding methods select only one path with reinforcement learning, such results imply the advantage of aggregating multiple paths in A*Net. A*Net also converges faster than all the other methods (Figure~\ref{fig:convergence}). Notably, unlike NBFNet that propagates through all nodes and edges, A*Net only propagates through 10\% nodes and 10\% edges on both datasets, which suggests that most nodes and edges are not important for path-based reasoning. Table~\ref{tab:transductive_efficiency} shows that A*Net reduces the number of messages by 14.1$\times$ and 42.9$\times$ compared to NBFNet on two datasets respectively. We observe a similar trend for A*Net on YAGO3-10, as shown in Table~\ref{tab:yago}. Note that the reduction in time and memory is less than the reduction in the number of messages, since A*Net operates on subgraphs with dynamic sizes and is harder to parallel than NBFNet on GPUs. We leave better parallel implementation as future work.

Table~\ref{tab:large-scale} shows the performance on ogbl-wikikg2, which has 2.5 million entities and 16 million triplets. While NBFNet faces out-of-memory (OOM) problem even for a batch size of 1, A*Net can perform reasoning by propagating through 0.2\% nodes and 0.2\% edges at each step. Surprisingly, even with such sparse propagation, A*Net outperforms embedding methods and achieves a new state-of-the-art result. Moreover, the validation curve in Figure~\ref{fig:ogbl-wikikg2} shows that A*Net converges significantly faster than embedding methods. Since A*Net only learns parameters for relations but not entities, it only uses 6.83 million parameters, which is 36.6$\times$ less than the best embedding method ComplEx+RP.

Table~\ref{tab:inductive} shows the performance on inductive knowledge graph reasoning. Note that embedding methods cannot deal with the inductive setting. A*Net is on par with NBFNet and significantly outperforms all the other methods. Additionally, as shown in Table~\ref{tab:inductive_efficiency}, A*Net reduces the number of messages by $3.1\times$ and $54.5\times$ on average for two datasets respectively.

\begin{table}[t]
    \centering
    \caption[Performance and efficiency on YAGO3-10]{Performance and efficiency on YAGO3-10, with $\alpha=10\%$ and $\beta=100\%$. Results of compared methods are from \cite{sun2019rotate}.}
    \vspace{-0.2em}
    \begin{subtable}[t]{0.5\textwidth}
        \centering
        \footnotesize
        \caption{Performance results.}
        \begin{tabular}{lcccc}
            \toprule
            \multirow{2}{*}{\bf{Method}} & \multicolumn{4}{c}{\bf{YAGO3-10}} \\
            & MRR & H@1 & H@3 & H@10 \\
            \midrule
            DistMult & 0.34 & 0.24 & 0.38 & 0.54 \\
            ComplEx & 0.36 & 0.26 & 0.40 & 0.55 \\
            RotatE & 0.495 & 0.402 & 0.550 & 0.670 \\
            \midrule
            NBFNet & \bf{0.563} & \bf{0.480} & \bf{0.612} & \bf{0.708} \\
            A*Net & \bf{0.556} & 0.470 & \bf{0.611} & \bf{0.707} \\
            \bottomrule
        \end{tabular}
    \end{subtable}
    \hspace{0.3em}
    \begin{subtable}[t]{0.47\textwidth}
        \centering
        \footnotesize
        \caption{Efficiency results.}
        \begin{adjustbox}{max width=\textwidth}
            \begin{tabular}{lccc}
                \toprule
                \multirow{2}{*}{\bf{Method}} &  \multicolumn{3}{c}{\bf{YAGO3-10}} \\
                & \#message & time & memory \\
                \midrule
                NBFNet & 2,158,080 & 51.3 min & 26.1 GiB \\
                A*Net & 134,793 & 20.8 min & 13.1 GiB \\
                \midrule
                Improvement & 16.0$\times$ & 2.5$\times$ & 2.0$\times$ \\
                \bottomrule
            \end{tabular}
        \end{adjustbox}
    \end{subtable}
    \label{tab:yago}
\end{table}

\begin{table}[t]
    \vspace{0.3em}
    \begin{minipage}[t]{0.62\textwidth}
        \centering
        \caption[Performance on inductive knowledge graph reasoning]{Performance on inductive knowledge graph reasoning (MRR). V1-v4 are 4 standard inductive splits. Results of compared methods are taken from \cite{zhang2022knowledge}. $\alpha=50\%$ and $\beta=100\%$ for FB15k237. $\alpha=5\%$ and $\beta=100\%$ for WN18RR.}
        \vspace{-0.2em}
        \footnotesize
        \begin{adjustbox}{max width=\textwidth}
            \begin{tabular}{lcccccccc}
                \toprule
                \multirow{2}{*}{\bf{Method}} & \multicolumn{4}{c}{\bf{FB15k-237}} & \multicolumn{4}{c}{\bf{WN18RR}} \\
                & v1 & v2 & v3 & v4 & v1 & v2 & v3 & v4 \\
                \midrule
                GraIL & 0.279 & 0.276 & 0.251 & 0.227 & 0.627 & 0.625 & 0.323 & 0.553 \\
                \midrule
                NeuralLP & 0.325 & 0.389 & 0.400 & 0.396 & 0.649 & 0.635 & 0.361 & 0.628 \\
                DRUM & 0.333 & 0.395 & 0.402 & 0.410 & 0.666 & 0.646 & 0.380 & 0.627 \\
                NBFNet & 0.422 & \bf{0.514} & \bf{0.476} & 0.453 & \bf{0.741} & \bf{0.704} & \bf{0.452} & 0.641 \\
                RED-GNN & 0.369 & 0.469 & 0.445 & 0.442 & 0.701 & 0.690 & 0.427 & 0.651 \\
                A*Net & \bf{0.457} & \bf{0.510} & \bf{0.476} & \bf{0.466} & 0.727 & \bf{0.704} & 0.441 & \bf{0.661} \\
                \bottomrule
            \end{tabular}
        \end{adjustbox}
        \label{tab:inductive}
    \end{minipage}
    \hspace{0.5em}
    \begin{minipage}[t]{0.34\textwidth}
        \vspace{-3.1em}
        \centering
        \caption[Ablation studies of A*Net on transductive FB15k-237]{Ablation studies of A*Net on transductive FB15k-237.}
        \vspace{-0.2em}
        \begin{subtable}{\textwidth}
            \centering
            \caption{Choices of priority function. \label{tab:handcraft}}
            \vspace{-0.2em}
            \begin{adjustbox}{max width=\textwidth}
                \begin{tabular}{lcccc}
                    \toprule
                    \bf{Priority} & \multicolumn{4}{c}{\bf{FB15k-237}} \\
                    \bf{Function} & MRR & H@1 & H@3 & H@10 \\
                    \midrule
                    PPR & 0.266 & 0.212 & 0.296 & 0.371 \\ 
                    Degree & 0.347 & 0.268 & 0.383 & 0.501  \\
                    Random & 0.378 & 0.288 & 0.413 & 0.556 \\
                    Neural & \bf{0.411} & \bf{0.321} & \bf{0.453} & \bf{0.586} \\ 
                    \bottomrule
                \end{tabular}
            \end{adjustbox}
            \label{tab:priority_function}
        \end{subtable}
        \begin{subtable}{\textwidth}
            \centering
            \caption{W/ or w/o sharing weights. \label{tab:share_weights}}
            \vspace{-0.2em}
            \begin{adjustbox}{max width=\textwidth}
                \begin{tabular}{lcccc}
                    \toprule
                    \bf{Sharing} & \multicolumn{4}{c}{\bf{FB15k-237}} \\
                    \bf{Weights} & MRR & H@1 & H@3 & H@10 \\
                    \midrule
                    No  & 0.374 & 0.282 & 0.413 & 0.557 \\
                    Yes & \bf{0.411} & \bf{0.321} & \bf{0.453} & \bf{0.586} \\ 
                    \bottomrule
                \end{tabular}
            \end{adjustbox}
            \label{tab:share_weights}
        \end{subtable}
    \end{minipage}
\end{table}

\begin{table}[!h]
    \centering
    \caption[Efficiency on inductive datasets]{Efficiency on inductive datasets. V1-v4 refer to the 4 standard splits.}
    \vspace{-0.2em}
    \footnotesize
    \begin{adjustbox}{max width=\textwidth}
        \begin{tabular}{lcccccccccccc}
            \toprule
            \multirow{2}{*}{\bf{Method}} 
            & \multicolumn{3}{c}{\bf{v1}} & \multicolumn{3}{c}{\bf{v2}} & \multicolumn{3}{c}{\bf{v3}} & \multicolumn{3}{c}{\bf{v4}} \\
            & \#msg. & time & memory & \#msg. & time & memory & \#msg. & time & memory & \#msg. & time & memory \\
            \midrule
            \multicolumn{13}{c}{\bf{FB15k-237}}\\
            \midrule
            NBFNet & 8,490 & 4.50 s & 2.79 GiB & 19,478 & 11.3 s & 4.49 GiB & 35,972 & 27.2 s & 6.28 GiB & 54,406 & 50.1 s & 7.99 GiB \\
            A*Net  & 2,644 & 3.40 s & 0.97 GiB & 6,316 & 8.90 s & 1.60 GiB & 12,153 & 18.9 s & 2.31 GiB & 18,501 & 33.7 s & 3.05 GiB \\
            \midrule
            Improvement & 3.2$\times$ & 1.3$\times$ & 2.9$\times$ & 3.1$\times$ & 1.3 $\times$ & 2.8$\times$ & 3.0$\times$ & 1.4$\times$ & 2.7$\times$ & 2.9$\times$ & 1.5$\times$ & 2.6$\times$ \\
            \midrule[0.08em]
            \multicolumn{13}{c}{\bf{WN18RR}} \\
            \midrule
            NBFNet & 10,820 & 8.80 s & 1.79 GiB & 30,524 & 30.9 s & 4.48 GiB & 51,802 & 78.6 s & 7.75 GiB & 7,940 & 13.6 s & 2.49 GiB \\ 
            A*Net & 210 & 2.85 s & 0.11 GiB & 478 & 8.65 s & 0.26 GiB & 704 & 13.2 s &  0.41 GiB & 279 & 4.20 s & 0.14 GiB \\
            \midrule
            Improvement & 51.8$\times$ & 3.1$\times$ & 16.3$\times$ & 63.9$\times$ & 3.6$\times$ & 17.2$\times$ & 73.6$\times$ & 6.0$\times$ & 18.9$\times$ & 28.5$\times$ & 3.2$\times$ & 17.8$\times$ \\
            \bottomrule
        \end{tabular}
    \end{adjustbox}
    \label{tab:inductive_efficiency}
\end{table}

\subsection{Ablation Studies}
\label{sec:ablation_astarnet}

\smallskip \noindent \textbf{Priority Function.} To verify the effectiveness of neural priority function, we compare it against three handcrafted priority functions: personalized PageRank (PPR), Degree and Random. PPR selects nodes with higher PPR scores w.r.t.\ the query head entity $u$. Degree selects nodes with larger degrees, while Random selects nodes uniformly. Table~\ref{tab:priority_function} shows that the neural priority function outperforms all three handcrafted priority functions, suggesting the necessity of learning a neural priority function.

\smallskip \noindent \textbf{Sharing Weights.} As discussed in Section~\ref{sec:method_neural}, we share the weights between the neural priority function and the reasoning predictor to help train the neural priority function. Table~\ref{tab:share_weights} compares A*Net trained with and without sharing weights. It can be observed that sharing weights is essential to training a good neural priority function in A*Net.

\begin{wrapfigure}{R}{0.53\textwidth}
    \vspace{-1.3em}
    \centering
    \hspace{-0.8em}
    \includegraphics[width=0.264\textwidth]{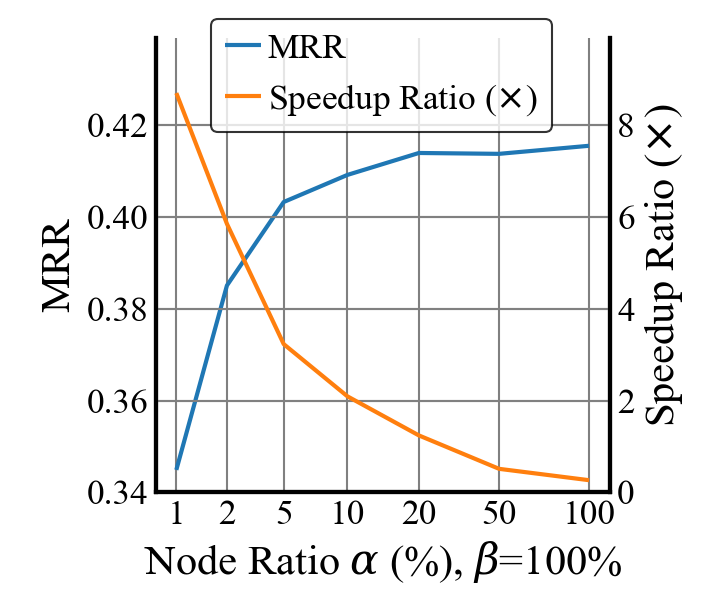}
    \hspace{0.1em}
    \includegraphics[width=0.253\textwidth]{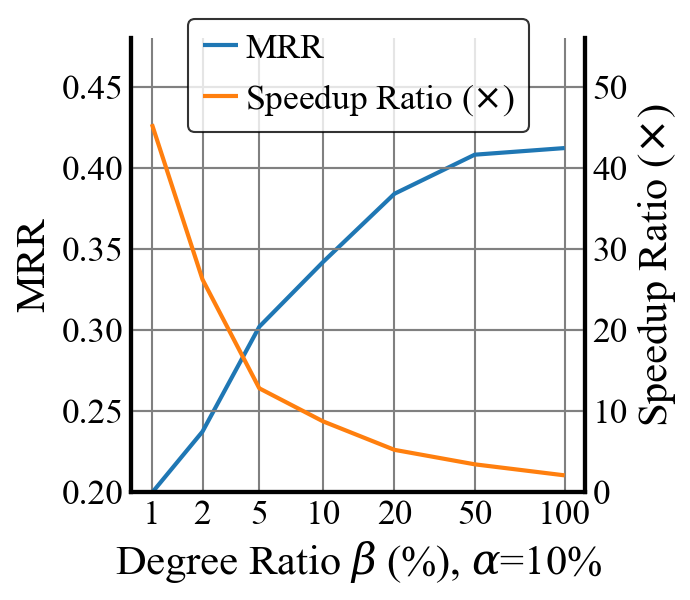}
    \caption[Performance and efficiency trade-off w.r.t.\ node ratio and degree ratio]{Performance and efficiency trade-off w.r.t.\ node ratio $\alpha$ and degree ratio $\beta$. Speedup ratio is relative to NBFNet.}
    \label{fig:trade_off}
    \vspace{-0.5em}
\end{wrapfigure}

\smallskip \noindent \textbf{Trade-off between Performance and Efficiency.} While A*Net matches the performance of NBFNet in less training time, one may further trade off performance and efficiency in A*Net by adjusting the ratios $\alpha$ and $\beta$. Figure~\ref{fig:trade_off} plots curves of performance and speedup ratio w.r.t.\ different $\alpha$ and $\beta$. If we can accept a performance similar to embedding methods (e.g.\ ConE~\cite{bai2021modeling}), we can set either $\alpha$ to 1\% or $\beta$ to 10\%, resulting in 8.7$\times$ speedup w.r.t.\ NBFNet.

\subsection{Visualization of Learned Important Paths}
\label{sec:vis_astarnet}

\begin{wrapfigure}{r}{0.5\textwidth}
    \vspace{-2.3em}
    \centering
    \includegraphics[width=0.5\textwidth]{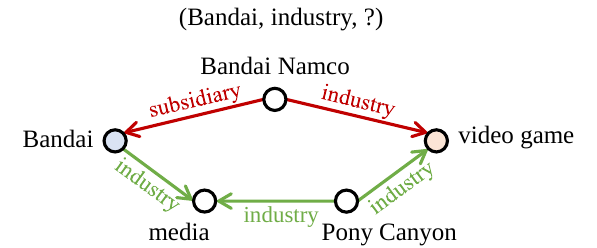}\vspace{-0.5em}
    \caption[Visualization of important paths learned by the neural priority function]{Visualization of important paths learned by the neural priority function.}
    \label{fig:visualization}
\end{wrapfigure}

We can extract the important paths from the neural priority function in A*Net for interpretation. For a given query $(u, q, ?)$ and a predicted entity $v$, we can use the node priority $s_{uq}^{(t)}(x)$ at each step to estimate the importance of a path. Empirically, the importance of a path $s_q(P)$ is estimated by
\begin{equation}
    s_q(P) = \frac{1}{|P|}\sum_{t=1, P^{(t)}=(x, r, y)}^{|P|} \frac{s_{uq}^{(t-1)}(x)}{S_{uq}^{(t-1)}}
\end{equation}
where $S_{uq}^{(t-1)} = \max_{x \in \gV^{(t-1)}} s_{uq}^{(t-1)}(x)$ is a normalizer to normalize the priority score for each step $t$. To extract the important paths with large $s_q(P)$, we perform beam search over the priority function $s_{uq}^{(t-1)}(x)$ of each step. Figure~\ref{fig:visualization} shows the important paths learned by A*Net for a test sample in FB15k-237. Given the query \emph{(Bandai, industry, ?)}, we can see both paths \emph{Bandai $\xleftarrow{\text{subsidiary}}$ Bandai Namco $\xrightarrow{\text{industry}}$ video game} and \emph{Bandai $\xrightarrow{\text{industry}}$ media$ \xleftarrow{\text{industry}}$ Pony Canyon $\xrightarrow{\text{industry}}$ video game} are consistent with human cognition.
\section{Theories and Proofs}
\label{app:method_A*}

Here we prove the correctness of path-based reasoning with A* algorithm.

\subsection{Iterative Path Selection for Computing Important Paths}
First, we prove that $\hat{\gP}_{u \leadsto v|q}^{(t)}$ computed by Equation~\ref{eqn:path_boundary} and \ref{eqn:path_selection} equals to the set of important paths and paths that are different from important paths in the last hop.

\begin{theorem}
\label{thm:superset}
If $m_q(\gP): 2^\gP \mapsto 2^\gP$ can select all important paths from a set of paths $\gP$, the set of paths $\hat{\gP}_{u \leadsto v|q}^{(t)}$ computed by Equation~\ref{eqn:path_boundary} and \ref{eqn:path_selection} equals to the set of important paths and paths that are different from important paths in the last hop of length $t$.
\begin{align}
    &\hat{\gP}_{u \leadsto v|q}^{(0)} \leftarrow \{(u, \text{self loop}, v)\}~\text{if}~u = v~\text{else}~\varnothing \tag{\ref{eqn:path_boundary}} \\
    &\hat{\gP}_{u \leadsto v|q}^{(t)} \leftarrow \bigcup_{\substack{x \in \gV \\ (x,r,v) \in \gE(v)}} \left\{P + \{(x,r,v)\} \middle| P \in m_q(\hat{\gP}_{u \leadsto x|q}^{(t-1)})\right\} \tag{\ref{eqn:path_selection}}
\end{align}
\end{theorem}
\begin{proof}
We use $\gQ_{u \leadsto v|q}^{(t)}$ to denote the set of important paths and paths that are different from important paths in the last hop of length $t$. For paths of length $0$, we define them to be important as they should be the prefix of some important paths. Therefore, $\gQ_{u \leadsto v|q}^{(0)} = \{(u, \text{self loop}, v)\}$ if $u = v$ else $\varnothing$. We use $P_{:-1}$ to denote the prefix of path $P$ without the last hop. The goal is to prove $\hat{\gP}_{u \leadsto v|q}^{(t)} = \gQ_{u \leadsto v|q}^{(t)}$. 

First, we prove $\hat{\gP}_{u \leadsto v|q}^{(t)} \subseteq \gQ_{u \leadsto v|q}^{(t)}$. It is obvious that $\hat{\gP}_{u \leadsto v|q}^{(0)} \subseteq \gQ_{u \leadsto v|q}^{(0)}$. In the case of $t > 0$, $\forall P \in \hat{\gP}_{u \leadsto v|q}^{(t)}$, we have $P_{:-1} \in m_q(\hat{\gP}_{u \leadsto v|q}^{(t-1)})$ according to Equation~\ref{eqn:path_selection}. Therefore, $P \in \gQ_{u \leadsto v|q}^{(t)}$.

Second, we prove $\gQ_{u \leadsto v|q}^{(t)} \subseteq \hat{\gP}_{u \leadsto v|q}^{(t)}$ by induction. For the base case $t = 0$, it is obvious that $\gQ_{u \leadsto v|q}^{(0)} \subseteq \hat{\gP}_{u \leadsto v|q}^{(0)}$. For the inductive case $t > 0$, $\forall Q \in \gQ_{u \leadsto v|q}^{(t)}$, $Q_{:-1}$ is an important path of length $t-1$ according to the definition of $\gQ_{u \leadsto v|q}^{(t)}$. $Q_{:-1} \in m_q(\gQ_{u \leadsto v|q}^{(t-1)}) \subseteq \gQ_{u \leadsto v|q}^{(t-1)}$ according to the definition of $m_q(\cdot)$ and $\gQ_{u \leadsto v|q}^{(t-1)}$. Based on the inductive assumption, we get $Q_{:-1} \in \hat{\gP}_{u \leadsto v|q}^{(t-1)}$. Therefore, $Q \in \hat{\gP}_{u \leadsto v|q}^{(t)}$ according to Equation~\ref{eqn:path_selection}.
\end{proof}

As a corollary of Theorem~\ref{thm:superset}, $\hat{\gP}_{u \leadsto v|q}$ is a slightly larger superset of the important paths $P_{u \leadsto v|q}$.

\begin{corollary}
If the end nodes of important paths are uniformly distributed in the knowledge graph, the expected size of $\hat{\gP}_{u \leadsto v|q}^{(t)}$ is $\left|\gP_{u \leadsto v|q}^{(t)}\right| + \frac{|\gE|}{|\gV|}\left|\gP_{u \leadsto v|q}^{(t-1)}\right|$.
\end{corollary}

\begin{proof}
Theorem~\ref{thm:superset} indicates that $\hat{\gP}_{u \leadsto v|q}^{(t)}$ contains two types of paths: important paths and paths that are different from important paths in the last hop of length $t$. The number of the first type is $\left|\gP_{u \leadsto v|q}^{(t)}\right|$. Each of the second type corresponds to an important path of length $t-1$. From an inverse perspective, each important path of length $t-1$ generates $d$ paths of the second type for $\hat{\gP}_{u \leadsto v|q}^{(t)}$, where $d$ is the degree of the end node in the path. If the end nodes are uniformly distributed in the knowledge graph, we have $\E\left[\hat{\gP}_{u \leadsto v|q}^{(t)}\right] = \left|\gP_{u \leadsto v|q}^{(t)}\right| + \frac{|\gE|}{|\gV|}\left|\gP_{u \leadsto v|q}^{(t-1)}\right|$. For real-world knowledge graphs, $\frac{|\gE|}{|\gV|}|$ is usually a small constant (e.g., $\leq 50$), and $\left|\hat{\gP}_{u \leadsto v|q}^{(t)}\right|$ is slightly larger than $\left|\gP_{u \leadsto v|q}^{(t)}\right|$ in terms of complexity.
\end{proof}

\subsection{From Iterative Path Selection to Iterative Node Selection}
Second, we demonstrate that Equation~\ref{eqn:path_selection} can be solved by Equation~\ref{eqn:node_selection} if paths with the same length and the same stop node can be merged.

\begin{proposition}
\label{prop:node_selection}
If $m_q(\gP)$ selects paths only based on the length $t$, the start node $u$ and the end node $x$ of each path, by replacing $m_q(\gP)$ with $n_{uq}^{(t)}(\gV)$, $\hat{\gP}_{u \leadsto v|q}^{(t)}$ can be computed as follows
\begin{align}
    \hat{\gP}_{u \leadsto v|q}^{(t)} \leftarrow \bigcup_{\substack{x \in n_{uq}^{(t-1)}(\gV) \\ (x,r,v) \in \gE(v)}} \left\{P + \{(x,r,v)\} \middle| P \in \hat{\gP}_{u \leadsto x|q}^{(t-1)}\right\} \tag{\ref{eqn:node_selection}}
\end{align}
\end{proposition}

This proposition is obvious. As a result of Proposition~\ref{prop:node_selection}, we merge paths by their length and stop nodes, which turns the exponential tree search to a polynomial dynamic programming algorithm.

\subsection{Reasoning with A* Algorithm}
Finally, we prove that the A* iteration (Equation~\ref{eqn:A*}) covers all important paths for reasoning (Equation~\ref{eqn:important_path}).

\begin{theorem}
\label{thm:A*}
If $n_{uq}^{(t)}(\gV): 2^\gV \mapsto 2^\gV$ can determine whether paths from $u$ to $x$ are important or not, and $\langle\oplus, \otimes\rangle$ forms a semiring~\cite{hebisch1998semirings}, the representation $\vh_q(u, v)$ for path-based reasoning can be computed by
\begin{align}
    \vh_q^{(t)}(u, v) \leftarrow \vh_q^{(0)}(u, v) \oplus \bigoplus_{\substack{x \in n_{uq}^{(t-1)}(\gV) \\ (x, r, v)\in \gE(v)}} \vh_q^{(t-1)}(u, x) \otimes \vw_q(x, r, v) \tag{\ref{eqn:A*}}
\end{align}
\end{theorem}

\begin{proof}
In order to prove Theorem~\ref{thm:A*}, we first prove a lemma for the analytic form of $\vh_q^{(t)}(u, v)$, and then show that $\lim_{t \rightarrow \infty}{\vh_q^{(t)}(u, v)}$ converges to the goal of path-based reasoning.

\begin{lemma}
\label{lem:A*}
Under the same condition as Theorem~\ref{thm:A*}, the intermediate representation $\vh_q^{(t)}(u, v)$ computed by Equation~\ref{eqn:boundary_nbfnet} and \ref{eqn:A*} aggregates all important paths within a length of $t$ edges, i.e.
\begin{equation}
    \vh_q^{(t)}(u, v) = \bigoplus_{P \in \hat{\gP}_{u \leadsto v|q}^{(\leq t)}}\bigotimes_{i=1}^{|P|}\vw_q(e_i)
\end{equation}
where $\hat{\gP}_{u \leadsto v|q}^{(\leq t)} = \bigcup_{k = 0}^{t} \hat{\gP}_{u \leadsto v|q}^{(k)}$.
\end{lemma}

\begin{proof}
We prove Lemma~\ref{lem:A*} by induction. Let $\ozero_q$ and $\oone_q$ denote the identity elements of $\oplus$ and $\otimes$ respectively. We have $\mathbbm{1}_q(u = v) = \oone_q$ if $u = v$ else $\ozero_q$. Note paths of length $0$ only contain self loops, and we define them as important paths, since they should be prefix of some important paths.

For the base case $t = 0$, we have $\vh_q^{(0)}(u, u) = \oone_q = \bigoplus_{P \in \gP_{u \leadsto u|q}: |P| \leq 0}\bigotimes_{i=1}^{|P|}\vw_q(e_i)$ since the only path from $u$ to $u$ is the self loop, which has the representation $\oone_q$. For $u \neq v$, we have $\vh_q^{(0)}(u, v) = \ozero_q = \bigoplus_{P \in \gP_{u \leadsto v|q}: |P| \leq 0}\bigotimes_{i=1}^{|P|}\vw_q(e_i)$ since there is no important path from $u$ to $v$ within length $0$.

For the inductive case $t > 0$, we have
\begingroup
\allowdisplaybreaks
\begin{align}
    \vh^{(t)}_q(u,v) 
    &= \vh^{(0)}_q(u,v) \oplus \bigoplus_{\substack{x \in n_{uq}^{(t-1)}(\gV) \\ (x, r, v)\in \gE(v)}}\vh_q^{(t-1)}(u,x) \otimes \vw_q(x, r, v)\\
    &= \vh^{(0)}_q(u,v) \oplus \bigoplus_{\substack{x \in n_{uq}^{(t-1)}(\gV) \\ (x, r, v)\in \gE(v)}}\left(\bigoplus_{P\in\hat{\gP}_{u \leadsto v|q}^{(\leq t-1)}} \bigotimes_{i=1}^{|P|} \vw_q(e_i)\right) \otimes \vw_q(x, r, v) \label{eqn:induction}\\
    &=
    \vh^{(0)}_q(u,v) \oplus \bigoplus_{\substack{x \in n_{uq}^{(t-1)}(\gV) \\ (x, r, v)\in \gE(v)}}\left[\bigoplus_{P\in\hat{\gP}_{u \leadsto v|q}^{(\leq t-1)}}\left(\bigotimes_{i=1}^{|P|} \vw_q(e_i)\right) \otimes \vw_q(x, r, v)\right] \label{eqn:distributive}\\
    &= \left( \bigoplus_{P\in\hat{\gP}_{u \leadsto v|q}^{(0)}} \bigotimes_{i=1}^{|P|} \vw_q(e_i)\right) \oplus \left(\bigoplus_{P\in\hat{\gP}_{u \leadsto v|q}^{(\leq t)}\setminus\hat{\gP}_{u \leadsto v|q}^{(0)}} \bigotimes_{i=1}^{|P|} \vw_q(e_i)\right) \label{eqn:associative}\\
    &= \bigoplus_{P\in\hat{\gP}_{u \leadsto v|q}^{(\leq t)}} \bigotimes_{i=1}^{|P|} \vw_q(e_i),
\end{align}
\endgroup
where Equation~\ref{eqn:induction} uses the inductive assumption, Equation~\ref{eqn:distributive} relies on the distributive property of $\otimes$ over $\oplus$, and Equation~\ref{eqn:associative} uses Proposition~\ref{prop:node_selection}. In the above equations, $\bigotimes$ and $\otimes$ are always applied before $\bigoplus$ and $\oplus$.
\end{proof}

Since $\gP_{u \leadsto v|q}^{(t)} \subseteq \hat{\gP}_{u \leadsto v|q}^{(t)}$, we have $\gP_{u \leadsto v|q} \subseteq \hat{\gP}_{u \leadsto v|q} \subseteq \gP_{u \leadsto v}$. Based on Lemma~\ref{lem:A*} and Equation~\ref{eqn:important_path}, it is obvious to see that
\begin{equation}
    \lim_{t \rightarrow \infty}{\vh_q^{(t)}(u, v)} = \bigoplus_{P \in \hat{\gP}_{u \leadsto v|q}}\vh_q(P) \approx \bigoplus_{P \in \gP_{u \leadsto v}}\vh_q(P) = \vh_q(u, v)
\end{equation}
Therefore, Theorem~\ref{thm:A*} holds.
\end{proof}

\section{Dataset Statistics}
\label{app:dataset_nbfnet}

Dataset statistics of two transductive settings, i.e.\ knowledge graph completion and homogeneous graph link prediction, are summarized in Table~\ref{tab:kg_statistics} and \ref{tab:homo_statistics}. Dataset statistics of inductive relation prediction is summarized in Table~\ref{tab:inductive_statistics}.

We use the standard transductive splits~\cite{toutanova2015observed, dettmers2018convolutional, mahdisoltani2014yago3} and inductive splits~\cite{teru2020inductive} for knowledge graphs. For homogeneous graphs, we follow previous works~\cite{kipf2016variational, davidson2018hyperspherical} and randomly split the edges into train/validation/test sets with a ratio of 85:5:10. All the homogeneous graphs used in this paper are undirected. Note that for inductive relation prediction, the original paper~\cite{teru2020inductive} actually uses a \emph{transductive valid set} that shares the same set of fact triplets as the training set for hyperparameter tuning. The \emph{inductive test set} contains entities, query triplets and fact triplets that never appear in the training set. The same data split is adopted in this paper for a fair comparison.

\begin{table}[!h]
    \centering
    \caption[Dataset statistics for knowledge graph completion]{Dataset statistics for knowledge graph completion.}
    \label{tab:kg_statistics}
    \footnotesize
    \begin{tabular}{lccccc}
        \toprule
        \multirow{2}{*}{\bf{Dataset}} & \multirow{2}{*}{\bf{\#Entity}} & \multirow{2}{*}{\bf{\#Relation}} & \multicolumn{3}{c}{\bf{\#Triplet}} \\
        & & & \bf{\#Train} & \bf{\#Validation} & \bf{\#Test} \\
        \midrule
        FB15k-237~\cite{toutanova2015observed} & 14,541 & 237 & 272,115 & 17,535 & 20,466 \\
        WN18RR~\cite{dettmers2018convolutional} & 40,943 & 11 & 86,835 & 3,034 & 3,134 \\
        YAGO3-10~\cite{mahdisoltani2014yago3} & 123,182 & 37 & 1,079,040 & 5000 & 5000 \\
        ogbl-biokg~\cite{hu2020ogb} & 93,773 & 51 & 4,762,678 & 162,886 & 162,870 \\
        ogbl-wikikg2~\cite{hu2020ogb} & 2,500,604 & 535 & 16,109,182 & 429,456 & 598,543 \\
        WikiKG90M~\cite{hu2021ogb} & 87,143,637 & 1,315 & 504,220,369 & 1,700,584 & 1,359,303 \\
        \bottomrule
    \end{tabular}
\end{table}

\begin{table}[!h]
    \centering
    \caption[Dataset statistics for homogeneous link prediction]{Dataset statistics for homogeneous link prediction.}
    \label{tab:homo_statistics}
    \footnotesize
    \begin{tabular}{lcccc}
        \toprule
        \multirow{2}{*}{\bf{Dataset}} & \multirow{2}{*}{\bf{\#Node}} & \multicolumn{3}{c}{\bf{\#Edge}} \\
        & & \bf{\#Train} & \bf{\#Validation} & \bf{\#Test} \\
        \midrule
        Cora~\cite{sen2008collective} & 2,708 & 4,614 & 271 & 544 \\
        CiteSeer~\cite{sen2008collective} & 3,327 & 4,022 & 236 & 474 \\
        PubMed~\cite{sen2008collective} & 19,717 & 37,687 & 2,216 & 4,435 \\
        \bottomrule
    \end{tabular}
\end{table}

\begin{table}[!h]
    \centering
    \caption[Dataset statistics for inductive relation prediction]{Dataset statistics for inductive relation prediction. Queries refer to the triplets that are used as training or test labels, while facts are the triplets used as training or test inputs. In the training sets, all queries are also provided as facts.}
    \label{tab:inductive_statistics}
    \begin{adjustbox}{max width=\textwidth}
        \begin{tabular}{llcccccccccc}
            \toprule
            \multirow{2}{*}{\bf{Dataset}} & & \multirow{2}{*}{\bf{\#Relation}} & \multicolumn{3}{c}{\bf{Train}} & \multicolumn{3}{c}{\bf{Validation}} & \multicolumn{3}{c}{\bf{Test}} \\
            & & & \bf{\#Entity} & \bf{\#Query} & \bf{\#Fact} & \bf{\#Entity} & \bf{\#Query} & \bf{\#Fact} & \bf{\#Entity} & \bf{\#Query} & \bf{\#Fact} \\
            \midrule
            \multirow{4}{*}{FB15k-237~\cite{teru2020inductive}}
            & v1 & 180 & 1,594 & 4,245 & 4,245 & 1,594 & 489 & 4,245 & 1,093 & 205 & 1,993\\
            & v2 & 200 & 2,608 & 9,739 & 9,739 & 2,608 & 1,166 & 9,739 & 1,660 & 478 & 4,145 \\
            & v3 & 215 & 3,668 & 17,986 & 17,986 & 3,668 & 2,194 & 17,986 & 2,501 & 865 & 7,406 \\
            & v4 & 219 & 4,707 & 27,203 & 27,203 & 4,707 & 3,352 & 27,203 & 3,051 & 1,424 & 11,714 \\
            \multirow{4}{*}{WN18RR~\cite{teru2020inductive}}
            & v1 & 9 & 2,746 & 5,410 & 5,410 & 2,746 & 630 & 5,410 & 922 & 188 & 1,618 \\
            & v2 & 10 & 6,954 & 15,262 & 15,262 & 6,954 & 1,838 & 15,262 & 2,757 & 441 & 4,011\\
            & v3 & 11 & 12,078 & 25,901 & 25,901 & 12,078 & 3,097 & 25,901 & 5,084 & 605 & 6,327\\
            & v4 & 9 & 3,861 & 7,940 & 7,940 & 3,861 & 934 & 7,940 & 7,084 & 1,429 & 12,334 \\
            \bottomrule
        \end{tabular}
    \end{adjustbox}
\end{table}
\section{Limitations and Future Work}
\label{sec:discussion}

There are a few limitations for NBFNet. First, the assumption of the generalized Bellman-Ford algorithm requires the operators $\langle\oplus, \otimes\rangle$ to satisfy a semiring. Due to the non-linear activation functions in neural networks, this assumption does not hold for NBFNet, and we do not have a theoretical guarantee on the loss incurred by this relaxation. Second, NBFNet is only verified on simple edge prediction, while there are other link prediction variants, e.g.\ complex logical queries with conjunctions ($\land$) and disjunctions ($\lor$)~\cite{hamilton2018embedding, ren2020query2box}. In the future, we would like to how NBFNet approximates the path formulation, as well as apply NBFNet to other link prediction settings.

One limitation for A*Net is that we focus on algorithm design rather than system design. As a result, the improvement in time and memory cost is much less than the improvement in the number of messages (Table~\ref{tab:transductive_efficiency}). In the future, we will co-design the algorithm and the system to further improve the efficiency.
\chapter[Representation Learning for Generalizing to Any Knowledge Graph]{Representation Learning for\linebreak Generalizing to Any Knowledge Graph}
\label{cha:ultra}

Generalization across knowledge structures plays a key role in the era of foundation models. The success of NBFNet has shown the potential of learning entity representations as a function of relations in a fixed vocabulary, which is just one step away from an ideal foundation model, since we only need to break the constraint on fixed relation vocabularies. If a model can generalize to both new entities and relations, it should be able to perform inference on any knowledge graph.

This chapter presents Ultra, a model that leverages invariances in knowledge graphs---interactions between relations---to learn representations for entities and relations. Ultra adopts two NBFNet instances, one for learning relations as a function of relation interactions, and another for learning entities as a function of relations. Ultra demonstrates strong zero-shot generalization performance on a large number of knowledge graphs of various domains and sizes, leading to the first foundation model for knowledge graph reasoning.

\smallskip \emph{This chapter is based on our work published at ICLR 2024~\cite{galkin2024towards}\footnote{The code is available at \url{https://github.com/DeepGraphLearning/ULTRA}}. Mikhail Galkin led the project and I contributed the core idea and model design in this project.}

\section{Overview}

Modern machine learning applications increasingly rely on the \emph{pre-training} and \emph{fine-tuning} paradigm. 
In this paradigm, a backbone model often trained on large datasets in a self-supervised fashion is commonly known as a \emph{foundation model}~\cite{bommasani2021opportunities}. After pre-training, foundation models can be fine-tuned on smaller downstream tasks. In order to transfer to a broad set of downstream tasks, foundation models leverage certain \emph{invariances} pertaining to a domain of interest, e.g.\ large language models like BERT~\cite{devlin2019bert}, GPT-4~\cite{achiam2023gpt}, Llama-2~\cite{touvron2023llama} operate on a fixed vocabulary of tokens; vision models operate on raw pixels~\cite{he2016deep, radford2021learning} or image patches~\cite{dosovitskiy2021image}; chemistry models~\cite{ying2021transformers, zheng2024predicting} learn a vocabulary of atoms from the periodic table.

Representation learning on knowledge graphs, however, has not yet witnessed the benefits of transfer learning despite a wide range of downstream applications such as precision medicine~\cite{chandak2023building}, materials science~\cite{venugopal2022matkg, statt2023materials}, virtual assistants~\cite{ilyas2022saga}, or product graphs in e-commerce~\cite{dong2018challenges}. The key problem is that different knowledge graphs typically have different entity and relation vocabularies. Classic \emph{transductive} embedding models~\cite{ali2021bringing} learn entity and relation embeddings tailored for each specific vocabulary and cannot generalize even to new nodes within the same graph. More recent efforts towards generalization across the vocabularies are known as \emph{inductive} learning methods~\cite{chen2023generalizing}. Most of the inductive methods~\cite{teru2020inductive, zhu2021neural, galkin2022nodepiece, zhang2022knowledge} generalize to new entities at inference time but require a fixed relation vocabulary to learn entity representations as a function of the relations. Such inductive methods still cannot transfer to knowledge graphs with a different set of relations, e.g.\ training on Freebase and inference on Wikidata.

The main research goal of this work is \emph{finding the invariances transferable across graphs with arbitrary entity and relation vocabularies}. Leveraging and learning such invariances would enable the \emph{pre-train and fine-tune} paradigm of foundation models for knowledge graph reasoning where a single model trained on one graph (or several graphs) with one set of relations would be able to \emph{zero-shot} transfer to any new, unseen graph with a completely different set of relations and relational patterns. 
Our approach to the problem is based on two key observations: (1) even if relations vary across the datasets, the interactions between the relations may be similar and transferable; (2) initial relation representations may be conditioned on this interaction bypassing the need for any input features. To this end, we propose \textsc{Ultra}, a method for \underline{u}nified, \underline{l}earnable, and \underline{tra}nsferable knowledge graph representations that leverages the invariance of the \emph{relational structure} and employs relative relation representations on top of this structure for parameterizing any unseen relation. Given any multi-relational graph, \textsc{Ultra} first constructs a graph of relations (where each node is a relation from the original graph) capturing their interactions.
Applying a graph neural network (GNN) with a \emph{labeling trick}~\cite{zhang2021labeling} over the graph of relations, \textsc{Ultra} obtains a unique \emph{relative} representation of each relation. The relation representations can then be used by any inductive learning method for downstream applications like knowledge graph completion. Since the method does not learn any graph-specific entity or relation embeddings nor requires any input entity or relation features, \textsc{Ultra} enables \emph{zero-shot} generalization to any other knowledge graph of any size and any relational vocabulary.

Experimentally, we show that \textsc{Ultra} paired with the NBFNet~\cite{zhu2021neural} link predictor pre-trained on three knowledge graphs (FB15k-237, WN18RR, and CoDEx-M derived from Freebase, WordNet, and Wikidata, respectively) generalizes to 50+ different knowledge graphs with sizes of 1,000--120,000 nodes and 5K--1M edges. \textsc{Ultra} demonstrates promising transfer learning capabilities where the zero-shot inference performance on those unseen graphs might exceed strong supervised baselines by up to $300\%$. The subsequent short fine-tuning of \textsc{Ultra} often boosts the performance even more.

\begin{figure}[!h]
    \centering
    \includegraphics[width=\textwidth]{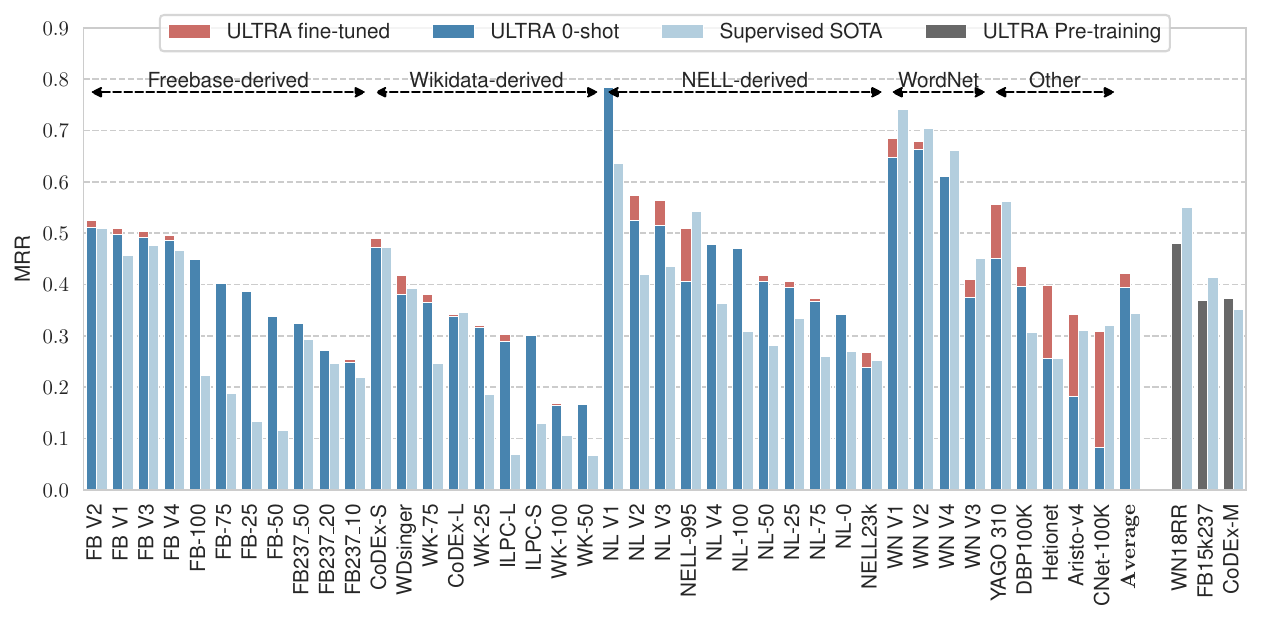}
    \caption[Zero-shot and fine-tuned performance of Ultra on 43 datasets]{Zero-shot and fine-tuned MRR (higher is better) of \textsc{Ultra} pre-trained on three graphs (FB15k-237, WN18RR, CoDEx-Medium). On average, zero-shot performance is better than best reported baselines trained on specific graphs (0.395 vs 0.344). More results in Figure~\ref{fig:mtdea} and Table~\ref{tab:main1}.}
    \label{fig:main_result}
\end{figure}
\section{Method}
\label{sec:method}

\begin{figure}[t]
    \centering
    \includegraphics[width=\textwidth]{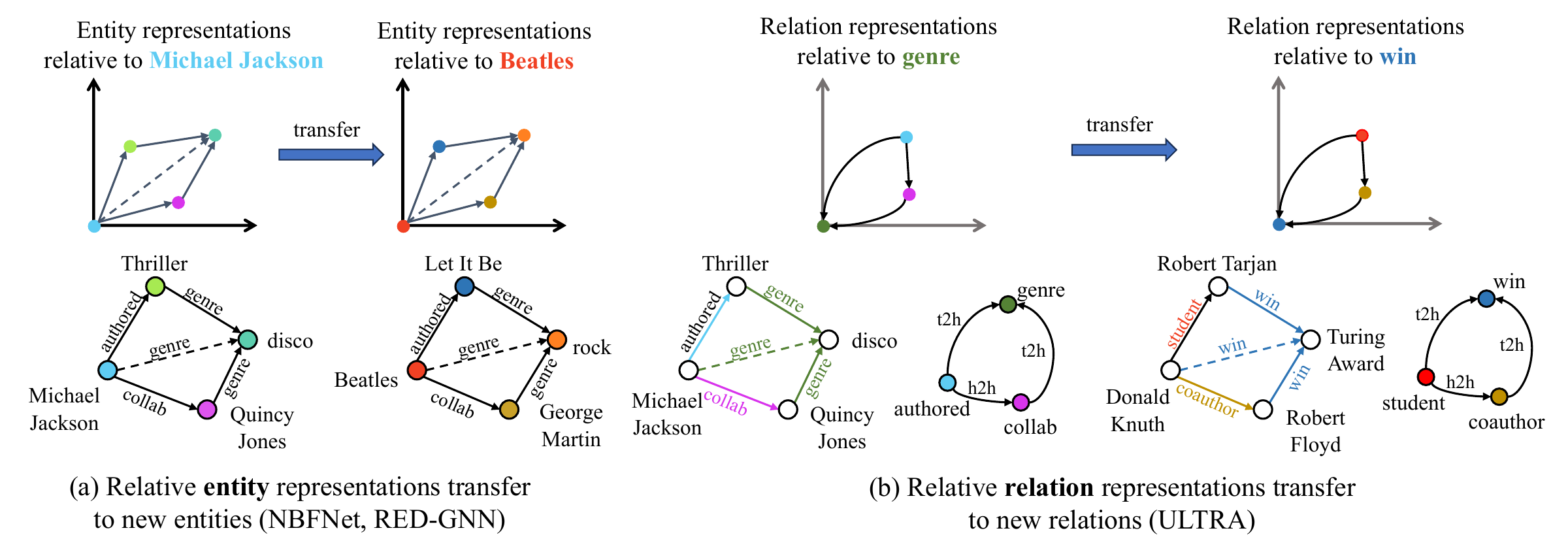}
    \caption[Relative representations generalize to new entities and relations]{(a) relative \textbf{entity} representations used in inductive models generalize to new entities; (b) relative \textbf{relation} representations based on a graph of relations generalize to both new relations and entities. The graph of relations captures four fundamental interactions (\emph{t2h}, \emph{h2h}, \emph{h2t}, \emph{h2h}) independent from any graph-specific relation vocabulary and whose representations can be learned.}
    \label{fig:approach}
\end{figure}

The key challenge of inductive inference with different entity and relation vocabularies is finding \emph{transferable invariances} that would produce entity and relation representations conditioned on the new graph (as learning entity and relation embedding matrices from the training graph is useless and not transferable). Most inductive GNN methods that transfer to new entities~\cite{zhu2021neural, zhang2022knowledge} learn \textbf{relative entity representations} conditioned on the graph structure as shown in Figure~\ref{fig:approach}(a). For example, given $a,b,c$ are variable entities and $a$ as a root node labeled with $\textsc{Indicator}()$, a structure $a \xrightarrow{\textit{authored}} b \xrightarrow{\textit{genre}} c \wedge a \xrightarrow{\textit{collab}} d \xrightarrow{\textit{genre}} c$ might imply existence of the edge $a \xrightarrow{\textit{genre}} c$. Learning such a structure on a training set with entities $\textit{Michael Jackson}  \xrightarrow{\textit{authored}} \textit{Thriller} \xrightarrow{\textit{genre}} \textit{disco}$ seamlessly transfers to new entities $\textit{Beatles}  \xrightarrow{\textit{authored}} \textit{Let It Be} \xrightarrow{\textit{genre}} \textit{rock}$ at inference time without learning entity embeddings thanks to the same relational structure and \emph{relative} entity representations. As training and inference relations are the same $\rtrain = \rinf$, such approaches learn relation embedding matrices and use \textbf{relations as invariants}.

In \textsc{Ultra}, we generalize knowledge graph reasoning to both new entities and relations (where $\rtrain \neq \rinf $) by leveraging a \emph{graph of relations}, i.e.\ a graph where each node corresponds to a distinct relation type\footnote{We also add inverse relations as nodes to the relation graph.} in the original graph. While relations at inference time are different, their interactions remain the same and are captured by the graph of relations. For example, Figure~\ref{fig:approach}(b), a \emph{tail} node of the \emph{authored} relation is also a \emph{head} node of the \emph{genre} relation. Hence, \emph{authored} and \emph{genre} nodes are connected by the \emph{tail-to-head} edge in the relation graph. Similarly, \emph{authored} and \emph{collab} share the same \emph{head} node in the entity graph and thus are connected with the \emph{head-to-head} edge in the relation graph. Overall, we distinguish \textbf{four} such core, \emph{fundamental} relation-to-relation interactions\footnote{Other strategies for capturing relation-to-relation interactions might exist beside those four types and we leave their exploration for future work.}: \emph{tail-to-head (t2h)}, \emph{head-to-head (h2h)}, \emph{head-to-tail (h2t)}, and \emph{tail-to-tail (t2t)}. Albeit relations in the inference graph in Figure~\ref{fig:approach}(b) are different, their graph of relations and relation interactions resemble that of the training graph. Hence, we could leverage the \textbf{invariance of the relational structure} and four fundamental relations to obtain relational representations of the unseen inference graph. As a typical knowledge graph reasoning task $(h, q, ?)$ is conditioned on a query relation $q$, it is possible to build representations of all relations \emph{relative} to the query $q$ by using a labeling trick on top of the graph of relations. Such \textbf{relative relation representations} do not need any input features and naturally generalize to any multi-relational graph.

Practically (Figure~\ref{fig:method_ultra}), given a query $(h, q, ?)$ over a graph $\graph$, \textsc{Ultra} employs a three-step algorithm that we describe in the following subsections. (1) Lift the original graph $\graph$ to the graph of relations $\grel$ -- Section~\ref{subsec:rel_graph}; (2) Obtain relative relation representations $\mR_q|(q, \grel)$ conditioned on the query relation $q$ in the relation graph $\grel$ -- Section~\ref{subsec:rel_representations}; (3) Using the relation representations $\mR_q$ as starting relation features, run inductive link prediction on the original graph $\graph$ -- Section~\ref{subsec:inductive_lp}.

\begin{figure}[t]
    \centering
    \includegraphics[width=\textwidth]{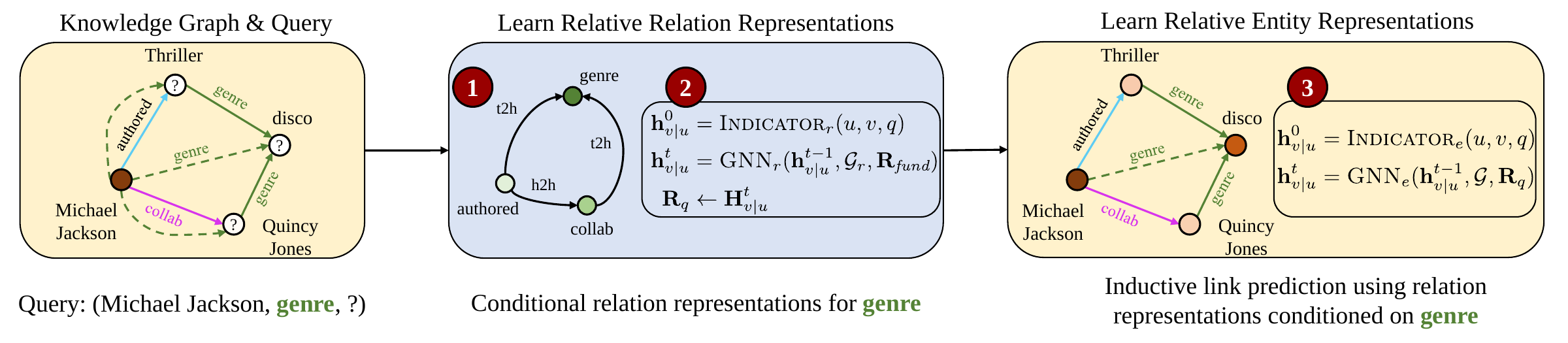}
    \caption[Overview of Ultra]{Given a query $(h,q,?)$ on graph $\gG$, \textsc{Ultra} (1) builds a graph of relations $\grel$ with four interactions $\rfund$ (Section~\ref{subsec:rel_graph}); (2) builds relation representations $\mR_q$ conditioned on the query relation $q$ and $\grel$ (Section~\ref{subsec:rel_representations}); (3) runs any inductive link predictor on $\gG$ using representations $\mR_q$ (Section~\ref{subsec:inductive_lp}). }
    \label{fig:method_ultra}
\end{figure}

\subsection{Relation Graph Construction}
\label{subsec:rel_graph}

Given a graph $\graph = (\gV, \gR, \gE)$, we first apply the lifting function $\grel = \textsc{Lift}(\graph)$ to build a graph of relations $\grel = (\gR, \rfund, \gE_r)$ where each node is a distinct relation type\footnote{$2|\rels|$ nodes after adding inverse relations to the original graph.} in $\graph$.
Edges $\gE_r  \in (\rels \times \rfund \times \rels)$ in the relation graph $\grel$ denote interactions between relations in the original graph $\graph$, and we distinguish four such fundamental relation interactions $\gR_{\textit{fund}}$:  \emph{tail-to-head (t2h)} edges, \emph{head-to-head (h2h)} edges, \emph{head-to-tail (h2t)} edges, and \emph{tail-to-tail (t2t)} edges. The full adjacency tensor of the relation graph is $\mA_r \in \sR^{|\rels| \times |\rels| \times 4}$.

\smallskip \noindent \textbf{Efficient Implementation via Sparse Matrix Multiplication.}
The graph of relations $\grel$ can be efficiently computed from the original graph $\gG$ with sparse matrix multiplications (spmm). Given the original graph $\gG$ with $|\gV|$ nodes and $|\gR|$ relation types, its adjacency matrix is $\mA \in \sR^{|\gV| \times |\gR| \times |\gV|}$.
For clarity, $\mA$ can be rewritten with \emph{heads} $\gH$ and \emph{tails} $\gT$ as $\mA \in \sR^{|\gH| \times |\gR| \times |\gT|}$. 
From $\mA$ we first build two sparse matrices $\mE_h \in \sR^{|\gH| \times |\gR|}$ and $\mE_t \in \sR^{|\gT| \times |\gR|}$ that capture the head-relation and tail-relation pairs, respectively. Computing interactions between relations is then equivalent to spmm operations between relevant adjacencies
\begin{align}
    \mA_{h2h} &= \text{spmm}(\mE_h^T, \mE_h) \in \sR^{|\gR| \times |\gR|} \\
    \mA_{t2t} &= \text{spmm}(\mE_t^T, \mE_t) \in \sR^{|\gR| \times |\gR|} \\
    \mA_{h2t} &= \text{spmm}(\mE_h^T, \mE_t) \in \sR^{|\gR| \times |\gR|} \\
    \mA_{t2h} &= \text{spmm}(\mE_t^T, \mE_h) \in \sR^{|\gR| \times |\gR|}  \\
    \mA_r &= [\mA_{h2h}, \mA_{t2t}, \mA_{h2t}, \mA_{t2h}] \in \sR^{|\gR| \times |\gR| \times 4}
\end{align}
where the final adjacency tensor $\mA_r$ is obtained by stacking adajcencies from four fundamental interactions.

\subsection{Conditional Relation Representations}
\label{subsec:rel_representations}

Given a query $(h, q, ?)$ and a relation graph $\grel$, we then obtain $d$-dimensional node representations $\mR_q \in \sR^{|\gR| \times d}$ of $\grel$ (corresponding to all edge types $\rels$ in the original graph $\graph$) conditioned on the query relation $q$. Practically, we implement conditioning by applying a labeling trick to initialize the node $q$ in $\grel$ through the $\textsc{Indicator}_r$ function and employ a  message passing GNN over $\grel$:
\begin{align}
    \vh^0_{v|q}  &= \textsc{Indicator}_r(v, q) = \mathbbm{1}_{v=q} * \1^d, \quad v \in \mathcal{G}_r\\
    \vh^{t+1}_{v|q} &= \textsc{Update} \Big( \vh^{t}_{v|q}, \textsc{Aggregate} \big( \textsc{Message}(\vh^t_{w|q}, \vr) | w \in \gN_r(v), r \in \rfund \big) \Big)
\end{align}
The indicator function is implemented as $\textsc{Indicator}_r(v,q) = \mathbbm{1}_{v=q} * \1^d$ that simply puts a vector of ones on a node $v$ corresponding to the query relation $q$, and zeros otherwise. Following \cite{huang2024theory}, we found that all-ones labeling with $\1^d$ generalizes better to unseen graphs of various sizes than a learnable vector. The GNN architecture (denoted as $\text{GNN}_r$ as it operates on the relation graph $\grel$) follows NBFNet~\cite{zhu2021neural} with a non-parametric DistMult~\cite{yang2015embedding} message function and sum aggregation. The only learnable parameters in each layer are embeddings of four fundamental interactions $\vrfund \in \sR^{4 \times d}$, a linear layer for the \textsc{Update} function, and an optional layer normalization. Note that our inductive setup assumes no given input entity or relation features, so our parameterization strategy can be used to obtain relational representations of \emph{any} multi-relational graph.

To sum up, each unique relation $q \in \gR$ in the query has its own matrix of conditional relation representations $\mR_q \in \sR^{|\rels| \times d}$ used by the entity-level reasoner for downstream applications.

\subsection{Entity-level Link Prediction}
\label{subsec:inductive_lp}

Given a query $(h, q, ?)$ over a graph $\graph$ and conditional relation representations $\mR_q$ from the previous step, it is now possible to adapt any off-the-shelf inductive link predictor that only needs relational features~\cite{zhu2021neural, zhang2022knowledge, zhu2023net, zhang2023adaprop} to balance between performance and scalability. We modify another instance of NBFNet ($\text{GNN}_e$ as it operates on the entity level) to account for separate relation representations per query:
\begin{align}
    \vh^0_{v|u}  &= \textsc{Indicator}_e(u, v, q) = \mathbbm{1}_{u=v} * \mR_q[q], \quad v \in \mathcal{G} \\
    \vh^{t+1}_{v|u} &= \textsc{Update} \Big( \vh^{t}_{v|u}, \textsc{Aggregate} \big( \textsc{Message}(\vh^t_{w|u}, g^{t+1}(\vr)) | w \in \gN_r(v), r \in \rels \big) \Big)
\end{align}
That is, we first initialize the head node $h$ with the query vector $q$ from $\mR_q$ whereas other nodes are initialized with zeros. Each $t$-th GNN layer applies a non-linear function $g^t(\cdot)$ to transform original relation representations to layer-specific relation representations as $\mR^{t} = g^t(\mR_q)$ from which the edge features are taken for the \textsc{Message} function. $g(\cdot)$ is implemented as a 2-layer MLP with ReLU. Similar to $\text{GNN}_r$ in Section~\ref{subsec:rel_representations}, we use sum aggregation and a linear layer for the \textsc{Update} function. After message passing, the final MLP $s: \sR^d \rightarrow \sR^1$ maps the node states to logits $p(h, q, v)$ denoting the score of a node $v$ to be a tail of the initial query $(h, q, ?)$.

\smallskip \noindent \textbf{Training.}
\textsc{Ultra} can be trained on any multi-relational graph or mixture of graphs thanks to the inductive and conditional relational representations. Following the standard practices in the literature~\cite{sun2019rotate, zhu2021neural}, \textsc{Ultra} is trained by minimizing the binary cross entropy loss over positive and negative triplets
\begin{align}
    \gL = - \log p(u, q, v) - \sum_{i=1}^{n} \frac{1}{n} \log (1 - p(u_i', q, v_i'))
    \label{eqn:training_loss}
\end{align}
where $(u,q,v)$ is a positive triple in the graph and $\{(u_i', q, v_i')\}^n_{i=1}$ are negative samples obtained by corrupting either the head $u$ or tail $v$ of the positive sample.
\section{Experiments}
\label{sec:experiments}

To evaluate the qualities of \textsc{Ultra} as a foundation model for knowledge graph reasoning, we explore the following questions: (1) Is pre-trained \textsc{Ultra} able to inductively generalize to unseen knowledge graphs in the zero-shot manner? (2) Are there any benefits from fine-tuning \textsc{Ultra} on a specific dataset?  (3) How does a single pre-trained \textsc{Ultra} model compare to models trained from scratch on each target dataset? (4) Do more graphs in the pre-training mix correspond to better performance?

\subsection{Setup and Datasets}
\label{subsec:datasets}

\smallskip \noindent \textbf{Datasets.}
We conduct a broad evaluation on 57 different knowledge graphs with reported, non-saturated results on the knowledge graph completion task. The datasets can be categorized into three groups:
\begin{itemize}[label=$\bullet$, leftmargin=*]
    \item \emph{Transductive} datasets (16 graphs) with the fixed set of entities and relations at training and inference time $(\gtrain = \ginf)$: FB15k-237~\cite{toutanova2015observed}, WN18RR~\cite{dettmers2018convolutional}, YAGO3-10~\cite{mahdisoltani2014yago3}, NELL-995~\cite{xiong2017deeppath}, CoDEx (Small, Medium, and Large)~\cite{safavi2020codex}, WDsinger, NELL23k, FB15k-237(10), FB15k-237(20), FB15k-237(50)~\cite{lv2020dynamic}, AristoV4~\cite{chen2021relation}, DBpedia100k~\cite{ding2018improving}, ConceptNet100k~\cite{malaviya2020commonsense}, Hetionet~\cite{himmelstein2017systematic}.
    \item \emph{Inductive entity} ($e$) datasets (18 graphs) with new entities at inference time but with the fixed set of relations $(\etrain \neq \einf, \rtrain = \rinf)$: 12 datasets from GraIL~\cite{teru2020inductive}, 4 graphs from INDIGO~\cite{liu2021indigo, hamaguchi2017knowledge}, and 2 ILPC 2022 datasets (Small and Large)~\cite{galkin2022open}. 
    \item \emph{Inductive entity and relation} ($e,r$) datasets (23 graphs) where both entities and relations at inference are new $(\etrain \neq \einf, \rtrain \neq \rinf)$: 13 graphs from \textsc{InGram}~\cite{lee2023ingram} and 10 graphs from MTDEA~\cite{zhou2023ood}.
\end{itemize}
In practice, however, a pre-trained \textsc{Ultra} operates in the \emph{inductive $(e,r)$} mode on all datasets (apart from those in the training mixture) as their sets of entities, relations, and relational structures are different from the training set. The dataset sizes vary from 1k to 120k entities and 1k-2M edges in the inference graph. We provide more details on the datasets in Section~\ref{app:dataset_ultra}.

\smallskip \noindent \textbf{Pretraining and Fine-tuning.}
\textsc{Ultra} is pre-trained on the mixture of 3 standard knowledge graphs (WN18RR, CoDEx-Medium, FB15k-237) to capture the variety of possible relational structures and sparsities in respective relational graphs $\grel$. \textsc{Ultra} is relatively small (177k parameters in total, with 60k parameters in $\text{GNN}_r$ and 117k parameters in $\text{GNN}_e$) and is trained for 200,000 steps with batch size of 64 with AdamW optimizer on 2 A100 (40 GB) GPUs. All fine-tuning experiments were done on a single RTX 3090 GPU.

\smallskip \noindent \textbf{Evaluation Protocol.}
We report Mean Reciprocal Rank (MRR) and Hits@10 (H@10) as the main performance metrics evaluated against the full entity set of the inference graph. For each triple, we report the results of predicting both heads and tails. Only in three datasets from \cite{lv2020dynamic} we report tail-only metrics similar to the baselines. In the zero-shot inference scenario, we run a pre-trained model on the inference graph and test set of triples. In the fine-tuning case, we further train the model on the training split of each dataset retaining the checkpoint of the best validation set MRR. We run zero-shot inference experiments once as the results are deterministic, and report an average of 5 runs for each fine-tuning run on each dataset. 

\smallskip \noindent \textbf{Baselines.}
On each graph, we compare \textsc{Ultra} against the reported state-of-the-art model (we list SOTA for all 57 graphs in Section~\ref{app:dataset_ultra}). To date, all of the reported SOTA models are trained end-to-end specifically on each target dataset. Due to the computational complexity of baselines, the only existing results on 4 MTDEA datasets~\cite{zhou2023ood} and 4 INDIGO datasets~\cite{liu2021indigo} report Hits@10 against 50 randomly chosen negatives. We compare \textsc{Ultra} against those baselines using this \emph{Hits@10 (50 negs)} metric as well as report the full performance on the whole entity sets.

\begin{figure}[t]
    \centering
    \includegraphics[width=\textwidth]{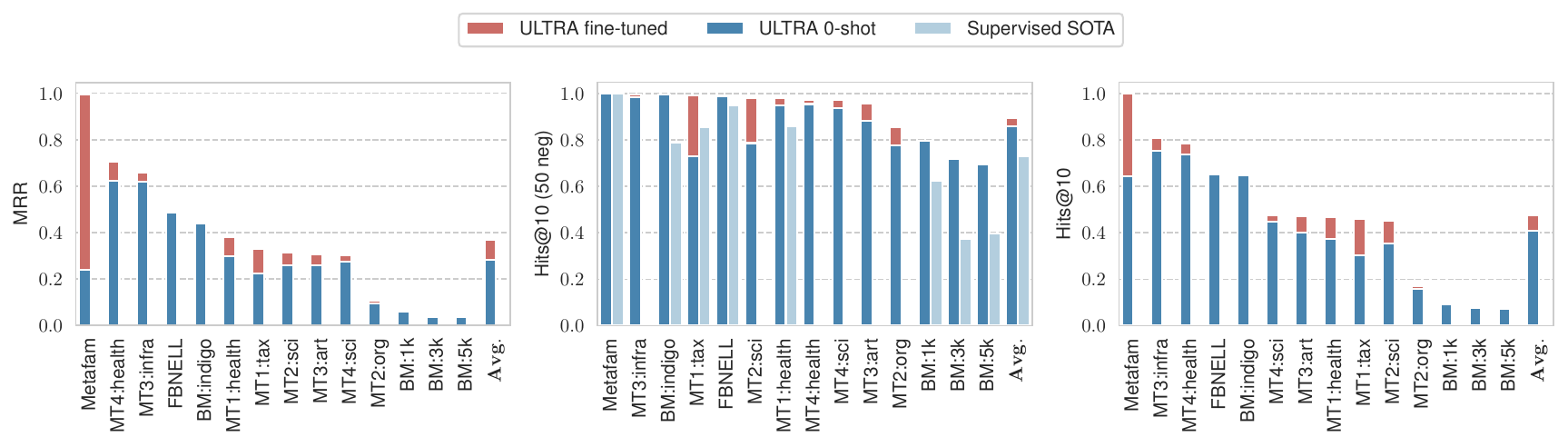}
    \caption[Performance of Ultra on 14 MTDEA and INDIGO datasets]{\textsc{Ultra} performance on 14 inductive datasets from MTDEA~\cite{zhou2023ood} and INDIGO~\cite{liu2021indigo} for 8 of which only an approximate metric \emph{Hits@10 (50 negs)} is available (center). We also report full MRR (left) and Hits@10 (right) computed on the entire entity sets demonstrating that Hits@10 (50 negs) overestimates the real performance.}
    \label{fig:mtdea}
\end{figure}

\subsection{Main Results}
\label{subsec:main_res}

The main experiment reports how \textsc{Ultra} pre-trained on 3 graphs inductively generalizes to 54 other graphs both in the zero-shot (0-shot) and fine-tuned cases. Figure~\ref{fig:main_result} compares \textsc{Ultra} with supervised SOTA baselines on 43 graphs that report MRR on the full entity set. Figure~\ref{fig:mtdea} presents the comparison on the rest 14 graphs including 8 graphs for which the baselines report \emph{Hits@10 (50 negs)}. The aggregated results on 51 graphs with available baseline results are presented in Table~\ref{tab:main1} and the complete evaluation on 57 graphs grouped into three families according to Section~\ref{subsec:datasets} is in Table~\ref{tab:main2}.

\begin{table}[t]
    \centering
    \caption[Zero-shot and fine-tuned performance of Ultra on 51 datasets]{Zero-shot and fine-tuned performance of \textsc{Ultra} compared to the published supervised SOTA on 51 datasets (as in Figure~\ref{fig:main_result} and Figure~\ref{fig:mtdea}). The zero-shot \textsc{Ultra} outperforms supervised baselines on average and on inductive datasets. Fine-tuning improves the performance even further. We report pre-training performance to the fine-tuned version.}
    \begin{adjustbox}{width=\textwidth}
        \begin{tabular}{lcccccc||cc||cc}\toprule
            \multirow{3}{*}{\bf{Model}} & \multicolumn{2}{c}{\bf{Inductive} $(e) + (e, r)$} & \multicolumn{2}{c}{\bf{Transductive} $e$} & \multicolumn{2}{c}{\bf{Total Avg}} & \multicolumn{2}{c}{\bf{Pretraining}} & \multicolumn{1}{c}{\bf{Inductive $(e) + (e, r)$}} \\  
            & \multicolumn{2}{c}{(27 graphs)} & \multicolumn{2}{c}{(13 graphs)} & \multicolumn{2}{c}{(40 graphs)} & \multicolumn{2}{c}{(3 graphs)} & \multicolumn{1}{c}{(8 graphs)} \\ \cmidrule(l){2-3} \cmidrule(l){4-5} \cmidrule(l){6-7} \cmidrule(l){8-9} \cmidrule(l){10-10}
            & \bf{MRR} & \bf{H@10} & \bf{MRR} & \bf{H@10} & \bf{MRR} & \bf{H@10} & \bf{MRR} & \bf{H@10} & \multicolumn{1}{c}{\bf{Hits@10 (50 negs)}} \\ 
            \midrule
            Supervised SOTA & 0.342 & 0.482 &0.348 & 0.494 &0.344 & 0.486 & \bf{0.439} & \bf{0.585} & 0.731 \\
            \textsc{Ultra} 0-shot & 0.435 &0.603 &0.312 &0.458 &0.395 &0.556 & - & - & 0.859 \\
            \textsc{Ultra} fine-tuned & \bf{0.443} & \bf{0.615} & \bf{0.379} & \bf{0.543} & \bf{0.422} & \bf{0.592}  & 0.407 & 0.568  & \bf{0.896} \\
            \bottomrule
        \end{tabular}
    \end{adjustbox}
    \label{tab:main1}
\end{table}

\begin{table}[t]
    \centering
    \caption[Performance of Ultra on 57 graphs grouped by the dataset category]{Zero-shot and fine-tuned \textsc{Ultra} results on the complete set of 57 graphs grouped by the dataset category. Fine-tuning especially helps on larger transductive datasets and boosts the total average MRR by 10\%. Additionally, we report as \emph{(train e2e)} the average performance of  dataset-specific \textsc{Ultra} models trained from scratch on each graph.}
    \begin{adjustbox}{width=\textwidth}
        \begin{tabular}{lcccccccc||ccc}
            \toprule
            \multirow{3}{*}{\bf{Model}} & \multicolumn{2}{c}{\bf{Inductive} $e, r$} & \multicolumn{2}{c}{\bf{Inductive} $e$} & \multicolumn{2}{c}{\bf{Transductive}} & \multicolumn{2}{c}{\bf{Total Avg}} & \multicolumn{2}{c}{\bf{Pretraining}} \\  
            & \multicolumn{2}{c}{(23 graphs)} & \multicolumn{2}{c}{(18 graphs)} & \multicolumn{2}{c}{(13 graphs)} & \multicolumn{2}{c}{(54 graphs)} & \multicolumn{2}{c}{(3 graphs)} \\ \cmidrule(l){2-3} \cmidrule(l){4-5} \cmidrule(l){6-7} \cmidrule(l){8-9} \cmidrule(l){10-11}
            & \bf{MRR} & \bf{H@10} & \bf{MRR} & \bf{H@10} & \bf{MRR} & \bf{H@10} & \bf{MRR} & \bf{H@10} & \multicolumn{1}{c}{\bf{MRR}} & \bf{H@10} \\ 
            \midrule
            \textsc{Ultra} (train e2e) &0.392 &0.552 &0.402 &0.559 &0.384 &0.545 &0.393 &0.552 & 0.403 & 0.562 \\ \midrule
            \textsc{Ultra} 0-shot &0.345 &0.513 &0.431 &0.566 &0.312 &0.458 &0.366 &0.518 & - & - \\
            \textsc{Ultra} fine-tuned & 0.397 & 0.556 & 0.442 & 0.582 &0.379 & 0.543 & 0.408 & 0.562 & 0.407 & 0.568 \\
            \bottomrule
        \end{tabular}
    \end{adjustbox}
\label{tab:main2}
\end{table}

On average, \textsc{Ultra} outperforms the baselines even in the 0-shot inference scenario both in MRR and Hits@10. The largest gains are achieved on smaller inductive graphs, e.g.\ on FB-25 and FB-50 0-shot \textsc{Ultra} yields almost $3\times$ better performance (291\% and 289\%, respectively). During pre-training, \textsc{Ultra} does not reach the baseline performance (0.407 vs 0.439 average MRR) and we link that with the lower 0-shot inference results on larger transductive graphs. However, fine-tuning \textsc{Ultra} effectively bridges this gap and surpasses the baselines. We hypothesize that in larger transductive graphs fine-tuning helps to adapt to different graph sizes (training graphs have 15-40k nodes while larger inference ones grow up to 123k nodes).

Following the sample efficiency and fast convergence of NBFNet~\cite{zhu2021neural}, we find that 1000-2000 steps are enough for fine-tuning \textsc{Ultra}. In some cases, fine-tuning brings marginal improvements or marginal negative effects. Averaged across 54 graphs (Table~\ref{tab:main2}), fine-tuned \textsc{Ultra} brings further 10\% relative improvement over the zero-shot version.

\subsection{Ablation Studies}
\label{subsec:ablation}

\begin{figure}[t]
    \centering
    \includegraphics[width=\textwidth]{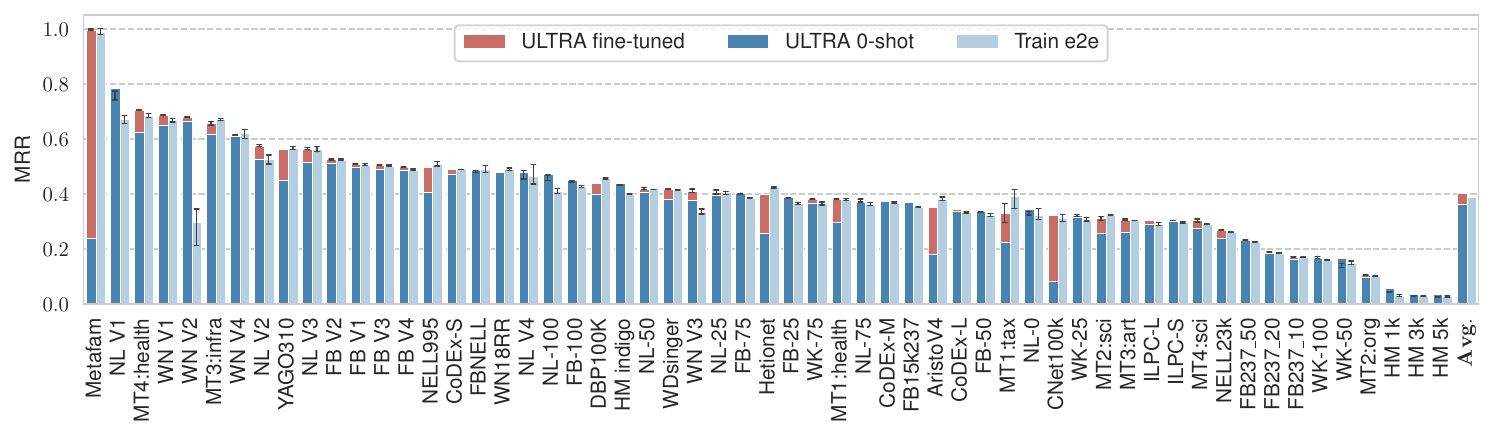}
    \caption[Comparison of Ultra against training a model from scratch]{Comparison of zero-shot and fine-tuned \textsc{Ultra} per-dataset performance against training a model from scratch on each dataset \emph{(Train e2e)}. Zero-shot performance of a single pre-trained model is on par with training from scratch while fine-tuning yields overall best results. }
    \label{fig:scratch}
\end{figure}

We performed several experiments to better understand the pre-training quality of \textsc{Ultra} and measure the impact of conditional relation representations on the performance.

\smallskip \noindent \textbf{Positive transfer from pre-training.}
We first study how a single pre-trained \textsc{Ultra} model compares to training instances of the same model separately on each graph end-to-end. For that, for each of 57 graphs, we train 3 \textsc{Ultra} instances of the same configuration and different random seeds until convergence and report the averaged results in Table~\ref{tab:main2} with per-dataset comparison in Figure~\ref{fig:scratch}. We find that, on average, a single pre-trained \textsc{Ultra} model in the zero-shot regime performs almost on par with the trained separate models, lags behind those on larger transductive graphs and exhibits better performance on inductive datasets. Fine-tuning a pre-trained \textsc{Ultra} shows overall the best performance and requires significantly less computational resources than training a model from scratch on every target graph.

\begin{wrapfigure}{r}{0.367\textwidth}
    \centering
    \vspace{-1.5em}
    \includegraphics[width=0.33\textwidth]{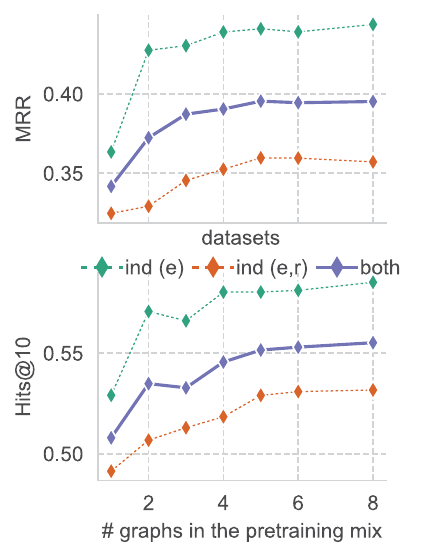}
    \caption[Zero-shot performance w.r.t.\ \# graphs in pre-training]{Averaged 0-shot performance on inductive datasets w.r.t.\ \# graphs in pre-training.}
    \label{fig:abl_num_graphs}
    \vspace{-1em}
\end{wrapfigure}

\smallskip \noindent \textbf{Number of graphs in the pre-training mix.}
We then study how inductive inference performance depends on the training mixture. While the main \textsc{Ultra} model was trained on the mixture of three graphs, here we train more models varying the amount of knowledge graphs in the training set from a single FB15k-237 to a combination of 8 transductive knowledge graphs. For the fair comparison, we evaluate pre-trained models in the zero-shot regime only on inductive datasets (41 graphs overall). The results are presented in Figure~\ref{fig:abl_num_graphs} where we observe the saturation of performance having more than three graphs in the mixture. We hypothesize that getting higher inference performance is tied up with model capacity, scale, and optimization. We leave that study along with more principled approached to selecting a pre-training mix for future work.

\smallskip \noindent \textbf{Conditional vs unconditional relation graph encoding.}
To measure the impact of the graph of relations and conditional relation representations, we pre-train three more models on the same mixture of three graphs varying several components: (1) we exclude four fundamental relation interactions (\emph{h2h}, \emph{h2t}, \emph{t2h}, \emph{t2t}) from the relation graph making it homogeneous and single-relational; (2) a homogeneous relation graph with an \emph{unconditional} GNN encoder following the R-GATv2 architecture from the previous SOTA approach, InGram~\cite{lee2023ingram}. The unconditional GNN needs input node features and we probed two strategies: Glorot initialization used in \cite{lee2023ingram} and initializing all nodes with a vector of ones $\1^d$. 

The results are presented in Table~\ref{tab:ablations} and indicate that ablated models struggle to reach the same pre-training performance and exhibit poor zero-shot generalization performance across all groups of graphs, e.g.\ up to 48\% relative MRR drop (0.192 vs 0.366) on the model with a homogeneous relation graph and randomly initialized node states with the unconditional R-GATv2 encoder. We therefore posit that conditional representations (both on relation and entity levels) are crucial for transferable representations for link prediction tasks that often require pairwise representations to break neighborhood symmetries. 

\begin{table}[t]
    \centering
    \caption[Ablation studies of Ultra]{Ablation studies: pre-training and zero-shot inference results of the main \textsc{Ultra}, \textsc{Ultra} without edge types in the relation graph (no etypes), \textsc{Ultra} without edge types and with InGram-like~\cite{lee2023ingram} unconditional GNN over relation graph where nodes are initialized with all ones (ones) or with Glorot initialization (random). Averaged results over 3 categories of datasets.}
    \label{tab:ablations}
    \begin{adjustbox}{width=\textwidth}
        \begin{tabular}{lcccccccc||ccc}
            \toprule
            \multirow{3}{*}{\bf{Model}} & \multicolumn{2}{c}{\bf{Inductive} $e, r$} & \multicolumn{2}{c}{\bf{Inductive} $e$} & \multicolumn{2}{c}{\bf{Transductive}} & \multicolumn{2}{c}{\bf{Total Avg}} & \multicolumn{2}{c}{\bf{Pretraining}} \\  
            & \multicolumn{2}{c}{(23 graphs)} & \multicolumn{2}{c}{(18 graphs)} & \multicolumn{2}{c}{(13 graphs)} & \multicolumn{2}{c}{(54 graphs)} & \multicolumn{2}{c}{(3 graphs)} \\ \cmidrule(l){2-3} \cmidrule(l){4-5} \cmidrule(l){6-7} \cmidrule(l){8-9} \cmidrule(l){10-11}
            & \bf{MRR} & \bf{H@10} & \bf{MRR} & \bf{H@10} & \bf{MRR} & \bf{H@10} & \bf{MRR} & \bf{H@10} & \multicolumn{1}{c}{\bf{MRR}} & \bf{H@10} \\ 
            \midrule
            \textsc{Ultra} &0.345 &0.513 &0.431 &0.566 &0.312 &0.458 &0.366 &0.518 &0.407 &0.568 \\
            - no etypes in rel. graph &0.292 &0.466 &0.389 &0.539 &0.258 &0.409 &0.316 &0.477 &0.357 &0.517 \\
            \begin{tabular}[c]{@{}l@{}} - no etypes, \\ \: - uncond. GNN (ones)\end{tabular}
            &0.187 &0.328 &0.262 &0.430 &0.135 &0.257 &0.199 &0.345 &0.263 &0.424 \\
            \begin{tabular}[c]{@{}l@{}} - no etypes, \\ \: - uncond. GNN (random)\end{tabular}
            &0.177 &0.309 &0.250 &0.417 &0.138 &0.255 &0.192 &0.332 &0.266 &0.433 \\
            \bottomrule
        \end{tabular}
    \end{adjustbox}
\end{table}
\section{Dataset Statistics}
\label{app:dataset_ultra}

We conduct evaluation on 57 openly available KGs of various sizes and three groups, i.e.\ tranductive, inductive with new entities, and inductive with both new entities and relations at inference time. The statistics for 16 transductive datasets are presented in Table~\ref{tab:app_datasets_transd}, 18 inductive entity datasets in Table~\ref{tab:app_datasets_inde}, and 23 inductive entity and relation datasets in Table~\ref{tab:app_datasets_indr_ultra}. For each dataset, we also list a currently published state-of-the-art model that, at the moment, are all trained specifically on each target graph. Performance of those SOTA models is aggregated as \emph{Supervised SOTA} in the results reported in the tables and figures. We omit smaller datasets (Kinships, UMLS, Countries, Family) with saturated performance as non-representative.

For the inductive datasets HM 1k, HM 3k, and HM 5k used in \cite{hamaguchi2017knowledge} and \cite{liu2021indigo}, we report the performance of predicting both heads and tails (noted as \emph{b-1K}, \emph{b-3K}, \emph{b-5K} in \cite{liu2021indigo}) and compare against the respective baselines. Some inductive datasets (MT2, MT3, MT4) from MTDEA~\cite{zhou2023ood} do not have reported entity-only KG completion performance. For Hetionet, we used the splits available in TorchDrug~\cite{zhu2022torchdrug} and compare with the baseline RotatE reported by TorchDrug.

\begin{table}[!h]
    \centering
    \caption[Transductive datasets (16) used in the experiments]{Transductive datasets (16) used in the experiments. Train, Valid, Test denote triples in the respective set. Task denotes the prediction task: \emph{h/t} is predicting both heads and tails, \emph{tails} is only predicting tails. SOTA points to the best reported result.}
    \label{tab:app_datasets_transd}
    \begin{adjustbox}{width=\textwidth}
        \begin{tabular}{lccccccccc}\toprule
            \bf{Dataset} & \bf{Entities} & \bf{Relations} & \bf{Train} & \bf{Valid} & \bf{Test} & \bf{Task} & \bf{SOTA} \\\midrule
            CoDEx Small~\cite{safavi2020codex} &2034 &42 &32888 &1827 &1828 & h/t & ComplEx RP~\cite{chen2021relation}\\
            WDsinger~\cite{lv2020dynamic} &10282 &135 &16142 &2163 &2203 & h/t & LR-GCN~\cite{he2023exploring} \\
            FB15k-237\_10~\cite{lv2020dynamic} &11512 &237 &27211 &15624 &18150 & tails & DacKGR~\cite{lv2020dynamic}\\
            FB15k-237\_20~\cite{lv2020dynamic} &13166 &237 &54423 &16963 &19776 & tails & DacKGR~\cite{lv2020dynamic}\\
            FB15k-237\_50~\cite{lv2020dynamic} &14149 &237 &136057 &17449 &20324 & tails & DacKGR~\cite{lv2020dynamic}\\
            FB15k-237~\cite{toutanova2015observed} &14541 &237 &272115 &17535 &20466 & h/t & NBFNet~\cite{zhu2021neural}\\
            CoDEx Medium~\cite{safavi2020codex} &17050 &51 &185584 &10310 &10311 & h/t & ComplEx RP~\cite{chen2021relation} \\
            NELL23k~\cite{lv2020dynamic} &22925 &200 &25445 &4961 &4952 & h/t & LR-GCN~\cite{he2023exploring} \\
            WN18RR~\cite{dettmers2018convolutional} &40943 &11 &86835 &3034 &3134 & h/t & NBFNet~\cite{zhu2021neural}\\
            AristoV4~\cite{chen2021relation} &44949 &1605 &242567 &20000 &20000 & h/t & ComplEx RP~\cite{chen2021relation} \\
            Hetionet~\cite{himmelstein2017systematic} &45158 &24 &2025177 &112510 &112510 & h/t & RotatE~\cite{sun2019rotate} \\
            NELL995~\cite{xiong2017deeppath} &74536 &200 &149678 &543 &2818 & h/t & RED-GNN~\cite{zhang2022knowledge}\\
            CoDEx Large~\cite{safavi2020codex} &77951 &69 &551193 &30622 &30622 & h/t & ComplEx RP~\cite{chen2021relation}\\
            ConceptNet100k~\cite{malaviya2020commonsense} &78334 &34 &100000 &1200 &1200 & h/t & BiQUE~\cite{guo2021bique} \\
            DBpedia100k~\cite{ding2018improving} &99604 &470 &597572 &50000 &50000 & h/t & ComplEx-NNE+AER~\cite{ding2018improving} \\
            YAGO310~\cite{mahdisoltani2014yago3} &123182 &37 &1079040 &5000 &5000 & h/t &  NBFNet~\cite{zhu2021neural}\\
            \bottomrule
        \end{tabular}
    \end{adjustbox}
\end{table}

\pagebreak

\begin{table}[!h]
    \caption[Inductive entity $(e)$ datasets (18) used in the experiments]{Inductive entity $(e)$ datasets (18) used in the experiments. Triples denote the number of edges of the graph given at training, validation, or test. Valid and Test denote triples to be predicted in the validation and test sets in the respective validation and test graph.}
    \label{tab:app_datasets_inde}
    \begin{adjustbox}{width=\textwidth}
        \begin{tabular}{lccccccccccc}\toprule
            \multirow{2}{*}{\bf{Dataset}} &\multirow{2}{*}{\bf{Relations}} &\multicolumn{2}{c}{\bf{Training Graph}} &\multicolumn{3}{c}{\bf{Validation Graph}} &\multicolumn{3}{c}{\bf{Test Graph}} & \multirow{2}{*}{\bf{SOTA}} \\ \cmidrule(l){3-4} \cmidrule(l){5-7} \cmidrule(l){8-10}
            & & \bf{Entities} & \bf{Triples} & \bf{Entities} & \bf{Triples} & \bf{Valid}  & \bf{Entities} & \bf{Triples} & \bf{Test}  \\\midrule
            FB v1~\cite{teru2020inductive} &180 &1594 &4245 &1594 &4245 &489 &1093 &1993 &411 & A*Net~\cite{zhu2023net} \\
            FB v2~\cite{teru2020inductive} &200 &2608 &9739 &2608 &9739 &1166 &1660 &4145 &947 & NBFNet~\cite{zhu2021neural} \\
            FB v3~\cite{teru2020inductive} &215 &3668 &17986 &3668 &17986 &2194 &2501 &7406 &1731 & NBFNet~\cite{zhu2021neural} \\
            FB v4~\cite{teru2020inductive} &219 &4707 &27203 &4707 &27203 &3352 &3051 &11714 &2840 & A*Net~\cite{zhu2023net} \\
            WN v1~\cite{teru2020inductive} &9 &2746 &5410 &2746 &5410 &630 &922 &1618 &373 & NBFNet~\cite{zhu2021neural}\\
            WN v2~\cite{teru2020inductive} &10 &6954 &15262 &6954 &15262 &1838 &2757 &4011 &852 & NBFNet~\cite{zhu2021neural} \\
            WN v3~\cite{teru2020inductive} &11 &12078 &25901 &12078 &25901 &3097 &5084 &6327 &1143 & NBFNet~\cite{zhu2021neural}\\
            WN v4~\cite{teru2020inductive} &9 &3861 &7940 &3861 &7940 &934 &7084 &12334 &2823 & A*Net~\cite{zhu2023net} \\
            NELL v1~\cite{teru2020inductive} &14 &3103 &4687 &3103 &4687 &414 &225 &833 &201 & RED-GNN~\cite{zhang2022knowledge}\\
            NELL v2~\cite{teru2020inductive} &88 &2564 &8219 &2564 &8219 &922 &2086 &4586 &935 & RED-GNN~\cite{zhang2022knowledge}\\
            NELL v3~\cite{teru2020inductive} &142 &4647 &16393 &4647 &16393 &1851 &3566 &8048 &1620 & RED-GNN~\cite{zhang2022knowledge}\\
            NELL v4~\cite{teru2020inductive} &76 &2092 &7546 &2092 &7546 &876 &2795 &7073 &1447 & RED-GNN~\cite{zhang2022knowledge}\\
            ILPC Small~\cite{galkin2022open} &48 &10230 &78616 &6653 &20960 &2908 &6653 &20960 &2902 & NodePiece~\cite{galkin2022open}\\
            ILPC Large~\cite{galkin2022open} &65 &46626 &202446 &29246 &77044 &10179 &29246 &77044 &10184 & NodePiece~\cite{galkin2022open} \\
            HM 1k~\cite{hamaguchi2017knowledge} &11 &36237 &93364 &36311 &93364 &1771 &9899 &18638 &476& R-GCN~\cite{liu2021indigo} \\
            HM 3k~\cite{hamaguchi2017knowledge} &11 &32118 &71097 &32250 &71097 &1201 &19218 &38285 &1349& Indigo~\cite{liu2021indigo} \\
            HM 5k~\cite{hamaguchi2017knowledge} &11 &28601 &57601 &28744 &57601 &900 &23792 &48425 &2124& Indigo~\cite{liu2021indigo} \\
            IndigoBM~\cite{liu2021indigo} &229 &12721 &121601 &12797 &121601 &14121 &14775 &250195 &14904& GraIL~\cite{liu2021indigo} \\
            \bottomrule
        \end{tabular}
    \end{adjustbox}
\end{table}

\pagebreak

\begin{table}[!h]
    \caption[Inductive entity and relation $(e,r)$ datasets (23) used in the experiments]{Inductive entity and relation $(e,r)$ datasets (23) used in the experiments. Triples denote the number of edges of the graph given at training, validation, or test. Valid and Test denote triples to be predicted in the validation and test sets in the respective validation and test graph.}
    \label{tab:app_datasets_indr_ultra}
    \begin{adjustbox}{width=\textwidth}
        \begin{tabular}{lccccccccccccc}\toprule
            \multirow{2}{*}{\bf{Dataset}} &\multicolumn{3}{c}{\bf{Training Graph}} &\multicolumn{4}{c}{\bf{Validation Graph}} &\multicolumn{4}{c}{\bf{Test Graph}} & \multirow{2}{*}{\bf{SOTA}} \\ \cmidrule(l){2-4} \cmidrule(l){5-8} \cmidrule(l){9-12}
            & \bf{Entities} & \bf{Relations} & \bf{Triples} & \bf{Entities} & \bf{Relations} & \bf{Triples} & \bf{Valid} & \bf{Entities} & \bf{Relations} & \bf{Triples} & \bf{Test} \\\midrule
            FB-25~\cite{lee2023ingram} &5190 &163 &91571 &4097 &216 &17147 &5716 &4097 &216 &17147 &5716 & InGram~\cite{lee2023ingram} \\
            FB-50~\cite{lee2023ingram} &5190 &153 &85375 &4445 &205 &11636 &3879 &4445 &205 &11636 &3879 &InGram~\cite{lee2023ingram} \\
            FB-75~\cite{lee2023ingram} &4659 &134 &62809 &2792 &186 &9316 &3106 &2792 &186 &9316 &3106 &InGram~\cite{lee2023ingram} \\
            FB-100~\cite{lee2023ingram} &4659 &134 &62809 &2624 &77 &6987 &2329 &2624 &77 &6987 &2329 &InGram~\cite{lee2023ingram} \\
            WK-25~\cite{lee2023ingram} &12659 &47 &41873 &3228 &74 &3391 &1130 &3228 &74 &3391 &1131 &InGram~\cite{lee2023ingram} \\
            WK-50~\cite{lee2023ingram} &12022 &72 &82481 &9328 &93 &9672 &3224 &9328 &93 &9672 &3225 &InGram~\cite{lee2023ingram} \\
            WK-75~\cite{lee2023ingram} &6853 &52 &28741 &2722 &65 &3430 &1143 &2722 &65 &3430 &1144 &InGram~\cite{lee2023ingram} \\
            WK-100~\cite{lee2023ingram} &9784 &67 &49875 &12136 &37 &13487 &4496 &12136 &37 &13487 &4496 &InGram~\cite{lee2023ingram} \\
            NL-0~\cite{lee2023ingram} &1814 &134 &7796 &2026 &112 &2287 &763 &2026 &112 &2287 &763 &InGram~\cite{lee2023ingram} \\
            NL-25~\cite{lee2023ingram} &4396 &106 &17578 &2146 &120 &2230 &743 &2146 &120 &2230 &744 &InGram~\cite{lee2023ingram} \\
            NL-50~\cite{lee2023ingram} &4396 &106 &17578 &2335 &119 &2576 &859 &2335 &119 &2576 &859 &InGram~\cite{lee2023ingram} \\
            NL-75~\cite{lee2023ingram} &2607 &96 &11058 &1578 &116 &1818 &606 &1578 &116 &1818 &607 &InGram~\cite{lee2023ingram} \\
            NL-100~\cite{lee2023ingram} &1258 &55 &7832 &1709 &53 &2378 &793 &1709 &53 &2378 &793  &InGram~\cite{lee2023ingram}\\
            \midrule
            Metafam~\cite{zhou2023ood} &1316 &28 &13821 &1316 &28 &13821 &590 &656 &28 &7257 &184 & NBFNet~\cite{zhou2023ood}\\
            FBNELL~\cite{zhou2023ood} &4636 &100 &10275 &4636 &100 &10275 &1055 &4752 &183 &10685 &597 & NBFNet~\cite{zhou2023ood} \\
            Wiki MT1 tax~\cite{zhou2023ood} &10000 &10 &17178 &10000 &10 &17178 &1908 &10000 &9 &16526 &1834 & NBFNet~\cite{zhou2023ood} \\
            Wiki MT1 health~\cite{zhou2023ood} &10000 &7 &14371 &10000 &7 &14371 &1596 &10000 &7 &14110 &1566 & NBFNet~\cite{zhou2023ood} \\
            Wiki MT2 org~\cite{zhou2023ood} &10000 &10 &23233 &10000 &10 &23233 &2581 &10000 &11 &21976 &2441 & N/A \\
            Wiki MT2 sci~\cite{zhou2023ood} &10000 &16 &16471 &10000 &16 &16471 &1830 &10000 &16 &14852 &1650 & N/A \\
            Wiki MT3 art~\cite{zhou2023ood} &10000 &45 &27262 &10000 &45 &27262 &3026 &10000 &45 &28023 &3113 & N/A \\
            Wiki MT3 infra~\cite{zhou2023ood} &10000 &24 &21990 &10000 &24 &21990 &2443 &10000 &27 &21646 &2405 & N/A \\
            Wiki MT4 sci~\cite{zhou2023ood} &10000 &42 &12576 &10000 &42 &12576 &1397 &10000 &42 &12516 &1388 & N/A \\
            Wiki MT4 health~\cite{zhou2023ood} &10000 &21 &15539 &10000 &21 &15539 &1725 &10000 &20 &15337 &1703 & N/A \\
            \bottomrule
        \end{tabular}
    \end{adjustbox}
\end{table}
\section{Limitations and Future Work}

Albeit \textsc{Ultra} demonstrates promising capabilities as a foundation model for knowledge graph reasoning in the zero-shot and fine-tuning regimes, there are several limitations and open questions. First, pre-training on more graphs does not often correspond to better inference performance. We hypothesize the reason might be in the overall small model size (177k parameters) and limited model capacity, i.e.\ with increasing the diversity of training data the model size should increase as well. On the other hand, our preliminary experiments did not show significant improvements of scaling the parameter count beyond 200k. We hypothesize it might be an issue of input normalization and model optimization. We plan to address those open questions in the future work.

\part{Multi-step Queries}
\label{part:multi-step}
\chapter[Solving Multi-hop Queries on Knowledge Graphs]{Solving Multi-hop Queries on\linebreak Knowledge Graphs}
\label{cha:gnn-qe}

Many reasoning applications require to deal with queries that inherently contain multiple steps. Often, such queries on knowledge graphs are handled by query embedding methods that model logical operations in an embedding space, which are not compatible with our strong NBFNet architecture. Hence, we are interested in a new framework for answering multi-hop queries with NBFNet-like models, which can also benefit from the inductive generalization abilities we developed in Part~\ref{part:inductive}.

In this chapter, we introduce GNN-QE, a framework that decomposes multi-hop queries into basic operations and parameterizes each operation with a graph neural network (GNN) or fuzzy logic operations. GNN-QE not only achieves significantly better performance than query embedding methods, but also requires less training samples and provides interpretability for intermediate variables. Like NBFNet, GNN-QE is inductive and can be applied to knowledge graphs with unseen entities. It can be further integrated with pretrained Ultra checkpoints to perform zero-shot query answering on any knowledge graph.

\smallskip \emph{This chapter is based on our work published at ICML 2022~\cite{zhu2022neural}\footnote{The code of GNN-QE is available at \url{https://github.com/DeepGraphLearning/GNN-QE}} and an arXiv paper~\cite{galkin2024zero}\footnote{The code of UltraQuery is available at \url{https://github.com/DeepGraphLearning/ULTRA}}. Mikhail Galkin led the project~\cite{galkin2024zero} and I contributed the model design. Some experiment results are from our work published at NeurIPS 2022~\cite{galkin2022inductive}.}

\section{Overview}

Knowledge graphs encapsulate knowledge about the world in a collection of relational edges between entities, and are widely adopted by many domains~\cite{miller1998wordnet, vrandevcic2014wikidata, himmelstein2017systematic, szklarczyk2019string}. Reasoning on knowledge graphs has attracted much attention in artificial intelligence, since it can be used to infer new knowledge or answer queries based on existing knowledge. One particular reasoning task we are interested in is answering complex First-Order Logic (FOL) queries on knowledge graphs, which involves logic operations like existential quantifier ($\exists$), conjunction ($\land$), disjunction ($\lor$) and negation ($\neg$). For example, the question ``\emph{Which universities do the Turing Award winners of deep learning work in?}'' can be represented as a FOL query, as showed in Figure~\ref{fig:method_gnn-qe}.

Traditionally, the problem of reasoning is handled by symbolic approaches, such as logic programming~\cite{lloyd2012foundations}, fuzzy logic~\cite{klir1995fuzzy} or probabilistic reasoning~\cite{pearl2014probabilistic}. In the same vein, several algorithms~\cite{dalvi2007efficient, schmidt2010foundations, zou2011gstore} have been developed for searching the answers to complex queries on graph databases. These methods traverse a graph and extract all possible assignments for intermediate variables, which provides good interpretation for each step. Besides, symbolic methods are guaranteed to produce the correct answer if all facts are given~\cite{stuart2016artificial}. However, many real-world knowledge graphs are known to be incomplete~\cite{nickel2015review}, which limits the usage of symbolic methods on knowledge graphs.

Recently, neural methods, such as embedding methods~\cite{bordes2013translating, trouillon2016complex, sun2019rotate} and graph neural networks (GNNs)~\cite{schlichtkrull2018modeling, vashishth2020composition, teru2020inductive, zhu2021neural}, have achieved significant progress in knowledge graph completion. Based on the success of these neural methods, many works have been proposed to solve FOL queries on incomplete graphs by learning an embedding for each FOL query~\cite{hamilton2018embedding, ren2020query2box, ren2020beta, chen2022fuzzy, zhang2021cone}. Typically, these methods 
translate the logic operations into neural logic operators in the embedding space. Nevertheless, it is hard to interpret what set of entities an intermediate embedding encodes, leaving the reasoning process unknown to users. The only interpretable method is CQD-Beam~\cite{arakelyan2021complex}, which applies beam search to a pretrained embedding model in the entity space. However, the complexity of exhaustive search prevents CQD-Beam from being trained directly on complex queries.

\smallskip \noindent \textbf{GNN-QE.} In this paper, we marry the advantages from both neural and symbolic approaches, and propose Graph Neural Network Query Executor (GNN-QE), a neural-symbolic method for answering FOL queries on incomplete knowledge graphs. Following symbolic methods that output a set of assignments for each intermediate variable, we decompose a complex FOL query into an expression over fuzzy sets (i.e.\ a continuous relaxation of sets), which attains interpretability for intermediate variables. Each basic operation in the expression is either a relation projection or a logic operation (e.g.\ conjunction, disjunction and negation). We design the relation projection to be a GNN that predicts the fuzzy set of tail entities given a fuzzy set of head entities and a relation. The logic operations are transformed to the product fuzzy logic operations over fuzzy sets, which satisfy logic laws and enable differentiation of logic operations. We also propose traversal dropout to regularize the model, and batch expression execution to speed up training and inference.

We evaluate our method on 3 transductive datasets for FOL queries. Experiments show that GNN-QE achieves new state-of-the-art performance on all datasets, with an average relative gain of 22.3\% on existential positive first-order (EPFO) queries and 95.1\% on negation queries (Section~\ref{sec:main_experiment}). By disentangling the contribution of knowledge graph completion and complex query framework, we find that GNN-QE achieves one of the best generalization performances from knowledge graph completion to EPFO queries among different methods. Additionally, the symbolic formulation of our method enables us to predict the number of answers without explicit supervision (Section~\ref{sec:cardinality}), and visualize intermediate variables (Section~\ref{sec:vis_gnn-qe}). The visualization provided by GNN-QE may help us better understand the reasoning process taken by the model, leading to more interpretable multi-hop reasoning.

To understand the inductive generalization capacity of GNN-QE, we construct a novel suite of 9 inductive datasets with different ratios of the inference graph size to the training graph. Compared to NodePiece-QE, a direct combination of inductive representations learned by NodePiece~\cite{galkin2022nodepiece} and CQD-Beam decoder~\cite{arakelyan2021complex}, we find that GNN-QE outperforms such a baseline by a large margin across all size ratios, though there is a performance decay for both models when the ratio grows larger. By probing the performance of the training queries, we observe that GNN-QE generalizes well to easy answers when we switch from the training graph to a larger inference graph, showing the superior \emph{faithfullness}~\cite{sun2020faithful} of GNN-QE. Additionally, GNN-QE consistently ranks easy answers higher than hard answers as expected.

\smallskip \noindent \textbf{UltraQuery.} With recent advancements in inductive knowledge graph reasoning~\cite{galkin2024towards, gao2023double}, we can further extend the framework of GNN-QE to answer complex queries over arbitrary knowledge graphs with new entities and relations at inference time. By parameterizing the relation projection with an \textsc{Ultra} model that does not memorize graph-specific entity nor relation embeddings, we develop \textsc{UltraQuery}, a complex query model with both inductive relation projection and inductive logical operations, thereby enabling zero-shot generalization to any knowledge graph. We curate a novel suite of 11 inductive query answering datasets where graphs and queries at inference time have new entity and relation vocabularies. Averaged across a total number of 22 transductive and inductive datasets, a single \textsc{UltraQuery} model outperforms the best reported baselines trained separately for each dataset by a relative gain of 50\% in MRR on both EPFO queries and negation queries.

\section{Method: GNN-QE}

\begin{figure*}
    \centering
    \includegraphics[width=\textwidth]{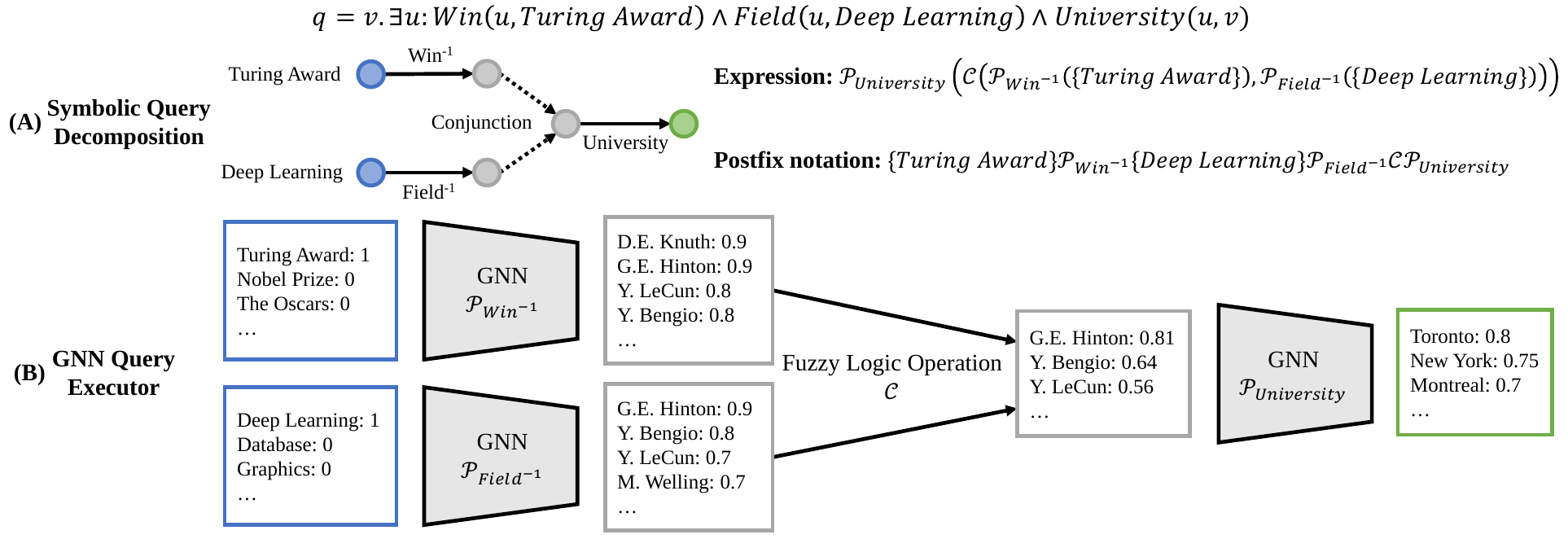}
    \caption[Overview of GNN-QE]{Overview of GNN-QE. \textbf{(A)} GNN-QE decomposes a FOL query into an expression of relation projections ($\gP$) and logic operations ($\gC, \gD, \gN$). We convert the query into an expression execution problem, where we use the postfix notation to efficiently batch multiple expressions. \textbf{(B)} The expression is executed with relation projection learned by GNNs and fuzzy logic operations. All the input, intermediate and output variables are fuzzy sets of entities. Best viewed in color.}
    \label{fig:method_gnn-qe}
\end{figure*}

Here we present our model, Graph Neural Network Query Executor (GNN-QE). The high-level idea of GNN-QE is to first decompose a FOL query into an expression of 4 basic operations (relation projection, conjunction, disjunction and negation) over fuzzy sets, then parameterize the relation projection with a GNN adapted from knowledge graph completion, and instantiate the logic operations with product fuzzy logic operations. Besides, we introduce traversal dropout to prevent the GNN from converging to a trivial solution, and batched expression execution for speeding up training and inference.

\subsection{Symbolic Query Decomposition}

Given a FOL query, the first step is to convert it into an expression of basic operations, so that we can retrieve answers by executing the expression. Previous works define basic operations as either relation projections and logic operations over \emph{embeddings}~\cite{ren2020query2box, ren2020beta, chen2022fuzzy, zhang2021cone}, or a score function over triplets~\cite{arakelyan2021complex}. To achieve better interpretability for intermediate variables, we explicitly define 4 basic operations over \emph{fuzzy sets of entities} as follows

\begin{itemize}[label=$\bullet$, leftmargin=*]
    \item \textbf{Relation Projection}: $\gP_q(\vx)$ computes the fuzzy set of \emph{tail} entities that are reachable by the input fuzzy set of \emph{head} entities through relation $q$. $\gP_{q^{-1}}(\vx)$ computes the fuzzy set of \emph{head} entities that can reach the input fuzzy set of \emph{tail} entities through relation $q$.
    \item \textbf{Conjunction}: $\gC(\vx, \vy)$ computes the logical conjunction for each element in $\vx$ and $\vy$.
    \item \textbf{Disjunction}: $\gD(\vx, \vy)$ computes the logical disjunction for each element in $\vx$ and $\vy$.
    \item \textbf{Negation}: $\gN(\vx)$ computes the logical negation for each element in $\vx$.
\end{itemize}

where $\vx, \vy \in [0,1]^\gV$ are two vector representations of fuzzy sets. We then decompose a FOL query into an expression of the above operations. For the example in Figure~\ref{fig:method_gnn-qe}, the corresponding expression is
\begin{equation}
    \label{eqn:decomposition}
    \gP_\textit{University}\left(\gC\left(\gP_{\textit{Win}^{-1}}(\{\textit{Turing Award}\}), \gP_{\textit{Field}^{-1}}(\{\textit{Deep Learning}\})\right)\right)
\end{equation}
where \{\textit{Turing Award}\} and \{\textit{Deep Learning}\} denote singleton sets of \emph{Turing Award} and \emph{Deep Learning}, respectively.

\subsection{Neural Relation Projection}

In order to solve complex queries on incomplete knowledge graphs, we learn a neural model to perform the relation projection $\vy = \gP_q(\vx)$. Specifically, the neural relation projection model should predict the fuzzy set of tail entities $\vy$ given the fuzzy set of head entities $\vx$ and a relation $q$ in the presence of missing links. This is in contrast to the common GNNs~\cite{schlichtkrull2018modeling, vashishth2020composition} and embedding methods~\cite{bordes2013translating, sun2019rotate} for knowledge graph completion, which operate on individual entities $x$ and $y$. While it is possible to apply such GNNs or embedding methods for relation projection, it takes at least $O(|\gV|^2d)$ time to compute them for every $x \in \vx$ and $y \in \vy$, which is not scalable.

Recently, \cite{zhu2021neural} introduced a new GNN framework for knowledge graph completion, which can predict the set of tail entities $\vy$ given an entity $x$ and a relation $q$ in $O(|\gV|d^2 + |\gE|d)$ time. Inspired by such a framework, we propose a scalable GNN solution for relation projection.

\textbf{Graph Neural Networks.}
Our goal is to design a GNN model that predicts a fuzzy set of tail entities given a fuzzy set of head entities and a relation. A special case of the input is a singleton set, where we need to model the probability $p_q(y|x)$ for every $y \in \vy$. Such a problem can be solved by GNNs in a single-source fashion~\cite{you2021identity, zhu2021neural}. For example, the recent work NBFNet~\cite{zhu2021neural} derives a GNN framework based on the generalized Bellman-Ford algorithm for single-source problems on graphs. Given a head entity $u$ and a projection relation $q$, we use the following iteration to compute a representation $\vh_v$ for each entity $v \in \gV$ w.r.t.\ the source entity $u$:
\begin{align}
    &\vh^{(0)}_v \leftarrow \textsc{Indicator}(u, v, q) \label{eqn:indicator} \\
    &\vh^{(t)}_v \leftarrow \textsc{Aggregate}(\{\textsc{Message}(\vh^{(t-1)}_z, (z, r, v)) | (z, r, v)\in\gE(v)\}) \label{eqn:bellman}
\end{align}
where the \textsc{Indicator} function initializes a relation embedding $\vq$ on entity $v$ if $v$ equals to $u$ and a zero embedding otherwise, and $\gE(v)$ is the set of edges going into $v$. The \textsc{Message} and \textsc{Aggregate} functions can be instantiated with any neural function from popular GNNs. To apply the above framework to a fuzzy set $\vx$ of head entities, we propose to replace Equation~\ref{eqn:indicator} with the following initialization
\begin{equation}
    \vh^{(0)}_v \leftarrow x_v \vq
    \label{eqn:fuzzy_set}
\end{equation}
where $x_v$ is the probability of entity $v$ in $\vx$. Intuitively, this GNN model initializes an embedding $\vq$ for the projection relation $q$ on all entities, where the scale of the initialization on an entity depends on its probability in the fuzzy set. The original \textsc{Indicator} function can be viewed as a special case of Equation~\ref{eqn:fuzzy_set}, with the fuzzy set being a singleton set.

For the \textsc{Aggregate} and the \textsc{Message} functions, we follow the design in NBFNet~\cite{zhu2021neural} and parameterize the \textsc{Message} function as
\begin{align}
    \textsc{Message}(\vh^{(t-1)}_z, (z, r, v)) = \vh^{(t-1)}_z \odot (\mW_r \vq + \vb_r)
    \label{eqn:message}
\end{align}
where $\mW^{(t)}_r$ and $\vb^{(t)}_r$ are the weight matrix and bias vector for relation $r$ in the $t$-th iteration respectively, and $\odot$ is the element-wise multiplication operator. The \textsc{Aggregate} function is parameterized as the principal neighborhood aggregation (PNA)~\cite{corso2020principal}. Our GNN has the same time complexity as NBFNet, and therefore takes $O(|\gV|d^2 + |\gE|d)$ time for each message passing iteration. Note it is possible to parameterize the framework with other GNN models, such as RGCN~\cite{schlichtkrull2018modeling} or CompGCN~\cite{vashishth2020composition}. See Section~\ref{sec:ablation} for experiments with different GNN models.

To apply the GNN framework for relation projection, we propagate the representations with Equation~\ref{eqn:bellman} for $T$ layers. Then we take the representations in the last layer, and pass them into a multi-layer perceptron (MLP) $f$ followed by a sigmoid function $\sigma$ to predict the fuzzy set of tail entities.
\begin{align}
    \gP_q(\vx) = \sigma(f(\vh^{(T)}))
    \label{eqn:predict}
\end{align}

\subsection{Fuzzy Logic Operations}

The logic operations (i.e.\ $\gC(\vx, \vy)$, $\gD(\vx, \vy)$, $\gN(\vx)$) glue multiple relation projection results and generate the input fuzzy set for the next relation projection. Ideally, they should satisfy certain logic laws, such as commutativity, associativity and non-contradiction. Most previous works~\cite{hamilton2018embedding, ren2020query2box, ren2020beta, zhang2021cone} propose dedicated geometric operations to learn these logic operations in the embedding space. Nevertheless, these neural operators are not guaranteed to satisfy most logic laws, which may introduce additional error when they are chained together. 

Here we model the conjunction, disjunction and negation with product fuzzy logic operations. Given two fuzzy sets $\vx, \vy \in [0,1]^\gV$, the operations are defined as follows
\begin{align}
    \gC(\vx, \vy) &= \vx \odot \vy \\
    \gD(\vx, \vy) &= \vx + \vy - \vx \odot \vy \\
    \gN(\vx) &= \bm{1} - \vx
\end{align}
where $\odot$ is the element-wise multiplication and $\bm{1}$ is a vector of all ones (i.e.\ the universe). Compared to geometric operations in previous works, such fuzzy logic operations satisfy many logic laws, e.g.\ De Morgan's laws $\gN(\gC(\vx, \vy)) = \gD(\gN(\vx), \gN(\vy))$, $\gN(\gD(\vx, \vy)) = \gC(\gN(\vx), \gN(\vy))$. Note FuzzQE~\cite{chen2022fuzzy} also adopts fuzzy logic operations and satisfies logic laws. However, FuzzQE applies fuzzy logic operations to \emph{embeddings}. By contrast, our GNN-QE applies fuzzy logic operations to \emph{fuzzy sets of entities}, which provides better interpretability (See Section~\ref{sec:vis_gnn-qe}).

\subsection{Learning}
Following previous works~\cite{ren2020query2box, ren2020beta, zhang2021cone}, we train our model to minimize the binary cross entropy loss.
\begin{equation}
    \gL = -\frac{1}{|\gA_Q|}\sum_{a \in \gA_Q}\log p(a|Q) - \frac{1}{|\gV\backslash\gA_Q|}\sum_{a' \in \gV\setminus\gA_Q}\log (1 - p(a'|Q))
    \label{eqn:loss}
\end{equation}
where $\gA_Q$ is the set of answers to the complex query $Q$ and $p(a|Q)$ is the probability of entity $a$ in the final output fuzzy set. Since GNN-QE always outputs the probability for all entities (Equation~\ref{eqn:predict}), we do not perform negative sampling and compute the loss with all negative answers.

\smallskip \noindent \textbf{Traversal Dropout.} One challenge in training GNN-QE is to let the model generalize to incomplete knowledge graphs at test time. This is because all the training queries are generated by assuming the training graph is complete~\cite{ren2020beta}. In other words, all the training queries can be perfectly solved by a simple relation traversal model on the training graph, without modeling any missing link. GNN models can easily discover this mode, which does not generalize to incomplete knowledge graphs at test time.

To solve this issue, we introduce traversal dropout to create an incomplete knowledge graph at training time. Specifically, we first run a relation traversal model to extract all the edges corresponding to the query. We then randomly mask out the traversed edges in each relation projection with probability $p$. Intuitively, the probability $p$ trades off between a simple relation traversal model and a full reasoning model. If $p$ is small, the GNN model may converge to a trivial relation traversal model, otherwise it is forced to encode non-trivial reasoning features. Since some of the edges in the test queries may be present in the knowledge graph, it is not always optimal to use a large $p$ to discourage a relation traversal model. In practice, we treat $p$ as a hyperparameter, and tune it based on the performance on the validation set. See Section~\ref{sec:ablation} for experiments with different values of $p$.

\begin{figure*}[!h]
    \centering
    \includegraphics[width=\textwidth]{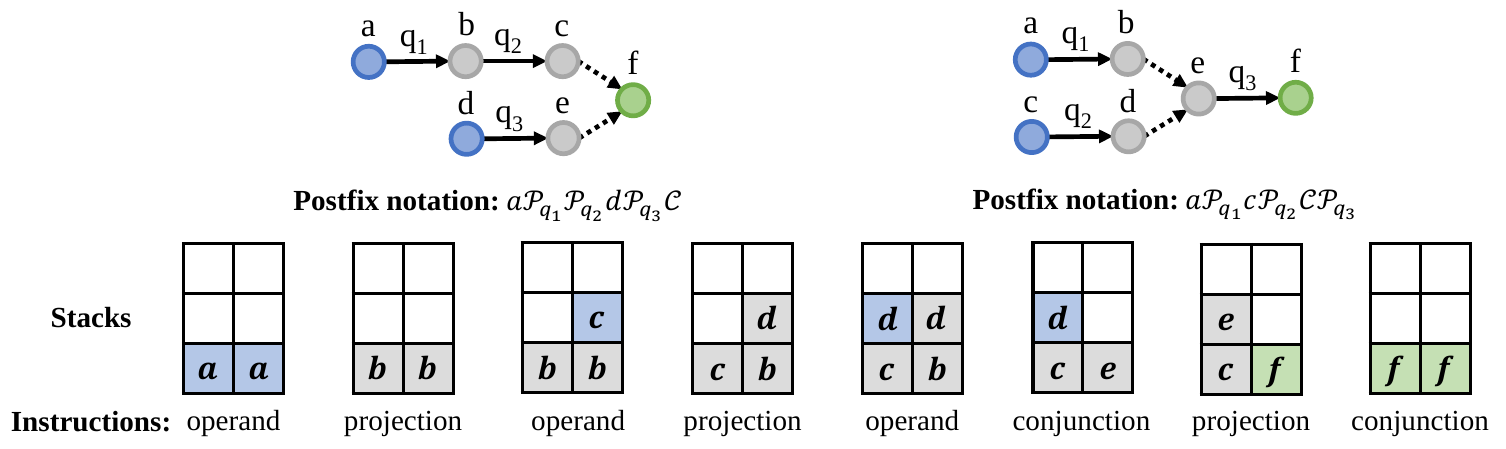}
    \caption[Illustration of batched expression execution over a batch of two queries]{Illustration of batched expression execution over a batch of two queries.}
    \label{fig:batch_execution}
\end{figure*}

\smallskip \noindent \textbf{Batched Expression Execution\footnote{Expression execution is formally known as expression evaluation in computer science. In this paper, we use the term ``expression execution'' to avoid ambiguity in machine learning contexts.}.}
Modern machine learning relies on batch processing on GPUs to accelerate the computation of neural (or even symbolic) models. However, it is challenging to batch the expressions of FOL queries, since different query structures require different recursive computation steps. Previous works~\cite{hamilton2018embedding, ren2020query2box, ren2020beta} divide a batch based on the query structure of each sample, and only batch the computation of samples that have the same structure. Nevertheless, such an implementation needs to enumerate every query structure, and is not scalable when the vocabulary of query structures grows large.

To solve this issue, we need to find a way to execute the expressions without recursion. This can be achieved by converting the expressions into postfix notation. The postfix notation, a.k.a.\ reverse Polish notation~\cite{lukasiewicz1951aristotle}, writes operators \emph{after} their operands in an expression. For example, the postfix expression of Equation~\ref{eqn:decomposition} is
\begin{align}
    \{\textit{Turing Award}\}\gP_{\textit{Win}^{-1}}\{\textit{Deep Learning}\}\gP_{\textit{Field}^{-1}}\gC\gP_\textit{University}
\end{align}
The advantage of postfix expressions is that they are unambiguous without parentheses, and therefore can be executed easily without recursion. To execute a postfix expression, we allocate a stack and scan the expression from left to right. When we encounter an operand, we push it into the stack. When we encounter an operator, we pop the corresponding number of operands from the stack, apply the operation and push the result into the stack. Such an algorithm can be easily batched for the same operator even in samples of different query types. Figure~\ref{fig:batch_execution} shows batched expression execution applied to two different query structures.

The overall time complexity of our batched execution is $O(t(|\gV|d^2 + |\gE|d))$, where $t$ is the maximal number of projections in a single query in the batch. Compared to existing implementation~\cite{hamilton2018embedding, ren2020query2box, ren2020beta} that scales \emph{linearly} w.r.t.\ the number of query types, batched expression execution scales \emph{independently} w.r.t.\ the number of query types, and can be applied to arbitrary large number of query types without scalability issues.
\section{Transductive Experiments of GNN-QE}

In this section, we evaluate GNN-QE by answering FOL queries on 3 standard datasets. Our experiments demonstrate that: (1) GNN-QE outperforms existing methods on both EPFO queries and queries with negation. (2) GNN-QE can predict the number of answers out-of-the-box without any explicit supervision. (3) We can visualize the intermediate variables of GNN-QE and interpret its reasoning process.

\subsection{Experiment Setup}

We evaluate our method on FB15k~\cite{bordes2013translating}, FB15k-237~\cite{toutanova2015observed} and NELL995~\cite{xiong2017deeppath} knowledge graphs. To make a fair comparison with baselines, we use the standard train, validation and test FOL queries generated by the BetaE paper~\cite{ren2020beta}, which consist of 9 EPFO query types and 5 query types with negation. We follow previous works~\cite{ren2020beta, chen2022fuzzy, zhang2021cone} and train our model with 10 query types (\emph{1p/2p/3p/2i/3i/2in/3in/inp/pni/pin}). The model is evaluated on 10 training query types, plus 4 query types (\emph{ip/pi/2u/up}) that have never been seen during training. A full list of query types and their statistics is provided in Section~\ref{app:dataset_gnn-qe}.

\smallskip \noindent \textbf{Evaluation Protocol.}
Following the evaluation protocol in \cite{ren2020query2box}, we separate the answers to each query into two sets: easy answers and hard answers. For test (validation) queries, easy answers are the entities that can be reached on the validation (train) graph via a symbolic relation traverse model. Hard answers are those that can only be reached with predicted links. In other words, the model must perform reasoning to get the hard answers. We compute the ranking of each hard answer against all non-answer entities. The performance is measured by mean reciprocal rank (MRR) and HITS at K (H@K) metrics.

\smallskip \noindent \textbf{Implementation Details.}
Our work is implemented based on the open-source codebase of GNNs for knowledge graph completion\footnote{\url{https://github.com/DeepGraphLearning/NBFNet}}. Following \cite{zhu2021neural}, we augment each triplet with a flipped one of its inverse relation, so that the GNN can propagate information in both directions. The neural relation projection model is set to a 4-layer GNN model. We train the model with the self-adversarial negative sampling~\cite{sun2019rotate}. Note we only instantiate 1 GNN model and share it across all neural relation projections in the query. For query types that contain multiple relation projections in a chain (\emph{2p/3p/inp/pni/pin}), we observe very noisy gradients for the relation projections early in the chain. Therefore, we zero out the gradients of those relation projections, and only update the GNN with gradients from the last relation projections close to the loss. Our model is trained with Adam optimizer~\cite{kingma2015adam} on 4 Tesla V100 GPUs.

\smallskip \noindent \textbf{Baselines.} We compare GNN-QE against both embedding methods and neural-symbolic methods. The embedding methods include GQE~\cite{hamilton2018embedding}, Q2B~\cite{ren2020query2box}, BetaE~\cite{ren2020beta}, FuzzQE~\cite{chen2022fuzzy} and ConE~\cite{zhang2021cone}. The neural-symbolic methods include CQD-CO~\cite{arakelyan2021complex} and CQD-Beam~\cite{arakelyan2021complex}. For CQD-CO and CQD-Beam, we obtain their performance using the codebase\footnotemark provided by the original authors.

\begin{table*}[!h]
    \centering
    \caption[Test results on answering FOL queries]{Test MRR results (\%) on answering FOL queries. avg$_p$ is the average MRR on EPFO queries ($\land$, $\lor$). avg$_n$ is the average MRR on queries with negation. Results of GQE and Q2B are taken from \cite{ren2020beta}. Results of BetaE, FuzzQE and ConE are taken from their original papers~\cite{ren2020beta, chen2022fuzzy, zhang2021cone}.}
    \begin{adjustbox}{width=\textwidth}
    \begin{tabular}{lcccccccccccccccc}
        \toprule
        \bf{Model} & \bf{avg$_p$} & \bf{avg$_n$} & \bf{1p} & \bf{2p} & \bf{3p} & \bf{2i} & \bf{3i} & \bf{pi} & \bf{ip} & \bf{2u} & \bf{up} & \bf{2in} & \bf{3in} & \bf{inp} & \bf{pin} & \bf{pni} \\
        \midrule
        \multicolumn{17}{c}{FB15k} \\
        \midrule
        GQE & 28.0 & - & 54.6 & 15.3 & 10.8 & 39.7 & 51.4 & 27.6 & 19.1 & 22.1 & 11.6 & - & - & - & - & - \\
        Q2B & 38.0 & - & 68.0 & 21.0 & 14.2 & 55.1 & 66.5 & 39.4 & 26.1 & 35.1 & 16.7 & - & - & - & - & - \\
        BetaE & 41.6 & 11.8 & 65.1 & 25.7 & 24.7 & 55.8 & 66.5 & 43.9 & 28.1 & 40.1 & 25.2 & 14.3 & 14.7 & 11.5 & 6.5 & 12.4 \\
        CQD-CO & 46.9 & - & \bf{89.2} & 25.3 & 13.4 & 74.4 & 78.3 & 44.1 & 33.2 & 41.8 & 21.9 & - & - & - & - & - \\
        CQD-Beam & 58.2 & - & \bf{89.2} & 54.3 & 28.6 & 74.4 & 78.3 & 58.2 & 67.7 & 42.4 & 30.9 & - & - & - & - & - \\
        ConE & 49.8 & 14.8 & 73.3 & 33.8 & 29.2 & 64.4 & 73.7 & 50.9 & 35.7 & 55.7 & 31.4 & 17.9 & 18.7 & 12.5 & 9.8 & 15.1 \\
        \midrule
        GNN-QE & \bf{72.8} & \bf{38.6} & 88.5 & \bf{69.3} & \bf{58.7} & \bf{79.7} & \bf{83.5} & \bf{69.9} & \bf{70.4} & \bf{74.1} & \bf{61.0} & \bf{44.7} & \bf{41.7} & \bf{42.0} & \bf{30.1} & \bf{34.3} \\
        \midrule[0.08em]
        \multicolumn{17}{c}{FB15k-237} \\
        \midrule
        GQE & 16.3 & - & 35.0 & 7.2 & 5.3 & 23.3 & 34.6 & 16.5 & 10.7 & 8.2 & 5.7 & - & - & - & - & - \\
        Q2B & 20.1 & - & 40.6 & 9.4 & 6.8 & 29.5 & 42.3 & 21.2 & 12.6 & 11.3 & 7.6 & - & - & - & - & - \\
        BetaE & 20.9 & 5.5 & 39.0 & 10.9 & 10.0 & 28.8 & 42.5 & 22.4 & 12.6 & 12.4 & 9.7 & 5.1 & 7.9 & 7.4 & 3.5 & 3.4 \\
        CQD-CO & 21.8 & - & \bf{46.7} & 9.5 & 6.3 & 31.2 & 40.6 & 23.6 & 16.0 & 14.5 & 8.2 & - & - & - & - & - \\
        CQD-Beam & 22.3 & - & \bf{46.7} & 11.6 & 8.0 & 31.2 & 40.6 & 21.2 & 18.7 & 14.6 & 8.4 & - & - & - & - & - \\
        FuzzQE & 24.0 & 7.8 & 42.8 & 12.9 & 10.3 & 33.3 & 46.9 & 26.9 & 17.8 & 14.6 & 10.3 & 8.5 & 11.6 & 7.8 & 5.2 & 5.8 \\
        ConE & 23.4 & 5.9 & 41.8 & 12.8 & 11.0 & 32.6 & 47.3 & 25.5 & 14.0 & 14.5 & 10.8 & 5.4 & 8.6 & 7.8 & 4.0 & 3.6 \\
        \midrule
        GNN-QE & \bf{26.8} & \bf{10.2} & 42.8 & \bf{14.7} & \bf{11.8} & \bf{38.3} & \bf{54.1} & \bf{31.1} & \bf{18.9} & \bf{16.2} & \bf{13.4} & \bf{10.0} & \bf{16.8} & \bf{9.3} & \bf{7.2} & \bf{7.8} \\
        \midrule[0.08em]
        \multicolumn{17}{c}{NELL995} \\
        \midrule
        GQE & 18.6 & - & 32.8 & 11.9 & 9.6 & 27.5 & 35.2 & 18.4 & 14.4 & 8.5 & 8.8 & - & - & - & - & - \\
        Q2B & 22.9 & - & 42.2 & 14.0 & 11.2 & 33.3 & 44.5 & 22.4 & 16.8 & 11.3 & 10.3 & - & - & - & - & - \\
        BetaE & 24.6 & 5.9 & 53.0 & 13.0 & 11.4 & 37.6 & 47.5 & 24.1 & 14.3 & 12.2 & 8.5 & 5.1 & 7.8 & 10.0 & 3.1 & 3.5 \\
        CQD-CO & \bf{28.8} & - & \bf{60.4} & 17.8 & 12.7 & 39.3 & 46.6 & 30.1 & 22.0 & 17.3 & \bf{13.2} & - & - & - & - & - \\
        CQD-Beam & 28.6 & - & \bf{60.4} & \bf{20.6} & 11.6 & 39.3 & 46.6 & 25.4 & \bf{23.9} & \bf{17.5} & 12.2 & - & - & - & - & - \\
        FuzzQE & 27.0 & 7.8 & 47.4 & 17.2 & 14.6 & 39.5 & 49.2 & 26.2 & 20.6 & 15.3 & 12.6 & 7.8 & 9.8 & 11.1 & 4.9 & 5.5 \\
        ConE & 27.2 & 6.4 & 53.1 & 16.1 & 13.9 & 40.0 & 50.8 & 26.3 & 17.5 & 15.3 & 11.3 & 5.7 & 8.1 & 10.8 & 3.5 & 3.9 \\
        \midrule
        GNN-QE & \bf{28.9} & \bf{9.7} & 53.3 & 18.9 & \bf{14.9} & \bf{42.4} & \bf{52.5} & \bf{30.8} & 18.9 & 15.9 & 12.6 & \bf{9.9} & \bf{14.6} & \bf{11.4} & \bf{6.3} & \bf{6.3} \\
        \bottomrule
    \end{tabular}
    \end{adjustbox}
    \label{tab:transductive_query}
\end{table*}

\begin{figure*}[!h]
    \centering
    \includegraphics[width=0.74\textwidth]{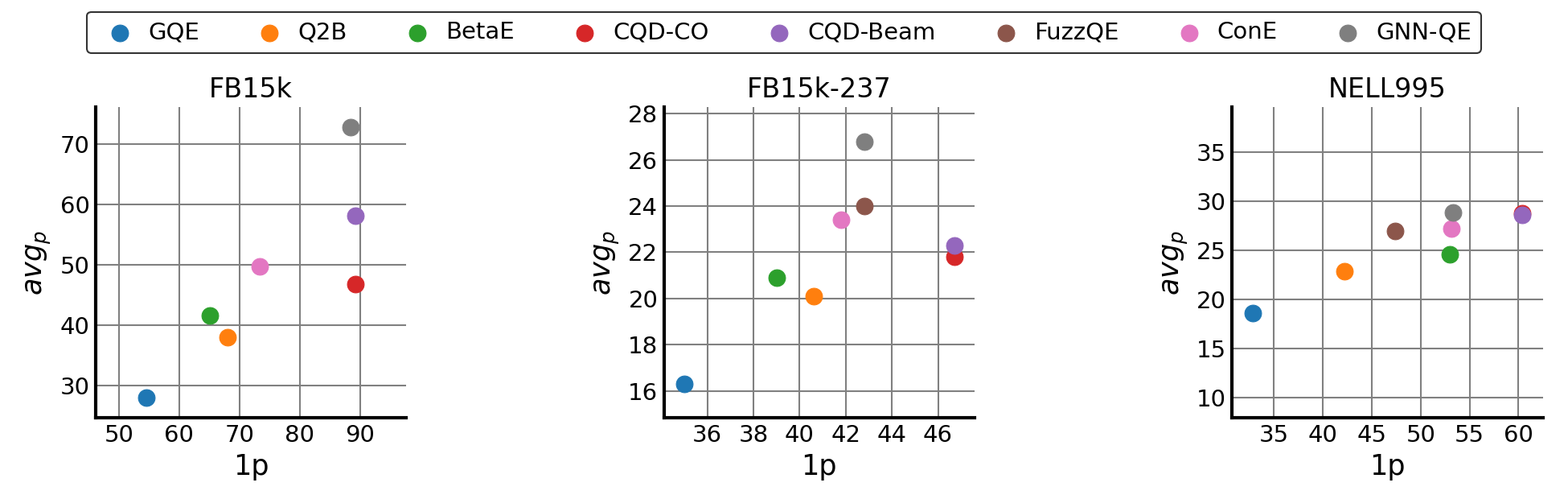}
    \caption[Results on EPFO queries w.r.t. results on knowledge graph completion]{MRR results on EPFO queries w.r.t.\ MRR results on knowledge graph completion (1p queries). Methods on the top left boundary of each plot generalize better from knowledge graph completion to EPFO queries. Best viewed in color.}
    \label{fig:generalization}
\end{figure*}
\footnotetext{\url{https://github.com/pminervini/KGReasoning}}

\subsection{Complex Query Answering}
\label{sec:main_experiment}

Table~\ref{tab:transductive_query} shows the MRR results of different models for answering FOL queries. GQE, Q2B, CQD-CO and CQD-Beam do not support queries with negation, so the corresponding entries are empty. We observe that GNN-QE achieves the best result for both EPFO queries and queries with negation on all 3 datasets. Notably, GNN-QE achieves an average relative gain of 22.3\% in avg$_p$ and 95.1\% in avg$_n$ compared to previous best model ConE. We attribute this gain to the advantage of fuzzy sets over geometric embeddings. Fuzzy sets can easily model intermediate variables with many possible assignments, while it is hard to embed a large number of entities in a low-dimensional vector. Such an advantage is especially useful for negation operations, since the output of a negation operation usually contains nearly $|\gV|$ entities.

Intuitively, the performance of complex query models should benefit from better knowledge graph completion performance, i.e.\ \emph{1p} queries. Here we disentangle the contribution of knowledge graph completion and complex query framework in answering EPFO queries. Figure~\ref{fig:generalization} plots the performance of EPFO queries w.r.t.\ the performance of knowledge graph completion on all datasets. Methods on the top-left corner of each plot show a better generalization from knowledge graph completion to EPFO queries, which implies their complex query frameworks are better. These include GQE, BetaE, FuzzQE, ConE and GNN-QE. By contrast, CQD-CO and CQD-Beam generalize worse than other methods, because they rely on a pretrained embedding model and cannot be trained for complex queries.

\subsection{Answer Set Cardinality Prediction}
\label{sec:cardinality}

\begin{table*}[t]
    \centering
    \caption[Results on answer set cardinality prediction]{MAPE (\%) of the number of answers predicted by GNN-QE. $avg$ is the average on all query types.}
    \label{tab:error}
    \footnotesize
    \begin{tabular}{lccccccccccccccc}
        \toprule
        \bf{Dataset} & \bf{avg} & \bf{1p} & \bf{2p} & \bf{3p} & \bf{2i} & \bf{3i} & \bf{pi} & \bf{ip} & \bf{2u} & \bf{up} & \bf{2in} & \bf{3in} & \bf{inp} & \bf{pin} & \bf{pni} \\
        \midrule
        FB15k & 37.1 & 34.4 & 29.7 & 34.7 & 39.1 & 57.3 & 47.8 & 34.6 & 13.5 & 26.5 & 31.4 & 50.3 & 50.3 & 39.4 & 29.8 \\
        FB15k-237 & 38.9 & 40.9 & 23.6 & 27.4 & 34.8 & 53.4 & 39.9 & 60.0 & 27.8 & 20.3 & 40.3 & 52.6 & 49.6 & 44.8 & 29.0 \\
        NELL995 & 44.0 & 61.9 & 38.2 & 47.1 & 56.6 & 72.3 & 49.5 & 45.8 & 19.9& 36.2& 30.0 & 47.0 & 42.3 & 39.8 & 29.4 \\
        \bottomrule
    \end{tabular}
\end{table*}

\begin{table*}[t]
    \centering
    \caption[Rank correlation between the model prediction and the number of answers]{Spearman's rank correlation between the model prediction and the number of ground truth answers on FB15k-237. $avg$ is the average correlation on all 12 query types in the table. Results of baselines are taken from \cite{zhang2021cone}.}
    \begin{adjustbox}{max width=\textwidth}
        \footnotesize
        \begin{tabular}{lccccccccccccc}
            \toprule
            \bf{Model} & \bf{avg} & \bf{1p} & \bf{2p} & \bf{3p} & \bf{2i} & \bf{3i} & \bf{pi} & \bf{ip} & \bf{2in} & \bf{3in} & \bf{inp} & \bf{pin} & \bf{pni} \\
            \midrule
            Q2B & - & 0.184 & 0.226 & 0.269 & 0.347 & 0.436 & 0.361 & 0.199 & - & - & - & - & - \\
            BetaE & 0.540 & 0.396 & 0.503 & 0.569 & 0.598 & 0.516 & 0.540 & 0.439 & 0.685 & 0.579 & 0.511 & 0.468 & 0.671 \\
            ConE & 0.738 & 0.70 & 0.71 & 0.74 & 0.82 & 0.72 & 0.70 & 0.62 & 0.90 & 0.83 & 0.66 & 0.57 & 0.88 \\
            \midrule
            GNN-QE & \bf{0.940} & \bf{0.948} & \bf{0.951} & \bf{0.895} & \bf{0.992} & \bf{0.970} & \bf{0.911} & \bf{0.937} & \bf{0.981} & \bf{0.968} & \bf{0.864} & \bf{0.880} & \bf{0.987} \\
            \bottomrule
        \end{tabular}
    \end{adjustbox}
    \label{tab:correlation}
\end{table*}

One advantage of GNN-QE is that it can predict the cardinality of the answer set (i.e.\ the number of answers) without explicit supervision. Specifically, the cardinality of a fuzzy set is computed as the sum of entity probabilities exceeding a certain threshold. We use 0.5 for the threshold as it is a natural choice for our binary classification loss (Equation~\ref{eqn:loss}). Table~\ref{tab:error} shows the mean absolute percentage error (MAPE) between our model prediction and the ground truth. Note none of existing methods can predict the number of answers without explicit supervision. \cite{ren2020beta} and \cite{zhang2021cone} observe that the uncertainty of Q2B, BetaE and ConE are positively correlated with the number of answers. We follow their setting and report the Spearman's rank correlation between our model prediction and the ground truth. As shown in Table~\ref{tab:correlation}, GNN-QE outperforms existing methods by a large margin on all query types.

\subsection{Intermediate Variables Visualization}
\label{sec:vis_gnn-qe}

\begin{table*}[t]
    \centering
    \caption[Visualization of a 3p query from FB15k-237]{Visualization of a \emph{3p} query from FB15k-237 test set.}
    \begin{minipage}{0.17\textwidth}
        \centering
        \includegraphics[width=\textwidth]{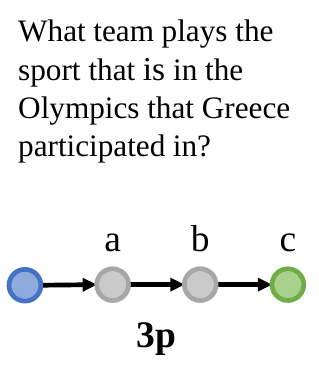}
    \end{minipage}
    \hspace{0.01\textwidth}
    \begin{minipage}{0.8\textwidth}
        \begin{adjustbox}{max width=\textwidth}
            \scriptsize
            \begin{tabular}{lccccc}
                \toprule
                \bf{Query} & \multicolumn{5}{l}{$q = ?c : \exists a, b: \text{ParticipateCountry}(a, \text{Greece}) \land \text{OlympicSports}(a, b) \land \text{TeamSports}(c, b)$} \\
                \midrule
                \multirow{2}{*}{\bf{Variable}} & \multicolumn{3}{c}{\bf{Top Predictions ($\geq 0.1$)}} & \bf{Random} & \bf{Filtered} \\
                & \bf{Easy} & \multicolumn{2}{c}{\bf{Hard}} & \bf{Ground Truth} & \bf{Ranking} \\
                \midrule
                a & \stack{\easy{1936 Summer Olympics} \\ \easy{1980 Winter Olympics} \\ \easy{2002 Winter Olympics}} & \stack{\tp{2010 Winter Olympics} \\ \fp{2012 Summer Olympics} \\ \fp{1920 Summer Olympics}} & \stack{\fp{1988 Summer Olympics} \\ \fp{1928 Summer Olympics} \\ \fp{1992 Summer Olympics}} & \stack{2010 Winter Olympics \\ (hard)} & 1 \\
                \midrule
                b & \stack{\easy{soccer} \\ \easy{track and field} \\ \easy{water polo}} & \stack{\tp{luge} \\ \tp{ice hockey} \\ \tp{short track speed skating}} & \stack{\fp{tennis} \\ - \\ - } & \stack{ice hockey \\ (hard)} & 1 \\
                \midrule
                c & \stack{\easy{Sacramento Kings} \\ \easy{Utah Jazz} \\ \easy{Seattle SuperSonics}} & \stack{\tp{Algeria soccer team} \\ \tp{Cincinnati Reds} \\ \tp{Washington Nationals}} & \stack{\tp{Chile soccer team} \\ \tp{Cardiff City} \\ \tp{Blackburn Rovers}} & \stack{Florida Panthers \\ (hard)} & 433 \\
                \bottomrule
            \end{tabular}
        \end{adjustbox}
    \end{minipage}
    \label{tab:vis_gnn-qe}
\end{table*}

Another advantage of GNN-QE is that we can interpret its reasoning process by investigating the intermediate variables. As the intermediate fuzzy sets may contain hundreds of entities, we consider two kinds of visualization to qualitatively analyze the precision and the recall of our model. The first one examines the entities with the top probabilities in each fuzzy set, and checks if they are an easy entity (i.e.\ those can be traversed on the training graph), a hard entity (i.e.\ those require reasoning) or a false positive one. For each fuzzy set, we visualize the top-3 easy entities and top-6 hard entities that have a minimum probability of 0.1. The second one draws a random ground truth assignment for each variable, such that the assignments form a valid grounding of the query and lead to a hard answer. We report the filtered ranking for each entity in the grounding.

Table~\ref{tab:vis_gnn-qe} shows the visualization of GNN-QE on a \emph{3p} query from FB15k-237 test set. Among the top hard entities, GNN-QE correctly predicts most of the intermediate entities, which indicates our method has a good precision for this sample. For the random ground truth assignments, GNN-QE recalls the first two hops (2010 Winter Olympics \& ice hockey) perfectly, but fails for the last hop. Such analysis would be beneficial to identify the steps where error occurs. 

\subsection{Ablation Studies}
\label{sec:ablation}

To provide a more comprehensive understanding of GNN-QE, we conduct three ablation studies on FB15k-237.

\smallskip \noindent \textbf{Traversal Dropout Probability $p$.} Figure~\ref{fig:dropout} shows the average MRR on EPFO queries of train and validation sets w.r.t.\ different probability $p$. The model can achieve a perfect training MRR of 1 when $p = 0$, which suggests that the model is able to learn the behavior of a relation traversal model. However, a relation traversal model cannot solve queries on incomplete graphs, which is revealed by its low performance on the validation set. With a non-zero probability $p$, traversal dropout makes the training problem more difficult, and enforces the model to learn a reasoning model that predicts the dropped link from its surrounding graph structure. However, it is not optimal to learn a fully reasoning model with $p = 1$, since it cannot perform relation traversal and some links in the validation queries can be perfectly solved by a relation traversal model.

\smallskip \noindent \textbf{Performance w.r.t.\ Number of Training Samples.}
Figure~\ref{fig:few_sample} plots the MRR curves of different query types in GNN-QE and BetaE under different number of training samples. It is observed that the performance of GNN-QE is not only better than BetaE, but also less sensitive to the number of training samples. Even with 1\% training samples (i.e.\ only 8,233 training queries for FB15k-237), GNN-QE achieves a comparative $avg_p$ and better $avg_n$ compared with BetaE trained with the full dataset. We conjecture the reason is that BetaE needs to learn a separate embedding for each entity, while our neural-symbolic method only learns relation embeddings (Equation~\ref{eqn:message}) for relation projection, which requires less samples to converge.

\begin{figure*}[t]
    \centering
    \begin{minipage}{0.49\textwidth}
        \centering
        \includegraphics[width=0.663\textwidth]{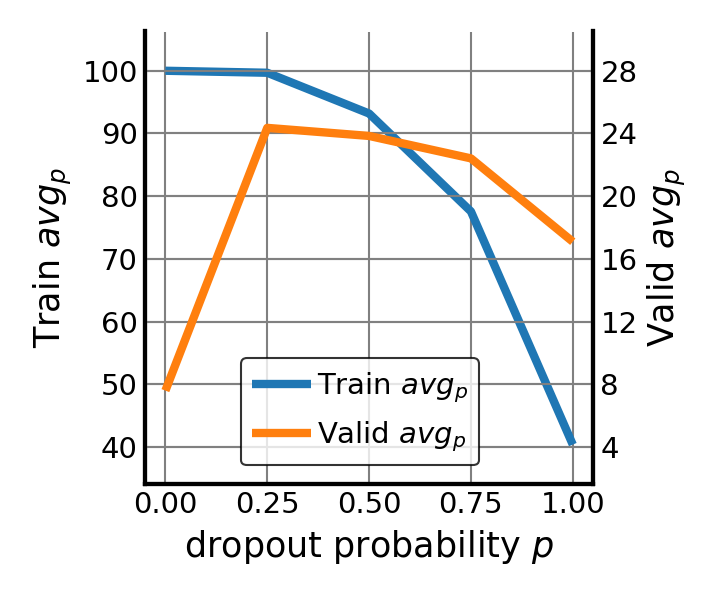}
        \caption[Results on EPFO queries w.r.t. traversal dropout probability]{Average MRR on EPFO queries (\%) of train / validation sets w.r.t. traversal dropout probability $p$. The best validation performance is achieved at $p = 0.25$.}
        \label{fig:dropout}
    \end{minipage}
    \hfill
    \begin{minipage}{0.48\textwidth}
        \centering
        \includegraphics[width=0.845\textwidth]{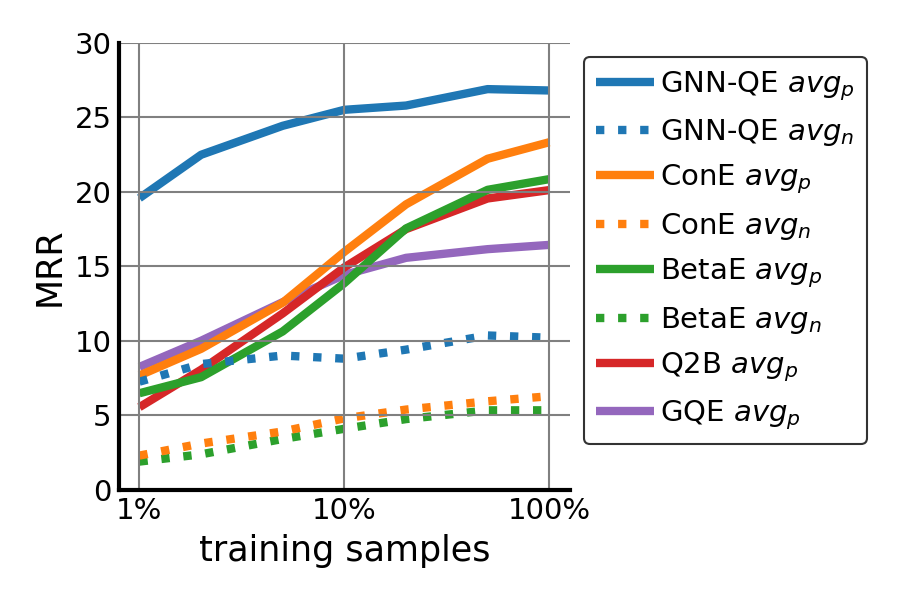}
        \caption[Test results w.r.t. the number of training samples]{Test MRR results w.r.t. number of training samples. GNN-QE is not only better than embeddings, but also less sensitive to the number of training samples.}
        \label{fig:few_sample}
    \end{minipage}
\end{figure*}

\begin{wraptable}{r}{0.38\textwidth}
    \vspace{-1.5em}
    \centering
    \caption[Test results w.r.t. GNN models]{Test MRR results (\%) w.r.t.\ GNN models. GNN-QE benefits from better GNN models.}
    \footnotesize
    \begin{tabular}{lcc}
        \toprule
        \bf{Model} & \bf{avg$_p$} & \bf{avg$_n$} \\
        \midrule
        BetaE   & 20.9 & 5.5 \\
        \midrule
        GNN-QE (RGCN)    & 20.9 & 7.3 \\
        GNN-QE (CompGCN) & 22.5 & 7.3 \\
        GNN-QE (NBFNet)  & \bf{26.8} & \bf{10.2} \\
        \bottomrule
    \end{tabular}
    \label{tab:gnn}
\end{wraptable}

\smallskip \noindent \textbf{GNN Parameterization.}
Table~\ref{tab:gnn} shows the MRR results of GNN-QE w.r.t.\ different GNN parameterizations. We consider three parameterizations for the \textsc{Message} and \textsc{Aggregate} functions in Equation~\ref{eqn:bellman}, namely RGCN~\cite{schlichtkrull2018modeling}, CompGCN~\cite{vashishth2020composition} and NBFNet~\cite{zhu2021neural}. It is observed that all three parameterizations outperform BetaE with significant improvement on $avg_n$, which suggests the advantages of fuzzy sets in modeling negation queries. Besides, GNN-QE benefits from stronger GNN models (NBFNet $>$ CompGCN $>$ RGCN). The performance of GNN-QE might be further improved with better GNN models.

\section{Inductive Experiments of GNN-QE}

In this section, we further evaluate GNN-QE on a suite of inductive datasets we construct for complex queries. We show that GNN-QE is able to: (1) answer complex logical queries over new unseen entities at inference time; (2) predict new correct answers for known \emph{training} queries when executed over larger inference graphs; (3) consistently ranks easy answers (i.e.\ those only requiring traversing existing edges) higher than hard answers (i.e.\ those requiring link prediction) in its prediction.

\subsection{Setup and Datasets}
\label{sec:dataset}

\smallskip \noindent \textbf{Datasets.}
Due to the absence of inductive logical query benchmarks, we create a novel suite of datasets based on FB15k-237~\cite{toutanova2015observed} (open license) and following the query generation process of BetaE~\cite{ren2020beta}. Given a source graph with $\gE$ entities, we sample $|\gE_{\textit{train}}| = r \cdot |\gE|, r \in [0.1, 0.9]$ nodes to induce a training graph $\gG_{\textit{train}}$. For validation and test graphs, we split the remaining set of entities into two non-overlapping sets each with $\frac{1-r}{2}|\gE|$ nodes. We then merge training and unseen nodes into the inference set of nodes $\gE_{\textit{inf}}$ and induce inference graphs for validation and test from those sets, respectively, i.e.\ $\gE_{\textit{inf}}^{\textit{val}} = \gE_{\textit{train}} \cup \gE_{\textit{val}}$ and $\gE_{\textit{inf}}^{\textit{test}} = \gE_{\textit{train}} \cup \gE_{\textit{test}}$. That is, validation and test inference graphs both extend the training graph but their sets of new entities are disjoint. Finally, we sample and remove 15\% of edges $\gT_{\text{pred}}$ in the inference graphs as missing edges for sampling queries with those missing edges. Overall, we sample 9 such datasets based on different choices of $r$, which result in the ratios of inference graph size to the training graph $\gE_{\textit{inf}} / \gE_{\textit{train}}$ from 106\% to 550\%.

For each dataset, we employ the query sampler from BetaE~\cite{ren2020beta} to extract 14 typical query types \emph{1p/2p/3p/2i/3i/ip/pi/2u/up/2in/3in/inp/pin/pni}. Training queries are sampled from the training graph $\gG_{\textit{train}}$, validation and test queries are sampled from their respective inference graphs $\gG_{\textit{inf}}$ where at least one edge belongs to $\gT_{\text{pred}}$ and has to be predicted at inference time.

As inference graphs extend training graphs, training queries are very likely to have new answers when executed over $\gG_{\textit{inf}}$ with simple graph traversal and without any link prediction. We create an additional set of true answers for all training queries executed over the test inference graph $\gG_{\textit{inf}}^{\textit{test}}$ to measure the entailment capabilities of query answering models. This is designed to be an inference task and extends the \emph{faithfullness} evaluation of \cite{sun2020faithful}. Dataset statistics can be found in Section~\ref{app:dataset_gnn-qe}.

\smallskip \noindent \textbf{Evaluation Protocol.}
Following the literature~\cite{ren2020beta}, query answers are separated into two sets: \emph{easy answers} that only require graph traversal over existing edges, and \emph{hard answers} that require inferring missing links to achieve the answer node. For the main experiment, evaluation involves ranking of \emph{hard} answers against all entities having easy ones filtered out. For evaluating training queries on inference graphs, we only have \emph{easy} answers and rank them against all entities. We report Hits@10 as the main performance metric on different query types.

\smallskip \noindent \textbf{Baselines.}
Since there do not exist inductive methods for answering complex queries, we create a baseline, NodePiece-QE, by combining CQD-Beam~\cite{arakelyan2021complex} with the inductive node representations learned by NodePiece~\cite{galkin2022nodepiece}. Specifically, NodePiece encodes each entity as a function of its incident relations, and is optimized by the objective and score function of ComplEx~\cite{trouillon2016complex}. The encoder can be either a plain MLP (denoted as NodePiece-QE), or a relational GNN~\cite{vashishth2020composition} aggregating the neighborhood of each entity (denoted as NodePiece-QE w/ GNN). We apply the learned encoder to materialize representations of entities in the inference graph $\mE \in \sR^{|\gE_{\textit{inf}}| \times d}$ and send them to CQD-Beam that decodes answers of complex FOL queries based on the representations.

Additionally, we consider an Edge-type Heuristic baseline, which finds all entities $e \in \gE$ that satisfy the relations in the last hop of $\gR_Q$ on the inference graph $\gG_{\textit{inf}}$. Intuitively, this baseline filters out entities that are not consistent with the query according to edge types, which is a necessary condition for the answers when the inference graph is reasonably dense. We use Edge-type Heuristic to show that inductive models learn non-trivial representations for complex FOL queries.

\subsection{Inductive Complex Query Answering}
\label{sec:test_queries}

Among 9 datasets with different ratios of the inference graph size to the training graph, we use the dataset with $\gE_{\textit{inf}} / \gE_{\textit{train}} = 175\%$ as a reference. Table~\ref{tab:inductive_query} summarizes the results on the reference dataset, while Figure~\ref{fig:main_exp} illustrates a bigger picture on all datasets. We observe that Edge-type Heuristic is able to attain 10.1\% for Hits@10, suggesting that some test queries may be easily answered by edge types. GNN-QE significantly outperforms both Edge-type Heuristic and NodePiece-QE by a large margin, and is also capable of answering negation queries that cannot be handled by CQD-Beam in NodePiece-QE.

Figure~\ref{fig:main_exp} shows a decreasing trend in the performance of both NodePiece-QE w/ GNN and GNN-QE as the inference graph grows and has more unseen entities. Both approaches achieve their best results at $\gE_{\textit{inf}} / \gE_{\textit{train}}$ ratio around 130\%, but their performance steadily deteriorate by up to 20 absolute Hits@10 points on EPFO queries and negation queries when the ratio grows to 550\%. We attribute this deterioration to a known generalization issue~\cite{knyazev2019understanding, yehudai2021local} of GNNs when performing inference over a larger graph than the model has seen during training. Recently, a few strategies have been proposed~\cite{buffelli2022sizeshiftreg, zhou2022ood} to alleviate this issue and we see it as a promising avenue for future work.

\begin{figure}[t]
    \centering
    \includegraphics[width=\textwidth]{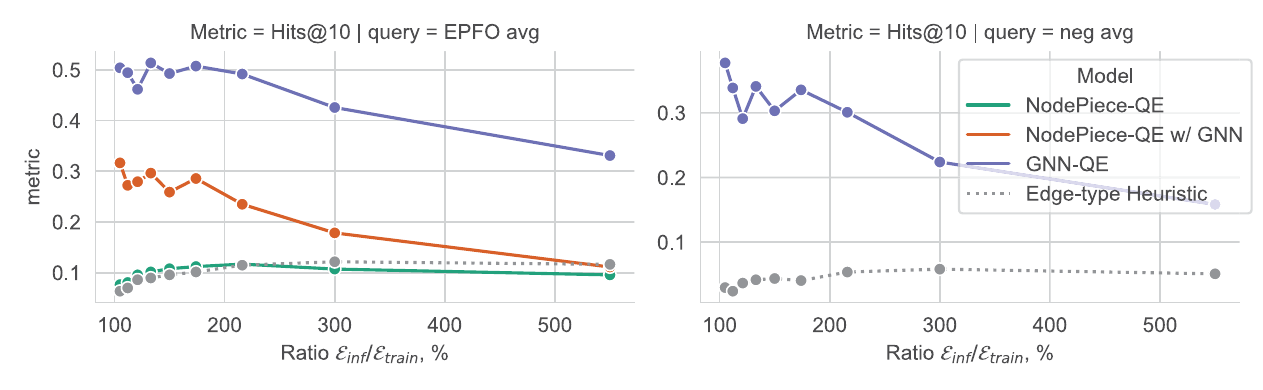}
    \caption[Performance of test queries on inference graphs of different ratios]{Aggregated Hits@10 performance of \textbf{test queries} (involving unseen entities) executed on inference graphs of different ratios compared to training graphs. Note that NodePiece-QE only supports EPFO queries but not negation queries.}
    \label{fig:main_exp}
\end{figure}

\begin{table}[t]
    \centering
    \caption[Test results of answering inductive FOL queries on the reference dataset]{Test Hits@10 results (\%) of answering inductive FOL queries when $\gE_{\textit{inf}} / \gE_{\textit{train}} = 175\%$. avg$_p$ is the average on EPFO queries ($\land$, $\lor$). avg$_n$ is the average on queries with negation.}
    \begin{adjustbox}{width=\textwidth}
    \begin{tabular}{lrrrrrrrrrrrrrrrr}
        \toprule
        \bf{Model} & \bf{avg$_p$} & \bf{avg$_n$} & \bf{1p} & \bf{2p} & \bf{3p} & \bf{2i} & \bf{3i} & \bf{pi} & \bf{ip} & \bf{2u} & \bf{up} & \bf{2in} & \bf{3in} & \bf{inp} & \bf{pin} & \bf{pni} \\
        \midrule
        Edge-type Heuristic & 10.1 & 4.1 & 17.7 & 8.2 & 9.9 & 10.7 & 13.0 & 9.8 & 8.2 & 5.3 & 8.5 & 2.6 & 2.9 & 8.4 & 3.8 & 2.7 \\
        NodePiece-QE & 11.2 & - & 25.5 & 8.2 & 8.4 & 12.4 & 13.9 & 9.9 & 8.7 & 7.0 & 6.8 & - & - & - & - & - \\
        NodePiece-QE w/ GNN & 28.6 & - & 45.9 & 19.2 & 11.5 & 39.9 & 48.8 & 29.4 & 22.6 & 25.3 & 14.6 & - & - & - & - & - \\
        \midrule
        GNN-QE & 50.7 & 33.6 & 65.4 & 36.3 & 31.6 & 73.8 & 84.3 & 56.5 & 41.5 & 39.3 & 28.0 & 33.3 & 46.4 & 29.2 & 24.9 & 34.0 \\
        \bottomrule
    \end{tabular}
    \end{adjustbox}
    \label{tab:inductive_query}
\end{table}

\subsection{Faithfullness on Larger Inference Graphs}
\label{sec:train_queries}

Simulating the incremental addition of new edges in graph databases, we evaluate inductive models with \emph{training} queries on the larger inference graph, with a focus on predicting \emph{easy answers} that require performing only graph traversal without predicting missing links. Particular challenges arising when executing training queries over a larger graph are: (1) the same queries can have more correct answers as more new nodes and edges satisfying the query pattern might have been added; (2) more new entities create a ``distractor'' setting with more false positives. Generally, this setting can be considered as an inductive extension of the \emph{faithfullness} evaluation~\cite{sun2020faithful} that captures how well a neural model can answer original training queries.

Figure~\ref{fig:train_queries} demonstrates the performance of baseline methods and GNN-QE. For reference, we plot the performance for the training queries on the training graph to show the generalization gap. Compared to the baseline methods, GNN-QE fits the training query data almost perfectly, which confirms the finding that NBFNet is able to perform graph traversal like symbolic algorithms~\cite{zhu2021neural}. Additionally, GNN-QE is able to find most new correct answers on inference graphs, with slight decrease in its performance when the inference graph gets larger. We attribute this decrease to the \emph{distractor} factor caused by new entities and the generalization issue mentioned above for larger inference graphs.

\begin{figure}[t]
    \centering
    \includegraphics[width=\textwidth]{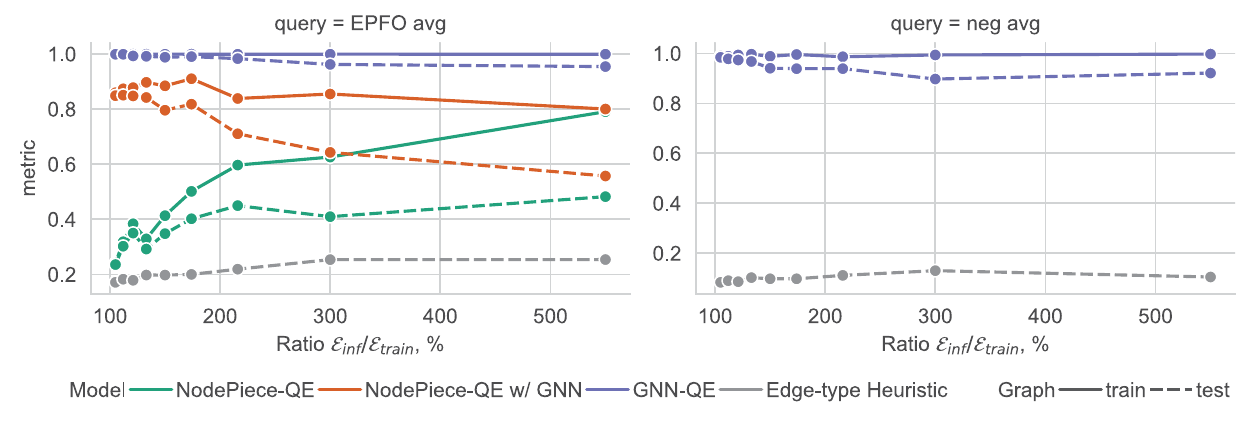}
    \caption[Performance of training queries on the training and inference graphs]{Aggregated Hits@10 performance of \textbf{training queries} on the original training and extended test inference graphs where queries have new correct answers. Note that NodePiece-QE only supports EPFO queries but not negation queries.}
    \label{fig:train_queries}
\end{figure}

\begin{table}[t]
    \centering
    \caption[Ranking of easy and hard answers on the reference dataset]{Macro-averaged AUROC score over \textbf{unfiltered} predictions on the reference $\gE_{\textit{inf}} / \gE_{\textit{train}} = 175\%$ dataset to measure if all easy answers are ranked higher than hard answers. Higher is better.}
    \begin{adjustbox}{width=\textwidth}
    \begin{tabular}{lcccccccccccccccc}
        \toprule
        \bf{Model} & \bf{avg$_p$} & \bf{avg$_n$} & \bf{1p} & \bf{2p} & \bf{3p} & \bf{2i} & \bf{3i} & \bf{pi} & \bf{ip} & \bf{2u} & \bf{up} & \bf{2in} & \bf{3in} & \bf{inp} & \bf{pin} & \bf{pni} \\
        \midrule
        NodePiece-QE & 0.692 & & 0.623 & 0.710 & 0.711 & 0.657 & 0.654 & 0.692 & 0.731 & 0.723 & 0.729 & - & - & - & - & - \\
        NodePiece-QE w/ GNN & 0.776 & & 0.783 & 0.783 & 0.739 & 0.758 & 0.733 & 0.760 & 0.801 & 0.841 & 0.787 & - & - & - & - & - \\
        \midrule
        GNN-QE & 0.973 & 0.885 & 0.998 & 0.992 & 0.986 & 0.969 & 0.962 & 0.967 & 0.969 & 0.938 & 0.978 & 0.879 & 0.859 & 0.926 & 0.914 & 0.847 \\
        \bottomrule
    \end{tabular}
    \end{adjustbox}
    \label{tab:auroc}
\end{table}

\subsection{Ranking of Easy and Hard Answers}
\label{sec:easy_vs_hard}

In addition to evaluating \emph{faithfullness} that measures how a model can recover easy answers, we also expect neural models have higher confidence in their predictions for easy answers than hard answers. To this end, we compute an AUROC metric over original \textbf{unfiltered} scores. The AUROC score measures how many hard answers are ranked \emph{after} easy answers. Note that the score only reflects the relative ranking between easy and hard answers, but does not depend on their actual ranking among all entities. Therefore, AUROC still needs to be paired with MRR to see how good these models are at predicting the correct answers.

We compute AUROC for each query and average them over each query type thus making it \textbf{macro-averaged AUROC}. Our experimental results on all query types using the models reported in Table~\ref{tab:inductive_query} on the reference 175\% dataset are compiled in Table~\ref{tab:auroc}. Compared to NodePiece with moderate AUROC scores, we observe that GNN-QE achieves nearly perfect AUROC scores, since GNN-QE aligns well with symbolic algorithms like subgraph matching.

\section{Method: UltraQuery}
\label{sec:method}

We aim at designing a single foundation model for complex FOL queries on any knowledge graph in the zero-shot fashion, i.e.\ without training on a target graph. In the complex query literature~\cite{hamilton2018embedding, ren2020query2box, ren2020beta, arakelyan2021complex, zhu2022neural}, it is common to break down query execution into a \emph{relation projection} to traverse graph edges and predict missing links, and \emph{logical operators} that model conjunction, disjunction, and union. The main challenge boils down to designing inductive projection and logical operators suitable for any entity and relation vocabulary.

\subsection{Inductive Relation Projection}
\label{subsec:ultra_proj}

The vast majority of complex query models are inherently transductive and implement relation projections as functions over entity and relation embeddings fixed to a certain knowledge graph vocabulary, e.g.\ with scoring functions from knowledge graph completion methods~\cite{hamilton2018embedding, arakelyan2021complex, bai2023answering}, geometric functions~\cite{ren2020query2box, zhang2021cone}, or pure neural methods~\cite{amayuelas2022neural, wang2023logical}. The only method inductive to new entities~\cite{zhu2022neural} learns relation embeddings and uses those as a labeling trick~\cite{zhang2021labeling} for a GNN that implements the projection operator.

As fixed relation embeddings do not transfer to new knowledge graphs with new relations, we adapt \textsc{Ultra}~\cite{galkin2024towards}, an inductive approach that builds relation representations dynamically using the invariance of \emph{relation interactions}, as the backbone of the relation projection operator thanks to its good zero-shot performance on simple knowledge graph completion tasks across a variety of graphs. \textsc{Ultra} leverages theoretical findings in multi-relational link prediction~\cite{barcelo2022weisfeiler, huang2024theory} and learns relation representations from a \emph{meta-graph} of relation interactions\footnote{The meta-graph can be efficiently obtained from any knowledge graph.}. The meta-graph includes four learnable edge types or meta-relations (\emph{head-to-tail}, \emph{tail-to-head}, \emph{head-to-head}, \emph{tail-to-tail}) which are independent from the relation vocabularies of knowledge graphs and therefore transfer across any graph. Practically, given a graph $\gG$ and projection query $(h, r, ?)$, \textsc{Ultra} employs labeling trick GNNs on two levels. First, it builds a meta-graph $\gG_r$ of relation interactions (a graph of relations where each node is a unique edge type in $\gG$) and applies a labeling trick to initialize the query node $r$. Running a message passing GNN over $\gG_r$ results in \emph{conditional relation representation} which are used as initial edge type features in the second, entity-level GNN. There, a starting node $h$ is initialized with a query vector from the obtained relation representations and running another GNN over the entity graph (with a final sigmoid readout) returns a scalar score in $[0, 1]$ representing a probability of each node to be a tail of a query $(h,r,?)$.

\begin{figure*}[t]
    \centering
    \includegraphics[width=\linewidth]{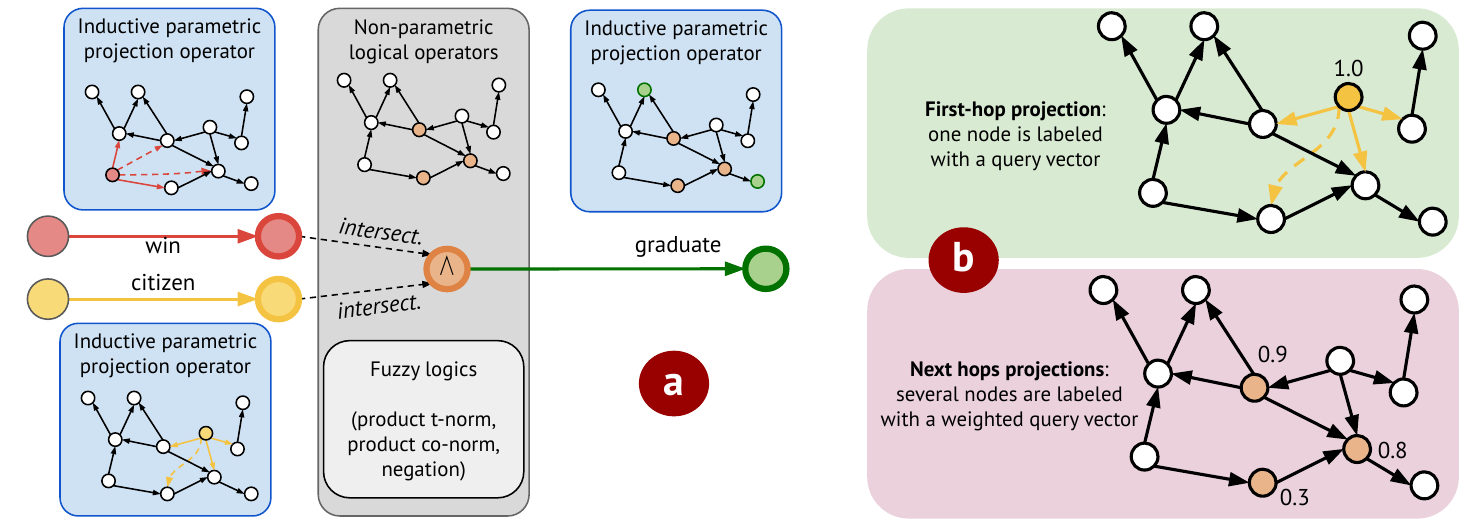}
    \caption[Overview of UltraQuery]{\textbf{(a)} Example of \emph{ip} query answering with \textsc{UltraQuery}: the inductive parametric projection operator (\Cref{subsec:ultra_proj}) executes relation projections on any graph and returns a scalar score for each entity; the scores are aggregated by non-parametric logical operators (\Cref{subsec:logic_ops}) implemented with fuzzy logics. Intermediate scores are used for weighted initializion of relation projections on the next hop. \textbf{(b)} The multi-source propagation issue with a pre-trained link predictor for relation projection: pre-training on  \emph{1p} link prediction is done in the single-source labeling mode (top) where only one query node is labeled with a non-zero vector; complex queries at later intermediate hops might have several plausible sources with non-zero initial weights (bottom) where a pre-trained operator fails. }
    \label{fig:ultraquery}
\end{figure*}

\autoref{fig:ultraquery}(a) illustrates the \emph{intersection-projection} query execution process where each projection step is tackled by the same inductive projection operator with initialization depending on the start anchor node or intermediate variables.

\smallskip \noindent \textbf{The multi-source propagation issue.}
While it is tempting to leverage \textsc{Ultra} pre-trained on multiple knowledge graph datasets for relation projection, there is a substantial distribution shift (\autoref{fig:ultraquery}(b)) between knowledge graph completion and complex queries. Specifically, knowledge graph completion is a special case of relation projection where the input always contains a single node. By comparison, in multi-hop complex queries, several likely nodes might have high intermediate scores and will be labeled with non-zero vectors leading to the \emph{multiple sources} propagation mode where a pre-trained operator is likely to fail. To alleviate the issue, we experimentally study two strategies: (1) short fine-tuning  of the pre-trained projection operator on complex queries (used in the main \textsc{UltraQuery} model), or (2) use the frozen pre-trained operator and threshold intermediate scores setting all scores below $0 < k < 1$ to zero (denoted as \textsc{UltraQuery LP}). The insight is to limit the propagation to one or a few source nodes, thereby reducing the discrepancy between training and test distributions.

\subsection{Inductive Logical Operations}
\label{subsec:logic_ops}
Learnable logical operators parameterized by neural nets in many complex query approaches~\cite{hamilton2018embedding, ren2020query2box, zhang2021cone, amayuelas2022neural} fit a particular embedding space and are not transferable. Instead, we resort to differentiable but non-parametric \emph{fuzzy logics}~\cite{van2022analyzing} that implement logical operators as algebraic operations (\emph{t-norms} for conjunction and \emph{t-conorms} for disjunction) in a bounded space $[0,1]$ and are used in several neuro-symbolic complex query approaches~\cite{arakelyan2021complex, zhu2022neural, arakelyan2024adapting, bai2023answering, yin2024rethinking}. \textsc{UltraQuery} employs fuzzy logical operators over \emph{fuzzy sets} $\vx \in [0,1]^{|\gV|}$ as the relation projection operator assigns a scalar in range $[0,1]$ for each entity in a graph. The choice of a fuzzy logic is often a hyperparameter although \cite{van2022analyzing} shows that the \emph{product logic} is the most stable. In product logic, given two fuzzy sets $\vx, \vy$, conjunction is element-wise multiplication $\vx \odot \vy$ and disjunction is $\vx + \vy - \vx \odot \vy$. Negation is often implemented as $\mathbf{1}-\vx$ where $\mathbf{1}$ is the \emph{universe} vector of all ones. For second- and later $i$-th hop projections, we obtain initial node states $\vh_v$ by weighting a query vector $\vr_i$ with their probability score $x_v$ from the fuzzy set of a previous step: $\vh_v = x_v \vr_i$.

\subsection{Training}
Following existing works~\cite{ren2020beta, zhu2022neural}, \textsc{UltraQuery} is trained on complex queries to minimize the binary cross entropy loss
\begin{align}
\hspace{-0.5em}\displaystyle{
    \gL = -\frac{1}{|\gA_{q}|}\sum_{a \in \gA_{q}}\log p(a|q) 
          -\frac{1}{|\gV\backslash\gA_{q}|}\sum_{a' \in \gV\backslash\gA_{q}}\log (1 - p(a'|q))}
    \label{eqn:loss}
\end{align}
where $\gA_{q}$ is the answer to the query $q$ and $p(a|q)$ is the probability of entity $a$ in the final output fuzzy set. \textsc{UltraQuery LP} uses a frozen checkpoint from knowledge graph completion and is not trained on complex logical queries.
\section{Experiments of UltraQuery}
\label{sec:experiments}

Our experiments focus on the following research questions: (1) How does a single \textsc{UltraQuery} model perform in the zero-shot inference mode on unseen graphs and queries compared to the baselines? (2) Does \textsc{UltraQuery} retain the quality metrics like \emph{faithfullness} and identify easy answers reachable by traversal? (3) How does the multi-source propagation issue affect the performance?

\subsection{Setup and Datasets}

\smallskip \noindent \textbf{Datasets.}
We employ 23 different complex query datasets each with 14 standard query types and its own underlying knowledge graph with different sets of entities and relations. We categorize the datasets into three groups (more statistics of the datasets and queries are provided in Section~\ref{app:dataset_gnn-qe}):
\begin{itemize}[label=$\bullet$, leftmargin=*]
    \item \emph{Transductive} (3 datasets) where training and inference graphs are the same $(\gtrain = \ginf)$ and test queries cover the same set of entities and relations: FB15k-237, NELL995 and FB15k all from \cite{ren2020beta} with at most 100 answers per query.
    \item \emph{Inductive entity} $(e)$ (9 datasets) from \cite{galkin2022inductive} where inference graphs extend training graphs $(\gtrain \subset \ginf)$ being up to 550\% larger in the number of entities. The set of relations is fixed in each training graph and does not change at inference making the setup inductive with respect to the entities. Training queries might have more true answers in the extended inference graph.
    \item \emph{Inductive entity and relation} $(e,r)$ (11 datasets): we sampled a novel suite of WikiTopics-CLQA datasets due to the absence of standard benchmarks evaluating the hardest inductive setup where inference graphs have both new entities and relations $(\gtrain \neq \ginf)$. The source graphs were adopted from the WikiTopics datasets~\cite{gao2023double}, we follow the \emph{BetaE setting} when sampling 14 query types with at most 100 answers.
\end{itemize} 

\smallskip \noindent \textbf{Implementation and Training.}
\textsc{UltraQuery} was trained on one FB15k-237 dataset with complex queries for 10,000 steps with batch size of 32 on 4 RTX 3090 GPUs for 2 hours (8 GPU-hours in total). We initialize the model weights with an available checkpoint of \textsc{Ultra} reported in \cite{galkin2024towards}. Following the standard setup in the literature, we train the model on 10 query types and evaluate on all 14 patterns. We employ \emph{product t-norm} and \emph{t-conorm} as non-parametric fuzzy logic operators to implement conjunction $(\wedge)$ and disjunction $(\lor)$, respectively, and use a simple $1-x$ negation. For the ablation study, \textsc{UltraQuery LP} uses the same frozen checkpoint (pre-trained on simple \emph{1p} link prediction) with scores thresholding to alleviate the multi-source propagation issue (\Cref{subsec:ultra_proj}).

\smallskip \noindent \textbf{Evaluation Protocol.}
As we train an \textsc{UltraQuery} model only on one FB15k-237 dataset and run zero-shot inference on other 22 graphs, the inference mode on those is \emph{inductive} $(e,r)$ since their entity and relation vocabularies are all different from the training set.

As common in the literature~\cite{ren2020beta, ren2023neural}, the answer set of each query is split into \emph{easy} and \emph{hard} answers. Easy answers are reachable by graph traversal and do not require inferring missing links whereas hard answers are those that involve at least one edge to be predicted at inference. In the rank-based evaluation, we only consider ranks of \emph{hard} answers and filter out easy ones and report filtered Mean Reciprocal Rank (MRR) and Hits@10 as main performance metrics.

Other qualitative metrics include: (1) \emph{faithfullness}~\cite{sun2020faithful}, i.e.\ the ability to recover \emph{easy} answers reachable by graph traversal. Here, we follow the setup in \cite{galkin2022inductive} and measure the performance of training queries on larger inference graphs where the same queries might have new true answers; (2) the AUROC score to estimate whether a model ranks easy answers higher than hard answers -- we compute AUROC over \emph{unfiltered} scores of easy answers as positive labels and hard answers as negative. (3) Mean Absolute Percentage Error (MAPE)~\cite{zhu2022neural} between the number of answers extracted from model's predictions and the number of ground truth answers (easy and hard combined) to estimate whether complex query models can predict the cardinality of the answer set.

\smallskip \noindent \textbf{Baselines.}
In transductive and inductive $(e)$ datasets, we compare a single \textsc{UltraQuery} model with the best reported models trained end-to-end on each graph (denoted as \emph{Best baseline} in the experiments): QTO~\cite{bai2023answering} for 3 transductive datasets (FB15k-237, FB15k, and NELL995) and GNN-QE~\cite{galkin2022inductive} for 9 inductive $(e)$ datasets. While a single \textsc{UltraQuery} model has 177k parameters, the baselines are several orders of magnitude larger with a parameters count depending on the number of entities and relations, e.g.\ a QTO model on FB15k-237 has 30M parameters due to having 2000$d$ entity and relation embeddings, and GNN-QE on a reference FB 175\% inductive $(e)$ dataset has 2M parameters. For a newly sampled suite of 11 inductive $(e, r)$ datasets, we compare against the edge-type heuristic baseline introduced in \cite{galkin2022inductive}. The heuristic selects the candidate nodes with the same incoming relation as the last hop of the query.

\subsection{Zero-shot Query Answering}

Here we measure the zero-shot query answering performance of \textsc{UltraQuery} trained on a fraction of complex queries of one FB15k-237 dataset. Figure~\ref{fig:main_fig1} and Table~\ref{tab:maintab1} illustrate the comparison with the best available baselines and ablated \textsc{UltraQuery LP} model on 23 datasets split into three categories (transductive, inductive $(e)$, and inductive $(e,r)$). For each dataset, we measure the average MRR on 9 EPFO queries with projection, intersection, and union operators, and 5 negation queries with the negation operator, respectively.

\begin{figure}[!t]
    \centering
    \includegraphics[width=\linewidth]{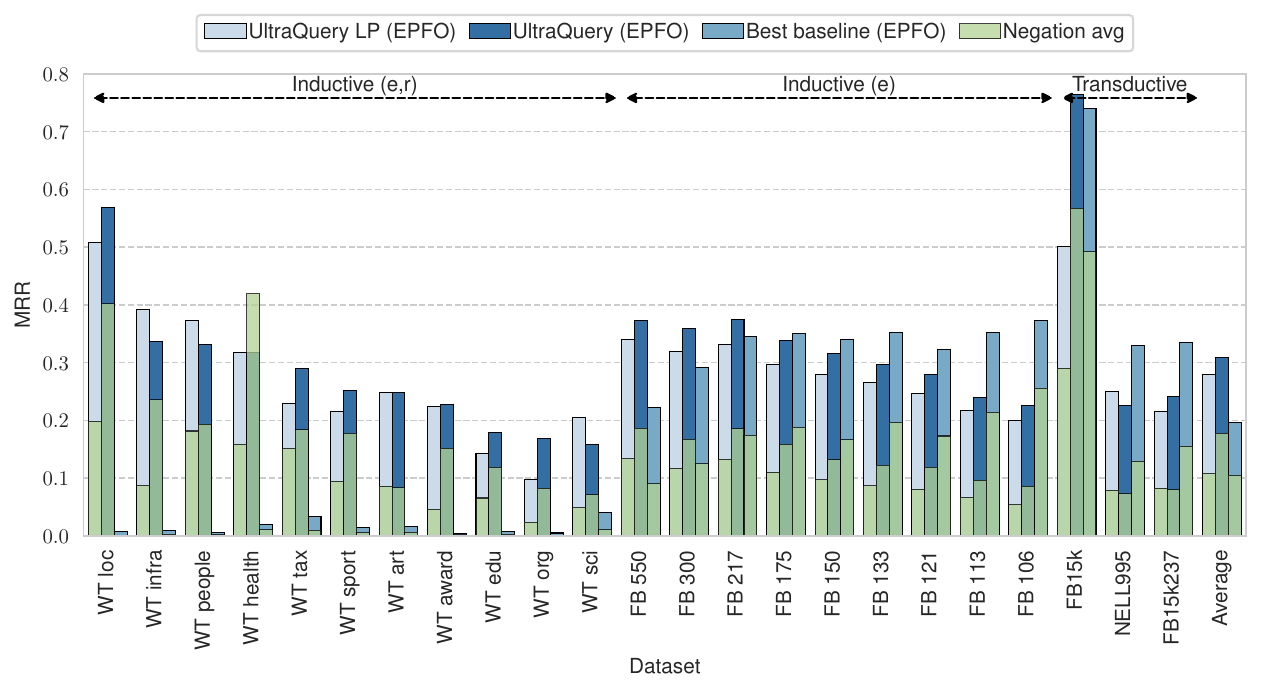}
    \caption[Zero-shot query answering performance of UltraQuery on 23 datasets]{Zero-shot query answering performance (MRR, higher is better) of a single \textsc{UltraQuery} model trained on one FB15k237 queries dataset compared to the best available baselines and ablated \textsc{UltraQuery LP} on 23 datasets. \emph{EPFO} is the average of 9 query types with $(\wedge, \lor)$ operators, \emph{Negation} is the average of 5 query types with the negation operator $(\neg)$. On average, a single \textsc{UltraQuery} model outperforms the best baselines trained specifically on each dataset.}
    \label{fig:main_fig1}
\end{figure}

\begin{table*}[!t]
    \centering
    \caption[Zero-shot results of UltraQuery and UltraQuery LP on 23 datasets]{Zero-shot inference results of \textsc{UltraQuery} and ablated \textsc{UltraQuery LP} on 23 datasets compared to the best reported baselines. \textsc{UltraQuery} was trained on one transductive FB15k-237 dataset, \textsc{UltraQuery LP} was only pre-trained on knowledge graph completion and uses scores thresholding. The \emph{no thrs.}\ version does not use any thresholding of intermediate scores (\Cref{subsec:ultra_proj}). The best baselines are trainable on each transductive and inductive $(e)$ dataset, and the non-parametric heuristic baseline on inductive $(e,r)$ datasets.}
    \begin{adjustbox}{width=\textwidth}
        \begin{tabular}{lrrrrrrrrrrrrrrrrr}
            \toprule
            \multirow{3}{*}{\bf{Model}} &\multicolumn{4}{c}{\bf{Inductive} $(e,r)$ (11 datasets)} &\multicolumn{4}{c}{\bf{Inductive} $(e)$ (9 datasets)} &\multicolumn{4}{c}{\bf{Transductive} (3 datasets)} &\multicolumn{4}{c}{\bf{Total Average} (23 datasets)} \\\cmidrule(l){2-5} \cmidrule(l){6-9} \cmidrule(l){10-13} \cmidrule(l){14-17}
            &\multicolumn{2}{c}{EPFO avg} &\multicolumn{2}{c}{neg avg} &\multicolumn{2}{c}{EPFO avg} &\multicolumn{2}{c}{neg avg} &\multicolumn{2}{c}{EPFO avg} &\multicolumn{2}{c}{neg avg} &\multicolumn{2}{c}{EPFO avg} &\multicolumn{2}{c}{neg avg} \\\cmidrule(l){2-3} \cmidrule(l){4-5} \cmidrule(l){6-7} \cmidrule(l){8-9} \cmidrule(l){10-11} \cmidrule(l){12-13} \cmidrule(l){14-15} \cmidrule(l){16-17}
            &\bf{MRR} &\bf{H@10} &\bf{MRR} &\bf{H@10} &\bf{MRR} &\bf{H@10} &\bf{MRR} &\bf{H@10} &\bf{MRR} &\bf{H@10} &\bf{MRR} &\bf{H@10} &\bf{MRR} &\bf{H@10} &\bf{MRR} &\bf{H@10} \\\midrule
            Best baseline &0.014 &0.029 &0.004 &0.007 &\bf{0.328} &\textbf{0.469} &\bf{0.176} &\bf{0.297} &\bf{0.468} &\bf{0.603} &\bf{0.259} &\bf{0.409} &0.196 &0.276 &0.105 &0.173 \\ \midrule
            \textsc{UltraQuery} 0-shot &\bf{0.280} &0.380 &\bf{0.193} &\bf{0.288} &0.312 &\bf{0.467} &0.139 &0.262 &0.411 &0.517 &0.240 &0.352 &\bf{0.309} &\bf{0.432} &\bf{0.178} &\bf{0.286} \\
            \textsc{UltraQuery LP} 0-shot &0.268 &\bf{0.409} &0.104 &0.181 &0.277 &0.441 &0.098 &0.191 &0.322 &0.476 &0.150 &0.263 &0.279 &0.430 &0.107 &0.195 \\
            \textsc{UltraQuery LP} no thrs. &0.227 &0.331 &0.080 &0.138 &0.246 &0.390 &0.085 &0.167 &0.281 &0.417 &0.127 &0.223 &0.242 &0.367 &0.088 &0.161 \\
            \bottomrule
        \end{tabular}
    \end{adjustbox}
    \label{tab:maintab1}
\end{table*}

Averaged across 23 datasets, \textsc{UltraQuery} outperforms available baselines by relative 50\% in terms of MRR and Hits@10 on EPFO and 70\% on negation queries (e.g.\ 0.31 vs 0.20 MRR on EPFO queries and 0.178 vs 0.105 on negation queries). The largest gains are achieved on the hardest inductive $(e,r)$ datasets where the heuristic baseline is not able to cope with the task. On inductive $(e)$ datasets, \textsc{UltraQuery} outperforms the trainable SOTA GNN-QE model on larger inductive inference graphs and performs competitively on smaller inductive versions. On transductive benchmarks, \textsc{UltraQuery} lags behind the SOTA QTO model which is expected and can be attributed to the sheer model size difference (177k of \textsc{UltraQuery} vs 30M of QTO) and the computationally expensive brute-force approach of QTO that materializes the whole $(\gV \times \gV \times \gR)$ 3D tensor of scores of all possible triplets. Pre-computing such tensors on three datasets takes considerable space and time, e.g.\ 8 hours for FB15k with heavy sparsification settings to fit onto a 24 GB GPU. Still, \textsc{UltraQuery} outperforms a much larger QTO model on the FB15k dataset on both EPFO and negation queries. The graph behind the NELL995 dataset is a collection of disconnected components which is disadvantageous for GNNs.

We note a decent performance of \textsc{UltraQuery LP} trained only on simple \emph{1p} link prediction and imbued with score thresholding to alleviate the multi-source message passing issue described in Section~\ref{subsec:ultra_proj}. Having a deeper look at other qualitative metrics in the following section, we reveal more sites where the issue incurs negative effects.

\begin{figure*}[t]
    \centering
    \includegraphics[width=\linewidth]{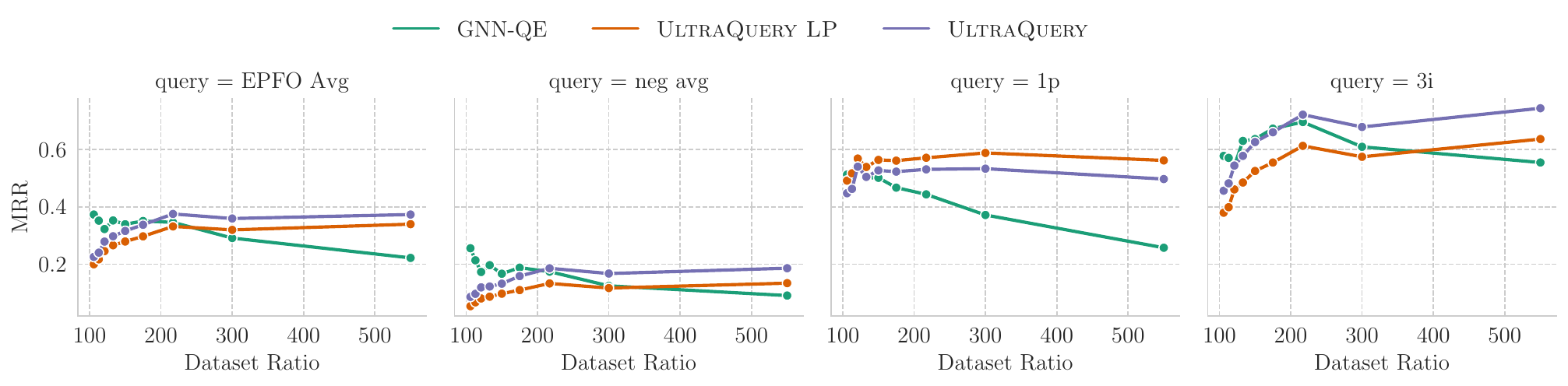}
    \caption[Mitigation of the multi-source message passing issue with UltraQuery]{Mitigation of the multi-source message passing issue (Section~\ref{sec:method}) with \textsc{UltraQuery}: while \textsc{UltraQuery LP} (pre-trained only on 1p link prediction) does reach higher 1p query performance (center right), it underperforms on negation queries (center left). \textsc{UltraQuery} adapts to the multi-source message passing scheme and trades a fraction of 1p query performance for better averaged EPFO, e.g.\ on the \emph{3i} query (right), and negation queries performance.}
    \label{fig:abl_multisource}
\end{figure*}

\subsection{Analysis}

\begin{figure*}[t]
    \centering
    \includegraphics[width=\linewidth]{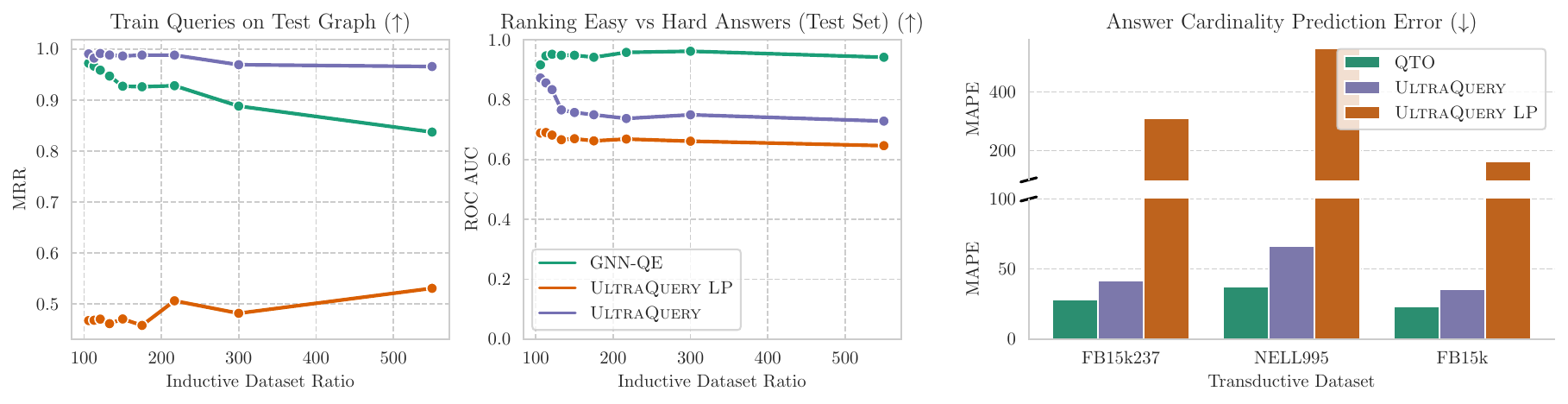}
    \caption[Qualitative analysis of UltraQuery on 9 inductive and 3 transductive datasets]{Qualitative analysis on 9 inductive $(e)$ and 3 transductive datasets averaged across all 14 query types. \textbf{Faithfullness, MRR (left):} \textsc{UltraQuery} successfully finds easy answers in larger inference graphs and outperforms trained GNN-QE baselines. \textbf{Ranking of easy vs hard answers, AUROC (center):} zero-shot inference methods slightly lag behind trainable GNN-QE due to assigning higher scores to hard answers. \textbf{Cardinality Prediction, MAPE (right):} \textsc{UltraQuery} is comparable to a much larger trainable baseline QTO. In all cases, \textsc{UltraQuery LP} is significantly inferior to the main model. }
    \label{fig:abl_quality}
\end{figure*}

Here, we study four aspects of model performance: the effect of the multi-source message passing issue mentioned in Section~\ref{subsec:ultra_proj}, the ability to recover answers achievable by edge traversal (\emph{faithfullness}), the ability to rank easy answers higher than hard answers, and the ability to estimate the cardinality of the answer set.

\smallskip \noindent \textbf{The multi-source message passing effect.}
The pre-trained \textsc{Ultra} checkpoint used in \textsc{UltraQuery LP} is tailored for singe-source message passing and struggles in the complex query setup on later hops with several initialized nodes (\autoref{tab:maintab1}). Training \textsc{UltraQuery} on complex queries alleviates this issue as shown in \autoref{fig:abl_multisource}, i.e.\ while \emph{1p} performance of \textsc{UltraQuery LP} is higher, the overall performance on EPFO and negative queries is lacking. In contrast, \textsc{UltraQuery} trades a fraction of \emph{1p} single-source performance to a much better performance on negative queries (about $2\times$ improvement) and better performance on many EPFO queries, for example, on \emph{3i} queries. Besides that, we note that the zero-shot performance of both \textsc{UltraQuery} models does not deteriorate from the increased size of the inference graph compared to the baseline GNN-QE.

\smallskip \noindent \textbf{Recovering easy answers on any graph.}
\emph{Faithfullness}~\cite{sun2020faithful} is the ability of a complex query model to return \emph{easy} query answers, i.e.\ the answers reachable by edge traversal in the graph without predicting missing edges. While faithfullness is a common problem for many complex query models, \autoref{fig:abl_quality} demonstrates that \textsc{UltraQuery} almost perfectly recovers easy answers on any graph size even in the zero-shot inference regime in contrast to the best baseline. Simple score thresholding does not help \textsc{UltraQuery LP} to deal with complex queries as all easy intermediate nodes have high scores above the threshold and the multi-source is more pronounced.

\smallskip \noindent \textbf{Ranking easy and hard answers.}
A reasonable complex query model is likely to score easy answers higher than hard ones that require inferring missing links~\cite{galkin2022inductive}.
Measuring that with AUROC (\autoref{fig:abl_quality}), \textsc{UltraQuery} is behind the baseline due to less pronounced decision boundaries (overlapping distributions of scores) between the scores of easy and hard answers. Still, due to scores filtering when computing ranking metrics, this fact does not have a direct negative impact on the overall performance.

\smallskip \noindent \textbf{Estimating the answer set cardinality.}
Neural-symbolic models like GNN-QE and QTO have the advantage of estimating the cardinality of the answer set based on the final scores without additional supervision. As shown in \autoref{fig:abl_quality}, \textsc{UltraQuery} is comparable to the larger and trainable QTO baseline on FB15k-237 (on which the model was trained) as well as on other datasets in the zero-shot inference regime. Since cardinality estimation is based on score thresholding, \textsc{UltraQuery LP} is susceptible to the multi-source propagation issue with many nodes having a high score and is not able to deliver a comparable performance.

\section{Dataset Statistics}
\label{app:dataset_gnn-qe}

\smallskip \noindent \textbf{Transductive Datasets.}
We use the complex query datasets generated by \cite{ren2020beta}. There is a total number of 14 query types, as showed in Figure~\ref{fig:query_type}. Statistics of all query types is summarized in Table~\ref{tab:statistics}.

\begin{figure*}[!h]
    \centering
    \includegraphics[width=0.98\textwidth]{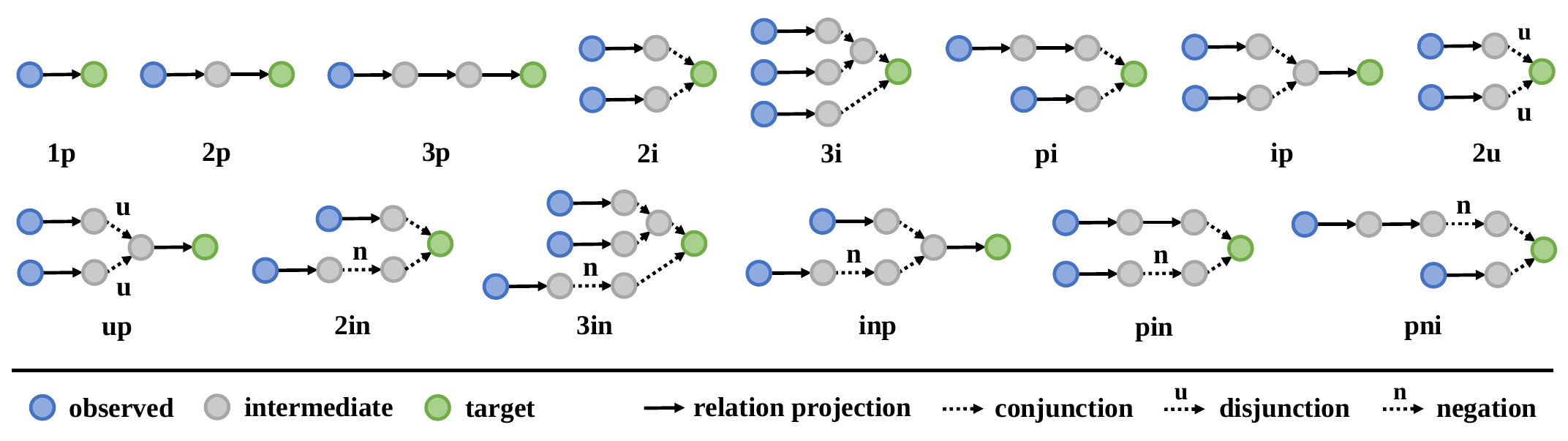}
    \caption[Types of complex FOL queries used in training and inference]{Types of complex FOL queries used in training and inference.}
    \label{fig:query_type}
\end{figure*}

\begin{table}[!h]
    \centering
    \caption[Statistics of different query types used in the transductive datasets]{Statistics of different query types used in the transductive datasets.}
    \footnotesize
    \begin{tabular}{llcccc}
        \toprule
        \bf{Split} & \bf{Query Type} & \bf{FB15k} & \bf{FB15k-237} & \bf{NELL995} \\
        \midrule
        \multirow{2}{*}{Train}
        & 1p/2p/3p/2i/3i & 273,710 & 149,689 & 107,982 \\
        & 2in/3in/inp/pin/pni & 27,371 & 14,968 & 10,798 \\
        \midrule
        \multirow{2}{*}{Valid}
        & 1p & 59,078 & 20,094 & 16,910 \\
        & Others & 8,000 & 5,000 & 4,000 \\
        \midrule
        \multirow{2}{*}{Test}
        & 1p & 66,990 & 22,804 & 17,021 \\
        & Others & 8,000 & 5,000 & 4,000 \\
        \bottomrule
    \end{tabular}
    \label{tab:statistics}
\end{table}

\smallskip \noindent \textbf{Inductive $(e)$ Datasets.}
We sampled 9 inductive datasets from the original FB15k-237~\cite{toutanova2015observed} with already added inverse edges. Statistics of the sampled graphs are presented in Table~\ref{tab:graphs}. The amount of new unique nodes is simply the difference $\gE_\textit{inf} - \gE_\textit{train}$ between entities in those graphs, e.g.\ for the dataset of ratio $175\%$, the validation inference graph contains $4,241$ new nodes and test inference graph contains $4,221$ news nodes. Note that the sets of new nodes introduced by the validation and test graphs are disjoint.

For each created inductive dataset, we sample queries of 14 query patterns following the BetaE~\cite{ren2020beta} procedure. We only retain queries that have less than 1000 answers. Table~\ref{tab:all_queries} summarizes the statistics of the sampled queries for each dataset ratio, graph and query type. In graphs with smaller inference graphs and smaller number of missing triplets, we sample fewer queries with negation (\emph{2in, 3in, inp, pin, pni}) for validation and test splits.

\begin{table}[!h]
    \centering
    \caption[Statistics of inductive datasets with various ratios $\gE_{\textit{inf}} / \gE_{\textit{train}}$]{Statistics of inductive datasets with various ratios $\gE_{\textit{inf}} / \gE_{\textit{train}}$. Originally inverse triplets are included. $\gR$ - number of unique relation types, $\gE$ - number of entities in various splits, $\gT$ - number of triplets. Validation and Test splits contain an inference graph $(\gE_\textit{inf}, \gT_\textit{inf})$ which is a superset of the training graph with new nodes, and missing edges to predict $\gT_\textit{pred}$.}
    \label{tab:graphs}
    \footnotesize
    \begin{adjustbox}{max width=\textwidth}
        \begin{tabular}{cccccccccccc}\toprule
            \multirow{2}{*}{\bf{Ratio}} &\multirow{2}{*}{$\gR$} &\multirow{2}{*}{$\gE_{\textit{total}}$} &\multicolumn{2}{c}{\bf{Training Graph}} &\multicolumn{3}{c}{\bf{Validation Graph}} &\multicolumn{3}{c}{\bf{Test Graph}} \\ \cmidrule(lr){4-5}  \cmidrule(lr){6-8}  \cmidrule(lr){9-11}
            & & &$\gE_{\textit{train}}$ & $\gT_{\textit{train}}$ & $\gE^{\textit{val}}_{\textit{inf}}$ & $\gT^{\textit{val}}_{\textit{inf}}$ & $\gT^{\textit{val}}_{\textit{pred}}$ & $\gE^{\textit{test}}_{\textit{inf}}$ & $\gT^{\textit{test}}_{\textit{inf}}$ & $\gT^{\textit{test}}_{\textit{pred}}$ \\ \midrule
            106\% &466 &14,512 &13,091 &493,425 &13,801 &551,336 &10,219 &13,802 &538,896 &8,023 \\
            113\% &468 &14,442 &11,601 &401,677 &13,022 &491,518 &15,849 &13,021 &486,068 &14,893 \\
            122\% &466 &14,444 &10,184 &298,879 &12,314 &413,554 &20,231 &12,314 &430,892 &23,289 \\
            134\% &466 &14,305 &8,634 &228,729 &11,468 &373,262 &25,477 &11,471 &367,810 &24,529 \\
            150\% &462 &14,333 &7,232 &162,683 &10,783 &311,462 &26,235 &10,782 &331,352 &29,755 \\
            175\% &436 &14,022 &5,560 &102,521 &9,801 &265,412 &28,691 &9,781 &266,494 &28,891 \\
            217\% &446 &13,986 &4,134 &52,455 &9,062 &227,284 &30,809 &9,058 &212,386 &28,177 \\
            300\% &412 &13,868 &2,650 &24,439 &8,252 &178,680 &27,135 &8,266 &187,156 &28,657 \\
            550\% &312 &13,438 &1,084 &5,265 &7,247 &136,558 &22,981 &7,275 &133,524 &22,503 \\
            \bottomrule
        \end{tabular}
    \end{adjustbox}
\end{table}

\smallskip \noindent \textbf{Inductive $(e,r)$ Datasets.}
The WikiTopics dataset introduced by~\cite{gao2023double} was used to evaluate link prediction model's zero-shot performance in the inductive $(e,r)$ setting, i.e.\ when the test-time inference graph contains \textit{both} new entities and new relations unseen in training. It grouped relations into 11 different topics, or domains, such as art, education, health care, and sport. Table~\ref{tab:app_datasets_indr_query} shows the statistics of the 11 topic-specific knowledge graphs in WikiTopics. We sample queries and answers of the 14 query patterns for WikiTopics following the procedure in BetaE~\cite{ren2020beta}. Table~\ref{tab:wikitopics} shows the statistics of complex queries generated for WikiTopics.

\pagebreak

\begin{table}[!h]
    \centering
    \caption[Statistics of different query types sampled for inductive datasets]{Statistics of different query types sampled for inductive datasets.}
    \label{tab:all_queries}
    \begin{adjustbox}{width=\textwidth}
        \begin{tabular}{lrrrrrrrrrrrrrrrr}\toprule
            \bf{Ratio} & \bf{Graph} & \multicolumn{1}{c}{\bf{1p}} & \multicolumn{1}{c}{\bf{2p}} & \multicolumn{1}{c}{\bf{3p}} & \multicolumn{1}{c}{\bf{2i}} & \multicolumn{1}{c}{\bf{3i}} & \multicolumn{1}{c}{\bf{pi}} & \multicolumn{1}{c}{\bf{ip}} & \multicolumn{1}{c}{\bf{2u}} & \multicolumn{1}{c}{\bf{up}} & \multicolumn{1}{c}{\bf{2in}} & \multicolumn{1}{c}{\bf{3in}} & \multicolumn{1}{c}{\bf{inp}} & \multicolumn{1}{c}{\bf{pin}} & \multicolumn{1}{c}{\bf{pni}} \\\midrule
            \multirow{3}{*}{106\%} &training & 135,613 &50,000 &50,000 &50,000 &50,000 &50,000 &50,000 &50,000 &50,000 &50,000 &40,000 &50,000 &50,000 &50,000 \\
            &validation & 6,582 &10,000 &10,000 &10,000 &10,000 &10,000 &10,000 &10,000 &10,000 &1,000 &1,000 &1,000 &1,000 &1,000 \\
            &test & 5,446 &10,000 &10,000 &10,000 &10,000 &10,000 &10,000 &10,000 &10,000 &1,000 &1,000 &1,000 &1,000 &1,000 \\ \midrule
            \multirow{3}{*}{113\%} &training & 115,523 &50,000 &50,000 &50,000 &50,000 &50,000 &50,000 &50,000 &50,000 &50,000 &40,000 &50,000 &50,000 &50,000 \\
            &validation & 10,256 &10,000 &10,000 &10,000 &10,000 &10,000 &10,000 &10,000 &10,000 &1,000 &1,000 &1,000 &1,000 &1,000 \\
            &test & 9,782 &10,000 &10,000 &10,000 &10,000 &10,000 &10,000 &10,000 &10,000 &1,000 &1,000 &1,000 &1,000 &1,000 \\ \midrule
            \multirow{3}{*}{122\%} &training & 91,228 &50,000 &50,000 &50,000 &50,000 &50,000 &50,000 &50,000 &50,000 &50,000 &40,000 &50,000 &50,000 &50,000 \\
            &validation & 12,696 &10,000 &10,000 &10,000 &10,000 &10,000 &10,000 &10,000 &10,000 &5,000 &5,000 &5,000 &5,000 &5,000 \\
            &test & 14,458 &10,000 &10,000 &10,000 &10,000 &10,000 &10,000 &10,000 &10,000 &5,000 &5,000 &5,000 &5,000 &5,000 \\ \midrule
            \multirow{3}{*}{134\%} &training & 75,326 &50,000 &50,000 &50,000 &50,000 &50,000 &50,000 &50,000 &50,000 &50,000 &40,000 &50,000 &50,000 &50,000 \\
            &validation & 15,541 &50,000 &50,000 &50,000 &50,000 &50,000 &50,000 &20,000 &20,000 &5,000 &5,000 &5,000 &5,000 &5,000 \\
            &test & 15,270 &50,000 &50,000 &50,000 &50,000 &50,000 &50,000 &20,000 &20,000 &5,000 &5,000 &5,000 &5,000 &5,000 \\ \midrule
            \multirow{3}{*}{150\%} &training & 56,114 &50,000 &50,000 &50,000 &50,000 &50,000 &50,000 &50,000 &50,000 &50,000 &40,000 &50,000 &50,000 &50,000 \\
            &validation & 16,229 &50,000 &50,000 &50,000 &50,000 &50,000 &50,000 &50,000 &50,000 &5,000 &5,000 &5,000 &5,000 &5,000 \\
            &test & 17,683 &50,000 &50,000 &50,000 &50,000 &50,000 &50,000 &50,000 &50,000 &5,000 &5,000 &5,000 &5,000 &5,000 \\ \midrule
            \multirow{3}{*}{175\%} &training & 38,851 &50,000 &50,000 &50,000 &50,000 &50,000 &50,000 &50,000 &50,000 &50,000 &40,000 &50,000 &50,000 &50,000 \\
            &validation & 17,235 &50,000 &50,000 &50,000 &50,000 &50,000 &50,000 &50,000 &50,000 &10,000 &10,000 &10,000 &10,000 &10,000 \\
            &test & 17,476 &50,000 &50,000 &50,000 &50,000 &50,000 &50,000 &50,000 &50,000 &10,000 &10,000 &10,000 &10,000 &10,000 \\ \midrule
            \multirow{3}{*}{217\%} & training & 22,422 &30,000 &30,000 &50,000 &50,000 &50,000 &50,000 &50,000 &50,000 &30,000 &30,000 &50,000 &50,000 &50,000 \\
            &validation & 18,168 &50,000 &50,000 &50,000 &50,000 &50,000 &50,000 &50,000 &50,000 &10,000 &10,000 &10,000 &10,000 &10,000 \\
            &test & 16,902 &50,000 &50,000 &50,000 &50,000 &50,000 &50,000 &50,000 &50,000 &10,000 &10,000 &10,000 &10,000 &10,000 \\ \midrule
            \multirow{3}{*}{300\%} &training & 11,699 &15,000 &15,000 &40,000 &40,000 &50,000 &50,000 &50,000 &50,000 &15,000 &15,000 &50,000 &40,000 &50,000 \\
            &validation & 16,189 &50,000 &50,000 &50,000 &50,000 &50,000 &50,000 &50,000 &50,000 &10,000 &10,000 &10,000 &10,000 &10,000 \\
            &test & 17,105 &50,000 &50,000 &50,000 &50,000 &50,000 &50,000 &50,000 &50,000 &10,000 &10,000 &10,000 &10,000 &10,000 \\ \midrule
            \multirow{3}{*}{550\%} &training & 3,284 &15,000 &15,000 &40,000 &40,000 &50,000 &50,000 &50,000 &50,000 &10,000 &10,000 &30,000 &30,000 &30,000 \\
            &validation & 13,616 &50,000 &50,000 &50,000 &50,000 &50,000 &50,000 &50,000 &50,000 &10,000 &10,000 &10,000 &10,000 &10,000 \\
            &test & 13,670 &50,000 &50,000 &50,000 &50,000 &50,000 &50,000 &50,000 &50,000 &10,000 &10,000 &10,000 &10,000 &10,000 \\ 
            \bottomrule
        \end{tabular}
    \end{adjustbox}
\end{table}

\pagebreak

\begin{table}[t]
    \centering
    \caption[Statistics of knowledge graphs in WikiTopics]{Statistics of knowledge graphs in WikiTopics. Triplets denote the number of edges of the graph given at training, validation, or test. Valid and Test denote triplets to be predicted in the validation and test sets respectively.}
    \label{tab:app_datasets_indr_query}
    \begin{adjustbox}{width=\textwidth}
        \begin{tabular}{lcccccccccccc}\toprule
            \multirow{2}{*}{\bf{Dataset}} &\multicolumn{3}{c}{\bf{Training Graph}} &\multicolumn{4}{c}{\bf{Validation Graph}} &\multicolumn{4}{c}{\bf{Test Graph}} \\ \cmidrule(l){2-4} \cmidrule(l){5-8} \cmidrule(l){9-12}
            &\bf{Entities} & \bf{Relations} & \bf{Triplets} & \bf{Entities} & \bf{Relations} & \bf{Triplets} & \bf{Valid} & \bf{Entities} & \bf{Relations} & \bf{Triplets} & \bf{Test} \\\midrule
            Art &10000 &65 &27262 &10000 &65 &27262 &3026 &10000 &65 &28023 &3113 \\
            Award &10000 &17 &23821 &10000 &13 &23821 &2646 &10000 &17 &25056 &2783 \\
            Education &10000 &19 &14355 &10000 &19 &14355 &1594 &10000 &19 &14193 &1575 \\
            Health &10000 &31 &15539 &10000 &31 &15539 &1725 &10000 &31 &15337 &1703 \\
            Infrastructure &10000 &37 &21990 &10000 &37 &21990 &2443 &10000 &37 &21646 &2405 \\
            Location &10000 &62 &85063 &10000 &62 &85063 &9451 &10000 &62 &80269 &8917 \\
            Organization &10000 &34 &33325 &10000 &34 &33325 &3702 &10000 &34 &31314 &3357 \\
            People &10000 &40 &55698 &10000 &40 &55698 &6188 &10000 &40 &58530 &6503 \\
            Science &10000 &66 &12576 &10000 &66 &12576 &1397 &10000 &66 &12516 &1388 \\
            Sport &10000 &34 &47251 &10000 &34 &47251 &5250 &10000 &34 &46717 &5190 \\
            Taxonomy &10000 &59 &18921 &10000 &59 &18921 &2102 &10000 &59 &19416 &2157 \\
            \bottomrule
        \end{tabular}
    \end{adjustbox}
\end{table}

\begin{table}[t]
    \centering
    \caption[Statistics of queries generated for knowledge graphs in WikiTopics]{Statistics of queries generated for knowledge graphs in WikiTopics. Numbers are the same for both the training and inference graph.}
    \begin{adjustbox}{width=\textwidth}
        \begin{tabular}{lrrrrrrrrrrrrrrr}\toprule
        \bf{Topics} & \bf{1p} & \bf{2p} & \bf{3p} & \bf{2i} & \bf{3i} & \bf{pi} & \bf{ip} & \bf{2in} & \bf{3in} & \bf{pin} & \bf{pni} & \bf{inp} & \bf{2u} & \bf{up} \\
        \midrule
        Art & 3113 & 10000 & 10000 & 10000 & 10000 & 10000 & 10000 & 1000 & 1000 & 1000 & 1000 & 1000 & 10000 & 10000 \\
        Award & 2783 & 10000 & 10000 & 10000 & 10000 & 10000 & 10000 & 1000 & 1000 & 1000 & 1000 & 1000 & 10000 & 10000 \\
        Education & 1575 & 10000 & 10000 & 10000 & 10000 & 10000 & 10000 & 1000 & 1000 & 1000 & 1000 & 1000 & 10000 & 10000 \\
        Health & 1703 & 10000 & 10000 & 10000 & 10000 & 10000 & 10000 & 1000 & 1000 & 1000 & 1000 & 1000 & 10000 & 10000 \\
        Infrastructure & 2405 & 10000 & 10000 & 10000 & 10000 & 10000 & 10000 & 1000 & 1000 & 1000 & 1000 & 1000 & 10000 & 10000 \\
        Location & 8000 & 8917 & 4000 & 8000 & 8000 & 8000 & 8000 & 1000 & 1000 & 1000 & 1000 & 1000 & 8000 & 8000 \\
        Organization & 3357 & 8000 & 4000 & 8000 & 8000 & 8000 & 8000 & 1000 & 1000 & 1000 & 1000 & 1000 & 8000 & 8000 \\
        People & 6503 & 10000 & 10000 & 10000 & 10000 & 10000 & 10000 & 1000 & 1000 & 1000 & 1000 & 1000 & 10000 & 10000 \\
        Science & 1388 & 10000 & 10000 & 10000 & 10000 & 10000 & 10000 & 1000 & 1000 & 1000 & 1000 & 1000 & 10000 & 10000 \\
        Sport & 5190 & 8000 & 4000 & 8000 & 8000 & 8000 & 8000 & 1000 & 1000 & 1000 & 1000 & 1000 & 8000 & 8000 \\
        Taxonomy & 2157 & 8000 & 8000 & 8000 & 8000 & 8000 & 8000 & 1000 & 1000 & 1000 & 1000 & 1000 & 8000 & 8000 \\
        \bottomrule
        \end{tabular}
    \end{adjustbox}
    \label{tab:wikitopics}
\end{table}
\section{Limitations and Future Work}

One limitation for GNN-QE is the scalability issue. GNN-QE has to materialize a vector representation for each entity in the graph within its relation projection, which takes much memory and computation time. Future work may address this issue with learned adaptive propagation like A*Net, or query optimization techniques that leverage the nature of complex queries to reduce the search space.

For \textsc{UltraQuery}, while we proposed two empirical solutions for mitigating the multi-source propagation issue, we lack a comprehensive understanding of the distribution shift between pre-trained single-hop models and relation projection required by multi-hop models. We expect future work to provide a better mathematical framework for unifying single-hop and multi-hop models. Besides, \textsc{UltraQuery} may be further improved by better multi-stage pre-training strategies and more pre-training datasets.

As highlighted in \cite{ren2023neural}, there is much potential work for complex query models as a part of neural graph databases. Some examples include better theoretical understanding of logical expressiveness bounds, supporting more query patterns beyond simple trees~\cite{yin2023text, yin2024rethinking}, queries without anchor nodes~\cite{barcelo2023neuro}, hyper-relational queries~\cite{alivanistos2022query} and queries with numerical literals~\cite{demir2023litcqd}. To extend the application of complex query models to natural language questions, it is also important to study end-to-end training of complex query models and semantic parsers without explicit query structure annotations.
\chapter[Solving Multi-step Queries with Large Language Models]{Solving Multi-step Queries with\linebreak Large Language Models}
\label{cha:htt}

There is a growing trend of applying large language models (LLMs) to solve reasoning problems, particularly in the form of few-shot chain-of-thought (CoT) prompting. With a few in-context examples and intermediate steps, an LLM is able to decompose a multi-step query into single steps and solve them sequentially. However, LLMs often fail to comprehend the knowledge implied by the examples, and instead rely on the implicit knowledge in their parameters, of which the errors may exacerbate over multiple steps. Can we improve or fix the knowledge of LLMs, even if they are black boxes?

In this chapter, we address this question with Hypotheses-to-Theories (HtT), a prompting method that learns explicit knowledge as textual rules, which help LLMs generalize better to problems longer than those in few-shot examples. The rules discovered by LLMs are not only aligned with human knowledge, but also naturally transferable to different models and to different forms of the same problem. HtT opens up a new learning paradigm that is transparent and interpretable to humans. We foresee that HtT will benefit many applications of LLMs, such as agents~\cite{deng2023mind2web}.

\smallskip \emph{This chapter is based on our work published at SoCal NLP Symposium~\cite{zhu2023large}.}

\section{Overview}

Coinciding with their tremendous growth in scale, large language models (LLMs)~\cite[][\textit{inter alia}]{brown2020language, chowdhery2022palm, achiam2023gpt, anil2023palm, team2023gemini} have demonstrated emergent capabilities across a wide range of reasoning tasks~\cite{wei2022emergent, bubeck2023sparks}, including program synthesis, arithmetic reasoning, symbolic reasoning and commonsense reasoning. Importantly, these abilities are commonly elicited by advanced prompting techniques~\cite{wei2022chain, zhou2023least, khot2023decomposed} that teach an LLM to decompose a complex problem into simple steps and perform reasoning step by step based on a small set of in-context examples.

For many reasoning problems, decomposing and conducting reasoning steps are not sufficient to solve the problem, since one needs domain-specific knowledge to generate correct steps. For instance, when inferring the relationship between two people in a family tree (Figure~\ref{fig:prompt}), an LLM should know the basic rules\footnote{In this paper, rules to refer to intermediate steps that are reusable across different samples. This is slightly different from the definition of rules in formal logic.} to merge family relations. However, LLMs are often prone to errors in generating such rules~\cite{zheng2023does, zhang2024language}, especially when the task deviates from requiring conventional knowledge (e.g.\ arithmetic in a non-decimal system)~\cite{tang2023large, wu2024reasoning}. To solve this challenge, one should find a way to equip LLMs with those domain-specific rules. While it is always possible to curate a dataset and inject required rules into an LLM via supervised finetuning~\cite{wang2021kepler, talmor2020leap, nye2021show}, we are interested in a generic solution that enables LLMs to automatically discover rules from standard datasets without rule annotation.

This paper proposes such a solution for LLMs to learn a library of textual rules and apply the rule library to solve new samples. Our framework, dubbed Hypotheses-to-Theories (HtT), consists of an induction stage and a deduction stage, akin to training and test stages of neural networks respectively. In the induction stage, an LLM is asked to generate rules for each example in the training set. The rules are verified by comparing the prediction of the LLM with the ground truth answer. We then filter the rules based on their number of occurrence and frequency of association with correct answers to construct the rule library for the deduction stage. In the deduction stage, the LLM is then asked to apply the learned rules to solve the reasoning problem, thereby reducing the chance of generating incorrect rules. To reduce the effort required for prompt engineering, we propose a technique called \emph{induction from deduction}, which fuses the rule generation and verification steps into a single deduction-like step. In this way, the prompts for both stages can be easily derived from existing few-shot prompting methods, such as chain-of-thought or least-to-most prompting.

Empirically, we verify the effectiveness of HtT with GPT-3.5 and GPT-4~\cite{achiam2023gpt} on the CLUTRR~\cite{sinha2019clutrr}, Arithmetic~\cite{wu2024reasoning} and List Functions~\cite{rule2020child} datasets, which correspond to relational reasoning, numerical reasoning and concept learning respectively. Experiments show that HtT consistently improves over baseline prompting methods across the models and datasets considered, with an absolute gain of 10-30\% in most cases. Moreover, the learned rules can be directly transferred to the textual version of CLUTRR, providing a practical advantage over previous reasoning approaches. Besides, We conduct extensive ablation studies to understand the properties of HtT, finding that the performance gain arises primarily from a reduction in the number of incorrect rules due to the use of the learned rules. We also observe a log-linear scaling law between accuracy and the number of training examples on all three datasets.

\begin{figure*}[t]
    \centering
    \includegraphics[width=\textwidth]{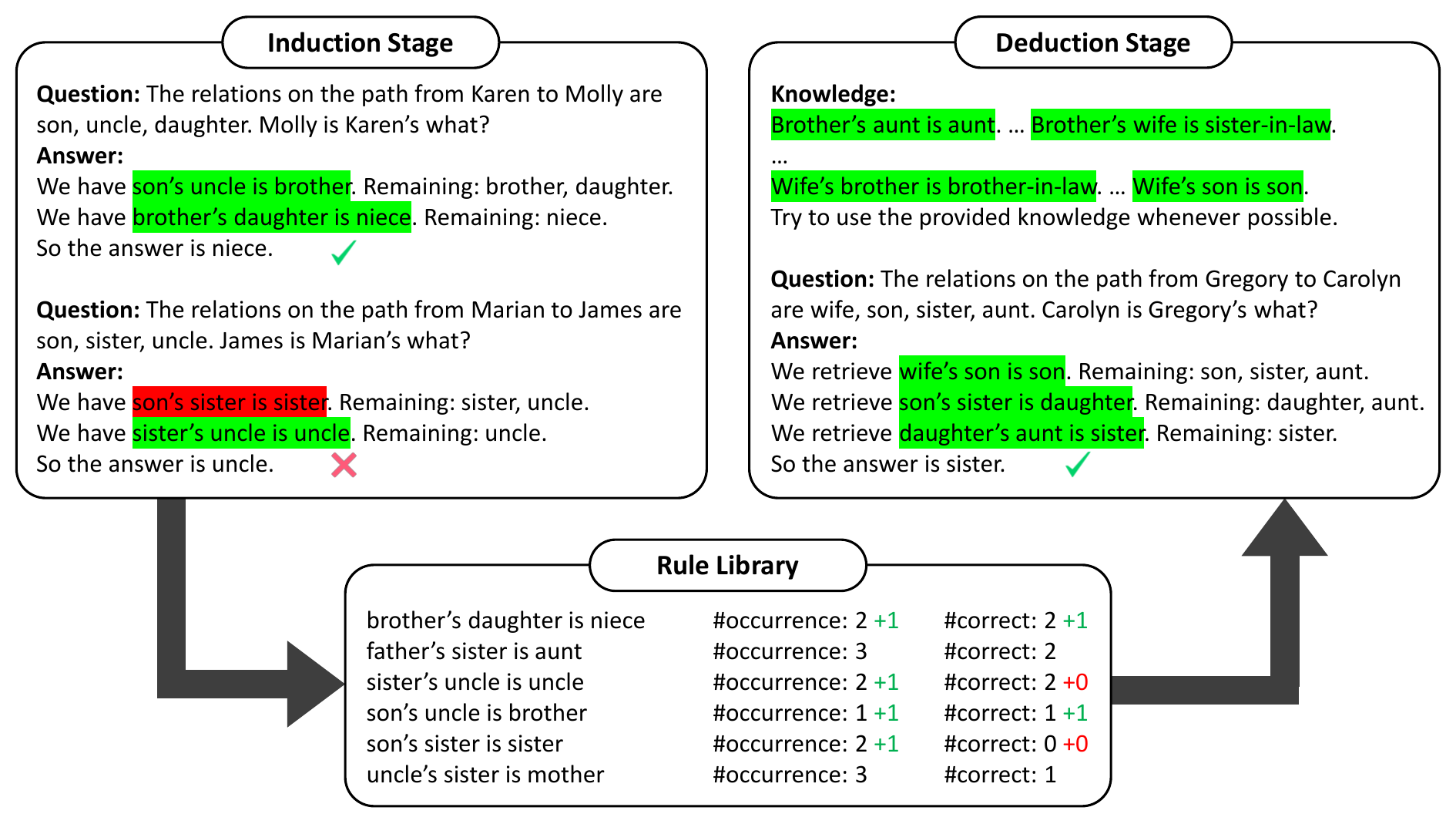}
    \caption[An example of Hypotheses-to-Theories on the relational reasoning problem]{An example of Hypotheses-to-Theories applied to chain-of-thought for the relational reasoning problem. Few-shot examples are omitted for brevity. The induction stage uses CoT to generate rules and verify them on the training samples. Rules are then collected and filtered to form the rule library. The deduction stage augments CoT with knowledge from the rule library. Correct and incorrect rules are marked with green and red respectively.}
    \label{fig:prompt}
\end{figure*}
\section{Method}

For many reasoning problems, the problem statements only contain the necessary facts within the context, but the rules are not explicitly stated. For instance, when an LLM is asked infer the relationship between two people, it is not given any kinship rules. An LLM pretrained on a massive corpus is able to recall certain commonsense rules from its parameters~\cite{petroni2019language, roberts2020much}, but due to the implicit nature of this process, it may often generate incorrect rules when solving reasoning problems. Our manual analysis indicates that such incorrect rules constitute 65\% and 81\% of the errors made by CoT on CLUTRR and base-16 Arithmetic respectively (Figure~\ref{fig:error_case}).

To solve the above challenge, we propose Hypotheses-to-Theories (HtT), a framework that learns a textual rule library from training examples and explicitly uses the rule library to solve test samples. HtT consists of an induction stage and a deduction stage, both implemented by few-shot prompting. In the induction stage, rules are generated and verified on a set of question-answer examples. The rules are then collected and filtered to form a library. In the deduction stage, we prompt the model to explicitly retrieve rules from the rule library to answer test questions. The two stages are similar to training and test stages of neural networks, except that we learn textual rules instead of model parameters.

\subsection{Induction Stage}

The induction stage aims to learn rules from training examples without rule annotation. For each training example (a question-answer pair), we ask an LLM to generate rules for answering the question. We extract rules from the output of the LLM with regular expressions, assuming the LLM follows the template of few-shot exemplars. Note that these rules can be either correct or incorrect. While we cannot judge these rules without golden rules, we can verify them by comparing the prediction of these rules against the ground truth answer. The insight is that if the model predicts the correct answer, it is very likely that all these rules are correct. Conversely, if the model predicts a wrong answer, it is very likely that at least one of the rules is incorrect. Due to the noisy nature of LLM reasoning, we collect rules and accuracy metrics from a reasonable number of training examples.

To filter the rules for the rule library, we follow the principles of rule mining~\cite{galarraga2013amie} and consider both coverage and confidence as criteria. The coverage of a rule tells how likely it will be reused, and the confidence of a rule indicates how likely it is correct. Specifically, we measure coverage based on the number of occurrence of each rule in the training set. Confidence is measured by the average accuracy of examples where a rule occurs. In practice, for each generated rule, we maintain two counters, number of occurrence and number of correct answers (Figure~\ref{fig:prompt} bottom). The confidence of each rule can be computed by dividing its number of correct answers by its number of occurrence. We fill the library with good rules exceeding a minimal coverage $k$ and a minimal confidence $p$.

\smallskip \noindent \textbf{Induction from Deduction.}
The induction stage introduces two sub-problems, rule generation and verification. Recent works on induction~\cite{yang2024language, qiu2024phenomenal} use two separate prompts for generation and verification, i.e.\ a prompt for generating rules based on the question and a prompt for applying the rules to deduce the answer. While it is possible to use two prompts here as well, this doubles the prompt engineering effort and complicates comparisons with other methods. Moreover, the multi-step nature of reasoning problems makes it challenging to generate rules for later steps at the beginning. Hence we propose induction from deduction, which adapts a deductive reasoning prompt (e.g.\ CoT, LtM) for both rule generation and verification (Figure~\ref{fig:prompt} left). The key idea is to explicitly declare a rule whenever a deduction is performed. In this way, both induction and deduction stages use the same base prompt, which is directly comparable to the base prompting method.

\subsection{Deduction Stage}
\label{sec:deduction}

In the deduction stage, we apply the rule library from the induction stage to solve test questions. Specifically, we prepend the rule library to a deductive reasoning prompt (e.g.\ CoT, LtM), and modify the exemplars to teach the LLM to retrieve rules from the library whenever it needs to generate a rule (Figure~\ref{fig:prompt} right). If all the rules required by a question are present in the library, the LLM should generate correct rules for each step without errors.

In practice, we find that even a strong LLM (e.g. GPT-4) somehow struggles to perform retrieval, especially when the library contains a large number of rules in an unstructured way. One workaround is to employ a pretrained passage retriever~\cite{karpukhin2020dense} and interleave retriever calls with LLM decoding~\cite{trivedi2023interleaving}. However, this introduces additional modules and makes the comparison against the base prompting method unfair. Here we propose a pure prompting solution to augment the retrieval ability of LLMs. 

\begin{wrapfigure}{R}{0.47\textwidth}
    \vspace{-1.6em}
    \centering
    \includegraphics[width=0.47\textwidth]{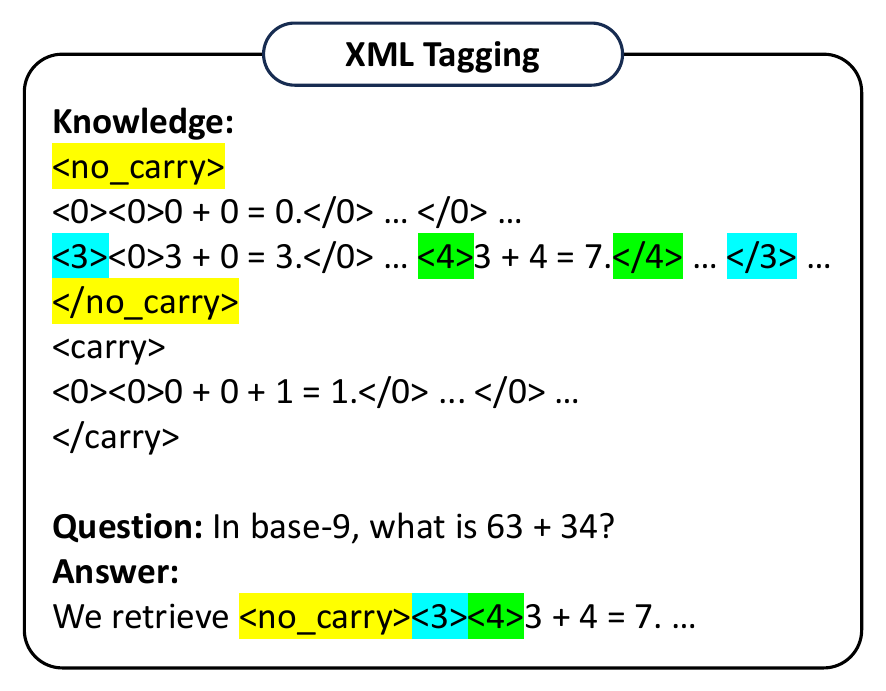}
    \caption[The XML tagging trick]{The XML tagging trick. With an XML hierarchy, we break down a hard retrieval problem into several easy retrieval problems.}
    \label{fig:xml_tag}
\end{wrapfigure}

\smallskip \noindent \textbf{In-Context Retrieval with XML Tags.}
We observe that retrieval often succeeds when the number of rules are limited, e.g.\ to at most 10 rules. Therefore, a natural idea is to organize the rule set into a hierarchy, such that each step in the hierarchy only involves a small number of options. We manually define a hierarchy by grouping similar rules together. Inspired by the XML tags used in prompting tutorials\footnotemark, we label each level of the hierarchy with pairs of XML tags like <carry> and </carry> (Figure~\ref{fig:xml_tag}). In order to index through the hierarchy, we ask the LLM to generate the tags for each level before outputting the retrieved rule. We find that the XML tagging trick significantly boosts the performance of HtT (Table~\ref{tab:ablation_htt}).

\subsection{Discussion}

\smallskip \noindent \textbf{Key Insights.}
It might look surprising that simply prompting the LLM and verifying the predictions on sufficient training examples can give us a library of good rules. Here we discuss the key insights behind HtT. While LLMs may occasionally generate incorrect rules, we conjecture they are able to produce correct rules on some examples with a non-trivial probability, similar to the assumption in \cite{wang2024chain}. With a sufficient set of training examples, we can extract most of the necessary rules for a problem class based on the coverage and confidence criteria. Since retrieving a rule is usually easier than generating the correct rule for an LLM, it will perform better on deductive reasoning when primed with a rule library.
\footnotetext{\url{https://docs.anthropic.com/claude/docs/constructing-a-prompt}}

\smallskip \noindent \textbf{Scope of Tasks.}
Generally, HtT has a similar scope of tasks to its base few-shot prompting method. To achieve substantial performance gain with HtT, two constraints apply: (1) To fit the library into the prompt, the number of rules required to solve most problems of the task should be moderately small ($\leq$500 in our experiments), excluding tasks that cannot be efficiently described by rules (e.g.\ natural language entailment). (2) To successfully induce most rules, the base prompting method should have a reasonable performance on the training examples ($\geq$20\% accuracy in our experiments), excluding tasks that are difficult for existing few-shot prompting methods, such as tasks requiring planning abilities~\cite{saparov2023language}. HtT does not impose constraints on the type of rules it learns. Our experiments show that HtT can learn kinship rules, numerical rules or even free-form rules that transform a list.
\section{Experiments}
\label{sec:experiment}

To evaluate HtT, we apply it as an augmentation to existing few-shot prompting methods. We benchmark the performance of HtT on relational reasoning and numerical reasoning that require multi-step reasoning and multiple rules, as well as concept learning that require a single complex rule. We also conduct ablation studies to thoroughly understand HtT.

\subsection{Implementation Details}
\label{sec:implementation}

We evaluate HtT and the baselines using two different LLMs, \texttt{gpt-3.5-turbo-0613} and \texttt{gpt-4-0613}. When the prompts exceed the 4k context length of \texttt{gpt-3.5-turbo-0613}, we use \texttt{gpt-3.5-turbo-16k-0613} instead. We use the default temperature of 1 for these models. Throughout the following, we will denote the two LLMs as GPT-3.5 and GPT-4 respectively. We further evaluate HtT with \texttt{gemini-1.0-pro}, \texttt{Mistral-7B-Instruct-v0.3} and \texttt{Meta-Llama-3-8B-Instruct} in Section~\ref{app:other_llms}.

On CLUTRR and Arithmetic, we perform the induction stage on 2,000 training examples for the proposed HtT. When the training set contains fewer than 2,000 examples, we resample the training examples. For each task in List Functions, we perform the induction stage on 20 training-validation splits sampled from the original data. We extract rules in the induction stage by searching string templates with regular expressions. We note that HtT does not rely on engineering of the string template, and any templates that can extract rules from the given few-shot examples suffice here. Even if the string templates recall wrong rules, they can be easily filtered by our minimal coverage criterion. We search the hyperparameters of HtT within the following grid: minimal coverage $k \in \{1, 2, 3\}$, minimal confidence $p \in \{0.1, 0.3, 0.5, 0.7, 0.9\}$. Due to the cost of LtM (3-5$\times$ compared to CoT), we induce rules and select the best hyperparameters based on CoT prompting, and only use LtM prompting for the deduction stage. This might slightly underestimate the performance of HtT for LtM prompting.

Since LLMs output free-form text to solve the problems, we evaluate models by matching the predicted text and the ground truth answer. We crop the last sentence from the predicted text, and check if the ground truth answer is present in that sentence. We only consider full word match and exclude partial matches like ``mother'' and ``grandmother''. If the LLM outputs more than one answer, we always consider it as wrong.

\subsection{Relational Reasoning}

We evaluate HtT on CLUTRR~\cite{sinha2019clutrr}, a relational reasoning dataset that queries the relationship between two family members in a family tree. CLUTRR comes in two forms: a symbolic version that only contains entities and their relationships, and a textual version that describes the relationships in a story. We evaluate HtT on both versions. We generate dataset splits by uniformly sampling the standard splits from \cite{sinha2019clutrr}. We use 2,000 samples of 2 and 3 hop examples for training, and 200 samples of 2 to 10 hop examples for both validation and test. For reference, we reproduce EdgeTransformer~\cite{bergen2021systematic}, one of the best domain-specific models on CLUTRR, in the same setting.

Table~\ref{tab:clutrr} shows the results on CLUTRR. Here HtT consistently improves both CoT and LtM prompting with both models by a margin of 11.1-16.4\% in average accuracy. Since induction is more challenging than deduction, we further evaluate GPT-3.5 with rules induced by GPT-4. Surprisingly, with rules from GPT-4, HtT increases the performance of CoT on GPT-3.5 by 27.2\%, doubling the performance of CoT. Compared to the supervised baseline EdgeTransformer, the performance of 5-shot CoT + HtT with GPT-4 is 7.5\% lower, which is reasonable since EdgeTransformer leverages forward chaining as a strong inductive bias and is specific to this problem. HtT has two advantages over such domain-specific models: (1) HtT does not require a predefined relation vocabulary; (2) the rules learned by HtT can directly transfer to textual inputs. As shown in Table~\ref{tab:clutrr_text}, rules learned from the symbolic version can also improve the performance of GPT-4 on the textual version. The improvement is not significant for GPT-3.5, since it often produces errors other than incorrect rules.

\begin{table}[t]
    \centering
    \caption[Results on the symbolic version of CLUTRR]{Results on the symbolic version of CLUTRR.}
    \label{tab:clutrr}
    \begin{adjustbox}{max width=\textwidth}
        \begin{tabular}{llcccccccccc}
            \toprule
            \bf{Model} & \bf{Prompt} & \bf{2 hops} & \bf{3 hops} & \bf{4 hops} & \bf{5 hops} & \bf{6 hops} & \bf{7 hops} & \bf{8 hops} & \bf{9 hops} & \bf{10 hops} & \bf{Average} \\
            \midrule
            \multicolumn{2}{l}{EdgeTransformer} & 100.0 & 94.4 & 96.8 & 88.0 & 68.8 & 61.9 & 50.0 & 50.0 & 36.0 & 71.8 \\
            \midrule
            & 0-shot CoT & 50.0 & 22.2 & 12.9 & 8.0 & 12.5 & 9.5 & 10.0 & 3.8 & 4.0 & 14.8 \\
            \cmidrule{2-12}
            & 5-shot CoT & 0.0  & 27.8 & 45.2 & 36.0 & 18.8 & 19.0 & 16.7 & 11.5 & 16.0 & 21.2 \\
            & + HtT & 87.5 & 38.9 & 35.5 & 44.0 & 37.5 & 14.3 & 33.3 & 11.5 & 36.0 & \bf{37.6 (+16.4)} \\
            GPT-3.5 & + HtT (GPT-4) & 100.0 & 55.6 & 32.3 & 60.0 & 50.0 & 47.6 & 43.3 & 19.2 & 28.0 & \bf{48.4 (+27.2)} \\
            \cmidrule{2-12}
            & 5-shot LtM & 37.5 & 22.2 & 29.0 & 36.0 & 25.0 & 14.3 & 10.0 & 23.1 & 20.0 & 24.1 \\
            & + HtT & 100.0 & 33.3 & 32.3 & 48.0 & 31.3 & 33.3 & 23.3 & 34.6 & 28.0 & \bf{40.5 (+16.4)} \\
            & + HtT (GPT-4) & 75 & 44.4 & 41.9 & 52.0 & 37.5 & 33.3 & 23.3 & 19.2 & 16.0 & \bf{38.1 (+14.0)} \\
            \midrule
            & 0-shot CoT & 50.0 & 22.2 & 22.6 & 32.0 & 37.5 &38.1 & 33.3 & 46.2 & 16.0 & 33.1 \\
            \cmidrule{2-12}
            & 5-shot CoT & 50.0 & 55.6 & 71.0 & 80.0 & 50.0 & 52.4 & 30.0 & 46.2 & 20.0 & 50.6 \\
            GPT-4 & + HtT & 100.0 & 61.1 & 74.2 & 84.0 & 75.0 & 38.1 & 56.7 & 53.8 & 36.0 & \bf{64.3 (+13.7)} \\
            \cmidrule{2-12}
            & 5-shot LtM & 62.5 & 38.9 & 58.1 & 68.0 & 50.0 & 38.1 & 43.3 & 34.6 & 28.0 & 46.8 \\
            & + HtT & 100.0 & 55.6 & 77.4 & 80.0 & 75.0 & 38.1 & 36.7 & 38.5 & 20.0 & \bf{57.9 (+11.1)} \\
            \bottomrule
        \end{tabular}
    \end{adjustbox}
\end{table}

\subsection{Numerical Reasoning}

We use the Arithmetic dataset~\cite{wu2024reasoning} to evaluate the LLMs on numerical reasoning in non-decimal systems. This dataset contains summation problems over 2 to 4 digits in several base systems. Since the rules in a non-decimal system are mostly different from those in the decimal system, arithmetic is considered to be a counterfactual setting that requires an LLM to perform reasoning rather than reciting. To prepare the dataset for HtT, we split it into training, validation and test. The training set contains 900 examples of 2 digit addition. Both the validation and test sets contain 100 examples of 2, 3 and 4 digit addition.

Table~\ref{tab:arithmetic} shows the results on Arithmetic. 0-shot CoT performs worst for both models in all base systems, because the LLMs with 0-shot CoT tend to convert non-decimal inputs to decimal, perform calculations, then revert to non-decimal, which is error prone due to the extra multiplications and divisions. For both CoT and LtM, HtT consistently improves the accuracy of two models by a large margin. The performance gain is less significant for GPT-3.5, since it is worse at inducing correct rules and retrieving rules from the library. This can be fixed by using a stronger model to induce the rules and offloading the retrieval steps to a separate prompt like in LtM. We observe a large improvement of LtM + HtT on GPT-3.5 with better rules from GPT-4, especially on base-11 and base-9 where GPT-3.5 struggles to induce correct rules. By contrast, there is no improvement for CoT + HtT with the better rules, because GPT-3.5 has a strong tendency to rely on its own beliefs (i.e.\ mostly decimal rules) when prompted with CoT, similar to the observation in \cite{longpre2021entity}.

\begin{table}[t]
    \centering
    \caption[Results on Arithmetic]{Results on Arithmetic. Note that GPT-4 (5-shot CoT) has 99.1\% accuracy on base-10.}
    \label{tab:arithmetic}
    \begin{adjustbox}{max width=\textwidth}
        \begin{tabular}{llcccccccccc}
            \toprule
            \multirow{2}{*}{\bf{Model}} & \multirow{2}{*}{\bf{Prompt}} & \multicolumn{3}{c}{\bf{Base-16}} & \multicolumn{3}{c}{\bf{Base-11}} & \multicolumn{3}{c}{\bf{Base-9}} & \multirow{2}{*}{\bf{Average}}\\
            & & \bf{2 digits} & \bf{3 digits} & \bf{4 digits} & \bf{2 digits} & \bf{3 digits} & \bf{4 digits} & \bf{2 digits} & \bf{3 digits} & \bf{4 digits} \\
            \midrule
            & 0-shot CoT & 30.6 & 10.5 & 0.0 & 5.6 & 5.3 & 0.0 & 11.1 & 10.5 & 0.0 & 8.2 \\
            \cmidrule{2-12}
            & 5-shot CoT & 83.3 & 34.2 & 11.5 & 5.6 & 2.6 & 0.0 & 25.0 & 13.2 & 11.5 & 20.8 \\
            & + HtT & 77.8 & 52.6 & 23.1 & 25.0 & 13.2 & 0.0 & 8.3 & 5.3 & 11.5 & \bf{24.1 (+3.3)} \\
            GPT-3.5 & + HtT (GPT-4) & 63.9 & 44.7 & 34.6 & 13.9 & 7.9 & 3.8 & 25.0 & 7.9 & 11.5 & \bf{23.7 (+2.9)} \\
            \cmidrule{2-12}
            & 5-shot LtM & 83.3 & 34.2 & 15.4 & 16.7 & 5.3 & 0.0 & 13.9 & 7.9 & 7.7 & 20.5 \\
            & + HtT & 80.6 & 39.5 & 26.9 & 16.7 & 2.6 & 3.8 & 19.4 & 5.3 & 3.8 & \bf{22.1 (+1.6)} \\
            & + HtT (GPT-4) & 72.2 & 31.6 & 30.8 & 47.2 & 15.8 & 11.5 & 44.4 & 21.1 & 15.4 & \bf{32.2 (+11.7)} \\
            \midrule
            & 0-shot CoT & 72.2 & 26.3 & 7.7 & 22.2 & 10.5 & 3.8 & 30.6 & 34.2 & 23.1 & 25.6 \\
            \cmidrule{2-12}
            & 5-shot CoT & 83.3 & 71.1 & 61.5 & 52.8 & 47.4 & 46.2 & 75.0 & 36.8 & 42.3 & 57.4 \\
            GPT-4 & + HtT & 100.0 & 94.7 & 84.6 & 88.9 & 71.1 & 46.2 & 86.1 & 68.4 & 65.4 & \bf{78.4 (+21.0)} \\
            \cmidrule{2-12}
            & 5-shot LtM & 88.9 & 81.6 & 61.5 & 52.8 & 47.4 & 30.8 & 52.8 & 31.6 & 11.5 & 51.0 \\
            & + HtT & 100.0 & 86.8 & 76.9 & 72.2 & 52.6 & 46.2 & 61.1 & 23.7 & 38.5 & \bf{62.0 (+11.0)} \\
            \bottomrule
        \end{tabular}
    \end{adjustbox}
\end{table}

\subsection{Concept Learning}

To assess the potential of HtT in learning complex rules, we further evaluate HtT on the concept learning problem using List Functions~\cite{rule2020child}. This dataset aims to identify a function that maps each input list to its corresponding output list, with 250 tasks grouped into 3 subsets: simple operations over numbers between 0 and 9 (P1), simple operations over numbers between 0 and 99 (P2), difficult operations over numbers between 0 and 99 (P3). For each task, we split 32 input-output pairs into 16 training samples and 16 test samples. For HtT, we further split the 16 training samples into 8 training samples and 8 validation samples to verify the rules based on the validation performance.

Table~\ref{tab:list_functions} shows the results on List Functions. Following \cite{qiu2024phenomenal}, we report both raw accuracy and task accuracy. Raw accuracy is the accuracy on test input-output pairs, while task accuracy is the ratio of tasks with all test input-output pairs correctly solved. HtT consistently improves 4-shot CoT on both models, with a gain of 18.5-18.7\% in raw accuracy and 10.2-14.5\% in task accuracy. Surprisingly, GPT-4 can discover some very complex rules in List Functions, as shown in Table~\ref{tab:learned_rules}. With rules learned by GPT-4, the task accuracy of GPT-3.5 can be boosted to 34.4\%, doubling the performance of GPT-3.5 with 4-shot CoT. This suggests that GPT-3.5 can understand most rules learned by GPT-4, and the challenge of concept learning lies more in induction than deduction. We also observe that the performance of GPT-3.5 decreases drastically on tasks involving large numbers (P2) or difficult operations (P3). By contrast, the decrease in performance is less significant for GPT-4, indicating that GPT-4 is more robust across various levels of difficulty.

\begin{table}[t]
    \centering
    \footnotesize
    \caption[Results on List Functions]{Results on List Functions.}
    \label{tab:list_functions}
    \begin{adjustbox}{max width=\textwidth}
        \begin{tabular}{llcccccccc}
            \toprule
            \multirow{2}{*}{\bf{Model}} & \multirow{2}{*}{\bf{Prompt}} & \multicolumn{4}{c}{\bf{Raw Accuracy}} & \multicolumn{4}{c}{\bf{Task Accuracy}} \\
            & & \bf{P1} & \bf{P2} & \bf{P3} & \bf{Average} & \bf{P1} & \bf{P2} & \bf{P3} & \bf{Average} \\
            \midrule
            & 0-shot CoT & 44.1 & 38.4 & 28.9 & 37.1 & 30.0 & 25.0 & 12.7 & 22.6 \\
            \cmidrule{2-10}
            \multirow{2}{*}{GPT-3.5} & 4-shot CoT & 32.8 & 32.2 & 19.3 & 28.1 & 22.5 & 15.0 & 8.0 & 15.2 \\
            & + HtT & 58.4 & 50.9 & 30.5 & \bf{46.6 (+18.5)} & 40.0 & 35.0 & 14.0 & \bf{29.7 (+14.5)} \\
            & + HtT (GPT-4) & 69.2 & 66.3 & 38.4 & \bf{58.0 (+29.9)} & 50.0 & 40.0 & 13.3 & \bf{34.4 (+19.2)} \\
            \midrule
            & 0-shot CoT & 69.3 & 56.6 & 48.3 & 58.1 & 51.3 & 45.0 & 26.0 & 40.8 \\
            \cmidrule{2-10}
            GPT-4 & 4-shot CoT & 67.3 & 60.9 & 43.9 & 57.4 & 53.8 & 55.0 & 29.3 & 46.0 \\
            & + HtT & 82.3 & 84.4 & 61.5 & \bf{76.1 (+18.7)} & 61.3 & 70.0 & 37.3 & \bf{56.2 (+10.2)} \\
            \bottomrule
        \end{tabular}
    \end{adjustbox}
\end{table}

\begin{table}[t]
    \centering
    \footnotesize
    \caption[Examples of complex rules learned by GPT-4 on List Functions]{Examples of complex rules learned by GPT-4 on List Functions.}
    \label{tab:learned_rules}
    \begin{adjustbox}{max width=\textwidth}
        \begin{tabular}{lp{17em}p{18em}}
            \toprule
            \bf{Task ID} & \bf{Ground Truth} & \bf{Top Learned Rule} \\
            \midrule
            c085 & remove all but element N + 1, N = element 1. & return a list with the element that corresponds to the first number in the list, where indexing starts at 1. \\
            \midrule
            c191 & repeat each element N times, where N is its tens digit, in order of appearance. & list each element as many times as its tens digit. \\
            \bottomrule
        \end{tabular}
    \end{adjustbox}
\end{table}

\subsection{Ablation Studies}
\label{sec:ablation_htt}

We conduct ablation studies with GPT-4, since GPT-3.5 sometimes struggles to induce and retrieve rules.

\begin{table}[t]
    \centering
    \footnotesize
    \begin{minipage}{0.45\textwidth}
        \centering
        \caption[Results on the textual version of CLUTRR]{Results on the textual CLUTRR w/ rules learned on the symbolic CLUTRR.}
        \vspace{-0.5em}
        \begin{tabular}{llc}
            \toprule
            \bf{Model} & \bf{Prompt} & \bf{Accuracy} \\
            \midrule 
            \multirow{2}{*}{GPT-3.5}
            & 5-shot CoT & 16.0 \\
            & + HtT & 16.3 (+0.3) \\
            \midrule
            \multirow{2}{*}{GPT-4}
            & 5-shot CoT & 48.7 \\
            & + HtT & \bf{59.1 (+10.4)} \\
            \bottomrule
        \end{tabular}
        \label{tab:clutrr_text}
        \\[0.5em]
        \caption[Ablation studies on random rules and the XML tagging trick]{Ablation studies on random rules and the XML tagging trick.}
        \vspace{-0.5em}
        \begin{adjustbox}{max width=\textwidth}
            \begin{tabular}{lcc}
                \toprule
                \bf{Prompt} & \bf{CLUTRR} & \bf{Arithmetic} \\
                \midrule
                5-shot CoT & 50.6 & 57.4 \\
                \midrule
                + random rules & 9.9 (-40.7) & 23.7 (-33.7)\\
                \midrule
                + HtT (unsorted) & 57.1 (+6.5) & 67.2 (+9.8) \\
                + HtT (sorted) & 60.0 (+9.4) & 72.5 (+15.1) \\
                + HtT (1 tag) & 59.6 (+9.0) & 74.8 (+17.4) \\
                + HtT (2 tags) & \bf{64.3} (+13.7) & 76.6 (+19.2) \\
                + HtT (3 tags) & N/A & \bf{78.4} (+21.0) \\
                \bottomrule
            \end{tabular}
        \end{adjustbox}
        \label{tab:ablation_htt}
    \end{minipage}
    \hfill
    \begin{minipage}{0.53\textwidth}
        \vspace{-1.8em}
        \centering
        \includegraphics[width=\textwidth]{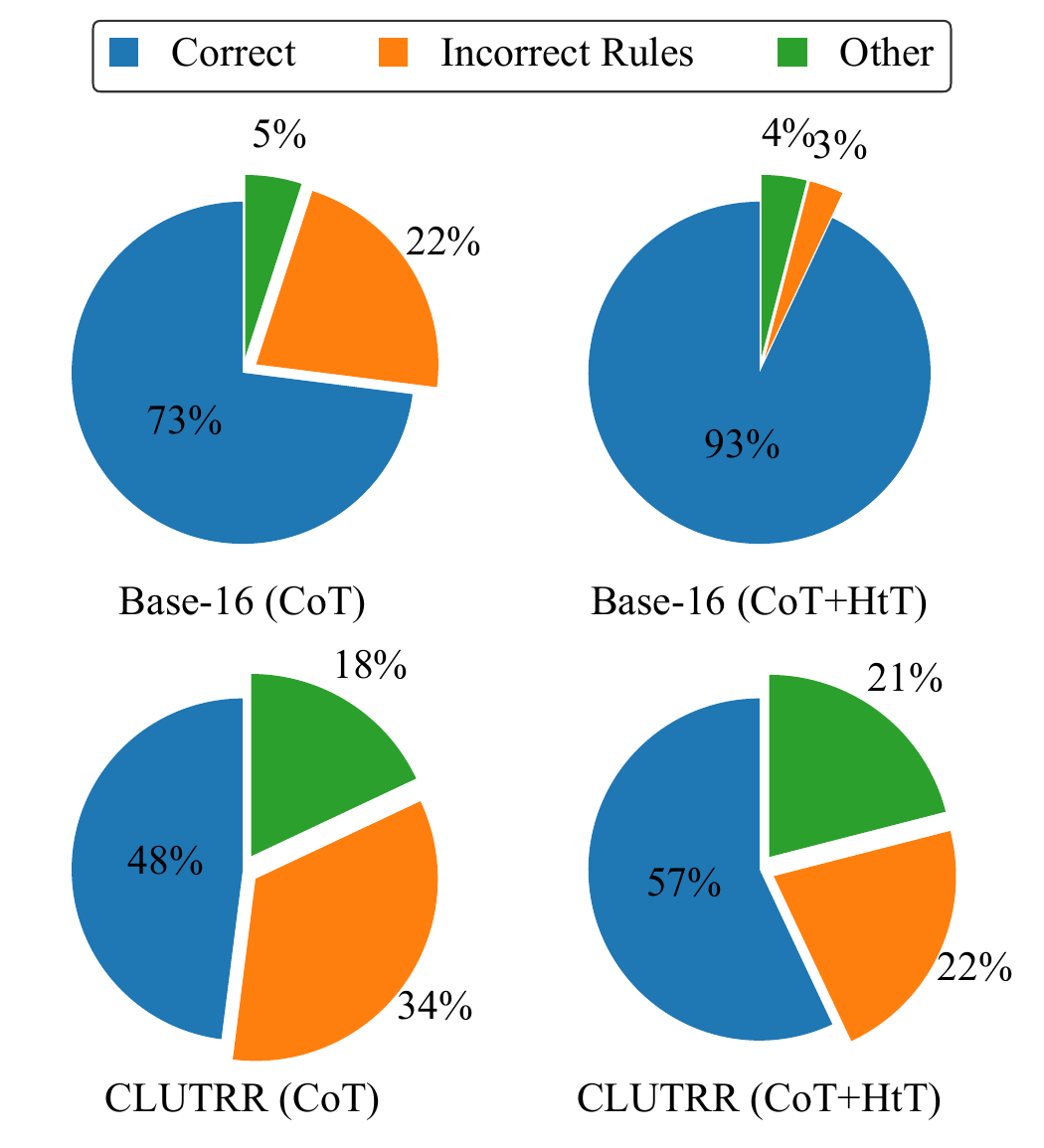}
        \captionof{figure}{Statistics of different error cases on CLUTRR and Arithmetic (base-16).}
        \label{fig:error_case}
    \end{minipage}
\end{table}

\smallskip \noindent \textbf{Does HtT reduce the occurrence of incorrect rules?}
Since an LLM generates free-form text to solve a problem, there can be multiple reasons for failure~\cite{zheng2023does}. While HtT boosts the overall performance on reasoning problems, it is not clear if the gain comes from less incorrect rules. We manually analyze the predictions of CoT and CoT + HtT on 100 test examples from CLUTRR and Arithmetic (base-16), and classify the predictions into 3 categories: correct, incorrect rules and other. Figure~\ref{fig:error_case} plots the distribution of error cases. We can see that most performance gain of HtT comes from reduction in incorrect rules.

\smallskip \noindent \textbf{Do the learned rules just hint the model about the rule space?}
A previous study~\cite{min2022rethinking} found that random labels perform similarly to gold labels in in-context learning. If that was the case for our problems, we could just generate random rules and do not have to learn the rule library. To answer this question, we replace the conclusions in the learned rules with random answers, e.g.\ changing 5 + A = E to 5 + A = 7 in base-16. Table~\ref{tab:ablation_htt} shows that random rules significantly hurt performance, indicating the necessity of learned rules in HtT. We conjecture that the contrary observation is because \cite{min2022rethinking} studied simple classification problems, whereas we are dealing with multi-step reasoning problems.

\smallskip \noindent \textbf{How do XML tags improve deductive reasoning?}
In Section~\ref{sec:deduction}, we introduce the XML tagging trick to augment the retrieval ability of an LLM in decoding. We use a 3-level hierarchy (carry, first addend, second addend) for Arithmetic and a 2-level hierarchy (first relation, second relation) for CLUTRR. Here we perform ablation studies to verify the importance of this design. We consider XML tagging with different levels of hierarchy. Since XML tagging requires the rules to be sorted, we also consider a variant with unsorted (i.e.\ randomly ordered) rules. As shown in Table~\ref{tab:ablation_htt}, the XML tagging trick significantly boosts performance, suggesting that even when the good rules are given, retrieval is an important ability in deductive reasoning.

\begin{figure}[t]
    \centering
    \includegraphics[width=\textwidth]{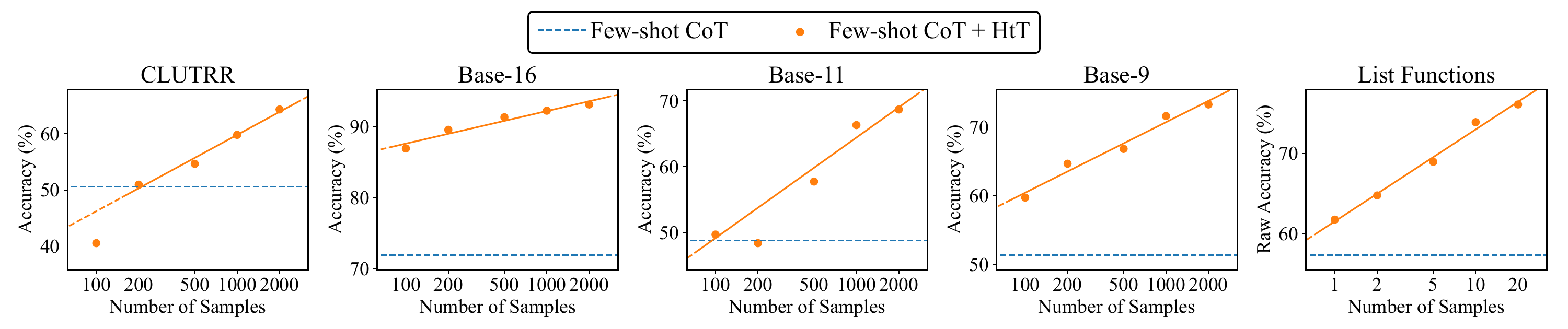}
    \caption[Performance of HtT w.r.t.\ the number of samples in the induction stage]{Performance of HtT w.r.t.\ the number of samples in the induction stage.}
    \label{fig:training_samples}
\end{figure}

\begin{figure}[t]
    \centering
    \includegraphics[width=0.66\textwidth]{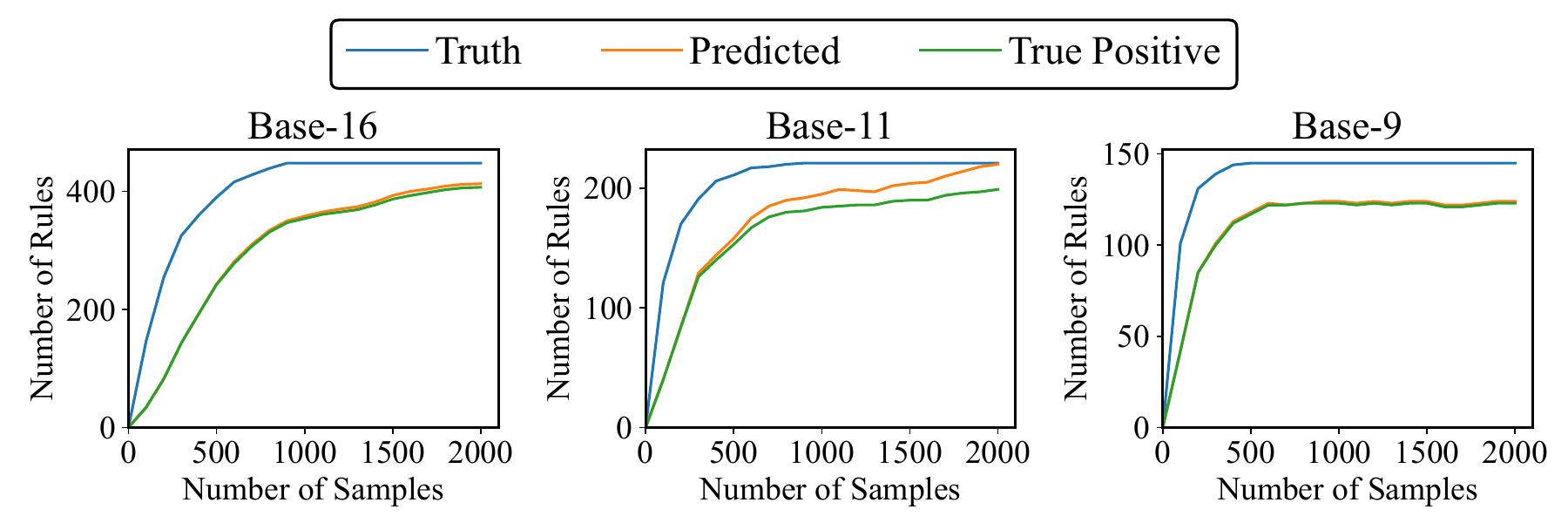}
    \caption[Number of rules discovered by HtT in the induction stage]{Number of rules discovered by HtT in the induction stage.}
    \label{fig:number_of_rules}
\end{figure}

\smallskip \noindent \textbf{How many samples does HtT need for the induction stage?}
One may be curious about how HtT scales with the number of samples and what is the minimal number of examples required. Here we conduct experiments with different numbers of examples for the induction stage. As shown in Figure~\ref{fig:training_samples}, there is a log-linear trend between performance and the number of examples, consistent with the scaling law for supervised learning~\cite{kaplan2020scaling}. The minimal number of examples varies across datasets. CLUTRR and base-11 require 500 examples to obtain a significant gain, while base-16 and base-9 only require 100 examples. On List Functions, 1 sample per task is enough to obtain some gain.

\smallskip \noindent \textbf{What proportion of rules are discovered by HtT?}
To investigate this question, we compare HtT with an oracle that always induces all necessary rules from an example. Note this comparison can only be made on Arithmetic, since rules in CLUTRR are not deterministic, e.g.\ a grandmother's daughter can be either a mother or an aunt. Figure~\ref{fig:number_of_rules} shows the number of rules induced by the oracle and HtT, as well as the number of true positive rules in HtT. We can see that HtT discovers more than 85\% of the rules in all datasets.

\subsection{Results with Other LLMs}
\label{app:other_llms}

Here we further evaluate HtT and CoT baselines using \texttt{gemini-1.0-pro}, \texttt{Mistral-7B-Instruct-v0.3} and \texttt{Meta-Llama-3-8B-Instruct}. We denote these models as Gemini Pro, Mistral 7B and Llama 3 8B respectively. The temperature is set to 0.9 for Gemini Pro, 0.7 for Mistral 7B and 0.6 for Llama 3 8B following their default values.

Table~\ref{tab:clutrr_other}, \ref{tab:arithmetic_other} and \ref{tab:list_functions_other} list the results of these LLMs on CLUTRR, Arithmetic and List Functions respectively. Generally, we observe that HtT consistently improves over CoT for all models on three datasets, similar to the trend of GPT models. The improvement of HtT is larger for models with higher CoT performance, such as Gemini Pro, since higher CoT performance indicates better rule induction abilities. For Mistral 7B and Llama 3 8B, we notice their performance with zero-shot CoT is higher than their few-shot counterparts, sometimes even stronger than CoT+HtT, e.g.\ Llama 3 8B on List Functions. We conjecture the reason is that these models may have been heavily tuned on instruction datasets and lost some of their in-context learning abilities. By comparing Table~\ref{tab:clutrr_other}-\ref{tab:list_functions_other} with Table~\ref{tab:clutrr}-\ref{tab:list_functions}, we conclude that Gemini Pro is slightly better than GPT-3.5, while Mistral 7B and Llama 3 8B are worse than GPT-3.5.

\begin{table}[!h]
    \centering
    \footnotesize
    \caption{Results on the symbolic version of CLUTRR.}
    \label{tab:clutrr_other}
    \begin{adjustbox}{max width=\textwidth}
        \begin{tabular}{llcccccccccc}
            \toprule
            \bf{Model} & \bf{Prompt} & \bf{2 hops} & \bf{3 hops} & \bf{4 hops} & \bf{5 hops} & \bf{6 hops} & \bf{7 hops} & \bf{8 hops} & \bf{9 hops} & \bf{10 hops} & \bf{Average} \\
            \midrule
            & 0-shot CoT & 12.5 & 50.0 & 9.7 & 8.0 & 12.5 & 4.8 & 20.0 & 3.8 & 8.0 & 14.4 \\
            \cmidrule{2-12}
            Gemini Pro & 5-shot CoT & 37.5 & 22.2 & 25.8 & 48.0 & 31.3 & 28.6 & 26.7 & 19.2 & 32.0 & 30.1 \\
            & + HtT & 100.0 & 55.6 & 51.6 & 68.0 & 43.8 & 19.0 & 43.3 & 34.6 & 28.0 & \bf{49.3 (+19.2)} \\
            \midrule
            & 0-shot CoT & 50.0 & 22.2 & 9.7 & 12.0 & 6.3 & 14.3 & 3.3 & 3.8 & 12.0 & \bf{14.8} \\
            \cmidrule{2-12}
            Mistral 7B & 5-shot CoT & 12.5 & 11.1 & 22.6 & 12.0 & 6.3 & 9.5 & 3.3 & 7.7 & 4.0 & 9.9 \\
            & + HtT & 37.5 & 22.2 & 22.6 & 24.0 & 0.0 & 14.3 & 10.0 & 3.8 & 4.0 & \bf{15.4 (+5.5)} \\
            \midrule
            & 0-shot CoT & 50.0 & 33.3 & 19.4 & 12.0 & 6.3 & 0.0 & 3.3 & 3.8 & 12.0 & 15.6 \\
            \cmidrule{2-12}
            Llama3 8B & 5-shot CoT & 50.0 & 0.0 & 22.6 & 36.0 & 50.0 & 19.0 & 16.7 & 3.8 & 12.0 & 23.3 \\
            & + HtT & 37.5 & 33.3 & 32.3 & 32.0 & 25.0 & 23.8 & 43.3 & 19.2 & 12.0 & \bf{28.7 (+5.4)} \\
            \bottomrule
        \end{tabular}
    \end{adjustbox}
\end{table}

\begin{table}[!h]
    \centering
    \footnotesize
    \caption{Results on Arithmetic.}
    \label{tab:arithmetic_other}
    \begin{adjustbox}{max width=\textwidth}
        \begin{tabular}{llcccccccccc}
            \toprule
            \multirow{2}{*}{\bf{Model}} & \multirow{2}{*}{\bf{Prompt}} & \multicolumn{3}{c}{\bf{Base-16}} & \multicolumn{3}{c}{\bf{Base-11}} & \multicolumn{3}{c}{\bf{Base-9}} & \multirow{2}{*}{\bf{Average}}\\
            & & \bf{2 digits} & \bf{3 digits} & \bf{4 digits} & \bf{2 digits} & \bf{3 digits} & \bf{4 digits} & \bf{2 digits} & \bf{3 digits} & \bf{4 digits} \\
            \midrule
            & 0-shot CoT & 16.7 & 2.6 & 0.0 & 0.0 & 0.0 & 0.0 & 13.9 & 13.2 & 0.0 & 5.2 \\
            \cmidrule{2-12}
            Gemini Pro & 5-shot CoT & 77.8 & 50.0 & 26.9 & 36.1 & 18.4 & 15.4 & 47.2 & 21.1 & 15.4 & 34.3 \\
            & + HtT & 91.7 & 57.9 & 38.5 & 55.6 & 36.8 & 23.1 & 77.8 & 28.9 & 23.1 & \bf{48.2 (+13.9)} \\
            \midrule
            & 0-shot CoT & 0.0 & 0.0 & 0.0 & 2.8 & 0.0 & 0.0 & 5.6 & 2.6 & 0.0 & 1.2 \\
            \cmidrule{2-12}
            Mistral 7B & 5-shot CoT & 2.8 & 0.0 & 0.0 & 16.7 & 0.0 & 0.0 & 13.9 & 0.0 & 0.0 & 3.7 \\
            & + HtT & 5.6 & 0.0 & 0.0 & 16.7 & 0.0 & 0.0 & 27.8 & 0.0 & 0.0 & \bf{5.9 (+2.2)} \\
            \midrule
            & 0-shot CoT & 0.0 & 0.0 & 0.0 & 0.0 & 0.0 & 0.0 & 2.8 & 0.0 & 0.0 & 0.3 \\
            \cmidrule{2-12}
            Llama3 8B & 5-shot CoT & 2.8 & 0.0 & 0.0 & 25.0 & 0.0 & 0.0 & 41.7 & 2.6 & 0.0 & 8.0 \\
            & + HtT & 8.3 & 0.0 & 0.0 & 33.3 & 7.9 & 0.0 & 44.4 & 5.3 & 0.0 & \bf{11.7 (+3.7)} \\
            \bottomrule
        \end{tabular}
    \end{adjustbox}
\end{table}

\pagebreak

\begin{table}[!h]
    \centering
    \footnotesize
    \caption{Results on List Functions.}
    \label{tab:list_functions_other}
    \begin{adjustbox}{max width=\textwidth}
        \begin{tabular}{llcccccccc}
            \toprule
            \multirow{2}{*}{\bf{Model}} & \multirow{2}{*}{\bf{Prompt}} & \multicolumn{4}{c}{\bf{Raw Accuracy}} & \multicolumn{4}{c}{\bf{Task Accuracy}} \\
            & & \bf{P1} & \bf{P2} & \bf{P3} & \bf{Average} & \bf{P1} & \bf{P2} & \bf{P3} & \bf{Average} \\
            \midrule
            & 0-shot CoT & 27.7 & 33.4 & 19.3 & 26.8 & 15.0 & 25.0 & 6.0 & 15.3 \\
            \cmidrule{2-10}
            Gemini Pro & 4-shot CoT & 43.4 & 37.5 & 22.6 & 34.5 & 25.0 & 30.0 & 9.3 & 21.4 \\
            & + HtT & 55.5 & 54.4 & 28.6 & \bf{46.1 (+11.6)} & 33.8 & 40.0 & 12.0 & \bf{28.6 (+7.2)} \\
            \midrule
            & 0-shot CoT & 26.1 & 27.8 & 11.2 & 21.7 & 11.3 & 20.0 & 2.7 & 11.3 \\
            \cmidrule{2-10}
            Mistral 7B & 4-shot CoT & 19.4 & 22.5 & 7.3 & 16.4 & 6.3 & 15.0 & 0.0 & 7.1 \\
            & + HtT & 30.2 & 32.8 & 10.0 & \bf{24.3 (+7.9)} & 16.3 & 25.0 & 3.3 & \bf{14.9 (+7.8)} \\
            \midrule
            & 0-shot CoT & 22.2 & 15.9 & 4.8 & 14.3 & 15.0 & 15.0 & 1.3 & \bf{10.4} \\
            \cmidrule{2-10}
            Llama3 8B & 4-shot CoT & 11.1 & 12.8 & 2.1 & 8.7 & 5.0 & 10.0 & 0.0 & 5.0 \\
            & + HtT & 15.2 & 23.4 & 9.6 & \bf{16.1 (+7.4)} & 8.8 & 15.0 & 2.7 & 8.8 (+3.8) \\
            \bottomrule
        \end{tabular}
    \end{adjustbox}
\end{table}
\section{Limitations and Future Work}

One limitation of HtT is that it requires the base model to have reasonably strong knowledge and retrieval ability. As shown in Table~\ref{tab:arithmetic}, the gain of HtT for GPT-3.5 is very marginal due to weak knowledge of non-decimal systems. Even with a rule library induced by GPT-4, GPT-3.5 has issues in retrieving correct rules, especially in very counterfactual settings like base-11 and base-9. Another limitation is that the number of rules is limited by the context length of the LLM. It remains an open problem to scale up deductive reasoning when the rule library cannot fit into the input context of an LLM.

\part{Systems}
\label{part:system}
\chapter[A Library for Structured Data and Applications]{A Library for\linebreak Structured Data and Applications}
\label{cha:torchdrug}

One hidden challenge in developing representation learning for structured data is the lack of proper infrastructure. As structured data often has flexible size and sparsity, they do not fit into modern machine learning frameworks designed for tensor computation on GPUs. In existing implementations, this issue is circumvented by padding structured data into grid data, which is both inefficient and counter-intuitive for debugging, posing a high barrier for average developers to enter domains related to structured data.

This chapter presents TorchDrug, a library we developed for accelerating development on structured data, and related applications, such as representation learning on graphs, molecules and proteins. TorchDrug simplifies the development of machine learning on structured data (reduce the lines of code by 20$\times$), and has brought many researchers and developers into the field of knowledge graph reasoning and drug discovery.

\smallskip \emph{This chapter is based on our work~\cite{zhu2022torchdrug}\footnote{The code is available at \url{https://github.com/DeepGraphLearning/torchdrug}}.}

\section{Overview}

Drug discovery is a long and costly process, taking on average 10 years and costing 2.5 billion US dollars to develop a new drug~\cite{dimasi2016innovation}. Machine learning has huge potential to accelerate the process of drug discovery by extracting evidence through mining and analyzing a large amount of data in the biomedical domain (e.g.\ scientific literature, bioassays, and clinical trials). Recently, machine learning methods have made significant progress in many drug discovery tasks, such as protein structure prediction~\cite{baek2021accurate, jumper2021highly}, molecular property prediction~\cite{duvenaud2015convolutional, hu2019strategies}, de novo molecular design and optimization~\cite{you2018graph, shi2020graphaf}, reaction prediction~\cite{jin2017predicting, bradshaw2019generative}, retrosynthesis prediction~\cite{dai2019retrosynthesis, shi2020graph}, and drug repurposing~\cite{wang2020covid, zhao2020biomedical}. However, it remains a challenge for machine learning researchers to work in this domain for a few reasons: (1) lacking domain knowledge of what are important tasks in the domain; (2) no standard benchmarks of different methods due to their completely different implementations; (3) the large cost of implementing complicated data preprocessing pipelines for each task.

To accelerate the process of drug discovery through machine learning, we see a critical need to develop an open-source machine learning platform for drug discovery. Here we present such a platform, called TorchDrug. TorchDrug provides a hierarchical interface to accommodate different demands in the development of drug discovery. At the low level, TorchDrug encapsulates graphs and molecules as basic data structures, and provides GPU-accelerated graph operations, along with standard datasets in a PyTorch-style interface. This significantly reduces the reliance on domain knowledge in model implementation. At the mid-level, TorchDrug supplies popular building blocks of graph representation learning models (e.g.\ MPNN~\cite{gilmer2017neural}), which can be used to quickly construct models for drug discovery. The high level contains reusable routines for a variety of important tasks in drug discovery, ranging from molecular property prediction, pretrained molecular representations, de novo molecule design and optimization, retrosynthesis prediction, knowledge graph reasoning (e.g.\ for drug repurposing), protein property prediction and protein-protein interaction prediction. Figure~\ref{fig:overview} presents an overview of the TorchDrug library. TorchDrug was released in 2021 and has received more than 50,000 downloads on PyPI and Anaconda by the year of 2024.

\begin{figure}[!h]
    \centering
    \includegraphics[width=0.85\textwidth]{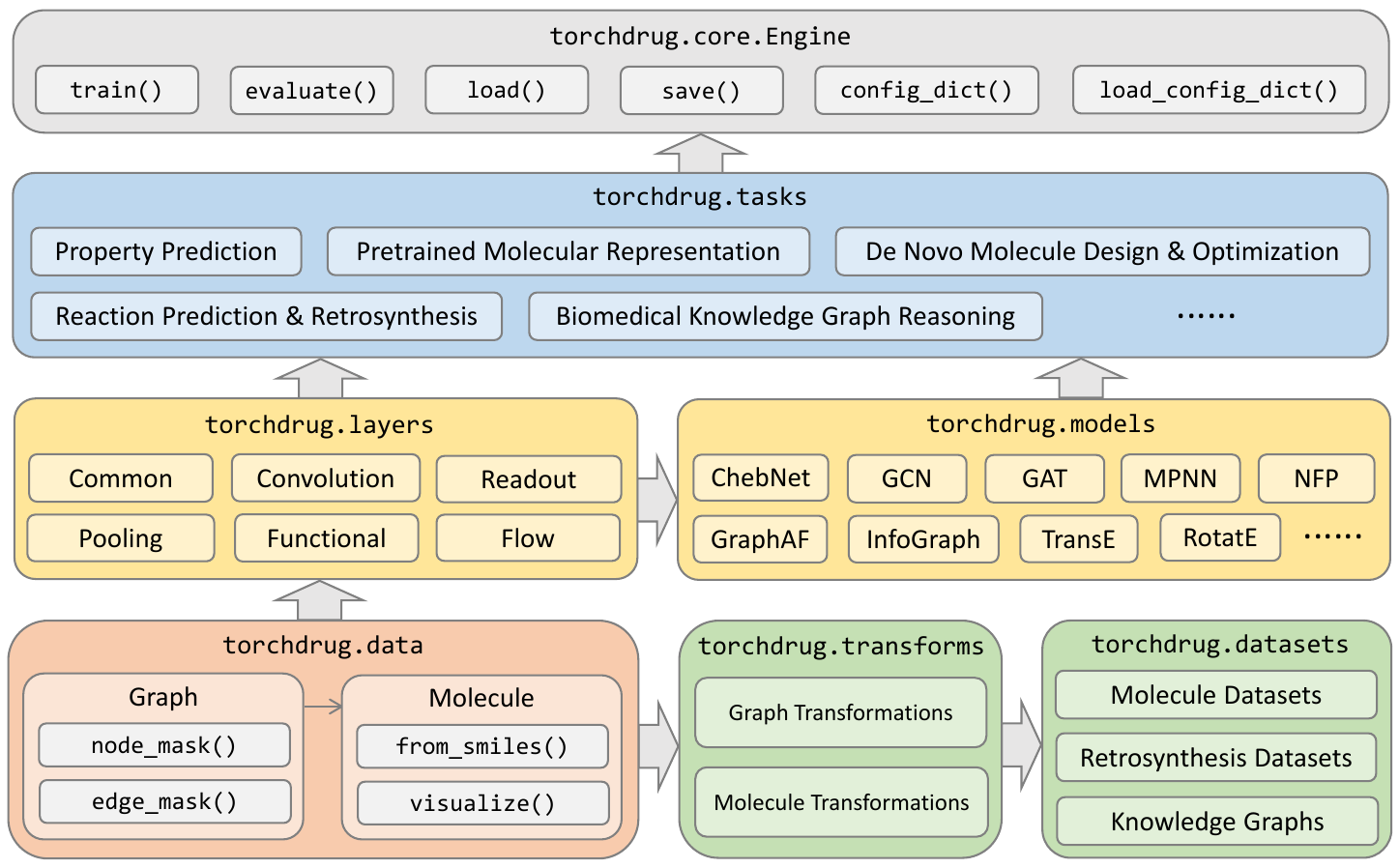}
    \caption[Overview of the hierarchy of TorchDrug]{Overview of the hierarchy of TorchDrug. The low level provides data structures and datasets, the mid level supplies representation learning layers and models, and the high level contains task routines. At the top of the hierarchy, \texttt{Engine} provides interface for scalable training and inference over multiple CPUs or GPUs.}
    \label{fig:overview}
\end{figure}
\section{Key Features}

TorchDrug offers two key features: (1) low-level data structures and graph operations that can be manipulated with minimal domain knowledge and GPU acceleration; (2) mid-level datasets, layers, models and high-level tasks that support rapid prototyping of ideas.

\subsection{Data Structures and Graph Operations}

The \texttt{data} module implements basic data structures for graph machine learning and drug discovery. It contains classes for homogeneous graphs, knowledge graphs (together in \texttt{data.Graph}), molecules (\texttt{data.Molecule}) and proteins (\texttt{data.Protein}). These data structures are designed to be the first-class citizen in TorchDrug, where many functions and classes take them as either input or output. Each class maintains the structure of a graph, with \texttt{data.Molecule} and \texttt{data.Protein} additionally supporting sanity check of the graph as a molecule or a protein respectively.
To deal with the diverse features used in drug discovery tasks, we also design a registration mechanism to support arbitrary number of node-level, edge-level or graph-level features in the data structures.

Graph operations are designed as member functions of the above data structures in an object-oriented programming style. While many GNN implementation relies on CPU-based libraries (e.g. NumPy~\footnote{\url{https://numpy.org/}} and NetworkX~\footnote{\url{https://networkx.org/}}\cite{hagberg2008exploring}) to perform graph operations, we directly implement graph operations as PyTorch operations. This allows us to seamlessly switch between CPUs and GPUs, as well as perform auto differentiation through the graph operations. For example, the following code snippet creates a batch of 4 molecules, sends it to a GPU, repeats the batch using GPU computation and visualizes the results.
\begin{figure}[!h]
    \centering
    \begin{minted}[bgcolor=mygray, fontsize=\small, baselinestretch=1.15]{python}
from torchdrug import data
smiles_list = ["N(Nc1ccccc1)c2ccccc2", "NC(=O)c1cccnc1"]
mols = data.PackedMolecule.from_smiles(smiles_list)
mols = mols.cuda()
mols = mols.repeat(2)
mols.visualize(num_row=1)
    \end{minted}
\end{figure}
\begin{figure}[!h]
    \centering
    \includegraphics[width=0.9\textwidth]{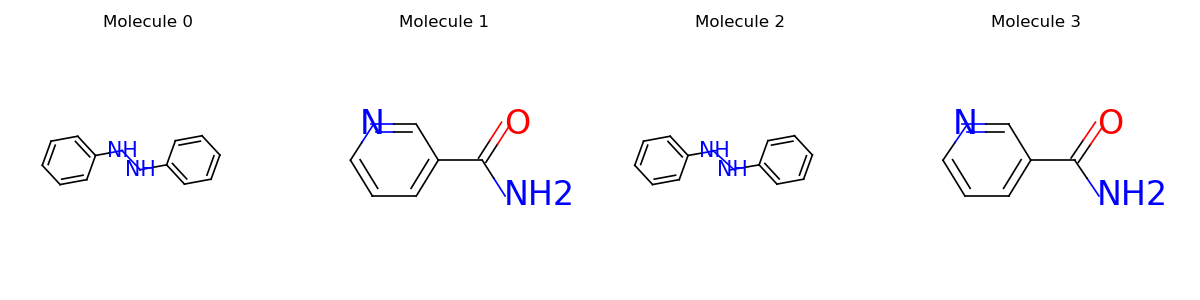}
    \vspace{-1em}
\end{figure}

Our data structures support a lot of graph operations, such as batching (\texttt{data.Graph.pack}) and de-batching (\texttt{data.PackedGraph.unpack}), node masking (\texttt{data.Graph.node\_mask}), edge masking (\texttt{data.Graph.edge\_mask}) and graph masking (\texttt{data.Graph.graph\_mask}). A list of core graph operations in TorchDrug is showed in Table~\ref{tab:graph_operation}. The data structures also contain several predefined node-level, edge-level and graph-level attributes that are useful for building machine learning models. For example, the type of atoms in a molecule may be used as an input feature to some property prediction model. For proteins, we additionally support residue-level attributes. Users may also register arbitrary attributes depending on their tasks. All the attributes are automatically maintained in all of our graph operations. The following example shows that the type of atoms and bonds are maintained after we mask out bonds without carbon atoms.
\begin{figure}[!h]
    \centering
    \begin{minted}[bgcolor=mygray, fontsize=\small, baselinestretch=1.15]{python}
import torchdrug as td
from torchdrug import data
smiles_list = ["CCSCCSP(=S)(OC)OC", "CCOC(=O)N"]
mols = data.PackedMolecule.from_smiles(smiles_list)
mols.visualize()
    \end{minted}
    \includegraphics[width=0.45\textwidth]{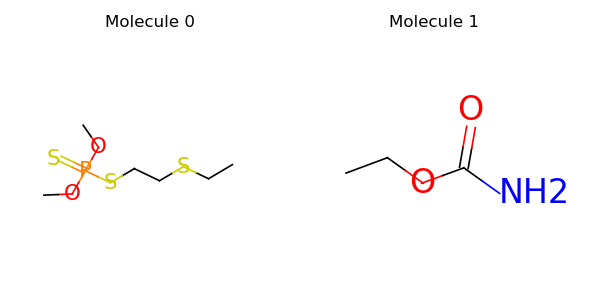}
    \vspace{-1em}
    \begin{minted}[bgcolor=mygray, fontsize=\small, baselinestretch=1.15]{python}
node_in, node_out, bond_type = mols.edge_list.t()
edge_mask = (mols.atom_type[node_in] == td.CARBON) | \
            (mols.atom_type[node_out] == td.CARBON)
mols = mols.edge_mask(edge_mask)
mols.visualize()
    \end{minted}
\end{figure}
\begin{figure}[!h]
    \centering
    \includegraphics[width=0.45\textwidth]{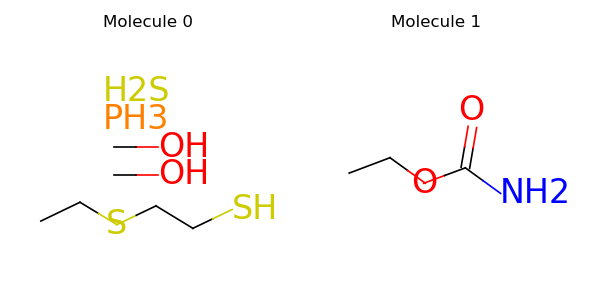}
    \vspace{-0.5em}
\end{figure}

To make TorchDrug more accessible to users, our data structures support a wide range of I/O interface, such as edge list, adjacency matrix, RDKit~\cite{landrum2006rdkit} molecules, or SMILES\footnotemark strings. For molecules and proteins, we additionally support common chemical feature functions for atoms, bonds and residues, so that users can easily obtain a good input representation for machine learning models.
\footnotetext{\url{https://en.wikipedia.org/wiki/Simplified_molecular-input_line-entry_system}}

\begin{table}[!h]
    \centering
    \caption[Core graph operations supported by TorchDrug]{Core graph operations supported by TorchDrug. All the operations of \texttt{data.Graph} are inherited by \texttt{data.Molecule} and \texttt{data.Protein}. All the operations over a single graph, molecule or protein are inherited by their corresponding batched variants.}
    \begin{adjustbox}{max width=\textwidth}
    \begin{tabular}{lll}
        \toprule
        \bf{Class} & \bf{API} & \bf{Graph Operation} \\
        \midrule
        \multirow{8}{*}{PyTorch-like}
        & \texttt{data.Graph.clone} & Clone this graph \\
        & \texttt{data.Graph.detach} & Detach this graph \\
        & \texttt{data.Graph.cpu} & Move this graph to CPU \\
        & \texttt{data.Graph.cuda} & Move this graph to GPU \\
        & \texttt{data.Graph.copy\_} & Copy data from another graph \\
        & \texttt{data.Graph.full} & Return a fully connected graph over nodes \\
        & \texttt{data.Graph.repeat} & Repeat this graph like \texttt{torch.repeat} \\
        & \texttt{data.PackedGraph.repeat\_interleave} & Repeat this graph like \texttt{torch.repeat\_interleave} \\
        \midrule
        \multirow{2}{*}{Node-level}
        & \texttt{data.Graph.node\_mask} & Mask out some nodes from this graph \\
        & \texttt{data.Graph.compact} & Remove isolated nodes \\
        \midrule
        \multirow{4}{*}{Edge-level}
        & \texttt{data.Graph.edge\_mask} & Mask out some edges from this graph \\
        & \texttt{data.Graph.directed} & Return a directed version of this graph \\
        & \texttt{data.Graph.undirected} & Return an undirected version of this graph \\
        & \texttt{data.Graph.match} & Search specific edges in this graph \\
        \midrule
        \multirow{7}{*}{Graph-level}
        & \texttt{data.Graph.connected\_components} & Split a graph into connected components \\
        & \texttt{data.Graph.split} & Split a graph into a batch of graphs \\
        & \texttt{data.Graph.pack} & Pack multiple graphs into a batch \\
        & \texttt{data.Graph.line\_graph} & Return a line graph of this graph \\
        & \texttt{data.PackedGraph.graph\_mask} & Mask out some graphs from this batch \\
        & \texttt{data.PackedGraph.merge} & Merge some graphs into a smaller batch \\
        & \texttt{data.PackedGraph.unpack} & Unpack a batch into multiple graphs \\
        \midrule
        \multirow{1}{*}{Molecule}
        & \texttt{data.Molecule.ion\_to\_molecules} & Convert ions to molecules \\
        \midrule
        \multirow{1}{*}{Protein}
        & \texttt{data.Protein.residue\_mask} & Mask out some residues from this protein \\
        \bottomrule
    \end{tabular}
    \end{adjustbox}
    \label{tab:graph_operation}
\end{table}

\subsection{Datasets, Layers, Models and Tasks}
\label{sec:modules}

\smallskip \noindent \textbf{Datasets.}
The \texttt{datasets} module provides 47 common datasets for 7 drug discovery tasks. These datasets inherit the \texttt{Dataset} class from PyTorch and further provide data loading and \texttt{\_\_getitem\_\_} functions, which facilitates the interaction with dataloaders in PyTorch. The following code snippet loads the ClinTox dataset for molecular property prediction and splits it into training, validation and test sets.
\begin{minted}[bgcolor=mygray, fontsize=\small, baselinestretch=1.15]{python}
import torch
from torchdrug import datasets
dataset = datasets.ClinTox("~/datasets/")
lengths = [int(0.8 * len(dataset)), int(0.1 * len(dataset))]
lengths += [len(dataset) - sum(lengths)]
train_set, valid_set, test_set = torch.utils.data.random_split(dataset, lengths)
\end{minted}

\smallskip \noindent \textbf{Layers and Models.}
The \texttt{layers} and \texttt{models} modules implement layers and models for representation learning respectively. This lets users switch between standard models or custom models from standard layers. Our interface follows the convention in PyTorch, which minimizes the cognitive load of users. Classes in layers (e.g.\ \texttt{GCNConv}) are similar to the layers in \texttt{torch.nn}, while classes in \texttt{models} (e.g.\ \texttt{GCN}) are similar to \texttt{torchvision.models}. We also include the common tricks used in previous state-of-the-art models~\cite{velivckovic2018graph, xu2018representation, xu2019powerful}, such as residual connection~\cite{he2016deep}, batch normalization~\cite{ioffe2015batch} and jumping knowledge~\cite{xu2018representation}, to provide more flexibility for users.

\smallskip \noindent \textbf{Tasks.}
The \texttt{tasks} module contains high-level routines of machine learning tasks in drug discovery. Typically, these include dataset preprocessing, prediction, training and evaluation. Each task is abstracted as a model-agnostic class in tasks, which can be used with any basic representation learning models (e.g.\ GIN~\cite{xu2019powerful}). Currently, TorchDrug supports 7 tasks: molecular property prediction, pretrained molecular representations, de novo molecule design, retrosynthesis, knowledge graph reasoning, protein property prediction and pretrained protein representations. A full list of tasks and models supported by TorchDrug are shown in Table~\ref{tab:task}. Here is an example of constructing a molecular property prediction task based on the GIN model~\cite{xu2019powerful}.
\begin{minted}[bgcolor=mygray, fontsize=\small, baselinestretch=1.15]{python}
from torchdrug import models, tasks
model = models.GIN(input_dim=dataset.node_feature_dim,
                   hidden_dims=[256, 256, 256, 256],
                   short_cut=True, batch_norm=True, concat_hidden=True)
task = tasks.PropertyPrediction(model, task=dataset.tasks,
                                criterion="bce", metric=("auprc", "auroc"))
\end{minted}

\begin{table}[!h]
    \centering
    \caption[Drug discovery tasks and models supported by TorchDrug]{Drug discovery tasks and models supported by TorchDrug.}
    \begin{adjustbox}{max width=\textwidth}
        \begin{tabular}{lllll}
            \toprule
            \bf{Task} & \bf{Model} \\
            \midrule
            & Neural Fingerprint~\cite{duvenaud2015convolutional} & ChebyNet~\cite{defferrard2016convolutional} & GCN~\cite{kipf2017semi} \\
            \bf{Molecular Property Prediction} & ENN-S2S~\cite{gilmer2017neural} & SchNet~\cite{schutt2017schnet} & GAT~\cite{velivckovic2018graph} \\
            & RGCN~\cite{schlichtkrull2018modeling} & GIN~\cite{xu2019powerful} \\
            \midrule
            \bf{Pretrained Molecular} & InfoGraph~\cite{sun2020infograph} & Edge Prediction~\cite{hamilton2017inductive} & Attribute Masking~\cite{hu2019strategies} \\
            \bf{Representations} & Context Prediction~\cite{hu2019strategies} \\
            \midrule
            \bf{De Novo Molecule Design}
            & GCPN~\cite{you2018graph} & GraphAF~\cite{shi2020graphaf} \\
            \midrule
            \bf{Retrosynthesis Prediction}
            & G2Gs~\cite{shi2020graph} \\
            \midrule
            & TransE~\cite{bordes2013translating} & DistMult~\cite{yang2015embedding} & ComplEx~\cite{trouillon2016complex} \\ 
            \bf{Knowledge Graph Reasoning}& NeuralLP~\cite{yang2017differentiable} & SimplE~\cite{kazemi2018simple} & RotatE~\cite{sun2019rotate} \\
            & KBGAT~\cite{nathani2019learning} \\
            \midrule
            \bf{Protein Property} & LSTM~\cite{rao2019evaluating} & ResNet~\cite{rao2019evaluating} & BERT~\cite{rao2019evaluating} \\
            \bf{Prediction} & CNN~\cite{shanehsazzadeh2020transfer} & ESM~\cite{rao2019evaluating} & GearNet~\cite{zhang2023protein} \\
            \midrule
            \bf{Pretrained Protein} & Distance Prediction~\cite{zhang2023protein} & Angle Prediction~\cite{zhang2023protein} & Dihedral Prediction~\cite{zhang2023protein} \\
            \bf{Representations} & Multiview Contrast~\cite{zhang2023protein} \\
            \bottomrule
        \end{tabular}
    \end{adjustbox}
    \label{tab:task}
\end{table}

\subsection{Training and Evaluation}

The \texttt{core} module links all the above modules and coordinates them for model training and evaluation with either single or multiple CPUs and GPUs. It wraps up all the ingredients of machine learning, such as models, datasets, optimizers, learning rate schedulers, together in a class, and automatically synchronize them when trained with multiple workers. It also provides convenient interface for saving and loading both parameters and hyperparameters. Using the dataset, model and task defined in Section~\ref{sec:modules}, we train the model and evaluate it on the validation set with the following code
\begin{minted}[bgcolor=mygray, fontsize=\small, baselinestretch=1.15]{python}
from torchdrug import core
optimizer = torch.optim.Adam(task.parameters(), lr=1e-3)
solver = core.Engine(task, train_set, valid_set, test_set, optimizer,
                     batch_size=1024, gpus=[0])
solver.train(num_epoch=100)
solver.evaluate("valid")
\end{minted}
\vspace{-1.8em}
\section{Model Benchmark}
\label{app:performance_benchmark}

We provide a comprehensive benchmark of models for 5 tasks implemented in TorchDrug.

\smallskip \noindent \textbf{Molecular Property Prediction.}
We benchmark 4 property prediction models: NFP~\cite{duvenaud2015convolutional}, GCN~\cite{kipf2017semi}, ENN-S2S~\cite{gilmer2017neural} and GIN~\cite{xu2019powerful}. The 15 benchmark datasets include QM9, QM8, BACE, BBBP, CEP, HIV, ClinTox, ESOL, FreeSolv, Lipophilicity, SIDER, Tox21, ToxCast, MUV and Malaria~\cite{wu2018moleculenet}. We consider both vanilla random split and scaffold-based random split for each dataset. The split for train/validation/test sets is 8:1:1. For each model on each dataset, we evaluate it with 5 different random splits and report the mean and the standard deviation of the performance. For datasets with a lot of tasks, only the first 16 tasks are plotted due to space limitations. We plot mean absolute error (MAE) and the coefficient of determination (R2) for regression tasks, AUROC and AUPRC for binary classification tasks. Figure~\ref{fig:property_prediction} and \ref{fig:property_prediction_scaffold} show the results on random splits and scaffold splits respectively.

\smallskip \noindent \textbf{Pretrained Molecular Representations.}
We benchmark 4 models: Edge Prediction~\cite{hamilton2017inductive}, InfoGraph~\cite{sun2020infograph}, Attribute Masking~\cite{hu2019strategies} and Context Prediction~\cite{hu2019strategies}. We follow the protocol in \cite{hu2019strategies} to first perform self-supervised pretraining on ZINC15~\cite{sterling2015zinc}, then perform supervised pretraining on ChEMBL~\cite{mayr2018large}. The models are finally finetuned and evaluated on standard property prediction datasets. Each property prediction dataset is split into train/validation/test with a ratio of 8:1:1 based on the scaffolds of molecules. For each model on each dataset, we evaluate with 10 random splits and report the mean and the standard deviation of the AUROC metric. Results are listed in Table~\ref{tab:pretrain}.

\begin{table}[!h]
    \centering
    \caption[Results of pretrained molecular representations]{Results of pretrained molecular representations.}
    \begin{adjustbox}{max width=\textwidth}
    \begin{tabular}{lccccccccc}
        \toprule
        \multirow{2}{*}{\bf{Method}} & \multicolumn{8}{c}{\bf{Dataset}} & \multirow{2}{*}{\bf{Average$\uparrow$}} \\
        & \bf{BBBP} & \bf{Tox21} & \bf{ToxCast} & \bf{Sider} & \bf{ClinTox} & \bf{MUV} & \bf{HIV} & \bf{Bace} \\
        \midrule
        No Pretrain & 67.1$\pm$2.9 & 75.0$\pm$0.2 & 60.6$\pm$0.7 & 58.9$\pm$0.8 & 60.8$\pm$3.9 & 64.3$\pm$3.4 & 76.4$\pm$1.6 & 66.5$\pm$9.0 & 66.2 \\
        Edge Prediction~\cite{hamilton2017inductive} & 67.1$\pm$2.6 & 74.6$\pm$0.7 & 69.8$\pm$0.5 & 59.4$\pm$1.5 & 59.0$\pm$2.6 & 66.8$\pm$1.0 & 76.3$\pm$2.0 & 68.4$\pm$3.9 & 67.7 \\
        InfoGraph~\cite{sun2020infograph} & 68.9$\pm$0.6 & 76.4$\pm$0.4 & 71.2$\pm$0.6 & 59.8$\pm$0.7 & 70.3$\pm$4.2 & 69.4$\pm$0.8 & 75.5$\pm$0.7 & 73.7$\pm$2.6 & 70.7\\
        Attribute Masking~\cite{hu2019strategies} & 65.2$\pm$0.9 & 75.8$\pm$0.5 & 70.6$\pm$0.6 & 58.9$\pm$0.9 & 79.0$\pm$2.3 & 68.3$\pm$2.1 & 76.9$\pm$0.9 & 78.1$\pm$0.8 & 71.6\\
        Context Prediction~\cite{hu2019strategies} & 71.1$\pm$1.8 & 75.6$\pm$0.3 & 71.1$\pm$0.3 & 61.7$\pm$0.5 & 65.9$\pm$1.9 & 68.5$\pm$0.6 & 77.1$\pm$0.3 & 78.6$\pm$0.5 & 71.2 \\
        \bottomrule
    \end{tabular}
    \end{adjustbox}
    \label{tab:pretrain}
\end{table}

\begin{wraptable}{R}{0.5\textwidth}
    \vspace{-1.5em}
    \centering
    \caption[Results of goal-directed property optimization]{Results of goal-directed property optimization.}
    \footnotesize
    \begin{tabular}{lcc}
        \toprule
        \bf{Method} & \bf{Penalized LogP$\uparrow$} & \bf{QED$\uparrow$} \\
        \midrule
        ZINC250k (Dataset) & 4.52 & 0.948 \\
        \midrule
        GCPN~\cite{you2018graph} & 6.560 & 0.948 \\
        GraphAF~\cite{shi2020graphaf} & 5.630 & 0.948 \\
        \bottomrule
    \end{tabular}
    \label{tab:molecule_generation}
\end{wraptable}

\smallskip \noindent \textbf{De Novo Molecule Design.}
We benchmark graph generative models for goal-directed property optimization, which aims to generate novel molecules with optimized chemical properties. We first pretrain the models on ZINC250k~\cite{irwin2012zinc} dataset, and then apply reinforcement learning algorithms to finetune the networks towards desired chemical properties, e.g. penalized logP and QED score. Penalized logP score is the octanol-water partition coefficient penalized by the synthetic accessibility score and the number of long cycles. QED score measures the drug-likeness of the molecule. We report the top-1 property scores of generated molecules by different models in Table~\ref{tab:molecule_generation}.

\smallskip \noindent \textbf{Retrosynthesis.}
We benchmark retrosynthesis models on the standard USPTO50k~\cite{lowe2012extraction} dataset, which contains 50k atom mapped reactions with 10 reaction types. We consider two settings, where the reaction type is given and unknown respectively. The top-k accuracy of the predictions is reported. Table~\ref{tab:retrosynthesis} shows the results for retrosynthesis models.

\begin{table}[!h]
    \centering
    \caption[Results of retrosynthesis]{Results of retrosynthesis.}
    \footnotesize
    \begin{tabular}{lcccccccc}
        \toprule
        \multirow{2}{*}{\bf{Method}}
        & \multicolumn{4}{c}{\bf{Given Reaction Class}} & \multicolumn{4}{c}{\bf{Unknown Reaction Class}} \\
        & \bf{Top-1$\uparrow$} & \bf{Top-3$\uparrow$} & \bf{Top-5$\uparrow$} & \bf{Top-10$\uparrow$} & \bf{Top-1$\uparrow$} & \bf{Top-3$\uparrow$} & \bf{Top-5$\uparrow$} & \bf{Top-10$\uparrow$} \\
        \midrule
        G2Gs~\cite{shi2020graph} & 0.639 & 0.852 & 0.904 & 0.938 & 0.438 & 0.677 & 0.748 & 0.822 \\
        \bottomrule
    \end{tabular}
    \label{tab:retrosynthesis}
\end{table}

\smallskip \noindent \textbf{Knowledge Graph Reasoning.}
We benchmark knowledge graph reasoning models on a biomedical knowledge graph Hetionet~\cite{himmelstein2017systematic}. The 7 models include TransE~\cite{bordes2013translating}, DistMult~\cite{yang2015embedding}, ComplEx~\cite{trouillon2016complex}, SimplE~\cite{kazemi2018simple}, RotatE~\cite{sun2019rotate}, KBGAT~\cite{nathani2019learning} and NeuralLP~\cite{yang2017differentiable}. Following standard knowledge graph completion evaluation, we report the mean rank (MR), mean reciprocal rank (MRR), hits at K (H@K) under filtered ranking setting. Table~\ref{tab:biokg_reasoning} summarizes the results of biomedical knowledge graph reasoning.

\begin{table}[!h]
    \centering
    \caption[Results of knowledge graph reasoning on Hetionet]{Results on knowledge graph reasoning on Hetionet.}
    \footnotesize
    \begin{tabular}{lccccc}
        \toprule
        \bf{Method} & \bf{MR$\downarrow$} & \bf{MRR$\uparrow$} & \bf{H@1$\uparrow$} & \bf{H@3$\uparrow$} & \bf{H@10$\uparrow$} \\
        \midrule
        TransE~\cite{bordes2013translating} & 1088 & 0.162 & 0.102 & 0.173 & 0.284 \\
        DistMult~\cite{yang2015embedding} & 941 & 0.187 & 0.128 & 0.199 & 0.304 \\
        ComplEx~\cite{trouillon2016complex} & 800 & 0.235 & 0.166 & 0.255 & 0.374 \\
        SimplE~\cite{kazemi2018simple} & 893 & 0.194 & 0.134 & 0.207 & 0.313 \\
        RotatE~\cite{sun2019rotate} & 744 & 0.257 & 0.185 & 0.282 & 0.403 \\
        KBGAT~\cite{nathani2019learning} & 1713 & 0.058 & 0.023 & 0.055 & 0.130 \\
        NeuralLP~\cite{yang2017differentiable} & 4017 & 0.175 & 0.128 & 0.182 & 0.273 \\
        \bottomrule
    \end{tabular}
    \label{tab:biokg_reasoning}
\end{table}

\pagebreak

\begin{figure}[!h]
    \centering
    \begin{subfigure}{0.9\textwidth}
        \includegraphics[width=\textwidth]{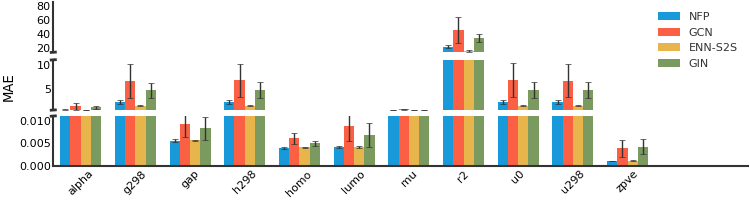}
        \includegraphics[width=\textwidth]{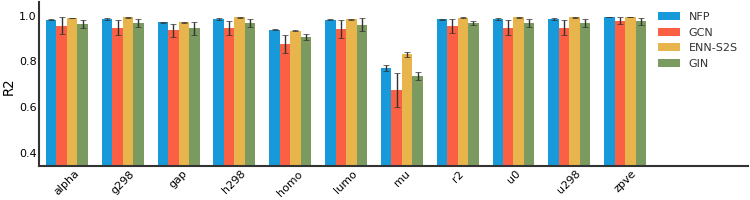}
        \caption{QM9}
    \end{subfigure}\
    \begin{subfigure}{0.9\textwidth}
        \includegraphics[width=\textwidth]{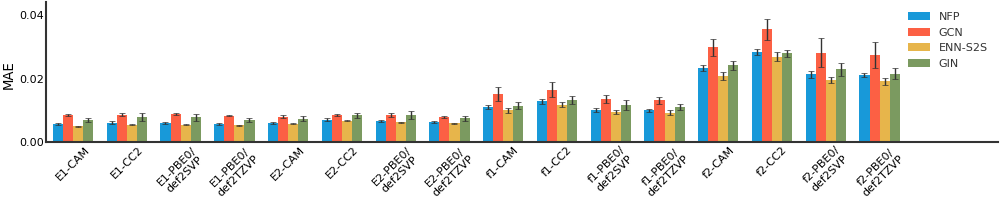}
        \includegraphics[width=\textwidth]{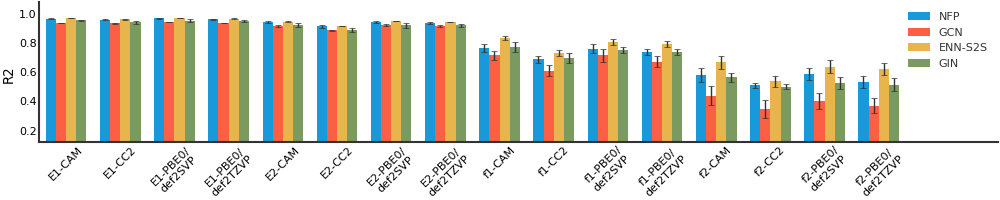}
        \caption{QM8}
    \end{subfigure}
    \begin{subfigure}{0.45\textwidth}
        \includegraphics[width=0.45\textwidth]{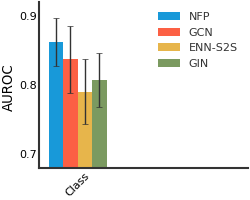}
        \includegraphics[width=0.45\textwidth]{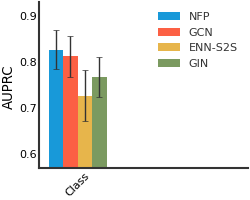}
        \caption{BACE}
    \end{subfigure}
    \begin{subfigure}{0.45\textwidth}
        \includegraphics[width=0.45\textwidth]{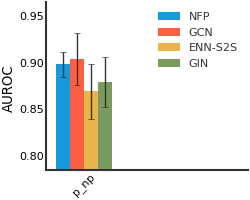}
        \includegraphics[width=0.45\textwidth]{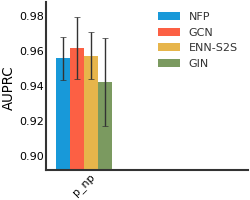}
        \caption{BBBP}
    \end{subfigure}
    \begin{subfigure}{0.45\textwidth}
        \includegraphics[width=0.45\textwidth]{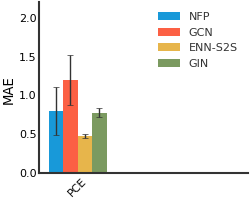}
        \includegraphics[width=0.45\textwidth]{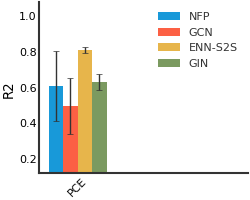}
        \caption{CEP}
    \end{subfigure}
    \begin{subfigure}{0.45\textwidth}
        \includegraphics[width=0.45\textwidth]{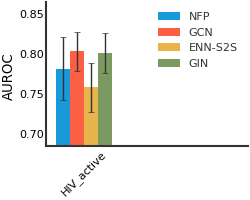}
        \includegraphics[width=0.45\textwidth]{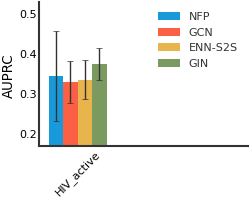}
        \caption{HIV}
    \end{subfigure}
    \caption[Molecular property prediction result on 15 datasets with random splits]{Molecular property prediction result on 15 datasets with random splits.}
    \label{fig:property_prediction}
\end{figure}
\begin{figure}[!h]\ContinuedFloat
    \begin{subfigure}{0.45\textwidth}
        \includegraphics[width=0.45\textwidth]{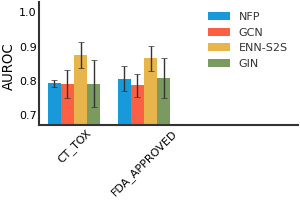}
        \includegraphics[width=0.45\textwidth]{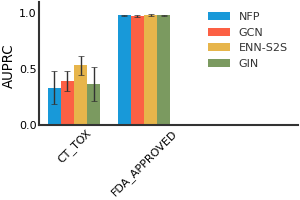}
        \caption{ClinTox}
    \end{subfigure}
    \begin{subfigure}{0.45\textwidth}
        \includegraphics[width=0.45\textwidth]{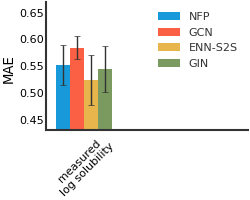}
        \includegraphics[width=0.45\textwidth]{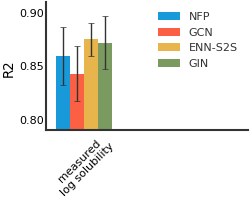}
        \caption{ESOL}
    \end{subfigure}
    \begin{subfigure}{0.45\textwidth}
        \includegraphics[width=0.45\textwidth]{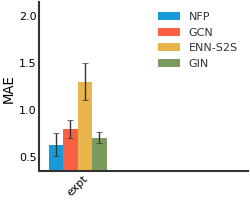}
        \includegraphics[width=0.45\textwidth]{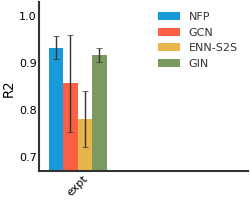}
        \caption{FreeSolv}
    \end{subfigure}
    \begin{subfigure}{0.45\textwidth}
        \includegraphics[width=0.45\textwidth]{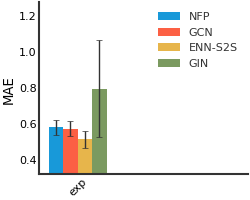}
        \includegraphics[width=0.45\textwidth]{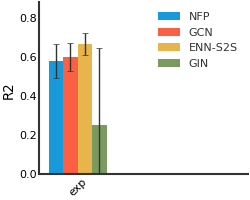}
        \caption{Lipophilicity}
    \end{subfigure}
    \begin{subfigure}{0.9\textwidth}
        \includegraphics[width=\textwidth]{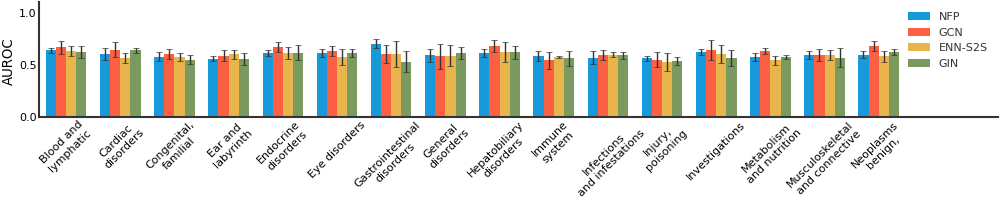}
        \includegraphics[width=\textwidth]{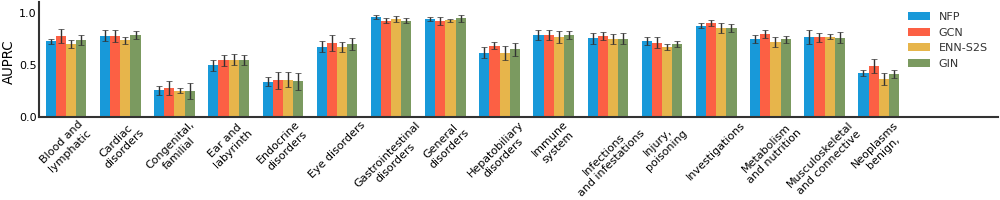}
        \caption{SIDER}
    \end{subfigure}
    \begin{subfigure}{0.9\textwidth}
        \includegraphics[width=\textwidth]{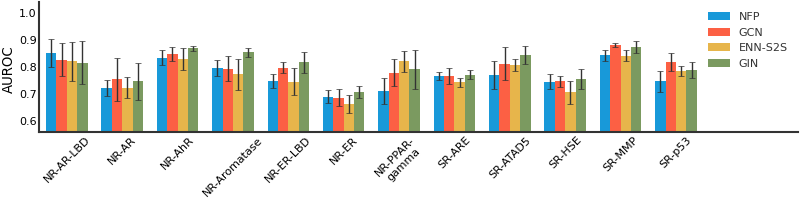}
        \includegraphics[width=\textwidth]{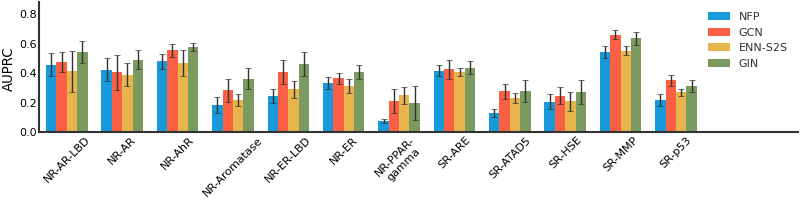}
        \caption{ToxCast}
    \end{subfigure}
    \caption[Molecular property prediction result on 15 datasets with random splits (cont. 1)]{Molecular property prediction result on 15 datasets with random splits (cont. 1).}
    \label{fig:property_prediction_1}
\end{figure}
\begin{figure}[!h]\ContinuedFloat
    \begin{subfigure}{0.9\textwidth}
        \includegraphics[width=\textwidth]{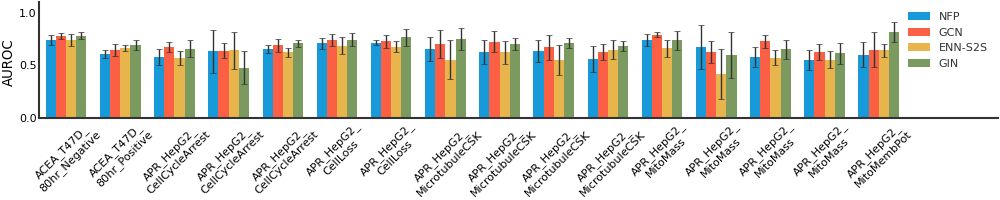}
        \includegraphics[width=\textwidth]{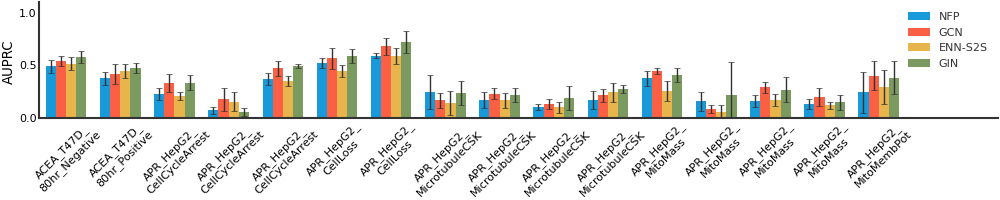}
        \caption{ToxCast}
    \end{subfigure}
    \begin{subfigure}{0.9\textwidth}
        \includegraphics[width=\textwidth]{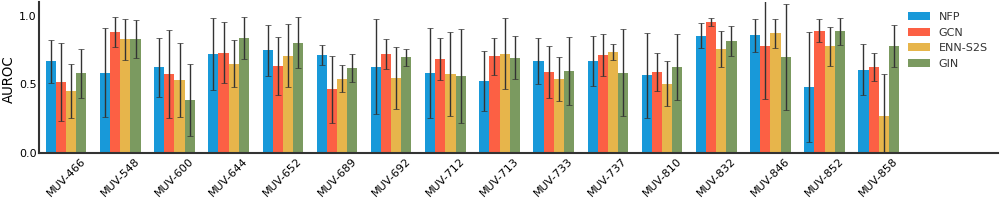}
        \includegraphics[width=\textwidth]{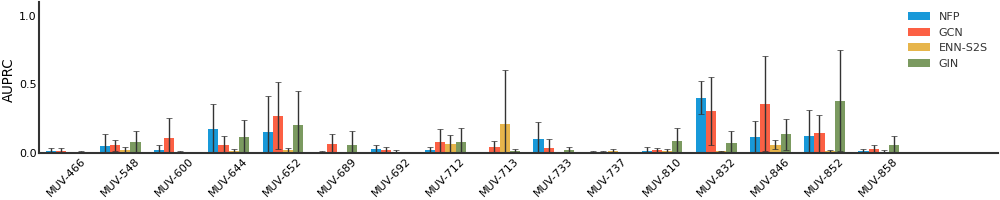}
        \caption{MUV}
    \end{subfigure}
        \begin{subfigure}{0.45\textwidth}
        \includegraphics[width=0.45\textwidth]{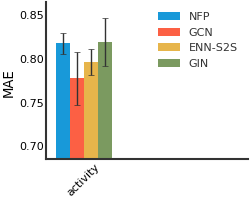}
        \includegraphics[width=0.45\textwidth]{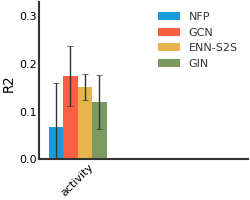}
        \caption{Malaria}
    \end{subfigure}
    \caption[Molecular property prediction result on 15 datasets with random splits (cont. 2)]{Molecular property prediction result on 15 datasets with random splits (cont. 2).}
    \label{fig:property_prediction_2}
\end{figure}

\begin{figure}[!h]
    \centering
    \begin{subfigure}{0.9\textwidth}
        \includegraphics[width=\textwidth]{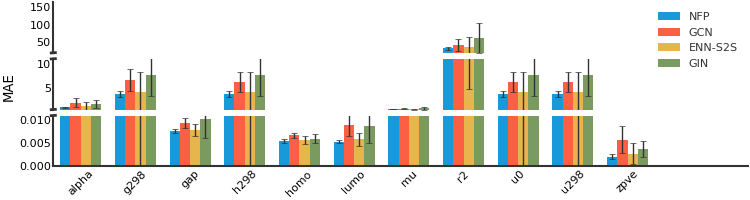}
        \includegraphics[width=\textwidth]{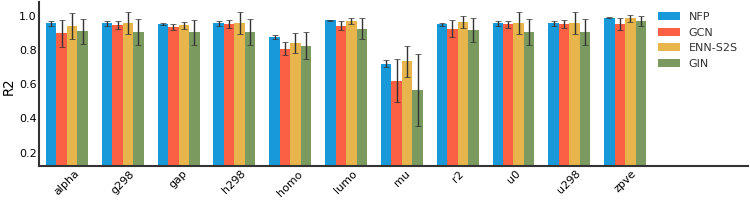}
        \caption{QM9}
    \end{subfigure}
    \begin{subfigure}{0.9\textwidth}
        \includegraphics[width=\textwidth]{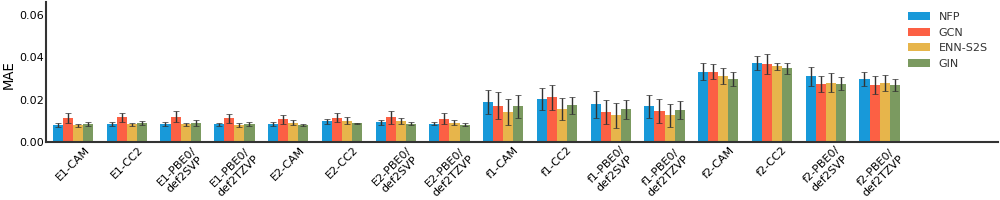}
        \includegraphics[width=\textwidth]{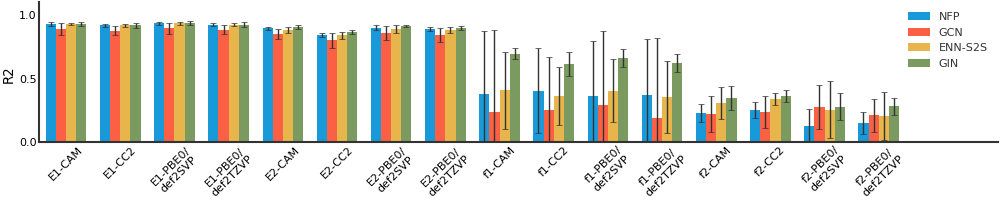}
        \caption{QM8}
    \end{subfigure}
    \begin{subfigure}{0.45\textwidth}
        \includegraphics[width=0.45\textwidth]{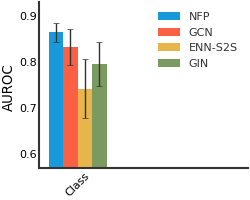}
        \includegraphics[width=0.45\textwidth]{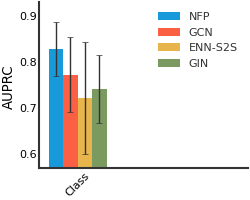}
        \caption{BACE}
    \end{subfigure}
    \begin{subfigure}{0.45\textwidth}
        \includegraphics[width=0.45\textwidth]{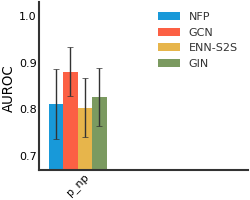}
        \includegraphics[width=0.45\textwidth]{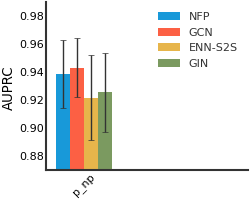}
        \caption{BBBP}
    \end{subfigure}
    \begin{subfigure}{0.45\textwidth}
        \includegraphics[width=0.45\textwidth]{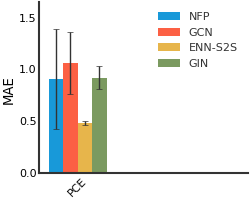}
        \includegraphics[width=0.45\textwidth]{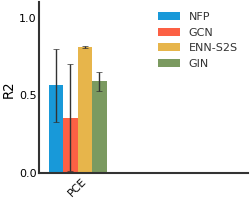}
        \caption{CEP}
    \end{subfigure}
    \begin{subfigure}{0.45\textwidth}
        \includegraphics[width=0.45\textwidth]{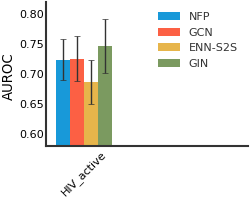}
        \includegraphics[width=0.45\textwidth]{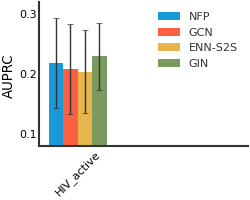}
        \caption{HIV}
    \end{subfigure}
    \caption[Molecular property prediction result on 15 datasets with scaffold splits]{Molecular property prediction result on 15 datasets with scaffold splits.}
    \label{fig:property_prediction_scaffold}
\end{figure}
\begin{figure}[!h]\ContinuedFloat
    \begin{subfigure}{0.45\textwidth}
        \includegraphics[width=0.45\textwidth]{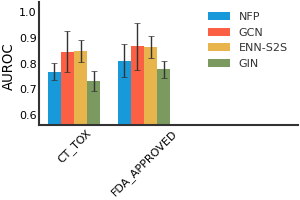}
        \includegraphics[width=0.45\textwidth]{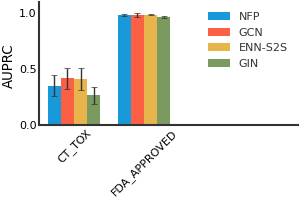}
        \caption{ClinTox}
    \end{subfigure}
    \begin{subfigure}{0.45\textwidth}
        \includegraphics[width=0.45\textwidth]{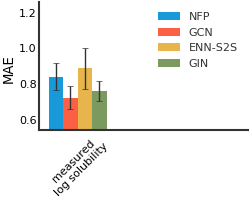}
        \includegraphics[width=0.45\textwidth]{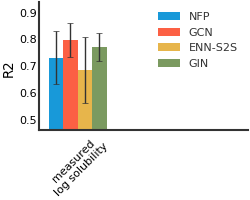}
        \caption{ESOL}
    \end{subfigure}
    \begin{subfigure}{0.45\textwidth}
        \includegraphics[width=0.45\textwidth]{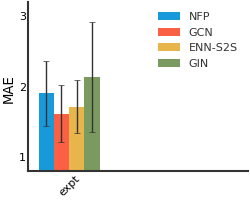}
        \includegraphics[width=0.45\textwidth]{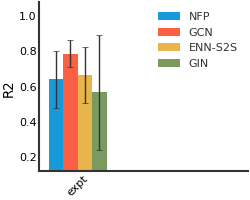}
        \caption{FreeSolv}
    \end{subfigure}
    \begin{subfigure}{0.45\textwidth}
        \includegraphics[width=0.45\textwidth]{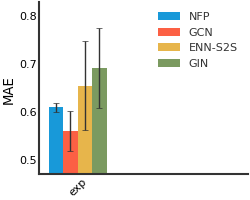}
        \includegraphics[width=0.45\textwidth]{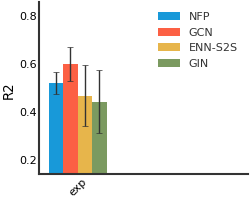}
        \caption{Lipophilicity}
    \end{subfigure}
    \begin{subfigure}{0.9\textwidth}
        \includegraphics[width=\textwidth]{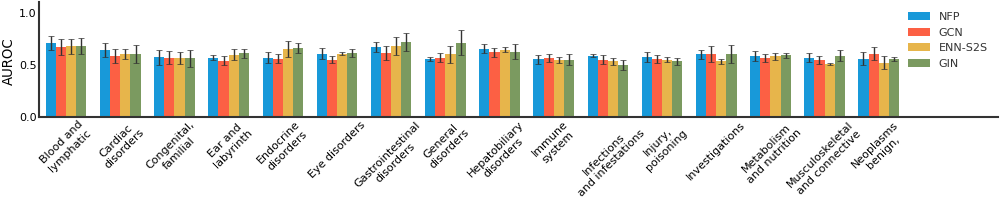}
        \includegraphics[width=\textwidth]{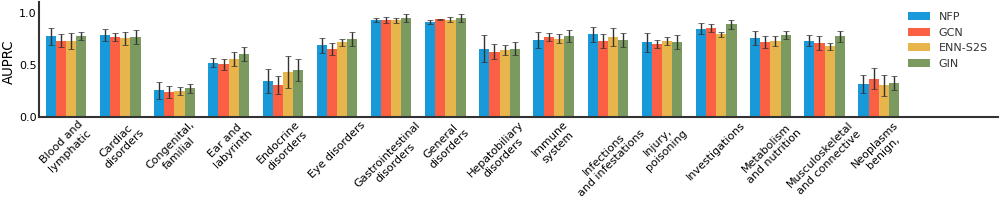}
        \caption{SIDER}
    \end{subfigure}
    \begin{subfigure}{0.9\textwidth}
        \includegraphics[width=\textwidth]{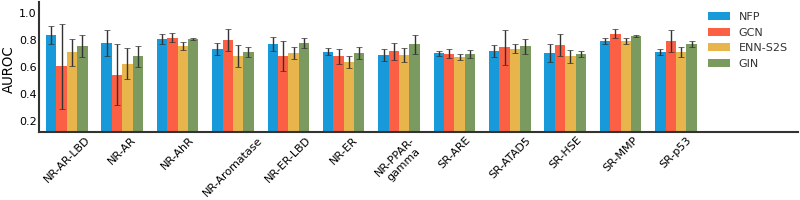}
        \includegraphics[width=\textwidth]{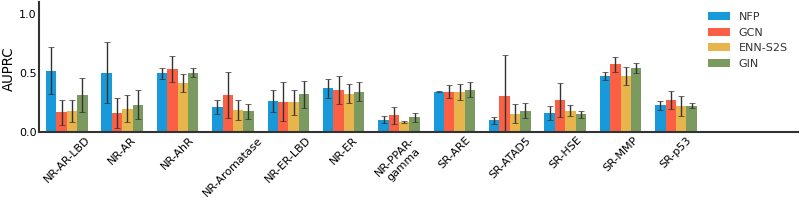}
        \caption{Tox21}
    \end{subfigure}
    \caption[Molecular property prediction result on 15 datasets with scaffold splits (cont. 1)]{Molecular property prediction result on 15 datasets with scaffold splits (cont. 1).}
    \label{fig:property_prediction_scaffold_1}
\end{figure}
\begin{figure}[!h]\ContinuedFloat
    \begin{subfigure}{0.9\textwidth}
        \includegraphics[width=\textwidth]{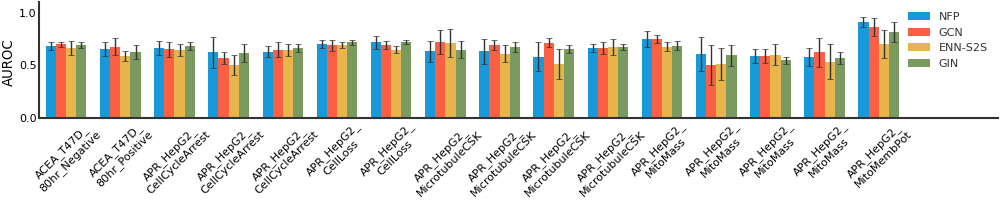}
        \includegraphics[width=\textwidth]{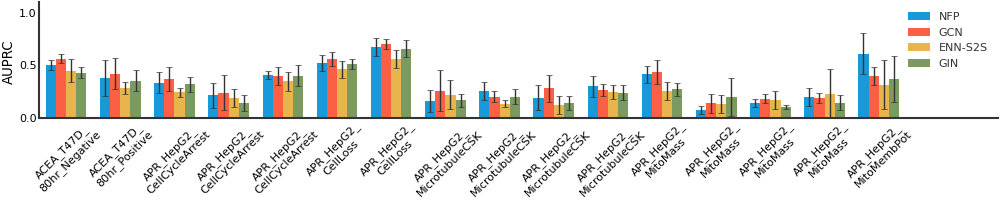}
        \caption{ToxCast}
    \end{subfigure}
    \begin{subfigure}{0.9\textwidth}
        \includegraphics[width=\textwidth]{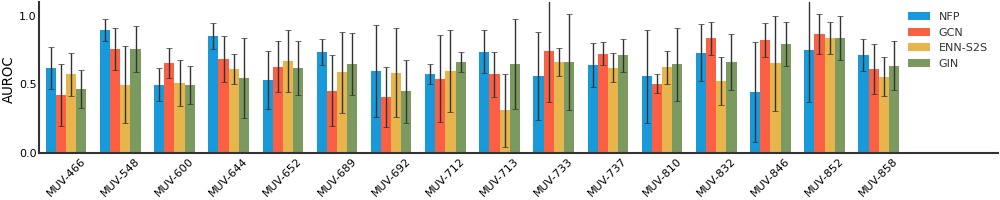}
        \includegraphics[width=\textwidth]{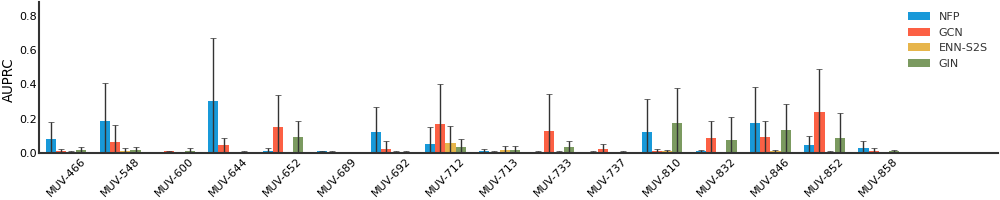}
        \caption{MUV}
    \end{subfigure}
        \begin{subfigure}{0.45\textwidth}
        \includegraphics[width=0.45\textwidth]{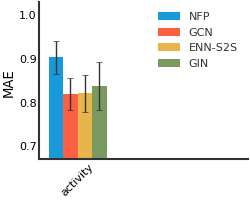}
        \includegraphics[width=0.45\textwidth]{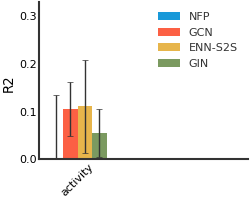}
        \caption{Malaria}
    \end{subfigure}
    \caption[Molecular property prediction result on 15 datasets with scaffold splits (cont. 2)]{Molecular property prediction result on 15 datasets with scaffold splits (cont. 2).}
    \label{fig:property_prediction_scaffold_2}
\end{figure}
\chapter[A System for Training Embeddings on Large Graphs]{A System for\linebreak Training Embeddings on Large Graphs}
\label{cha:graphvite}

Many real-world reasoning applications require dealing with large graphs, often at the scale of tens of millions or even billions of edges. While embedding methods are widely used in these applications, most existing embedding implementations utilize either multiple CPUs or a single GPU, both of which have limitations for large graphs. Additionally, embedding methods do not benefit much from the standard data parallelism on multiple GPUs. How can we harness the computation power of multiple GPUs to scale up embedding methods?

In this chapter, we introduce GraphVite, a system that scales up embedding methods on both homogeneous graphs and knowledge graphs with multiple CPUs and GPUs. GraphVite can handle large graphs whose embedding matrices cannot fit into the memory of a single GPU, supporting up to billion-scale homogeneous graphs or ten-million-scale knowledge graphs on a single machine. As of the year 2024, embedding methods have largely been replaced by inductive models, and there is no longer strong demand for training embeddings on large graphs. Readers interested in parallelism and system techniques may still find this chapter valuable.

\smallskip \emph{This chapter is based on our work published at WWW 2019~\cite{zhu2019graphvite}\footnote{The code is available at \url{https://github.com/DeepGraphLearning/graphvite}}.}

\section{Overview}

Graphs are ubiquitous in the real world. Examples like social networks~\cite{mislove2007measurement}, citation networks~\cite{sen2008collective}, protein-protein interaction networks~\cite{szklarczyk2016string} and many more cover a wide range of applications. In graph analysis, it is critical to have effective representations for nodes, as these representations largely determine the performance of many downstream tasks. Recently, there is a growing interest in unsupervised learning of continuous representations for nodes and edges, which is aimed at preserving the structure of graphs in a low-dimensional space. This kind of approaches has been proven successful in various applications, such as node classification~\cite{perozzi2014deepwalk}, link prediction~\cite{liben2007link}, and graph visualization~\cite{tang2016visualizing}.

Many works have been proposed on this stream, including DeepWalk~\cite{perozzi2014deepwalk}, LINE~\cite{tang2015line}, and node2vec~\cite{grover2016node2vec}. These methods learn effective embeddings by predicting the neighbors of each node and can be efficiently optimized by asynchronous stochastic gradient descent (ASGD)~\cite{recht2011hogwild}. On a single machine with multi-core CPUs, they are capable of processing graphs with one or a few millions of nodes. Given that real-world graphs easily go to tens of millions nodes and nearly billions of edges, how to adapt embedding methods to graphs of such large scales remains very challenging. One may think of exploiting computer clusters for training large-scale graphs. However, it is a non-trivial task to extend existing methods to distributed settings. Even if distributed algorithms are available, the cost of large CPU clusters is still prohibitive for many users. Therefore, we are wondering whether it is possible to scale embedding methods to very large graphs on a single machine, which should be particularly valuable for common users.

Inspired by the recent success of training deep neural networks with GPUs~\cite{ciresan2011flexible, krizhevsky2012imagenet}, we would like to utilize such highly parallel hardware to accelerate the training of embeddings. However, directly adopting GPUs for embedding methods could be inefficient, since the sampling procedure in embedding methods requires excessive random memory access on the graph structure, which is at the disadvantage of GPUs. Compared to GPUs, CPUs are much more capable of performing random memory access. Therefore, it would be wise to use both CPUs and GPUs for training embeddings. Along this direction, a straightforward solution is to follow the mini-batch stochastic gradient descent (mini-batch SGD) paradigm utilized in existing deep learning frameworks (e.g.\ TensorFlow~\cite{abadi2016tensorflow} and PyTorch~\cite{paszke2017automatic}). Different from deep neural networks, the training of embeddings involves much more memory access per computation. As a result, mini-batch SGD would suffer from severe memory latency on the bus before it benefits from fast GPU computation. Therefore, other than mini-batch SGD, we need to design a system that leverages distinct advantages of CPUs and GPUs and uses them collaboratively to train embeddings efficiently.
\begin{itemize}[label=$\bullet$, leftmargin=*]
    \item{\textbf{Limited GPU Memory.} The embedding matrices are quite large while the memory of a single GPU is very small. Modern GPUs usually have a capacity of 12GB or 16GB.}
    \item{\textbf{Limited Bus Bandwidth.} The bandwidth of the bus is much slower than the computation speed of GPUs. There will be severe latency if GPUs exchange data with the main memory frequently. }
    \item{\textbf{Large Synchronization Cost.} A lot of data are transferred between CPUs and GPUs. Both the CPU-GPU or inter-GPU synchronizations are very costly.}
\end{itemize}
Our multi-GPU system, GraphVite, addresses the above challenges by decomposing embedding methods into an edge augmentation stage and an embedding training stage, and deploying them to multiple CPUs and multiple GPUs respectively. In the edge sampling stage, we propose parallel online augmentation to augment the graph with random walks and generate augmented edge samples with multiple CPUs in an online fashion. In the embedding training stage, we propose parallel negative sampling to partition the training workload among GPUs, and assign GPUs to non-overlapping partitions so that multiple GPUs can perform gradient updates simultaneously without much inter-GPU synchronization. The parallel negative sampling also significantly reduces GPU memory usage as each GPU only stores the subset of embeddings corresponding to its partition. We further introduce a collaboration strategy to reduce the synchronization cost between CPUs and GPUs.

We evaluate GraphVite on 4 homogeneous graphs and 3 knowledge graphs of different scales. On a single machine with 4 Tesla P100 GPUs, our system only takes one minute to train embeddings on a homogeneous graph with 1 million nodes and 5 million edges. Compared to the current fastest system~\cite{tang2015line}, GraphVite is 51 times faster and does not sacrifice any performance. GraphVite can scale up to homogeneous graphs with 65 million nodes and 2 billion edges, or knowledge graphs with 5 million entities and 21 million triplets. We also investigate the speed of GraphVite under different hardware configurations. Even on economic GPUs like GeForce GTX 1080, GraphVite is able to achieve a speedup of 29 times compared to the current fastest system.
\section{Preliminary}
\label{sec:preliminary}

Here we review existing embedding methods on both knowledge graphs and homogeneous graphs from a system perspective. Without loss of generality, we consider a knowledge graph $\gG=(\gV, \gE, \gR)$, while homogeneous graphs can be viewed as knowledge graphs with the same relation for every edge.

The goal of embedding methods is to learn a low-dimensional representation for each entity and each relation in a knowledge graph, or just for each node in a homogeneous graphs. Towards this goal, existing embedding methods train embeddings to distinguish the edges in $\gE$ (i.e.\ positive samples) from some randomly corrupted edges (i.e.\ negative samples). In other words, edges are essentially utilized as training data.

Since real-world graphs may be sparse, some existing embedding methods~\cite{perozzi2014deepwalk, tang2015line, grover2016node2vec, guu2015traversing, garcia2015composing} conduct random walks on the original graph to introduce more connectivity. Specifically, they connect nodes within a specific distance on a random walk path as additional positive edges~\cite{perozzi2014deepwalk, tang2015line, grover2016node2vec}. For knowledge graphs, the additional edges may be further associated with an augmented relation according to the chain of relations \pagebreak on the path~\cite{guu2015traversing, garcia2015composing}. In the case of homogeneous graphs, the augmented relation is the same as other relations.

Once the graph is augmented, embeddings are trained with the samples from the augmented graph. Typically, there are three sets of embeddings, namely $\mathbf{head}$ embedding matrix, $\mathbf{relation}$ embedding matrix and $\mathbf{tail}$ embedding matrix. For an edge sample \edge{u, q, v}, a score function is computed based on $\mathbf{head}[u]$, $\mathbf{relation}[q]$ and $\mathbf{tail}[v]$ to predict the likelihood of the edge. In the case of homogeneous graphs, the relation is trivial and only $\mathbf{head}[u]$ and $\mathbf{tail}[v]$ are used in the score function. The embeddings are then optimized according to the label of the edge. In some cases, the embedding matrices $\mathbf{head}$ and $\mathbf{tail}$ may share their weights.

Overall, the computation procedures of these embedding methods can be divided into two stages: \textbf{edge augmentation} and \textbf{embedding training}. Algorithm~\ref{alg:embedding} summarizes the general framework of existing embedding methods. Note that the first stage can be easily parallelized, and the second stage can be parallelized via ASGD. In multi-CPU systems~\cite{perozzi2014deepwalk, tang2015line, grover2016node2vec}, these two stages are executed in a sequential order, with each stage parallelized by a bunch of CPU threads. In single-GPU systems~\cite{thunlp2017openne, han2018openke, luca2019ampligraph}, the first stage is executed by one or more CPUs, while the second stage is executed by the GPU.

\begin{algorithm}[!h]
    \captionsetup{font=footnotesize}\caption{General framework of embedding methods.}
    \begin{algorithmic}[1]
        \footnotesize
        \State{$\gE' \gets \gE$}
        \For{$u \in \gV$} \Comment{optional, parallelizable}
            \State{$path \gets \Call{RandomWalk}{u}$}
            \For{$v \in path$}
                \State{$q \gets \Call{AugmentRelation}{path, u, v}$} \Comment{trivial relation for homogeneous graphs}
                \State{$\gE' \gets \gE' \cup \{\langle u, q, v\rangle\}$}
            \EndFor
        \EndFor
        \For{each iteration} \Comment {parallelizable}
            \State{$u, q, v \gets \Call{PositiveSampling}{\gE'}$}
            \State{$\Call{Train}{\mathbf{head}[u], \mathbf{relation}[q], \mathbf{tail}[v], \text{label}=1}$} \Comment{ignore relation for homogeneous graphs}
            \For{$v' \in \Call{NegativeSampling}{\gV}$}
                \State{$\Call{Train}{\mathbf{head}[u], \mathbf{relation}[q], \mathbf{tail}[v'], \text{label}=0}$} \Comment{ignore relation for homogeneous graphs}
            \EndFor
        \EndFor
    \end{algorithmic}
    \label{alg:embedding}
\end{algorithm}
\vspace{-0.5em}
\section{Method}

GraphVite combines multiple CPUs and GPUs to scale up the training of embedding methods. Specifically, for the edge augmentation stage, we propose parallel online augmentation to efficiently generate augmented positive samples on multiple CPUs. For the embedding training stage, we propose parallel negative sampling to train the embeddings on multiple GPUs with minimal synchronization. A collaboration strategy is further proposed to reduce the synchronization cost. Figure~\ref{fig:system} shows the overview of our system.

\begin{figure}[t]
    \centering
    \includegraphics[width=0.7\textwidth]{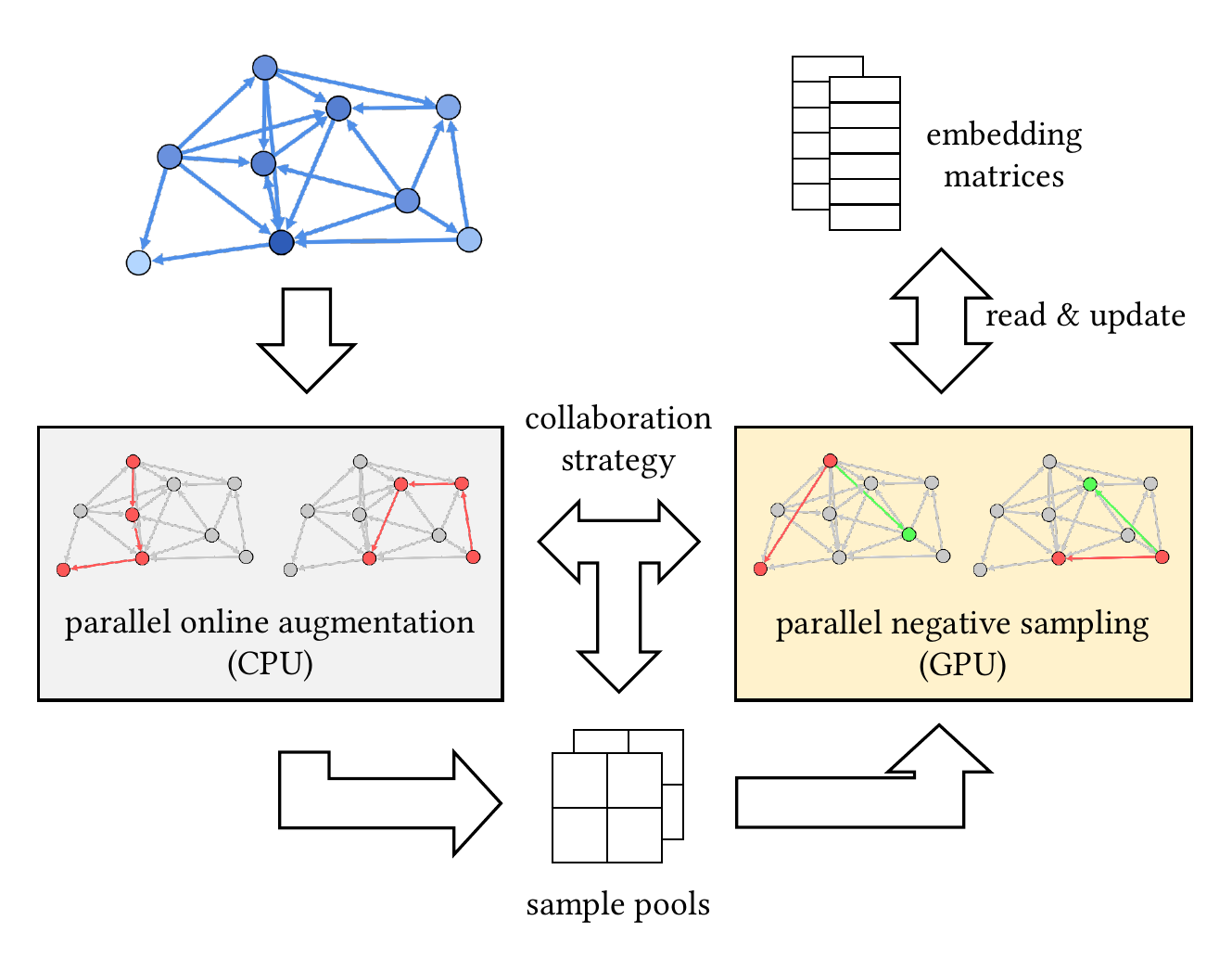}
    \caption[Overview of GraphVite]{Overview of GraphVite. The gray and yellow boxes correspond to the stages of edge augmentation and embedding training respectively. The former is performed by parallel online augmentation on CPUs, while the latter is performed by parallel negative sampling on GPUs. The two stages are executed asynchronously with our collaboration strategy.}
    \label{fig:system}
\end{figure}

\subsection{Parallel Online Augmentation}
\label{sec:parallel_augmentation}

In the edge augmentation stage, the augmented edge set $\gE'$ is usually one or two magnitude larger than the original edge set $\gE$, which makes it impossible to pre-compute the augmented edge set in the main memory for large-scale graphs. Therefore, we introduce a parallel online augmentation, which generates augmented edge samples on the fly without explicit edge augmentation. Our method can be viewed as an online extension of the augmentation and edge sampling method used in LINE~\cite{tang2015line}.

First, we draw a departure node with the probability proportional to the degree of each node. Then we perform a random walk from the departure node, and pick node pairs within a specific distance $s$ as edge samples. When $s$ is set to 1, the edge samples are equivalent to the samples drawn from the original graph, i.e.\ without augmentation. Note that edge samples generated in the same random walk are highly correlated and may degrade the performance of optimization. Inspired by the experience replay technique widely used in reinforcement learning~\cite{lin1993reinforcement, mnih2013playing}, we collect edge samples into a sample pool, and shuffle the sample pool before using it for embedding training.

The proposed edge sampling method can be parallelized when each CPU thread is allocated with an independent sample pool in advance. Algorithm~\ref{alg:parallel_online_augmentation} gives the process of parallel online augmentation in details.

\smallskip \noindent \textbf{Pseudo Shuffle.}
While shuffling the sample pool is important to optimization, it slows down the graph augmentation stage (see Table~\ref{tab:shuffle}). The reason is that a general shuffle consists of lots of random memory access and cannot be accelerated by the CPU cache. The loss in speed will be even worse if the server has more than one CPU socket. To mitigate this issue, we propose a pseudo shuffle technique that shuffles correlated samples in a much more cache-friendly way and improves the speed of the system significantly. Note that most correlation comes from edge samples that share the source node or the target node in the same random walk. As such correlation occurs in a group of $s$ samples for an augmentation distance $s$, we divide the sample pool into $s$ continuous blocks, and scatter correlated samples into different blocks. For each block, we always append samples sequentially at the end, which can benefit a lot from CPU cache. The $s$ blocks are concatenated to form the final sample pool.

\begin{algorithm}[!h]
    \captionsetup{font=footnotesize}\caption{Parallel Online Augmentation}
    \begin{algorithmic}[1]
        \footnotesize
        \Function{ParallelOnlineAugmentation}{$num\_CPU$}
            \For{$i \gets 0$ to $num\_CPU - 1$} \Comment{paralleled}
                \State{$pool[i] \gets \varnothing$}
                \While{$pool$ is not full}
                    \State{$x \gets \Call{DepartureSampling}{G}$}
                    \For{$u, v \in \Call{RandomWalkSampling}{x}$}
                        \If{$Distance(u, v) <= s$}
                            \State{$q \gets \Call{AugmentRelation}{path, u, v}$} \Comment{trivial relation for homogeneous graphs}
                            \State{$pool.append((u, q, v))$}
                        \EndIf
                    \EndFor
                \EndWhile
                \State{$pool[i] \gets \Call{Shuffle}{pool[i]}$}
            \EndFor
            \State{\Return${\Call{Concatenate}{pool[\cdot]}}$}
        \EndFunction
    \end{algorithmic}
    \label{alg:parallel_online_augmentation}
\end{algorithm}

\vspace{-0.5em}
\subsection{Parallel Negative Sampling}
\label{sec:parallel_training}

In the embedding training stage, we divide the training task into fragments and distribute them to multiple GPUs. The sub tasks are necessarily designed with little shared parameters to minimize the inter-GPU synchronization cost. \pagebreak To see how parameters can be distributed to multiple GPUs without overlap, we introduce a definition of \emph{$\epsilon$-gradient exchangeable}.
\begin{definition}
    \textbf{$\epsilon$-gradient exchangeable}. A loss function $\gL(X;\theta)$ is \emph{$\epsilon$-gradient exchangeable} on two sets of training data $X_1$, $X_2$ if for $\epsilon \geq 0$, $\forall \theta_0 \in \Theta$ and $\forall \alpha \in \mathbb{R}^+$, exchanging the order of two gradient descent steps results in a vector difference with norm no more than $\epsilon$.
    \begin{equation}
        \begin{cases}
            & \theta_1 \gets \theta_0 - \alpha \nabla \gL(X_1;\theta_0) \\
            & \theta_2 \gets \theta_1 - \alpha \nabla \gL(X_2;\theta_1)
        \end{cases}
        \label{eq:order1}
    \end{equation}
    \vspace{-0.5em}
    \begin{equation}
        \begin{cases}
            & \theta'_1 \gets \theta_0 - \alpha \nabla \gL(X_2;\theta_0) \\
            & \theta'_2 \gets \theta'_1 - \alpha \nabla \gL(X_1;\theta'_1)
        \end{cases}
        \label{eq:order2}
    \end{equation}
    i.e.\ $\lVert \theta_2 - \theta'_2 \rVert \leq \epsilon$ is true for the above equations.
\end{definition}
Particularly, we abbreviate \emph{0-gradient exchangeable} to \emph{gradient exchangeable}. Due to the sparse nature of embedding training, there are many sets that form \emph{gradient exchangeable} pairs in the graph. For example, for two edge sample sets $X_1, X_2 \subseteq E$, if they do not share any source nodes or target nodes, $X_1$ and $X_2$ are \emph{gradient exchangeable}. Even if $X_1$ and $X_2$ share some nodes, they can still be \emph{$\epsilon$-gradient exchangeable} if the learning rate $\alpha$ and the number of iterations are bounded.

Based on the gradient exchangeability observed in embeddings, we propose a parallel negative sampling algorithm for the embedding training stage. For $n$ GPUs, we partition rows of $\mathbf{head}$ and $\mathbf{tail}$ into $n$ partitions respectively (see the top-left corner of Figure~\ref{fig:parallel_negative_sampling}). We do not partition $\mathbf{relation}$ as it is usually very small and can easily fit into a GPU. This results in an $n \times n$ partition grid for the sample pool, where each edge belongs to one of the blocks. In this way, any pair of blocks that does not share row or column is \emph{gradient exchangeable}. Blocks in the same row or column are \emph{$\epsilon$-gradient exchangeable}, as long as we restrict the number of iterations on each block.

We define \emph{episode} as a block-level step used in parallel negative sampling. During each episode, we send $n$ orthogonal blocks and their corresponding $\mathbf{head}$ and $\mathbf{tail}$ partitions to $n$ GPUs respectively. Each GPU then updates its own embedding partitions with ASGD. Because these blocks are mutually \emph{gradient exchangeable} and do not share any row in the parameter matrices, multiple GPUs can perform ASGD concurrently without any synchronization. At the end of each episode, we gather the updated parameters from all GPUs and assign another $n$ orthogonal blocks. In the case of shared weights between $\mathbf{head}$ and $\mathbf{tail}$, we double the number of partitions to $2n$, and send $n$ blocks that do not share any index for rows and columns during each episode.

Here \emph{$\epsilon$-gradient exchangeable} is controlled by the number of total samples in $n$ orthogonal blocks, which we define as \emph{episode size}. The smaller episode size, the better \emph{$\epsilon$-gradient exchangeable} we will have for embedding training. However, smaller episode size will also induce more frequent synchronization. Hence the episode size is tuned so that there is a good trade off between the speed and \emph{$\epsilon$-gradient exchangeable} (see experiments in Section~\ref{sec:ablation_graphvite}). Figure~\ref{fig:parallel_negative_sampling} gives an example of parallel negative sampling with 4 partitions.

Typically, embedding methods generate negative edges with a tail node sampled from all possible nodes. However, it could be very time-consuming if GPUs have to communicate with each other to get the embeddings of their negative samples. To avoid this cost, we restrict that the tail node can only be drawn from the $\mathbf{tail}$ rows on the current GPU. Though this seems a little problematic, we find it works well in practice. An intuitive explanation is that with parallel online augmentation, every node is likely to have positive samples with nodes from all context partitions. As a result, every node can potentially form negative samples with all possible nodes.

Note that although we demonstrate with the number of partitions equal to $n$, the parallel negative sampling can be easily generalized to cases with any number of partitions greater than $n$, simply by processing the orthogonal blocks in subgroups of $n$ during each episode. Algorithm~\ref{alg:parallel_negative_sampling} illustrates the hybrid system for multiple GPUs.

\begin{figure}[t]
    \includegraphics[width=0.95\textwidth]{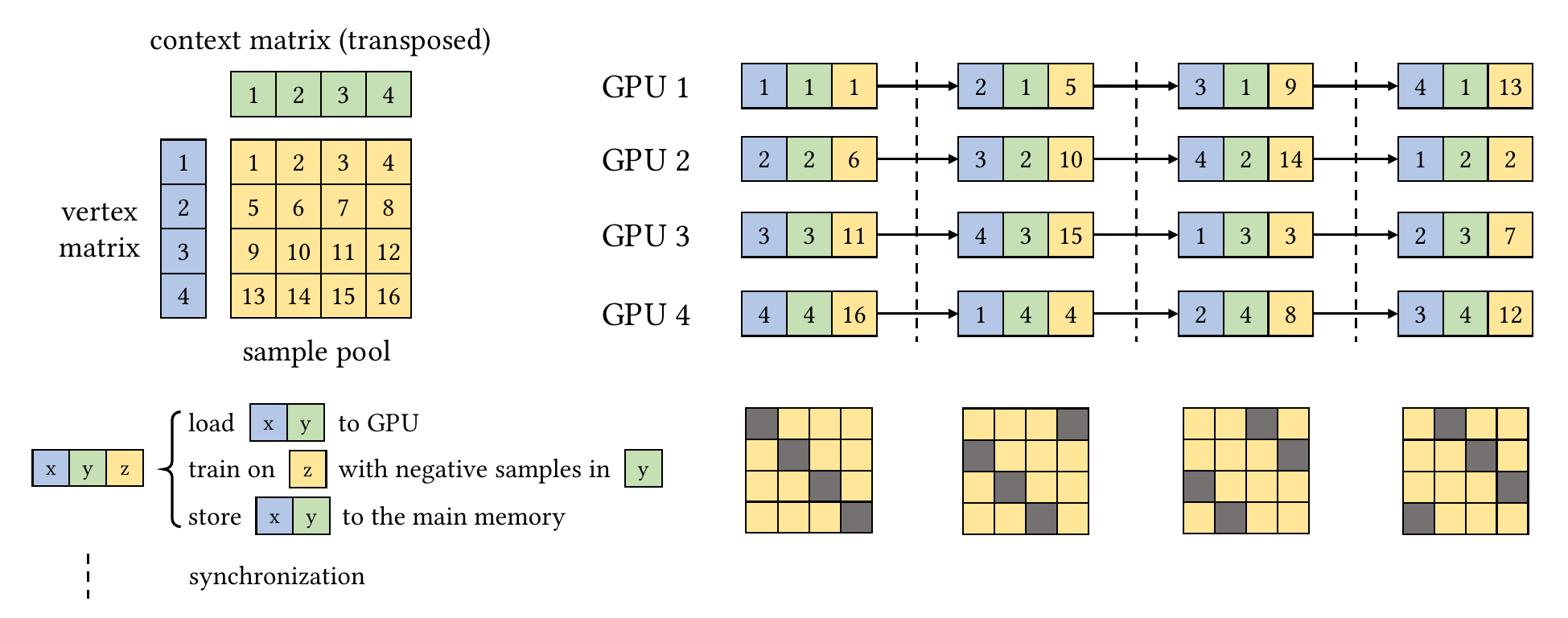}
    \vspace{-1em}
    \caption[Illustration of parallel negative sampling]{Illustration of parallel negative sampling on 4 GPUs. During each episode, GPUs take orthogonal blocks from the sample pool. Each GPU trains embeddings with negative samples drawn from its own tail partition. Each GPU updates its own copy of relation embeddings independently. The entity and relation embeddings are only synchronized between every two episodes.}
    \label{fig:parallel_negative_sampling}
    \vspace{-1em}
\end{figure}

\begin{algorithm}[!h]
    \captionsetup{font=footnotesize}\caption{Parallel Negative Sampling}
    \begin{algorithmic}[1]        
        \footnotesize
        \Function{ParallelNegativeSampling}{$num\_GPU$}        
            \State{$head\_partitions \gets \Call{Partition}{\mathbf{head}}$}
            \State{$tail\_partitions \gets \Call{Partition}{\mathbf{tail}}$}
            \While{not converge}
                \State{$pool \gets \Call{ParallelOnlineAugmentation}{num\_CPU}$}
                \State{$block[\cdot][\cdot] \gets \Call{Redistribute}{pool}$}
                \For{$\mathit{offset} \gets 0$ to $num\_GPU - 1$}
                    \For{$i \gets 0$ to $num\_GPU - 1$} \Comment{paralleled}
                        \State{$hid \gets i$}
                        \State{$tid \gets (i + \mathit{offset}) \mod num\_GPU$}
                        \State{send $head\_partitions[hid]$ to GPU $i$}
                        \State{send $tail\_partitions[tid]$ to GPU $i$}
                        \State{$relation\_copy[i] \gets relation$}
                        \State{send $relation\_copy[i]$ to GPU $i$} \Comment{ignore for homogeneous graphs}
                        \State{train $block[hid][tid]$ on GPU $i$}
                        \State{receive $head\_partitions[hid]$ from GPU $i$}
                        \State{receive $tail\_partitions[tid]$ from GPU $i$}
                        \State{receive $relation\_copy[i]$ from GPU $i$} \Comment{ignore for homogeneous graphs}
                        \State{$relation\_update[i] \gets relation\_copy[i] - relation$}
                        \State{$relation \gets relation + relation\_update[i]$}
                    \EndFor
                \EndFor
            \EndWhile
        \EndFunction
    \end{algorithmic}
    \label{alg:parallel_negative_sampling}
\end{algorithm}

\subsection{Collaboration Strategy}
\label{ref:collaboration_strategy}

Our parallel negative sampling enables different GPUs to train embeddings concurrently, with GPU synchronization only required between episodes. However, it should be noticed that the sample pool is also shared between CPUs and GPUs. If they synchronize on the sample pool, then GPUs are idle when CPUs produce the samples, and CPUs are idle when GPUs consume the samples. To maximize the usage of hardware, we propose a collaboration strategy to hide the synchronization cost. We allocate two sample pools in the main memory, and let CPUs and GPUs always work on different pools. CPUs first fill up a sample pool and pass it to GPUs. After that, parallel online augmentation and parallel negative sampling are performed concurrently on CPUs and GPUs respectively. The two pools are swapped when CPUs fill up a new pool. Figure~\ref{fig:system} illustrates this procedure. With the collaboration strategy, the throughput of our hybrid system is nearly doubled.

\subsection{Discussion}
Here we further discuss some practical details of our hybrid system.

\smallskip \noindent \textbf{Batched Transfer.} In parallel negative sampling, the sample pool is assigned to GPUs by block, which is sometimes very large for the memory of a GPU. Instead of copying the whole sample block to a GPU, we transfer the sample block by a small granularity. In this way, the memory cost of edge samples on GPUs becomes negligible.

\smallskip \noindent \textbf{CPU-GPU co Usage Optimization.} When the number of partitions equals the number of GPUs, we can further optimize the bus usage by fixing the tail partition for each GPU. In this way, we save the synchronization of $\mathbf{tail}$ matrix and reduce nearly half of the bus usage.

\smallskip \noindent \textbf{Single GPU Case.} Although parallel negative sampling is proposed for multiple GPUs, our hybrid system is compatible with a single GPU. Typically a GPU can hold at most 12 million embeddings. So a single GPU is sufficient for training embeddings on graphs that contain no more than 12 million nodes.
\section{Experiments}
\label{sec:experiment}

\subsection{Experiment Setup}

We compare GraphVite with existing systems on both homogeneous graph embeddings, and knowledge graph embeddings. For homogeneous graph embeddings, we evaluate different systems on Youtube~\cite{mislove2007measurement} dataset. We also report the time and performance of GraphVite on three larger homogeneous graphs, of which the scale is beyond existing embedding systems. For knowledge graph embeddings, we evaluate different systems on FB15k-237~\cite{toutanova2015observed} and WN18RR~\cite{dettmers2018convolutional}. We additionally report the time and performance of GraphVite on a larger knowledge graph, Wikidata5m~\cite{wang2021kepler}.

\smallskip \noindent \textbf{Implementation Details.}
Our implementation generally follows the open source codes for homogeneous graph embeddings~\footnote{\url{https://github.com/tangjianpku/LINE}}\footnote{\url{https://github.com/phanein/deepwalk}} and knowledge graph embeddings~\footnote{\url{https://github.com/DeepGraphLearning/KnowledgeGraphEmbedding}}. We adopt the asynchronous SGD~\cite{recht2011hogwild} in GPU training, and leverage the on-chip shared memory of GPU for fast forward and backward computation. We also utilize the alias table trick~\cite{tang2015line, grover2016node2vec} to boost parallel online augmentation and parallel negative sampling.

\begin{wrapfigure}{r}{0.48\textwidth}
    \vspace{-1.4em}
    \centering
    \includegraphics[width=0.48\textwidth]{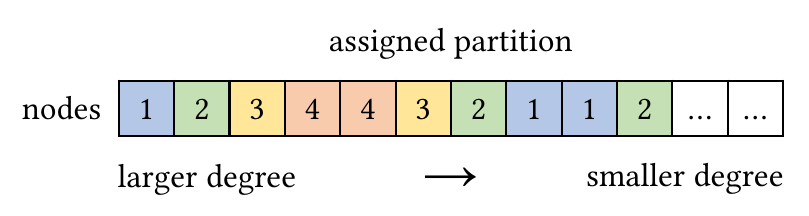}
    \caption[Degree-guided partition strategy]{Degree-guided partition strategy. Nodes are first sorted by their degrees, and then partitioned in a zig-zag fashion.}
    \label{fig:zig-zag_partition}
\end{wrapfigure}

We largely follow existing embedding methods~\cite{bordes2013translating, sun2019rotate, tang2015line, perozzi2014deepwalk} to set the hyperparameters for their multi-GPU implementation in GraphVite, except we tune the learning rate, the batch size, the episode size and other hyperparameters that are unique to GraphVite. For knowledge graph embeddings, we follow previous works~\cite{bordes2013translating, sun2019rotate} and directly sample positive edges without augmentation. Weights of $\mathbf{head}$ and $\mathbf{tail}$ are tied in knowledge graph embeddings. For homogeneous graph embeddings, we use edge augmentation with random walks of 40 edges. To balance the workload in different blocks, we partition $\mathbf{head}$ and $\mathbf{tail}$ matrices according to their degrees. Specifically, we first sort nodes by their degrees and then assign them into different partitions in a zig-zag fashion, as illustrated in Figure~\ref{fig:zig-zag_partition}. During the embedding training stage, negative samples are generated uniformly for knowledge graph embeddings~\cite{bordes2013translating, sun2019rotate} and with a probability proportional to the $3/4$ power of node degrees for homogeneous graph embeddings~\cite{tang2015line, perozzi2014deepwalk, grover2016node2vec}. We adopt the standard \textit{O3} compliation optimization in \emph{g++} and \emph{nvcc}.

\smallskip \noindent \textbf{Evaluation.}
For homogeneous graphs, we report the wall time of preprocessing and training for all systems. The performance of learned embeddings are evaluated on the multi-label node classification task~\cite{perozzi2014deepwalk, tang2015line}, where embeddings are frozen and fed into SVMs~\cite{cortes1995support} to predict labels. We report the micro-F1 as well as the macro-F1 scores for multi-label node classification. For datasets without node labels, we evaluate learned embeddings on the link prediction task~\cite{grover2016node2vec} and report AUROC scores. For knowledge graphs, since they do not require edge augmentation, we report the wall time of training for all systems. We evaluate the performance of systems on knowledge graph completion, and additionally report the evaluation time as it is commonly required to deploy large-scale knowledge graph completion in real-world scenarios. We follow the filtered ranking protocol~\cite{bordes2013translating} to evaluate knowledge graph embeddings. For every test triplet \edge{u, q, v}, we rank it against all negative triplets \edge{u, q, v'} or \edge{u', q, v} that do not appear in the knowledge graph. Mean rank (MR), mean reciprocal rank (MRR) and HITS at N (H@N) are reported for knowledge graph completion.

\smallskip \noindent \textbf{Baselines.}
We compare GraphVite against existing systems or implementations on embedding methods. For homogeneous graph embeddings, these include the official multi-CPU implementations of DeepWalk~\cite{perozzi2014deepwalk}, LINE~\cite{tang2015line} and node2vec~\cite{grover2016node2vec}, as well as a single-GPU implementation of LINE in OpenNE~\cite{thunlp2017openne}\footnote{OpenNE has been upgraded a lot through the years. Experiments are carried out on the OpenNE in 2018.}. For knowledge graph embeddings, the baselines include the single-GPU implementation of TransE and RotatE in the open source codebase for knowledge graph embeddings~\cite{sun2019rotate}.

\subsection{Results on Homogeneous Graphs}

\smallskip \noindent \textbf{Standard Datasets.}
Table~\ref{tab:time_youtube} presents the time of different systems. Among all existing systems, LINE~\cite{tang2015line} takes the minimal total time to run. However, the GPU implementation of LINE in OpenNE is even worse than its CPU counterpart, possibly due to the mini-batch SGD paradigm it uses. Compared to the current fastest system, LINE, GraphVite is much more efficient. With 4 GPUs, our system completes training emebddings on a million-scale graph in only one and a half minutes. Even on a single GPU, GraphVite takes no more than 4 minutes and is still 19 times faster than LINE.

One may be curious about the performance of embeddings learned by GraphVite. Table~\ref{tab:performance_youtube} summarizes the performance over different percentages of training data. It is observed that GraphVite achieves the best or competitive results in most settings, showing that GraphVite does not sacrifice any performance. In some small percentage cases, GraphVite falls a little behind DeepWalk. This is because GraphVite uses negative sampling for optimization, while DeepWalk uses both hierarchical softmax~\cite{mikolov2013efficient} and negative sampling~\cite{mikolov2013distributed}, which could be more robust to few labeled data.

\begin{table}[t]
    \centering
    \caption[Results of time of different systems on Youtube]{Results of time of different systems on Youtube. The preprocessing time refers to all the overhead before training, including graph input and offline edge augmentation. The preprocessing time of OpenNE is not comparable since it does not have the edge augmentation stage. The speedup ratio of GraphVite is computed w.r.t.\ LINE, which is the current fastest system.}
    \footnotesize
    \begin{tabular}{lcccc}
         \toprule
         \bf{System} & \bf{\#CPU Thread} & \bf{\#GPU} & \bf{Training Time} & \bf{Preprocessing Time} \\
         \midrule
         LINE~\cite{tang2015line}               & 20            & -         & 1.24 hrs                          & 17.4 mins                         \\
         DeepWalk~\cite{perozzi2014deepwalk}    & 20            & -         & 1.56 hrs                          & 14.2 mins                         \\
         node2vec~\cite{grover2016node2vec}     & 20            & -         & 47.7 mins                         & 25.9 hrs                          \\
         LINE in OpenNE~\cite{thunlp2017openne} & 1             & 1         & > 1 day                           & 2.14 mins                         \\
         GraphVite                              & 6             & 1         & \bf{3.98 mins}($18.7\times$)      & \bf{7.37 s}                       \\
         GraphVite                              & 24            & 4         & \bf{1.46 mins}($50.9\times$)      & \bf{16.0 s}                       \\
         \bottomrule
    \end{tabular}
    \label{tab:time_youtube}
\end{table}

\begin{table}[t]
    \centering
    \caption[Results of node classification on Youtube]{Results of node classification on Youtube.}
    \begin{adjustbox}{max width=\textwidth}
    \begin{tabular}{llcccccccccc}
        \toprule
                                        & \bf{\% Labeled Nodes} & \bf{1\%} & \bf{2\%} & \bf{3\%} & \bf{4\%} & \bf{5\%}
                                        & \bf{6\%} & \bf{7\%} & \bf{8\%} & \bf{9\%} & \bf{10\%} \\
        \midrule
        \multirow{4}{*}{\bf{Micro-F1(\%)}}   & LINE\cite{tang2015line}               & 32.98         & 36.70         & 38.93         & 40.26         & 41.08
                                                                                & 41.79         & 42.28         & 42.70         & 43.04         & 43.34 \\
                                        & LINE\cite{tang2015line}+augmentation  & 36.78         & 40.37         & 42.10         & 43.25         & 43.90
                                                                                & 44.44         & 44.83         & 45.18         & 45.50         & 45.67 \\
                                        & DeepWalk\cite{tang2015line}           & \bf{39.68}  & 41.78         & 42.78         & 43.55         & 43.96
                                                                                & 44.31         & 44.61         & 44.89         & 45.06         & 45.23 \\
                                        & GraphVite                             & 39.19         & \bf{41.89}  & \bf{43.06}  & \bf{43.96}  & \bf{44.53}
                                                                                & \bf{44.93}  & \bf{45.26}  & \bf{45.54}  & \bf{45.70}  & \bf{45.86}  \\ 
        \midrule
        \multirow{4}{*}{\bf{Macro-F1(\%)}}   & LINE\cite{tang2015line}               & 17.06         & 21.73         & 25.28         & 27.36         & 28.50
                                                                                & 29.59         & 30.43         & 31.14         & 31.81         & 32.32 \\
                                        & LINE\cite{tang2015line}+augmentation  & 22.18         & 27.25         & 29.87         & 31.88         & 32.86
                                                                                & 33.73         & 34.50         & 35.15         & 35.76         & 36.19 \\
                                        & DeepWalk\cite{tang2015line}           & \bf{28.39}  & \bf{30.96}  & \bf{32.28}  & \bf{33.43}  & \bf{33.92}
                                                                                & 34.32         & 34.83         & 35.27         & 35.54         & 35.86         \\
                                        & GraphVite                             & 25.61         & 29.46         & 31.32         & 32.70         & 33.81
                                                                                & \bf{34.59}  & \bf{35.27}  & \bf{35.82}  & \bf{36.14}  & \bf{36.49} \\ 
        \bottomrule
    \end{tabular}
    \end{adjustbox}
    \label{tab:performance_youtube}
\end{table}

\smallskip \noindent \textbf{Larger Datasets.}
To demonstrate the scalability of GraphVite, we further test GraphVite on three larger graphs. We learn the embeddings of Friendster-small with 1 GPU and 4 GPUs. For Hyperlink-PLD and Friendster, since their embedding matrices cannot fit into the memory of a single GPU, we only evaluate them with 4 GPUs. Table~\ref{tab:time_all} gives the training time of GraphVite on these datasets. The training time of baseline systems is not reported here, as all existing systems cannot solve such large graphs in a week, except LINE~\cite{tang2015line} on Friendster-small. Compared to them, GraphVite takes less than 1 day to train embeddings on the largest dataset Friendster with 1.8 billion edges, showing that GraphVite can be an efficient tool for embedding billion-scale graphs.

We also evaluate the performance of the learned embeddings on these datasets. Figure~\ref{fig:curves} presents the performance of GraphVite over different training epochs on these datasets. On Friendster-small, we also plot the performance of LINE for reference. Due to the long training time, we only report the performance of LINE by the end of all training epochs. It is observed that GraphVite converges on all these datasets. On the Friendster-small dataset, GraphVite significantly outperforms LINE. On the Hyperlink-PLD, we get an AUROC of 0.943. On Friendster, the Micro-F1 reaches about 81.0\%. All the above observations verify the performance of our system.

\begin{table}[t]
    \centering
    \caption[Results of time on larger datasets]{Results of time on larger datasets. The embedding matrices of Hyperlink-PLD and Friendster cannot fit into the memory of a single GPU.}
    \vspace{-0.2em}
    \footnotesize
    \begin{tabular}{lccc}
        \toprule
                    & \bf{Friendster-small} & \bf{Hyperlink-PLD} & \bf{Friendster}  \\
        \midrule
        1 GPU       & 8.78 hrs                      & -                         & -                     \\
        4 GPU       & 2.79 hrs                      & 5.36 hrs                  & 20.3 hrs              \\
        \bottomrule
    \end{tabular}
    \label{tab:time_all}
\end{table}

\begin{figure*}[t]
    \centering
    \begin{subfigure}{0.3\textwidth}
        \includegraphics[width=\textwidth]{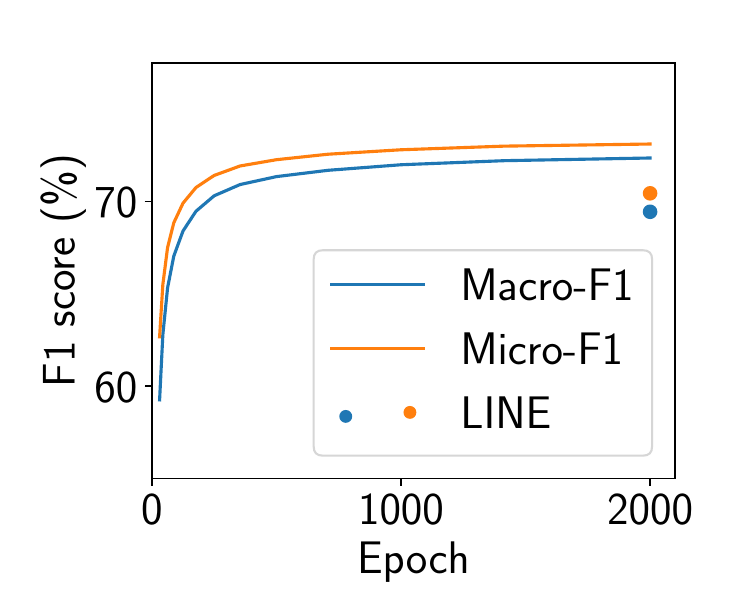}
        \caption{Friendster-small}
    \end{subfigure}
    \hfill
    \begin{subfigure}{0.3\textwidth}
        \includegraphics[width=\textwidth]{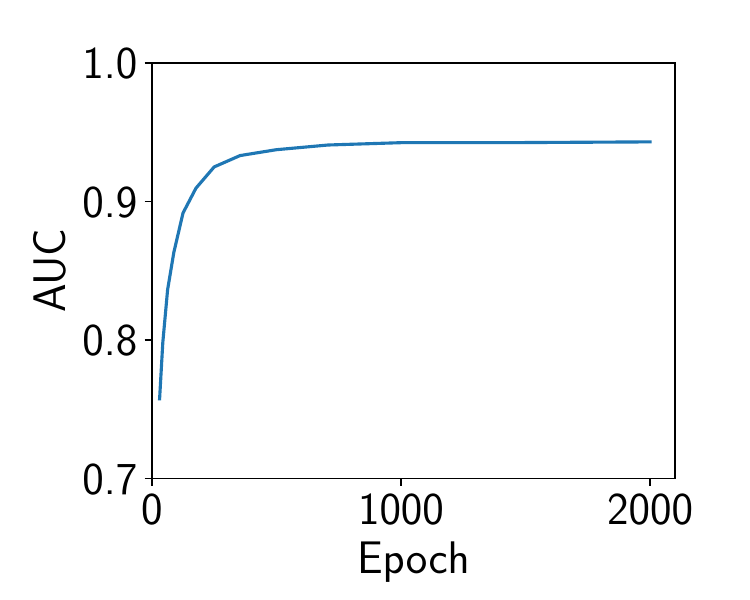}
        \caption{Hyperlink-PLD}
    \end{subfigure}
    \hfill
    \begin{subfigure}{0.3\textwidth}
        \includegraphics[width=\textwidth]{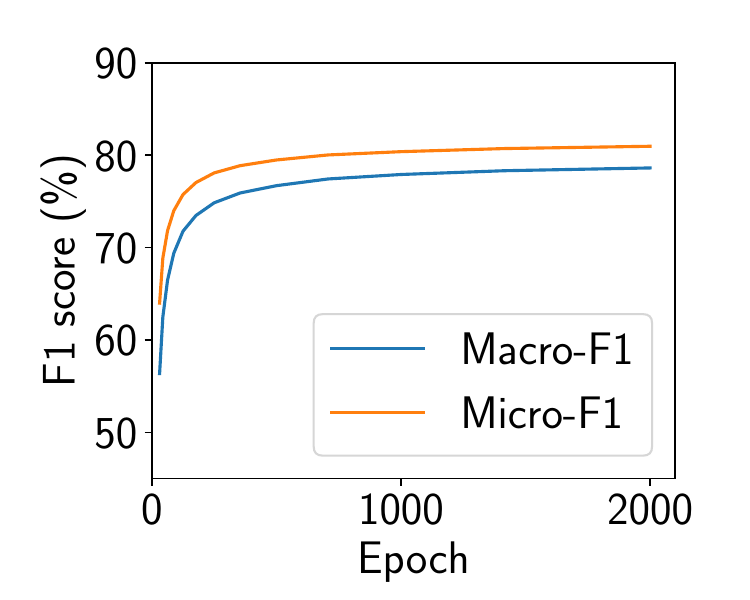}
        \caption{Friendster}
    \end{subfigure}
    \vspace{-0.2em}
    \caption[Performance curves of GraphVite on larger datasets]{Performance curves of GraphVite on larger datasets. For Friendster, we plot the results of LINE for reference. The other systems cannot solve any of these datasets within a week.}
    \label{fig:curves}
\end{figure*}

\subsection{Ablation Studies}
\label{sec:ablation_graphvite}

\smallskip \noindent \textbf{Contribution of Main Components.}
In GraphVite, parallel online augmentation, parallel negative sampling, and the collaboration strategy are the main components in the system. Here we study how these components contribute to the performance of our system. We compare GraphVite with a strong baseline system with single GPU. Specifically, the baseline has the same GPU implementation as GraphVite, while it uses the standard parallel edge sampling instead of parallel online augmentation, and executes two stages sequentially.

Table~\ref{tab:main_components} shows the result of the ablation study. Parallel online augmentation helps improve the quality of learned embeddings, since it introduces more connectivity to the sparse graph. Besides, parallel online augmentation also accelerates the system a little, as it reuses nodes and reduces the amortized cost of each sample. With parallel negative sampling, we can employ multiple GPUs for training, and the speed is boosted by about 3 times. Moreover, the collaboration strategy even improves the speed and does not impact the performance.

\begin{table*}
    \centering
    \caption[Ablation of main components in GraphVite]{Ablation of main components in GraphVite. Note that the baseline has the same GPU implementation with GraphVite and parallel edge sampling on CPU. The baseline should be regarded as a very strong one.}
    \begin{adjustbox}{max width=\textwidth}
    \begin{tabular}{lcccccc}
        \toprule
                            & \bf{Parallel Online}  & \bf{Parallel Negative} & \bf{Collaboration} & \multirow{2}{*}{\bf{Micro-F1}} & \multirow{2}{*}{\bf{Macro-F1}} & \multirow{2}{*}{\bf{Training Time}}  \\
                            & \bf{Augmentation}     & \bf{Sampling (4 GPUs)} & \bf{Strategy} \\
        \midrule
        Single-GPU Baseline &                       &                   &               & 35.26     & 20.38     & 8.61 mins         \\
                            & \checkmark            &                   &               & 41.48     & 29.80     & 6.35 mins         \\
                            &                       & \checkmark        &               & 34.38     & 19.81     & 2.66 mins         \\
                            & \checkmark            & \checkmark        &               & 41.75     & 29.30     & 2.24 mins         \\
        \midrule
        GraphVite           & \checkmark            & \checkmark        & \checkmark    & 41.89     & 29.46     & \bf{1.46 mins}  \\
        \bottomrule
    \end{tabular}
    \end{adjustbox}
    \label{tab:main_components}
\end{table*}

\smallskip \noindent \textbf{Pseudo Shuffle.}
In parallel online augmentation, GraphVite performs pseudo shuffle to decorrelate the augmented edge samples, while some existing systems~\cite{perozzi2014deepwalk, grover2016node2vec} do not shuffle their samples. We compare the proposed pseudo shuffle with three baselines, including no shuffle, a full random shuffle and an index mapping algorithm. The index mapping algorithm preprocesses a random mapping on the indexes of samples and saves the time of computing random variables.

Table~\ref{tab:shuffle} gives the results of different shuffle algorithms on a single GPU. It is observed that all shuffle algorithms are about 1 percent better than the no shuffle baseline. However, different shuffle algorithms vary largely in their speed. Compared to the no shuffle baseline, the random shuffle and index mapping algorithms slow down the system by several times, while our pseudo shuffle has only a little overhead. Therefore, we conclude that pseudo shuffle is the best practice considering both speed and performance.

\begin{table}[t]
    \centering
    \caption[Results of performance and speed by different shuffle algorithms]{Results of performance and speed by different shuffle algorithms. The proposed pseudo shuffle algorithm achieves the best trade off between performance and speed.}
    \footnotesize
    \begin{tabular}{lcc}
        \toprule
        \bf{Shuffle Algorithm}   & \bf{Micro-F1(\%)}  & \bf{Training Time}     \\
        \midrule
        None                & 40.41         & 3.60 mins         \\
        Random shuffle      & 41.61         & 17.1 mins         \\
        Index mapping       & 41.21         & 12.1 mins         \\
        \midrule
        Pseudo shuffle      & \bf{41.52}  & \bf{3.98 mins}  \\
        \bottomrule
    \end{tabular}
    \label{tab:shuffle}
\end{table}

\smallskip \noindent \textbf{Choice of Episode Size.}
In parallel negative sampling, GraphVite relies on the property of gradient exchangeability to ensure its approximation to standard SGD. While the smaller episode size provides better exchangebility, it will increase the frequency of synchronization over rows of the embedding matrices, and thus slows down embedding training. To quantify such influence in speed and performance, we examine our system on 4 GPUs with different episode sizes. 

\begin{figure}[t]
    \begin{minipage}[b]{0.55\textwidth}
        \centering
        \hspace{-0.5em}
        \begin{subfigure}{0.49\textwidth}
            \includegraphics[width=\textwidth]{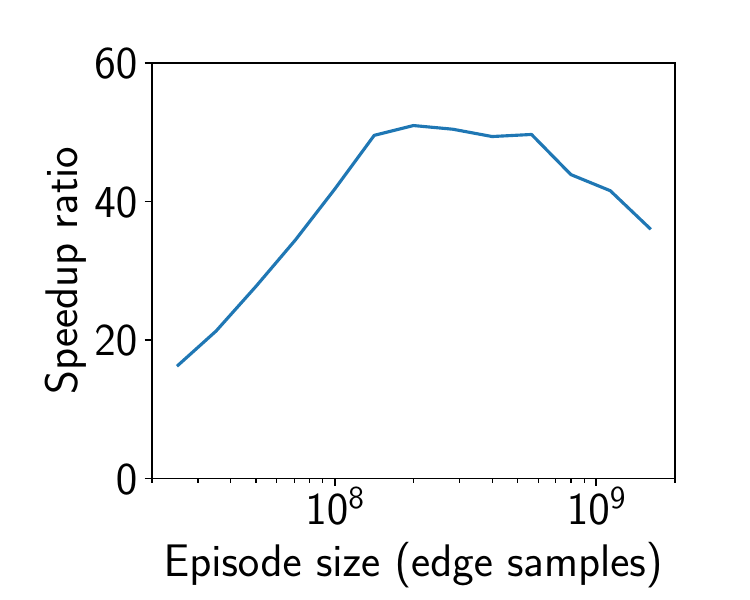}
        \end{subfigure}
        \hfill
        \begin{subfigure}{0.49\textwidth}
            \includegraphics[width=\textwidth]{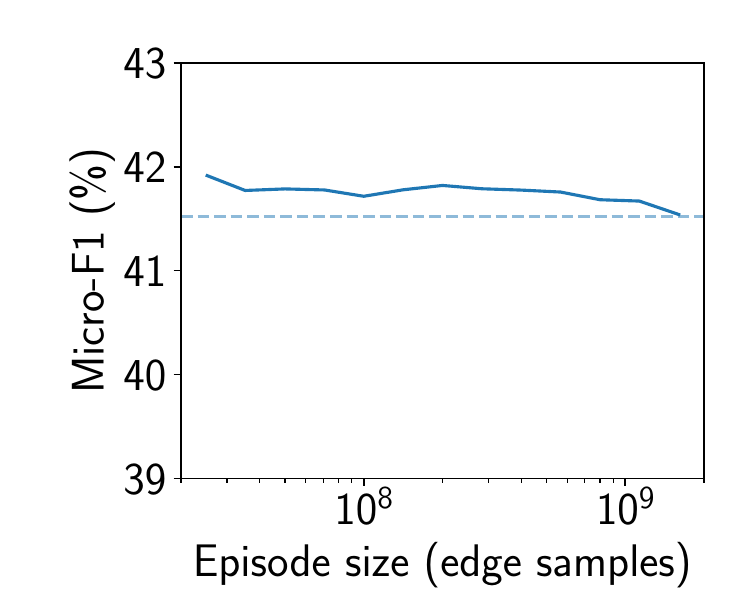}
        \end{subfigure}
        \caption[Speed and performance of GraphVite with respect to different episode sizes]{Speed and performance of GraphVite with respect to different episode sizes. The dashed line represents the single GPU baseline without parallel negative sampling.}
        \label{fig:episode_size}
    \end{minipage}
    \hfill
    \begin{minipage}[b]{0.42\textwidth}
        \centering
        \includegraphics[width=0.95\textwidth,trim={0 {.1\textwidth} 0 {.1\textwidth}},clip]{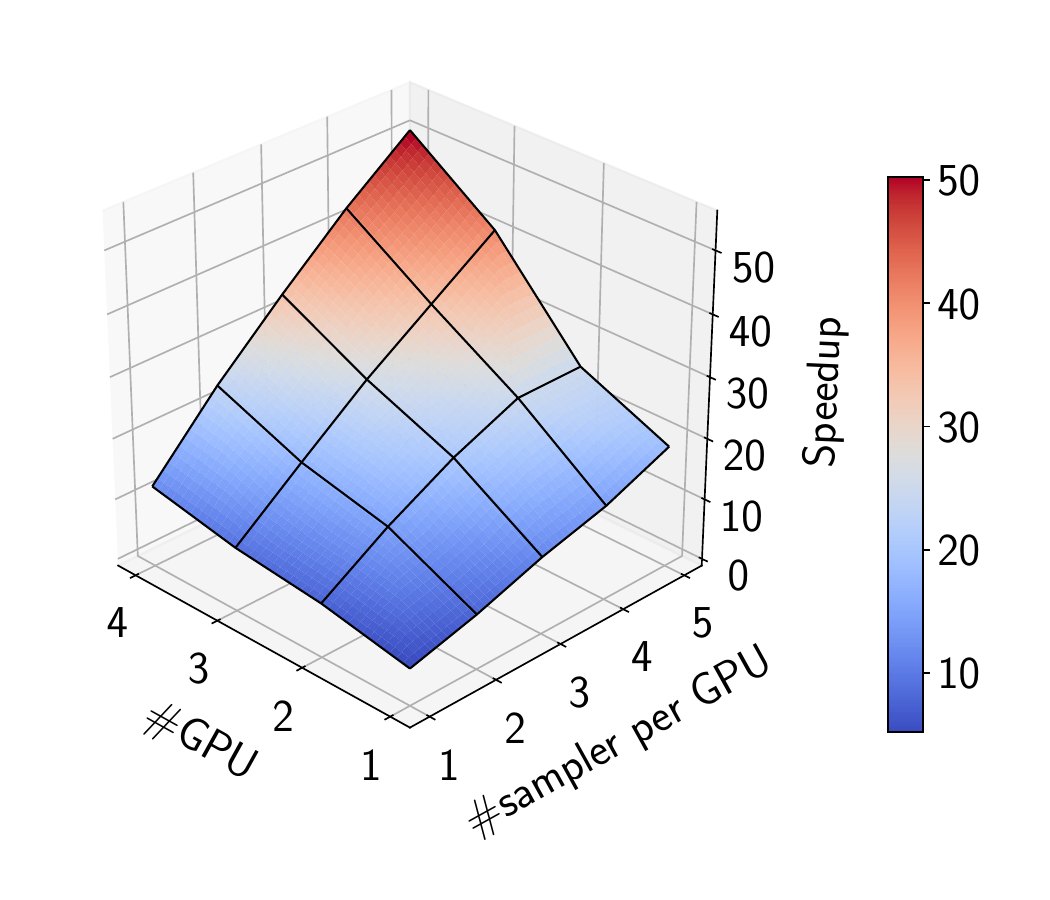}
        \caption[Results of speedup under different number of hardware]{Results of speedup under different number of hardware. It is observed that the speedup is almost linear to the number of CPUs and GPUs.}
        \label{fig:scalability}
    \end{minipage}
\end{figure}

Figure~\ref{fig:episode_size} plots the curves of speed and performance with respect to different episode sizes. On the performance side, we notice that the performance of GraphVite is insensitive to the choice of the episode size. Compared to the single GPU baseline, parallel negative sampling achieves competitive or slightly better results, probably due to the regularization effect introduced by partition. On the speed side, larger episode size achieves more speedup since it reduces the amortized burden of the bus. The speed drops at very large episode size, as there becomes only a few episodes in training. Therefore, we choose an episode size of $2 \times 10^8$ edge samples for Youtube. Generally, the best episode size is proportional to $|V|$, so one can set the episode size for other graphs accordingly.

\smallskip \noindent \textbf{Speedup w.r.t.\ the Number of CPUs and GPUs.}
In GraphVite, both online augmentation and negative sampling can be parallelized on multiple CPUs or GPUs, \pagebreak and synchronization is only required between episodes. Therefore, our system should have great scalability. To verify that point, we investigate our system with different 
number of CPU and GPU. We change the number of GPU from 1 to 4, and vary the number of sampler per GPU from 1 to 5. The effective number of CPU threads is $\#GPU \times (\#sampler~per~GPU + 1)$ as there is one scheduler thread for each GPU.

Figure~\ref{fig:scalability} plots the speedup ratio with respect to different number of CPUs and GPUs. The speedup ratio almost forms a plane over both variables, showing that our system scales almost linearly to the hardware. Quantitatively, GraphVite achieves a relative speedup of 11$\times$ when the hardware is scaled to 20$\times$. The speedup is about half of its theoretical maximum. We believe this is mainly due to the increased synchronization cost, as well as increased load on shared main memory and bus when we use more CPUs and GPUs.

\smallskip \noindent \textbf{Speed under Economic Hardware Configurations.}
Up to now, all experiments are conducted on a server with Xeon E5 CPUs and Tesla P100 GPUs. One might wonder whether such a high performance depends on the specific hardware configuration. Therefore, we further test our system on an economic server with Core i7 CPUs and GTX 1080 GPUs.

Table~\ref{tab:hardware} compares the results from two configurations. Different hardware configurations do have difference in speed, but the gap is marginal. The time only increases to 1.6$\times$ when we move to the economic server. Note that this two configurations are almost the best and the worst in current machine learning servers, so one could expect a running time between these two configurations on their own hardware.

\begin{table}[t]
    \centering
    \caption[Training time of GraphVite under different hardware configurations]{Training time of GraphVite under different hardware configurations. Generally GraphVite may take a time between these two configurations on most hardware.}
    \footnotesize
    \begin{tabular}{lcccc}
        \toprule
        \bf{Hardware}                      & \bf{CPU threads} & \bf{GPU} & \bf{Training time} \\
        \midrule
        \multirow{2}{*}{Tesla P100 server} & 6                & 1        & 3.98 mins          \\
                                           & 24               & 4        & 1.46 mins          \\
        \midrule
        \multirow{2}{*}{GTX 1080 server}   & 3                & 1        & 6.28 mins          \\
                                           & 12               & 4        & 2.48 mins          \\
        \bottomrule
    \end{tabular}
    \label{tab:hardware}
\end{table}

\subsection{Results on Knowledge Graphs}

\smallskip \noindent \textbf{Standard Datasets.}
Table~\ref{tab:time_kg} summarizes the time of GraphVite and an open source PyTorch implementation on knowledge graph embeddings. For both TransE and RotatE embedding methods, GraphVite outperforms existing implementation on both training time and evaluation time under the same single GPU setting. With 4 GPUs, GraphVite additionally accelerates training by nearly 4 times, reaching $6\times$ speedup for TransE and $7.5\times$ speedup for RotatE. The speedup in multi-GPU evaluation is not so significant, since single-GPU evaluation is already very fast on such datasets, and the bottleneck of GraphVite lies in the preprocessing for evaluation.

The performance of different systems is demonstrated in Table~\ref{tab:performance_kg}. GraphVite is slightly behind the official implementation of RotatE on FB15k-237, but achieves better results on WN18RR. We believe such discrepancy comes from the parallel negative sampling used in GraphVite. Nevertheless, GraphVite achieves competitive performance compared to other methods on knowledge graph completion.

\begin{table}[t]
    \centering
    \caption[Result of time on FB15k-237 and WN18RR]{Result of time on FB15k-237 and WN18RR.}
    \begin{adjustbox}{max width=\textwidth}
    \begin{tabular}{llcccccc}
        \toprule
        \multirow{2}{*}{\bf{Method}} & \multirow{2}{*}{\bf{System}} & \multirow{2}{*}{\bf{\#CPU Thread}} & \multirow{2}{*}{\bf{\#GPU}} & \multicolumn{2}{c}{\bf{FB15k-237}} & \multicolumn{2}{c}{\bf{WN18RR}} \\
        & & & & \bf{Training} & \bf{Evaluation} & \bf{Training} & \bf{Evaluation} \\
        \midrule
        \multirow{3}{*}{TransE~\cite{bordes2013translating}}
        & RotatE~\cite{sun2019rotate} & 6 & 1 & 44.6 mins & 1.0 min & 38.6 mins & 24.7 s \\
        & GraphVite & 6 & 1 & \bf{33.8 mins}($1.32\times$) & \bf{25.7s} & \bf{29.1 mins}($1.33\times$) & \bf{9.95 s} \\
        & GraphVite & 24 & 4 & \bf{6.98 mins}($6.39\times$) & \bf{19.5 s} & \bf{6.65 mins}($5.80\times$) & \bf{8.50 s} \\
        \midrule
        \multirow{3}{*}{RotatE~\cite{sun2019rotate}}
        & RotatE~\cite{sun2019rotate} & 6 & 1 & 2.10 hrs & 1.30 mins & 1.85 hrs & 23.0 s \\
        & GraphVite & 6 & 1 & \bf{1.04 hrs}($2.02\times$) & \bf{28.8 s} & \bf{55.4 mins}($2.0\times$) & \bf{10.1 s} \\
        & GraphVite & 24 & 4 & \bf{16.4 mins}($7.68\times$) & \bf{20.8 s} & \bf{14.8 mins}($7.5\times$) & \bf{8.90 s} \\
        \bottomrule
    \end{tabular}
    \end{adjustbox}
    \label{tab:time_kg}
\end{table}

\begin{table}[t]
    \centering
    \caption[Results of knowledge graph completion on FB15k-237 and WN18RR]{Results of knowledge graph completion on FB15k-237 and WN18RR.}
    \begin{adjustbox}{max width=\textwidth}
    \begin{tabular}{lcccccccccc}
        \toprule
        \multirow{2}{*}{\bf{Method}} & \multicolumn{5}{c}{\bf{FB15k-237}} & \multicolumn{5}{c}{\bf{WN18RR}} \\
        & \bf{MR} & \bf{MRR} & \bf{HITS@1} & \bf{HITS@3} & \bf{HITS@10} & \bf{MR} & \bf{MRR} & \bf{HITS@1} & \bf{HITS@3} & \bf{HITS@10} \\
        \midrule
        TransE~\cite{bordes2013translating} & 357 & 0.294 & - & - & 0.465 & 3384 & 0.226 & - & - & 0.501 \\
        DistMult~\cite{yang2015embedding} & 254 & 0.241 & 0.155 & 0.263 & 0.419 & 5110 & 0.43 & 0.39 & 0.44 & 0.49 \\
        ComplEx~\cite{trouillon2016complex} & 244 & 0.325 & 0.237 & 0.356 & 0.501 & 5261 & 0.44 & 0.41 & 0.46 & 0.51 \\
        RotatE~\cite{sun2019rotate} & \bf{177} & \bf{0.338} & \bf{0.241} & \bf{0.375} & \bf{0.533} & 3340 & 0.476 & 0.428 & 0.492 & 0.571 \\
        GraphVite (RotatE) & 201 & 0.314 & 0.218 & 0.348 & 0.506 & \bf{2359} & \bf{0.500} & \bf{0.455} &\bf{0.518} & \bf{0.589} \\
        \bottomrule
    \end{tabular}
    \end{adjustbox}
    \label{tab:performance_kg}
\end{table}

\smallskip \noindent \textbf{Larger Datasets.}
We further report the result of GraphVite on Wikidata5m, a large-scale knowledge graph that contains 5 million entities and 21 million triplets. We compare the results of 5 popular embedding methods in GraphVite. Each method is trained for 1000 epochs with 4 Tesla V100 GPUs and 24 CPU threads. Table~\ref{tab:wikidata5m} shows the time and performance of different embedding methods. It takes about 2 hours to train these methods on such a large graph with 4 GPUs. We can see that the performance of different methods generally follows their order in literatures, except that SimplE is slightly worse than ComplEx.

\begin{table}[t]
    \centering
    \caption[Result of time and performance on Wikidata5m]{Result of time and performance on Wikidata5m.}
    \footnotesize
    \begin{tabular}{lcccccc}
        \toprule
        \bf{Method} & \bf{Training Time} & \bf{MR} & \bf{MRR} & \bf{HITS@1} & \bf{HITS@3} & \bf{HITS@10} \\
        \midrule
        GraphVite (TransE) & 2.26 hrs & 109,370 & 0.253 & 0.170 & 0.311 & 0.392 \\
        GraphVite (DistMult) & 2.29 hrs & 279,091 & 0.248 & 0.204 & 0.273 & 0.331 \\
        GraphVite (ComplEx) & 2.23 hrs & 244,540 & 0.281 & 0.228 & 0.310 & 0.373 \\
        GraphVite (SimplE) & 2.22 hrs & 123,400 & 0.263 & 0.212 & 0.287 & 0.358 \\
        GraphVite (RotatE) & 2.10 hrs & 89,459 & 0.290 & 0.234 & 0.322 & 0.390 \\
        \bottomrule
    \end{tabular}
    \label{tab:wikidata5m}
\end{table}
\section{Dataset Statistics}

The following datasets are used in our experiments. Statistics of homogenenous graphs and knowledge graphs are summarized in Table~\ref{tab:homogeneous_graphs} and \ref{tab:knowledge_graphs} respectively. 

\begin{itemize}[label=$\bullet$, leftmargin=*]
    \begin{item}
        Youtube~\cite{mislove2007measurement} is a large-scale social network in the Youtube website. It contains 1 million nodes and 5 million edges. For some of the nodes, they have labels that represent the type of videos users enjoy.
    \end{item}
    \begin{item}
        Friendster-small~\cite{yang2015defining} is a sub-graph induced by all the labeled nodes in Friendster. It has 8 million nodes and 447 million edges. The node labels in this graph are the same as those in Friendster.
    \end{item}
    \begin{item}
        Hyperlink-PLD~\cite{meusel2015graph} is a hyperlink graph extracted from the Web corpus~\footnote{\url{http://commoncrawl.org/}}. We use the pay-level-domain aggregated version of the graph. It has 43 million nodes and 623 million edges. This dataset does not contain any label.
    \end{item}
    \begin{item}
        Friendster~\cite{yang2015defining} is a very large social network in an online gaming site. It has 65 million nodes and 1.8 billion edges. Some nodes have labels that represent the group users join.
    \end{item}
    \begin{item}
        FB15k-237~\cite{toutanova2015observed} is a subset of the encyclopedia knowledge graph FB15k~\cite{bordes2013translating} with duplicate and inverse relations removed. It has 15 thousand entities, 237 relations and 272 thousand triplets.
    \end{item}
    \begin{item}
        WN18RR~\cite{dettmers2018convolutional} is a subset of the ontology graph WN18~\cite{miller1998wordnet} constructed following a similar process of FB15k-237. It has 41 thousand entities, 11 relations and 87 thousand triplets.
    \end{item}
    \begin{item}
        Wikidata5m~\cite{wang2021kepler} is a very large encyclopedia knowledge graph extracted from Wikidata~\cite{vrandevcic2014wikidata}. It has 5 million entities, 822 relations and 21 million triplets.
    \end{item}
\end{itemize}

\begin{table}[!h]
    \centering
    \caption[Statistics of homogeneous graphs used in experiments]{Statistics of homogeneous graphs used in experiments.}
    \footnotesize
    \begin{tabular}{lccc}
        \toprule
        \bf{Dataset} & \bf{\#Nodes} & \bf{\#Edges} & \bf{Evaluation Task} \\
        \midrule
        Youtube & 1,138,499 & 4,945,382 & 47-class node classification \\
        Friendster-small & 7,944,949 & 447,219,610 & 100-class node classification \\
        Hyperlink-PLD & 39,497,204 & 623,056,313 & link prediction \\
        Friendster & 65,608,376 & 1,806,067,142 & 100-class node classification \\
        \bottomrule
    \end{tabular}
    \label{tab:homogeneous_graphs}
\end{table}

\begin{table}[!h]
    \centering
    \caption[Statistics of knowledge graphs used in experiments]{Statistics of knowledge graphs used in experiments.}
    \footnotesize
    \begin{tabular}{lcccc}
        \toprule
        \bf{Dataset} & \bf{\#Entities} & \bf{\#Relations} & \bf{\#Triplets} & \bf{Evaluation Task} \\
        \midrule
        FB15k-237 & 14,541 & 237 & 272,115 & knowledge graph completion \\
        WN18RR & 40,943 & 11 & 86,835 & knowledge graph completion \\
        Wikidata5m & 4,594,485 & 822 & 20,614,279 & knowledge graph completion \\
        \bottomrule
    \end{tabular}
    \label{tab:knowledge_graphs}
\end{table}

\chapter{Conclusion}

In this thesis, we have explored several representation learning models in reasoning domains, with a focus on generalization across structures. Different from prevalent embedding methods that memorize information of each element in a structure, we devised models to learn representations as functions of fundamental elements shared by different structures. Such elements include paths, relation-relation interactions, relation projections and rules. Our models offer advantages in generalization across various knowledge and query structures, such as knowledge graphs with arbitrary entity and relation vocabularies, multi-step queries in both graph and text modalities. Consequently, we unified different graph structures and developed the first foundation model for both single- and multi-step queries on knowledge graphs. Alongside these models, we have released two systems to facilitate machine learning development and accelerate representation learning models on structured data.

The assets we developed in this thesis have already influenced the community, leading to several follow-up works built on top of our models or systems. Some notable examples include expressiveness of GNNs~\cite{zhang2021labeling, huang2024theory}, distance-based propagation~\cite{shomer2023distance} and generalization across discrete attributes~\cite{shen2024zero}. TorchDrug has been adopted by industrial open-source software such as ChemicalX from AstraZeneca and GT4SD from IBM. The idea of GraphVite has been integrated into DGL-KE from Amazon. We expect our assets to continue benefiting both graph machine learning and reasoning communities.

While techniques may soon become outdated, there is some intangible heritage from our works, which we believe will have a relatively long-term impact. First, our works accelerated the transition from transductive models to inductive ones in the community, eliminating the need of training models separately for each dataset or knowledge update. This simplifies the development pipeline and also reduces carbon footprint. Second, our works reveal that the key to generalization in reasoning problems is having a proper inductive bias for models, and we provided several examples on how to draw inspiration from symbolic algorithms and inject certain inductive bias in models. Finally, our works empirically prove that many reasoning problems, even if they come from different datasets or domains, can be unified. This will serve as an important lesson and influence future models in the reasoning community.

Looking into the future, reasoning and generalization across structures will be mostly explored in the realm of large language models. Compared to knowledge graphs that are constructed based on handcrafted schemas (e.g.\ definitions of triplets), text is more general and flexible in terms of representing knowledge, queries and answers. However, this does not render our efforts on structures in vain. As pointed out in Chapter~\ref{cha:introduction}, structures are inherent in reasoning problems regardless of their modality. Our methods have matured in the testbed of explicit structures like knowledge graphs, and the future is to bring these ideas to LLMs to solve textual reasoning problems with implicit structures. Some future directions include generalization across task structures, updating knowledge in LLMs and systems for LLM reasoning.

\smallskip \noindent \textbf{Generalization across task structures.}
As LLMs unify all sequence-to-sequence tasks in natural languages, it is crucial to understand how they generalize across tasks with different structures and how to improve this generalization. Enhancing this capacity will significantly impact the scaling law~\cite{kaplan2020scaling}, as better generalization leads to better performance under the same amount of training data. Currently, LLMs acquire their zero-shot generalization ability after instruction tuning, but there lacks a clear understanding of how instructions transfer to new questions or domains, which causes the failure observed in Figure~\ref{fig:gpt_failure}. GPT-4 generally has the necessary knowledge for the questions if we prompt separately, but cannot follow the instruction to answer the question. To solve this challenge, we may follow the success in Chapter~\ref{cha:ultra} to design models based on meta structures extracted from each task.

A more explicit unification of tasks with different structures is adopted by function calling nowadays. LLMs have been trained on massive code snippets, and they have strong abilities in performing reasoning related to code. In function calling, we abstract task-specific procedures as functions, and leverage LLMs to perform task-independent reasoning strategies conditioned on the function signatures. If we can improve the generalization of LLMs over new tasks and tools, this will significantly extend the applicability of function calling. Additionally, generalization across tasks under distribution shift is of great interest, since data collection in some domains is challenging and we have to pretrain LLMs based on synthetic data or data from other domains.

\smallskip \noindent \textbf{Updating knowledge in large language models.}
Before the advent of LLMs, knowledge was usually stored in natural languages, or handcrafted forms such as knowledge graphs and databases. LLMs have emerged as a new source of knowledge in the form of parameters, which brings up a new challenge: how can we update the parameteric knowledge in LLMs? For example, we may want to inject commonsense knowledge or up-to-date information into models, and remove toxic knowledge from the model. Addressing this challenge is particularly important given the extremely high cost of re-training language models. A recent work~\cite{kandpal2023large} points out that LLMs cannot memorize all long-tail knowledge, even with reasonable scaling in model size. Consequently, controlling what to memorize is crucial to maximize the utility of LLMs. We estimate that such a goal may be achieved through techniques like knowledge editing~\cite{meng2022locating}, model merging~\cite{matena2022merging} or dynamic LoRA composition~\cite{huang2023lorahub}.

\smallskip \noindent \textbf{Systems for large language model reasoning.}
Systems have always been playing an important role in the development and deployment of machine learning algorithms, and this principle applies equally to large language model reasoning. Generally, we need systems for two purposes: (1) encapsulating existing routines into functions to support ideas of a higher level; (2) reducing the computation and memory cost to unlock new opportunities for research and applications. Currently, many reasoning pipelines involve multiple calls to an LLM, which requires a lot of trial and error during development. To streamline this process, it is crucial to develop automatic evaluation systems and automatic prompt optimization systems for LLMs. An interesting remark is that these systems may be implemented based on LLM themselves. Additionally, optimizing the computation and memory cost of a single LLM call benefits all kinds of applications, and we foresee this challenge will be resolved by automatic model optimization systems, such as compilation.

\bibliographystyle{plain}
\bibliography{reference}

\appendix

\end{document}